\pdfoutput=1
\documentclass[twoside]{article}

\usepackage[preprint]{aistats2027}

\usepackage[round]{natbib}

\usepackage[utf8]{inputenc}
\usepackage[T1]{fontenc}
\usepackage{url}
\usepackage{booktabs}
\usepackage{float}
\usepackage{placeins}
\usepackage{amsfonts}
\usepackage{nicefrac}
\usepackage{microtype}
\usepackage{xcolor}
\definecolor{bestcell}{RGB}{190,214,238}
\definecolor{approxcell}{RGB}{226,237,248}
\newcommand{\cellbest}[1]{{\setlength{\fboxsep}{1pt}\colorbox{bestcell}{#1}}}
\newcommand{\cellapprox}[1]{{\setlength{\fboxsep}{1pt}\colorbox{approxcell}{#1}}}
\newif\ifch \chfalse
\newcommand{\ch}[1]{\ifch\textcolor{magenta}{#1}\else#1\fi}
\usepackage{amsmath,amssymb,amsthm}
\usepackage{mathtools}
\usepackage{graphicx}
\usepackage{tikz}
\usepackage{pgfplots}
\usepgfplotslibrary{groupplots}
\pgfplotsset{compat=1.18}
\usepackage{multirow}
\usepackage{longtable}
\usepackage{tabularx}
\usepackage{algorithm}
\usepackage{algpseudocode}
\usepackage[nohyperlinks,nolist]{acronym}

\usepackage[hypertexnames=false,hidelinks]{hyperref}
\hypersetup{
  pdftitle={Playing the Network Backward: A Game-Theoretic Attribution Framework},
  pdfauthor={Jakob Paul Zimmermann, Jim Berend, Georg Loho, Sebastian Lapuschkin, Wojciech Samek}
}
\newcommand{\coderepo}{\url{https://github.com/JayPiZimmermann/Playing_The_Network_Backward_Official_Repository}}

\newcommand{\NN}{\mathbb{N}}

\newcommand{\RR}{\mathbb{R}}
\renewcommand{\P}{\mathbb{P}}
\providecommand{\E}{}
\renewcommand{\E}{\mathbb{E}}
\newcommand{\relu}{\text{ReLU}}
\newcommand{\1}{\mathbf{1}}
\DeclareMathOperator{\argmax}{argmax}

\DeclareMathOperator{\softmax}{softmax}

\newtheorem{theorem}{Theorem}
\newtheorem{proposition}{Proposition}
\newtheorem{lemma}{Lemma}

\newtheorem{definition}{Definition}
\newtheorem{remark}{Remark}

\newif\ifprelim \prelimtrue
\global\prelimfalse
\newcommand{\vitRGlocAL}{0.711}

\newcommand{\vitRGlocPFtwenty}{7.48}

\newcommand{\vitRGlocMaxS}{0.729}
\newcommand{\vitRGlocTime}{0.17}
\newcommand{\vitRGfaithPFtwenty}{6.78}

\newcommand{\vitRGanchorMaxS}{0.166}
\newcommand{\vitRGanchorTime}{0.10}
\newcommand{\vitRGEqlocAL}{0.698}
\newcommand{\vitRGEqlocTK}{0.836}

\newcommand{\vitRGEqanchorMaxS}{0.170}
\newcommand{\vitBestBaseAL}{0.644}
\newcommand{\vitBestBaseALname}{LRP-$\gamma$}
\newcommand{\vitBestBasePFtwenty}{6.43}

\newcommand{\vitBestBaseTK}{0.823}
\newcommand{\vitBestBaseTKname}{KernelSHAP}
\newcommand{\vitBestBasePG}{0.961}
\newcommand{\vitBestBasePGname}{KernelSHAP}
\newcommand{\vitBestBaseMaxS}{0.208}
\newcommand{\vitBestBaseMaxSname}{AR}
\newcommand{\vitRGEqTime}{8.66}

\newcommand{\vitSHAPTime}{6.71}
\newcommand{\vocRGlocAL}{0.459}
\newcommand{\vocRGEqlocAL}{0.456}
\newcommand{\vocBestBaseAL}{0.413}
\newcommand{\vocBestBaseALname}{TIBAV}
\newcommand{\nValGrid}{2{,}021}
\newcommand{\vggRGlocAL}{0.685}
\newcommand{\vggSGlocMaxS}{0.278}
\newcommand{\vggBestBaseAL}{0.613}
\newcommand{\vggBestBaseALname}{LRP-$\varepsilon$}
\newcommand{\vggBestBaseMaxS}{0.314}
\newcommand{\vggBestBaseMaxSname}{GCAM++}
\newcommand{\resRGlocAL}{0.719}
\newcommand{\resBestBaseAL}{0.711}
\newcommand{\resBestBaseALname}{LRP-$\gamma$}

\newcommand{\sanityResLRPgammaRho}{0.254}

\newcommand{\sanityResLRPepsRho}{0.250}

\newcommand{\sanityVggLRPepsRho}{0.341}
\newcommand{\sanityVggLRPepsHOG}{0.457}
\newcommand{\sanityVitRGanchorRho}{0.454}
\newcommand{\sanityVitRGanchorHOG}{0.469}
\newcommand{\sanityVitRGlocRho}{0.450}

\newcommand{\drDistinctImages}{336}
\newcommand{\drFloorMin}{0.94}
\newcommand{\drFloorMax}{0.97}
\newcommand{\perImageCrossMax}{0.19}
\newcommand{\perImageN}{565}

\newcommand{\PrelimTag}{\ifprelim\ \textcolor{red}{\textbf{[PRELIMINARY NUMBERS]}}\fi}

\makeatletter
\providecommand{\hyper@natlinkstart}[1]{}
\providecommand{\hyper@natlinkend}{}
\providecommand{\hyper@natlinkbreak}[2]{#1}
\providecommand{\hyper@natanchorstart}[1]{}
\providecommand{\hyper@natanchorend}{}
\def\hyper@natlinkstart#1{}
\def\hyper@natlinkend{}
\def\hyper@natlinkbreak#1#2{#1}
\def\hyper@natanchorstart#1{}
\def\hyper@natanchorend{}
\providecommand{\captionof}[1]{\def\@captype{#1}\caption}
\makeatother

\begin{document}

\runningtitle{Playing the Network Backward: A Game-Theoretic Attribution Framework}
\runningauthor{Zimmermann, Berend, Loho, Lapuschkin, Samek}

\twocolumn[
\aistatstitle{Playing the Network Backward: A Game-Theoretic\\ Attribution Framework}
\aistatsauthor{Jakob Paul Zimmermann$^{1,2,3}$ \quad Jim Berend$^{1}$\\[2pt]
  \bfseries Georg Loho$^{3}$ \quad Sebastian Lapuschkin$^{1,5}$ \quad Wojciech Samek$^{1,2,4}$}
\aistatsaddress{\small $^{1}$Fraunhofer Heinrich-Hertz-Institute, Berlin, Germany\\
  \small $^{2}$Technische Universit\"at Berlin, Berlin, Germany\\
  \small $^{3}$Freie Universit\"at Berlin, Berlin, Germany\\
  \small $^{4}$BIFOLD -- Berlin Institute for the Foundations of Learning and Data, Berlin, Germany\\
  \small $^{5}$Technological University Dublin, Dublin, Ireland}
]

\begin{figure*}[t]
    \centering
    \includegraphics[angle=270,width=0.86\textwidth]{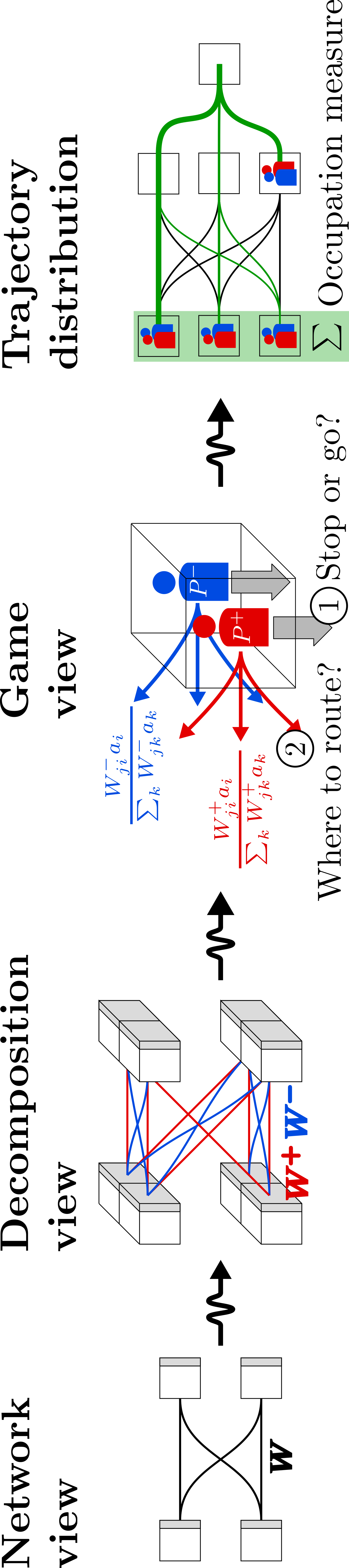}
    \caption{The backward pass as a two-player game. The Stopping Decomposition splits every weight by
    sign, $W = W^+ - W^-$; at each state of the extended graph the players decide (1) stop or continue
    and (2) where to route. Their equilibrium policies induce a distribution over backward
    trajectories, and the discounted occupation measure of that distribution at the input states is the
    attribution map.}
    \label{fig:game-trajectory-distribution}
\end{figure*}

\begin{acronym}[Bhattacharyya]
    \acro{XAI}{Explainable Artificial Intelligence}
    \acro{ADF}{Assumed Density Filtering}
    \acro{CDF}{Cumulative Distribution Function}
    \acro{DC}{Difference-of-Convex}
    \acro{CPWL}{Continuous Piecewise Linear}
    \acro{IB}{Information Bottleneck}
    \acro{KL}{Kullback-Leibler}
    \acro{LMP}{Layered Markov Process}
    \acroplural{LMP}[LMPs]{Layered Markov Processes}
    \acro{MDL}{Minimum Description Length}
    \acro{MDP}{Markov Decision Process}
    \acroplural{MDP}[MDPs]{Markov Decision Processes}
    \acro{MP}{Markov Process}
    \acroplural{MP}[MPs]{Markov Processes}
    \acro{SAC}{Soft Actor-Critic}
    \acro{SPNE}{Subgame Perfect Nash Equilibrium}
    \acroplural{SPNE}[SPNEs]{Subgame Perfect Nash Equilibria}
    \acro{TV}{Total Variation}

    \acro{AV}{Attention-Value}
    \acro{CNN}{Convolutional Neural Network}
    \acro{GELU}[GeLU]{Gaussian Error Linear Unit}
    \acro{ILSVRC}{ImageNet Large Scale Visual Recognition Challenge}
    \acro{LN}{Layer Normalization}
    \acro{MLP}{Multi-Layer Perceptron}
    \acro{QK}{Query-Key}
    \acro{ReLU}{Rectified Linear Unit}
    \acro{ResNet}{Residual Network}
    \acro{SiLU}{Sigmoid Linear Unit}
    \acro{VGG}{Visual Geometry Group}
    \acro{ViT}{Vision Transformer}

    \acro{AttrLoc}[AL]{Attribution Localization}
    \acro{AUC}{Area Under the Curve}
    \acroplural{AUC}[AUCs]{Areas Under the Curve}
    \acro{AvgS}{Average Sensitivity}
    \acroplural{AvgS}[AvgS]{Average Sensitivities}
    \acro{BC}{Bhattacharyya Coefficient}
    \acro{HOG}{Histogram of Oriented Gradients}
    \acroplural{HOG}[HOGs]{Histograms of Oriented Gradients}
    \acro{MaxS}{Maximum Sensitivity}
    \acroplural{MaxS}[MaxS]{Maximum Sensitivities}
    \acro{PF}{Pixel Flipping}
    \acro{PFAUC}[PF-AUC]{Pixel-Flipping Area Under the Curve}
    \acroplural{PFAUC}[PF-AUCs]{Pixel-Flipping Areas Under the Curve}
    \acro{PG}{Pointing Game}
    \acro{Sel}{Selectivity}
    \acroplural{Sel}[Sel]{Selectivities}
    \acro{SSIM}{Structural Similarity Index Measure}
    \acro{TK}{Top-$k$ Intersection}

    \acro{AttnLRP}{Attention-Aware Layer-wise Relevance Propagation}
    \acro{CAM}{Class Activation Mapping}
    \acro{CPLRP}[CP-LRP]{Conservative Propagation LRP}
    \acro{DAVE}{Distribution-aware Attribution via ViT Gradient Decomposition}
    \acro{DeepLift}[DeepLift]{Deep Learning Important FeaTures}
    \acro{GCAM}[GCAM]{Grad-CAM}
    \acro{GCAMpp}[GCAM++]{Grad-CAM++}
    \acro{GrInp}[Gr$\times$Inp]{Gradient times Input}
    \acro{HiResCAM}{High-Resolution Class Activation Mapping}
    \acro{IntGrad}{Integrated Gradients}
    \acro{KShap}[KShap]{KernelSHAP}
    \acro{KernelSHAP}{Kernel SHapley Additive exPlanations}
    \acro{LCAM}[LCAM]{LayerCAM}
    \acro{LIME}{Local Interpretable Model-agnostic Explanations}
    \acro{LRP}{Layer-wise Relevance Propagation}
    \acro{RGEq}[RG+Eq]{Routing Game with equivariant Reynolds averaging}
    \acro{RG}{Routing Game}
    \acro{SG}{Stopping Game}
    \acro{SHAP}{SHapley Additive exPlanations}
    \acro{SmGrad}{SmoothGrad}
    \acro{absLRP}[absLRP]{absolute-value Layer-wise Relevance Propagation}
    \acro{AR}{Attention Rollout}
    \acro{SOTA}{state-of-the-art}
    \acro{RawAttn}{Raw Attention}
    \acro{GAE}{Generic Attention Explainability}
    \acro{TIBAV}{Transformer Interpretability Beyond Attention Visualization}
\end{acronym}

\ifprelim\noindent\textcolor{red}{\textbf{PRELIMINARY NUMBERS: regenerate tables/generated/ before
submission.}}\fi
\begin{abstract}
Attribution methods explain which input features drive a model's prediction, making them central to model debugging and mechanistic interpretability. Yet backward attribution methods, including
gradients, Layer-wise Relevance Propagation (LRP), and transformer-specific rules, lack a shared
framework in which to compare the underlying backward calculations. We introduce such a framework
by recasting backward attribution as a two-player game on an extended network graph, building on
Gaubert and Vlassopoulos' ReLU Net Game. Gradients and the full $\alpha\beta$-LRP family arise as
integrals over game trajectories under specific equilibria, so attribution maps become projections
of trajectory distributions rather than the primary object. \ch{Desired explanation properties,
such as finer localization or robustness to input noise, become game-theoretic choices, such as policy regularization or risk aversion, and yield new adaptations of the known backward rules.}
\ch{In a vision case study, one such adaptation of $\alpha\beta$-LRP outperforms all baselines,
including transformer-specific methods, in attribution localization on ViT-B/16.}
Code is available at \coderepo.
\end{abstract}
\section{Introduction}\label{sec:intro}

Attribution has shifted from a tool for inspecting model decisions to a practical mechanism for
improving and controlling models: recent work uses attributions to reveal and correct spurious
behavior in vision models~\citep{Pahde2023R2R}, identify retrieval and parametric heads in language
models~\citep{Kahardipraja2025Atlas}, steer decoding toward instructions or
evidence~\citep{Komorowski2026AGD}, and localize circuit components through attribution
patching~\citep{Syed2024AttributionPatching,Kramar2024AtPStar,Hanna2024Faithfulness}. This makes
critiques of saliency more consequential: parameter insensitivity, input dependence, perturbation
fragility, and misleading modified-gradient structure~\citep{Adebayo2018sanity, Kindermans2018,
Ghorbani2019fragile, Nie2018guided} are no longer only failures of visual explanation, but
potential failure modes of downstream procedures built on attribution. Understanding these methods
therefore requires a common account of the backward calculation itself.

Backward attribution, however, is usually presented as a zoo of recipes---gradients as local
sensitivities, \ac{LRP} as redistribution, Deep Taylor as local decomposition, and
transformer-specific rules as handcrafted extensions~\citep{Bach2015, Montavon2017, Binder2016,
Kohlbrenner2020, Chefer2021transformer, Ali2022, achtibat2024attnlrp, Wrobel2024}. \ch{\citet{Ancona2018}
unify several gradient-based rules as modified gradients, a map-level unification that leaves out
the $\alpha\beta$ rule; we instead seek a signed trajectory-space representation of gradients and
the $\alpha\beta$ rule that exposes their routing laws. This need} is especially acute on \acp{ViT}: \ch{LayerNorm's feature-wise coupling through its
input-dependent mean and variance} and the QK-softmax break \ac{LRP} conservation~\citep{Ali2022}.

\paragraph{Network calculation as a backward game.}
\citet{Gaubert2025} prove that a ReLU network's forward output is the game value of a two-player
Stopping Game on its graph: layers are rounds, neurons are states, input values are terminal
payoffs, and the scalar output $f(x)$ is the game value.
We reinterpret backward attribution
as a \textbf{distribution of weighted game trajectories} through the network
\ch{(Figure~\ref{fig:game-trajectory-distribution})}, yielding a framework to compare different backward
attribution methods. We adapt the game to a novel Stopping
Decomposition, separating paths with positive and negative contributions to the prediction
(Lemma~\ref{lem:parity-path})---information that the original game formulation averages out.
\ch{In this framework,} propagation rules become \textbf{equilibrium policies}, and attribution
maps are recovered from the \textbf{discounted occupation measure} of the induced trajectory
distribution---an integral over the path distribution assigning relevance to each game state.
We introduce two games:
the \textbf{\ac{SG}}, whose occupation measures recover the Jacobian (Theorem~\ref{thm:gradient}),
and the entropy-regularized \textbf{\ac{RG}} at temperature $\tau$, whose occupation measures
recover the full $\alpha\beta$-\ac{LRP} family at its anchor $\tau{=}1$
(Theorem~\ref{thm:lrp-recovery}) on ReLU networks. Guiding the player's policies at attention
blocks by reference policies of an oracle player recovering the QK-softmax extends the framework to
attention modules; the same entropy regularization extends \ac{SG} from \ac{ReLU} to Softplus
networks.

\paragraph{Consequences of the reformulation.}

\emph{(i)} We interpret the trajectory distribution as the network's backward calculation.
The Hellinger distance $\mathrm{H}(\Pi_A,\Pi_B)$ between two such distributions---computable by a
linear-time backward pass---therefore gives a metric between calculations rather than between
heatmaps.
\ch{Under cascading randomization~\citep{Adebayo2018sanity}, this distance separates trained from
randomized calculations even where their pixel projections stay similar: the two diagnostics detect
different aspects of parameter sensitivity.}
\emph{(ii)} The game framework yields \textbf{analytical modes}: ideas like risk aversion, action
space augmentation or policy regularization map back to adaptations of classical backpropagation
rules that change attribution maps in a controlled way.
\emph{(iii)} $\alpha\beta$-\ac{LRP} parameters have game-space duals: $(\alpha,\beta)$ set
occupation-measure path weights, while the $\varepsilon$-stabilizer becomes a stop option with
payoff $\log(\varepsilon)$.

\paragraph{Contributions.}
\ch{\emph{(i)} We present the \acf{SG}, adapting the ReLU Net Game of~\citet{Gaubert2025} to the
Stopping Decomposition, and the novel entropy-regularized \acf{RG}.
\emph{(ii)} We prove that known methods arise as equilibria: \ac{SG} occupation measures recover
the Jacobian, and \ac{RG} occupation measures at $\tau{=}1$ recover the full $\alpha\beta$-\ac{LRP}
family on ReLU networks.
\emph{(iii)} The Hellinger distance between trajectory distributions separates trained from
randomized calculations under cascading weight randomization~\citep{Adebayo2018sanity}, even when
their attribution maps look alike
(\S\ref{sec:hellinger-sanity}).
\emph{(iv)} Finally, we show the utility of the framework beyond unification: the game-level
changes such as
risk aversion or policy regularization derive new attribution methods in a principled way, whose performance we demonstrate in a case study on ViT-B/16, VGG-16 and ResNet-50.}
\section{Related Work}\label{sec:related-work}

\textbf{Game and trajectory views.} The closest mathematical precursor is the ReLU Net Game of
\citet{Gaubert2025}, which identifies the forward evaluation of a ReLU network with the value of a
backward Stopping Game but does not connect that game to the backward pass or to attribution. We
reformulate the game to draw the connection to a novel Stopping Decomposition, whose inherent
weight decomposition is related to the \ac{DC} decomposition of neural networks---the decomposition
\citet{hiddenMono} use directly for attribution. Our \ac{SG} and \ac{RG} are different games
utilizing this weight decomposition, chosen to recover backward attribution rules rather than the
pure DC decomposition.

\textbf{Relevance propagation and stochastic explanations.} \ac{LRP}, Deep Taylor, and Excitation
Backprop constitute the core relevance propagation family~\citep{Bach2015, Montavon2017,
Zhang2018excitation}, with later variants modifying redistribution for contrast, stability, or
signed relevance~\citep{Binder2016, Kohlbrenner2020, Gu2019, Nam2020, vukadin2024abslrp}. We supply
a common trajectory-space semantics for this family: \ac{SG} recovers gradients, \ac{RG} recovers
the generalized $\alpha\beta$ rule, and its analytical modes map back to principled modifications
of existing backward rules. Excitation Backprop~\citep{Zhang2018excitation} is an absorbing Markov
chain recovering the $\alpha\beta$-\ac{LRP} rule for the special case $(\alpha, \beta) = (1,0)$
known as the $z^+$-rule, which only considers the positive weights in the relevance propagation.
Our game extends it with an opponent player (enabling negative-weight influence) and replaces
input-dependent transition probabilities with soft Bellman equations.
Deep
Taylor~\citep{Montavon2017} justifies that same $z^+$-rule and explicitly leaves the mathematical
modeling of the general $\alpha\beta$ rule as an open problem; \ac{RG} closes this.

\textbf{Shapley-value attribution methods.} SHAP and KernelSHAP frame explanation as a cooperative
game over feature coalitions~\citep{Lundberg2017} in a \emph{black-box} setting; they are therefore
a distinct, complementary game that explains marginal input-feature contributions relative to a
baseline, whereas our backward game explains how relevance is routed through network internals in a
\emph{white-box} setting.

\textbf{Robust explanations via Softplus.} \citet{Dombrowski2020} propose Softplus replacement as a
curvature-based defense against gradient-explanation manipulation.
Appendix~\ref{app:softplus-stopping-game} shows that entropy-regularizing the activation-state
stop/continue choice in \ac{SG} recovers Softplus-network gradients, situating their construction
within our framework.

\textbf{Transformer and ViT explainability.} Attention Rollout / Flow read attention as a routing
graph~\citep{Abnar2020attention}; transformer-\ac{LRP} methods propagate through attention and skip
connections~\citep{Chefer2021transformer, Ali2022, achtibat2024attnlrp}; DAVE addresses ViT
tokenization artifacts via conditioned operators and equivariant averaging~\citep{Wrobel2024}. The
conditioned-operator transformer constructions arise naturally inside our game semantics: the
QK-softmax is the oracle player's reference policy, while attribution lives on the Value-routing
trajectories. Hence, our method is compatible with Reynolds averaging, which in our \ac{RG}+Eq mode
removes the same grid artifact that DAVE targets---however, with more control via the different
game modes.
\section{Backward Attribution}\label{sec:background}

Consider a ReLU network with layers $l = 1,\dots,L$, input $a^{(0)} = x$, pre-activations
$z^{(l)}_j = \sum_i a^{(l-1)}_i W^{(l)}_{ij} + b^{(l)}_j$ and activations $a^{(l)}_j =
    \relu(z^{(l)}_j)$, where $W^{(l)}_{ij}$ connects neuron $i$ of layer $l{-}1$ to neuron $j$ of layer
$l$. We explain a scalar output $f(x) = a^{(L)}_u$. We write $[u]^{+} = \max(u,0)$ and $[u]^{-} =
    \max(-u,0)$, so $u = [u]^+ - [u]^-$, and $W^{(l,\pm)} = [W^{(l)}]^{\pm}$ entrywise.

\paragraph{Layer-wise relevance propagation.} \ac{LRP}~\citep{Bach2015} explains $f(x)$ by redistributing it backward, layer by layer: starting
from $R^{(L)}_u = f(x)$, every neuron passes its relevance $R^{(l)}_j$ on to its inputs in
proportion to their contributions $a^{(l-1)}_i W^{(l)}_{ij}$ to its pre-activation, so that
relevance is conserved from layer to layer up to the share absorbed by biases and stabilizers.
Rules differ in how the contributions are weighted. The $\alpha\beta$-rule is parameterized by
$\alpha,\beta \ge 0$, with $\alpha-\beta = 1$ for conservation (attributions sum to $f(x)$), and a
stabilization parameter $\varepsilon \ge 0$ \ch{(added to each sign denominator, as in common
implementations~\citep{Anders2021zennit})}: the relevance of neuron $j$ is split into a
positive-contribution fraction scaled by $\alpha$ and a negative-contribution fraction scaled by
$\beta$, each normalized by its own denominator,
\begin{align}\label{eq:lrp-ab}
    R^{(l-1)}_i &= \textstyle\sum_j \Bigl( \alpha\, \tfrac{[a^{(l-1)}_i
    W^{(l)}_{ij}]^+}{\Sigma^{(l,+)}_j + \varepsilon} - \beta\, \tfrac{[a^{(l-1)}_i
    W^{(l)}_{ij}]^-}{\Sigma^{(l,-)}_j + \varepsilon} \Bigr) R^{(l)}_j, \nonumber\\
    \Sigma^{(l,\pm)}_j &= \textstyle\sum_k [a^{(l-1)}_k W^{(l)}_{kj}]^{\pm} + [b^{(l)}_j]^{\pm}.
\end{align}
Relevance assigned to the bias and to $\varepsilon$ is absorbed. $(\alpha,\beta) = (1,0)$ gives the
$z^+$-rule of Deep Taylor decomposition~\citep{Montavon2017}, which turns out to be the same rule
as Excitation Backprop~\citep{Zhang2018excitation}; other rules (LRP-$\varepsilon$, LRP-$\gamma$)
change the redistribution fractions.

\paragraph{Transformers.} In a Vision Transformer (\acs{ViT}), \ch{LayerNorm couples the features of a token through its
input-dependent mean and variance,} and the attention matrix $A = \softmax(QK^\top/\sqrt{d})$
depends on the input, which breaks conservation~\citep{Ali2022}.
\ch{Transformer-specific rules therefore \emph{condition} on the forward pass: conservative
propagation~\citep{Ali2022} and DAVE~\citep{Wrobel2024} hold $A$ and the LayerNorm statistics fixed
and propagate relevance through the Value path only;
AttnLRP~\citep{achtibat2024attnlrp} also propagates through the softmax and the QK product.}

All these rules share one structure: a backward recursion that routes mass from each neuron to its
inputs. The games below make the routing explicit.
\section{Attribution as a Game}\label{sec:game}

\paragraph{Setup and common state-space structure.}
With the notation of Section~\ref{sec:background}, we build a finite-horizon two-player
game~\citep{OsborneRubinstein1994} played \emph{backward} from the output: a state records a
position in the backward traversal (layer, neuron, plus game-specific bookkeeping for player turn
and stream sign), and at each non-terminal state the active player either stops, moving to an
absorbing cemetery state~$\dagger$, or continues. Skip connections become random transitions to
summand states; max-pooling sends players greedily to maximal-payoff states.
Appendix~\ref{app:notation} collects the bookkeeping labels (Table~\ref{tab:notation}).
The \ac{SG} value $V_p(s)$ is the $\gamma$-discounted expected payoff difference $r_p - r_{p'}$ of
player $p$ and opponent $p'$ along trajectories from $s$ under the policies, expressed for the
player $p$ in turn at $s$ by the Bellman recursion
\begin{align}\label{eq:bellman}
    V_p(s) = \max_{a \in \mathcal{A}_x(s)} \Big\{ &r_p(s,a) - r_{p'}(s,a) \nonumber\\
    &+ \gamma(s) \textstyle\sum_{s'} \P_x(s' \mid s,a)\, V_p(s') \Big\},
\end{align}
whereas the \acl{RG} replaces the hard maximization at routing states by an entropy-regularized
maximization over mixed routing decisions, whose value is the soft game
value~\eqref{eq:soft-bellman}.
Any player policies induce a backward trajectory measure $\Pi_x$ on trajectories $\zeta=(\zeta_0,
    \zeta_1, \dots)$ with a game-specific cumulative discount $d_t(\zeta)$ up to time $t$.
The \textbf{discounted occupation measure}~\citep{Puterman1994}
\begin{align}\label{eq:attribution}
    \Gamma_x(v) \coloneqq \E_{\zeta \sim \Pi_x}\!\left[ \textstyle\sum_t d_t(\zeta)\,\1\{\zeta_t = v\}
    \right]
\end{align}
quantifies how much weight the trajectories going through a state have.

\subsection{The Stopping Game}\label{sec:stopping}

With each neuron $j$ in layer $l$, we associate two states $s^{(l)}_{j,+}$ and $s^{(l)}_{j,-}$,
where the sign encodes which player is in turn.
The biases become the immediate payoffs for continuing; under continuation, weight magnitudes drive
the environment transitions with discount $\gamma_j^{(l)} = \sum_i |W_{ij}^{(l)}|$ (positive- and
negative-weight predecessors follow $\P \propto W^{(l,\pm)}/\gamma^{(l)}_j$, and a negative weight
passes the turn to the opponent).
The players' expected payoffs realize the novel Stopping Decomposition, separating positive and
negative paths to each neuron activation (Lemma~\ref{lem:parity-path}). The gradient is then the
difference of two \emph{non-negative} occupation measures:

\begin{theorem}[Gradient representation]\label{thm:gradient}
The Stopping Game allows a subgame-perfect Nash equilibrium, unique up to tie-breaking at
zero-activation neurons.
Under the stop-at-tie-breaks equilibrium the discounted occupation measure $\Gamma_x(s_{i,p}^{(l)})
    \ge 0$ equals the output derivative with respect to the corresponding \emph{Stopping Decomposition}
activation.
The gradient of $f$ with respect to the input is then the difference $\Gamma_x(s_{i,+}^{(0)}) -
    \Gamma_x(s_{i,-}^{(0)})$.
\end{theorem}
Appendix~\ref{app:relu-net-game} proves the theorem, works through an example, gives the full state
space with skip connections and max-pooling, and compares the Stopping Decomposition with the
\ac{DC} decomposition.

Modeling $\alpha\beta$-\ac{LRP} or attention modules via the \ac{SG} would require
activation-dependent transitions. The \ac{RG} moves activation-dependence into the players'
policies, opening a continuous family of deformations.

\subsection{The Routing Game}\label{sec:routing}

In the \ac{RG} we model every neuron by three game states: an activation state
$s^{(l,\mathrm{act})}_{j,q}$ and two linear routing states $s^{(l,\mathrm{lin})}_{j,q,+}$ and
$s^{(l,\mathrm{lin})}_{j,q,-}$, the $+$ stream state and the $-$ stream state, modeling the
positive part $W^{(l,+)}$ and the negative part $W^{(l,-)}$ of the affine layer respectively. At an
activation state the player in turn stops or continues to both streams with probability $\tfrac12$
each; the $-$ stream passes the turn to the opponent, \ch{and discounting the game value by $2$}
compensates the split.
At a \emph{linear routing state} $s^{(l,\mathrm{lin})}_{j,q,\kappa}$ for neuron $j$ in layer $l$ on
sign stream $\kappa \in \{+,-\}$ with player $q$ in turn, the active player mixes over predecessors
$i$, a bias node and an outside option $\perp_{\mathrm{lin}}$ under a policy
$\pi_{q}(s_{j,q,\kappa}^{(l,\text{lin})})$; the bias node and $\perp_{\mathrm{lin}}$ absorb the
trajectory, with log-payoffs $\log[b^{(l)}_j]^\kappa$ and $\log\varepsilon$.
Writing the game in log space turns the active player $r$'s advantage of moving to predecessor $i$
into an additive action--value,
\begin{align}\label{eq:q-value}
    Q_r\bigl(s^{(l,\mathrm{lin})}_{j,q,\kappa}, i\bigr) = \begin{cases} \log W_{ij}^{(l,\kappa)} {+} V_r(s^{(l-1,\mathrm{act})}_{i,q}), & r{=}q,\, l{\ge}2,\\[0.2em] \log \bigl[x_i W_{ij}^{(1)}\bigr]^{\kappa}, & r{=}q,\, l{=}1, \\
        -\infty, & r \neq q,
    \end{cases}
\end{align}
where $V_r(\cdot)$ denotes the \emph{game value} at the successor activation state from the
perspective of player $r$. At each linear routing subgame we regularize the players' policies with
Shannon entropy~\citep{Geist2019regularized}, which yields the following equilibrium value and
Gibbs best response:
\begin{align}
    V_q\!\bigl(s^{(l,\mathrm{lin})}_{j,q,\kappa}\bigr) &= \tau \log \textstyle\sum_i
    \exp\!\bigl(Q_q(s^{(l,\mathrm{lin})}_{j,q,\kappa}, i)/\tau\bigr), \label{eq:soft-bellman}\\
    \pi_{q}(s_{j,q,\kappa}^{(l,\text{lin})})(i) &\propto
    \exp\!\bigl(Q_q(s^{(l,\mathrm{lin})}_{j,q,\kappa}, i)/\tau\bigr), \label{eq:gibbs-policy}
\end{align}
where the sums include the bias node and $\perp_{\mathrm{lin}}$. At $\tau>0$ the game admits an
SPNE whose occupation measure at $\tau=1$ is the $\alpha\beta$-\ac{LRP} attribution:

\begin{theorem}[$\alpha\beta$-\ac{LRP} representation]\label{thm:lrp-recovery}
For $\tau>0$, the Gibbs policies of Eq.~\eqref{eq:gibbs-policy} and stopping decisions according to
neuron activations yield SPNEs of the entropy-regularized \ac{RG}, unique up to tie-breaking at
zero-activation neurons. Under the stop-at-tie-breaks equilibrium at $\tau=1$ the
$\alpha\beta$-\ac{LRP} attribution map in~\eqref{eq:lrp-ab} equals $R_i^{(l)} =
    f(x)\,\bigl(\Gamma_x(s_{i,+}^{(l)}) - \Gamma_x(s_{i,-}^{(l)})\bigr)$.
\end{theorem}
Appendix~\ref{app:routing-formal} gives the full game definition, the proofs, and a worked example
of the local routing dynamics (Figure~\ref{fig:Routing-Game-illustration}).

In practice we use $\tau$ as a \emph{myopic} control: each linear subgame applies the $\tau$-scaled
Gibbs policy to its own action values, but inherits game values from the soft Bellman equations at
$\tau{=}1$. We do not re-solve at $\tau$ downstream: that would change the forward activations,
while we want trajectories that explain the network's actual forward pass. The myopic setting at
$\tau\neq 1$ is a controlled temperature deformation of the occupation measure.
Low $\tau$ concentrates routing on the dominant pathway.

\subsection{Attention via a detached-logit oracle}\label{sec:attention}

Two \ac{ViT} subtleties break conservative transport: \ch{LayerNorm's input-dependent statistics}
and the data-dependent QK-softmax. LayerNorm statistics are detached so the layer becomes a pointwise linear
map~\citep{Ali2022, Wrobel2024}; the QK-softmax block is replaced by an \emph{oracle game} emitting
a detached reference policy, the softmax row $A_{t\cdot}$, with occupation-measure trajectories
flowing only through the Value route. \ch{Like conservative propagation~\citep{Ali2022}, this explains the network conditioned on its
realized attention pattern.}

\begin{theorem}[Attention, informal]\label{thm:attention}
With reference policy $A_{t\cdot}$, the KL-regularized Value-routing subgame has an SPNE whose
occupation measure at the block input equals $\alpha\beta$-\ac{LRP} on the conditioned effective
map $x \mapsto \sum_k A_{tk}\, x_k W_v + b_v$ (Theorem~\ref{thm:attn-ab-recovery}).
\end{theorem}
Our implementation applies this rule to the attention product and the Value projection as one
composite step; a full gradient fallback through the attention block degrades localization (AL
$0.624$ vs.\ \vitRGlocAL{} at the localization-selected configuration). Equivariant Reynolds
averaging optionally suppresses tokenization-grid artifacts~\citep{Wrobel2024} (\ac{RG}+Eq). Full
KL formulation, oracle embedding, occupation-measure theorem and implementation are in
Appendix~\ref{app:attention-subgame}.
\section{Game-Theoretic Modes}\label{sec:modes}\label{sec:controls}

Game-theoretic concepts like risk aversion and entropy regularization map back to novel modes of
the backward rules for probing and tuning explanation properties (Table~\ref{tab:modes}). The exact
representations of Section~\ref{sec:game} are the defaults; every mode other than the \ac{LRP}
parameters $(\alpha,\beta,\varepsilon)$ is a deformation that is off at the anchor.
\ch{We use five:
the temperature $\tau$ (focus), risk aversion $\lambda$ at activation states and its
attention-logit variant $\lambda_{\mathrm{sm}}$, the probit gate $\tilde\Phi$ (gate smoothing),
and equivariant averaging over $N$ input transformations (grid artifacts).}

\ch{\emph{Focus.} The myopic temperature of Section~\ref{sec:routing} rescales the Gibbs policy at every linear state while keeping the $\tau{=}1$ game values. Low $\tau$ concentrates routing on dominant predecessors; high $\tau$ spreads mass more uniformly, exposing broader contextual features as well as architectural priors.}
\ch{\emph{Risk aversion.} \Acl{ADF} (\acs{ADF};~\citealp{gast2018lightweight}) propagates a
Gaussian mean $\mu$ and variance $\sigma^2$ for every pre-activation $z$ from an input variance
$\sigma^2$ in one additional moment pass. The risk-averse player replaces the observed $z$ at
activation states by the lower-confidence bound $\mu+\lambda\sigma$ with $\lambda<0$, the
pessimistic counterpart of confidence-bound exploration~\citep{auer2002finite, zhang2021risk}, and
so attenuates routes through high-variance neurons; linear-state routing is unchanged. On attention
blocks, $\lambda_{\mathrm{sm}}$ applies the same shift to the oracle's attention logits before the
softmax, downweighting high-variance key tokens.}
\ch{\emph{Gate smoothing.} The risk-averse stop/continue rule is the hard gate
$\1[\mu/\sigma+\lambda>0]$. The probit gate $\tilde\Phi$ lets the player instead continue with
probability $\Phi(\mu/\sigma+\lambda)$, the marginal of a player who perceives the continuation
payoff as $z+\eta$, $\eta\sim\mathcal N(0,1)$, and continues iff it is positive. At $\mu{=}z$,
$\sigma{=}1$, $\lambda{=}0$ this probit-regularized choice recovers the Gaussian-smoothed ReLU
$\E[\relu(z+\eta)]$, just as entropy regularization recovers
Softplus (Appendix~\ref{app:softplus-stopping-game}); on ReLU networks it is a smoothing
deformation. Because the gate uses the \ac{ADF} $\sigma$, $\sigma^2$ is active whenever
$\tilde\Phi$ is on.}
\ch{\emph{Exact modes and cost.} $(\alpha,\beta)$ weight the two streams of the occupation
measure, and the \ac{LRP} stabilizer $\varepsilon$ adds an outside option $\perp_{\mathrm{lin}}$
with log-score $\log\varepsilon$ at every linear state, leaking weak-evidence neurons into the
cemetery while leaving strong-evidence ones near-unaffected (Theorem~\ref{thm:ab-recovery}). The
exact modes and $\tau$ add no cost; $\lambda$, $\lambda_{\mathrm{sm}}$ and $\tilde\Phi$ add one
\ac{ADF} pass ($\vitRGanchorTime\to\vitRGlocTime$\,s per ViT-B/16 image); averaging multiplies the
cost by $N$ ($4.93$\,s at $N{=}50$, $\vitRGEqTime$\,s with $\lambda_{\mathrm{sm}}$).
Appendix~\ref{app:modes} derives the modes and gives defaults, costs and an order of use.}

\begin{table*}[t]
    \centering
    \caption{Modes of the games. ``Exact'': the mode stays within a represented rule; ``deform.'':
    optional deformation, off (default value) at the anchor. $\sigma^2$: input variance of the \ac{ADF}
    moment propagation; $N$: number of transformed inputs averaged by \ac{RG}+Eq.}\label{tab:modes}
    \small
    \setlength{\tabcolsep}{4pt}
    \begin{tabularx}{\textwidth}{@{}>{\raggedright\arraybackslash}p{0.08\textwidth}>{\raggedright\arraybackslash}X>{\raggedright\arraybackslash}p{0.08\textwidth}>{\raggedright\arraybackslash}p{0.24\textwidth}@{}}
        \toprule
        \textbf{Mode} & \textbf{Effect on the game} & \textbf{Default} & \textbf{Property} \\
        \midrule
        $(\alpha,\beta)$ & branch discounts of the two streams (exact) & $(2,1)$ & support / oppose \\
        $\varepsilon$ & outside option $\perp_{\mathrm{lin}}$, payoff $\log\varepsilon$ (exact) & $0.5$ &
        damps low-inflow neurons \\
        $\tau$ & entropy regularization of routing (deform.) & $1$ & sharp / diffuse \\
        $\lambda$, $\sigma^2$ & values $\to$ lower-confidence bounds (deform.) & $\lambda{=}0$ & noise /
        fragility \\
        $\lambda_{\mathrm{sm}}$ & attention logits shifted by $\lambda_{\mathrm{sm}}$ times a per-key
        \ac{ADF} std (deform.) & $0$ & stable attention \\
        $\tilde\Phi$ & noise on player observations (deform.) & off & gate smoothing \\
        $N$ & average over input transformations (deform.) & off & grid artifacts \\
        \bottomrule
    \end{tabularx}
\end{table*}

\section{Trajectory Sanity Check}\label{sec:hellinger-sanity}\label{sec:hellinger}

The cascading-randomization sanity check of~\citet{Adebayo2018sanity} asks whether an explanation
meaningfully depends on the learned parameters, and answers by comparing attribution maps in pixel
space. But the object of interest is the backward calculation itself---a distribution $\Pi$ over
graph trajectories---and the attribution map is only its pixel projection.
The \textbf{Hellinger distance}~\citep{Hellinger1909} $\mathrm{H}(\Pi_A, \Pi_B) = (1 -
    \mathrm{BC})^{1/2}$, with Bhattacharyya coefficient~\citep{Bhattacharyya1943} $\mathrm{BC} =
    \sum_\zeta \sqrt{\Pi_A(\zeta)\,\Pi_B(\zeta)}$, measures the routing logic directly rather than its
visual projection. This section probes the representation claim: parameter randomization should
move $\Pi$ even when its pixel projection remains stable. We remark that $\mathrm{H}=0$ iff $\Pi_A
    = \Pi_B$, and the metric takes its maximal value $\mathrm{H}=1$ iff their supports are disjoint.
Unlike KL-type divergences, $\mathrm{H}$ stays finite when randomization opens or closes routes,
and since both laws are layered Markov chains on the same state graph, BC factorizes along layers:
one backward pass computes $\mathrm{H}$ exactly, in time linear in the number of edges, and
$\mathrm{H}^2$ decomposes into per-input-state contributions (Appendix~\ref{app:path-divergence}).

\subsection{Trajectories change, maps need not}\label{sec:hellinger-cascading}

\begin{figure*}[t]
    \centering
    \input{figures/plots/sanity_check/hellinger_live_plots.inc.tex}
    \caption{Hellinger trajectory distance under cascading parameter
    randomization~\citep{Adebayo2018sanity}. Hellinger distance $\mathrm{H}$ (vertical axis, from
    $0.7$) against the number of cascadingly re-initialized layers, output to input (horizontal axis;
    $\mathrm{H}=0$ before the first layer is randomized). Solid: plain $\mathrm{H}$; dotted:
    $\mathrm{H}_{\mathrm{surv}}$, conditioned on survival to an ordinary input-layer state. Thin
    horizontal reference lines mark the fully-random-model values from
    Table~\ref{tab:double-random-hellinger} (solid: plain $\mathrm{H}$; dotted:
    $\mathrm{H}_{\mathrm{surv}}$), in the color of the corresponding game.}
    \label{fig:hellinger-sanity}
\end{figure*}

Under cascading randomization (layers cumulatively re-initialized output$\to$input), both \ac{SG}
and \ac{RG} diverge from the first randomized layer and plateau near the fully-random-model
reference of Table~\ref{tab:double-random-hellinger} (Figure~\ref{fig:hellinger-sanity}). The
reference itself sits below the ceiling $\mathrm{H}=1$ because two fully randomized models still
share some short trajectories: both may stop early at structurally dead neurons and route mass to
the common cemetery state. Such short paths have a higher chance of coinciding across models than
full trajectories reaching the input layer. Conditioning on survival removes this shared cemetery
mass and gives the strict ceiling $\mathrm{H}_{\mathrm{surv}} \approx 1$. Trajectory- and map-level
readouts are per-image orthogonal (Spearman $|\rho|\le \perImageCrossMax$ between $\mathrm{H}$ and
pixel-level Gradient or $\alpha\beta$-\ac{LRP} similarity): images with similar attribution maps
need not have similar calculations. Per-game values, the conditioning, the floor mechanism and the
per-image analysis are in Appendix~\ref{app:sanity-check}.

\paragraph{What the map hides.}\label{sec:pixel-sanity}
Cascading randomization on the maps reproduces~\citet{Adebayo2018sanity,
sixt2020whenexplanationslie}: some methods retain map structure under full randomization. On
ViT-B/16, the \ac{RG} anchor itself retains Spearman $|\rho| = \sanityVitRGanchorRho$, the highest
of all methods, and HOG $r = \sanityVitRGanchorHOG$; on VGG-16, LRP-$\varepsilon$ retains
$\sanityVggLRPepsRho$ and $\sanityVggLRPepsHOG$ (Table~\ref{tab:sanity-check}). The
localization-selected \ac{RG} modes are slightly more model sensitive than the anchor on all three
models. \ch{Read against Figure~\ref{fig:hellinger-sanity}, the two readouts disagree: $\mathrm{H}$ sits at
the fully random reference while the projection retains structure. The map-level failure stands;
$\mathrm{H}$ locates the loss between calculation and projection.}
\section{Case Study}\label{sec:experiments}\label{sec:case-study}

\begin{table*}[t]
    \centering
    \caption{ViT-B/16 on the ImageNet-S test split ($n{=}565$; MaxS/AvgS on a fixed random subset of
    200 test images). \textbf{Bold}, dark shading: best in column; $^{\approx}$, light shading: not
    significantly different from the best (paired bootstrap / McNemar, Holm-corrected, $p\ge0.05$);
    small numbers: 95\% percentile-bootstrap confidence intervals over test images
    (Appendix~\ref{app:stats}). One row per method, selected on the validation split by the
    localization rank-sum (faithfulness-selected rows and further baselines:
    Appendix~\ref{app:results}); the RG anchor is the exact $\alpha\beta$-\ac{LRP} representation
    (Theorem~\ref{thm:lrp-recovery}). $t$/img: seconds per image on one GPU. Config: SmoothGrad
    averages $N$ samples at noise level $\sigma$ relative to the input range; LRP-$\gamma$ has
    stabilizer $\varepsilon$ and positive-weight boost $\gamma$; AttnLRP has layer-type parameters
    $\gamma_{\mathrm{conv}}, \gamma_{\mathrm{lin}}, \varepsilon_{\mathrm{attn}}$; KernelSHAP draws $N$
    samples for $M$ superpixels; \ac{RG} modes as in Table~\ref{tab:modes}.\PrelimTag}
    \label{tab:eval-vit}
    \scriptsize
    \resizebox{\textwidth}{!}{\begin{tabular}{@{}l l c c c c c c c c c@{}}
\toprule
Method & Config & AL$\uparrow$ & PG$\uparrow$ & TK$\uparrow$ & PF$_5$$\downarrow$ & PF$_{20}$$\downarrow$ & Sel$\downarrow$ & MaxS$\downarrow$ & AvgS$\downarrow$ & $t$/img \\
\midrule
Gradient & -- & \begin{tabular}[t]{@{}c@{}}0.411\\[-1pt]{\tiny[0.389,0.432]}\end{tabular} & \begin{tabular}[t]{@{}c@{}}0.361\\[-1pt]{\tiny[0.322,0.402]}\end{tabular} & \begin{tabular}[t]{@{}c@{}}0.386\\[-1pt]{\tiny[0.359,0.413]}\end{tabular} & \begin{tabular}[t]{@{}c@{}}8.45\\[-1pt]{\tiny[8.33,8.56]}\end{tabular} & \begin{tabular}[t]{@{}c@{}}8.08\\[-1pt]{\tiny[7.94,8.22]}\end{tabular} & \begin{tabular}[t]{@{}c@{}}6.033\\[-1pt]{\tiny[5.905,6.154]}\end{tabular} & \begin{tabular}[t]{@{}c@{}}1.290\\[-1pt]{\tiny[1.233,1.351]}\end{tabular} & \begin{tabular}[t]{@{}c@{}}0.922\\[-1pt]{\tiny[0.893,0.953]}\end{tabular} & \textbf{0.02} \\
SmGrad & $N{=}50,\sigma/\mathrm{range}{=}0.2$ & \begin{tabular}[t]{@{}c@{}}0.517\\[-1pt]{\tiny[0.498,0.536]}\end{tabular} & \begin{tabular}[t]{@{}c@{}}0.706\\[-1pt]{\tiny[0.667,0.743]}\end{tabular} & \begin{tabular}[t]{@{}c@{}}0.675\\[-1pt]{\tiny[0.654,0.694]}\end{tabular} & \begin{tabular}[t]{@{}c@{}}8.10\\[-1pt]{\tiny[7.96,8.24]}\end{tabular} & \begin{tabular}[t]{@{}c@{}}7.02\\[-1pt]{\tiny[6.83,7.20]}\end{tabular} & \begin{tabular}[t]{@{}c@{}}5.029\\[-1pt]{\tiny[4.894,5.163]}\end{tabular} & \begin{tabular}[t]{@{}c@{}}0.597\\[-1pt]{\tiny[0.569,0.630]}\end{tabular} & \begin{tabular}[t]{@{}c@{}}0.468\\[-1pt]{\tiny[0.451,0.485]}\end{tabular} & 1.12 \\
\midrule
LRP-$\gamma$ & $\varepsilon{=}0.1,\gamma{=}2$ & \begin{tabular}[t]{@{}c@{}}0.644\\[-1pt]{\tiny[0.628,0.660]}\end{tabular} & \begin{tabular}[t]{@{}c@{}}0.873\\[-1pt]{\tiny[0.844,0.899]}\end{tabular} & \begin{tabular}[t]{@{}c@{}}0.816\\[-1pt]{\tiny[0.800,0.832]}\end{tabular} & \begin{tabular}[t]{@{}c@{}}7.93\\[-1pt]{\tiny[7.78,8.08]}\end{tabular} & \begin{tabular}[t]{@{}c@{}}6.64\\[-1pt]{\tiny[6.44,6.83]}\end{tabular} & \cellbest{\begin{tabular}[t]{@{}c@{}}\textbf{4.451}\\[-1pt]{\tiny[4.324,4.572]}\end{tabular}} & \begin{tabular}[t]{@{}c@{}}0.300\\[-1pt]{\tiny[0.287,0.314]}\end{tabular} & \begin{tabular}[t]{@{}c@{}}0.218\\[-1pt]{\tiny[0.210,0.226]}\end{tabular} & 0.08 \\
AttnLRP & $\gamma_{\mathrm{conv}}{=}0.25,\gamma_{\mathrm{lin}}{=}0.05,\varepsilon_{\mathrm{attn}}{=}0.25$ & \begin{tabular}[t]{@{}c@{}}0.644\\[-1pt]{\tiny[0.628,0.660]}\end{tabular} & \begin{tabular}[t]{@{}c@{}}0.855\\[-1pt]{\tiny[0.825,0.883]}\end{tabular} & \begin{tabular}[t]{@{}c@{}}0.809\\[-1pt]{\tiny[0.794,0.825]}\end{tabular} & \cellbest{\begin{tabular}[t]{@{}c@{}}\textbf{7.77}\\[-1pt]{\tiny[7.61,7.93]}\end{tabular}} & \cellbest{\begin{tabular}[t]{@{}c@{}}\textbf{6.43}\\[-1pt]{\tiny[6.23,6.62]}\end{tabular}} & \begin{tabular}[t]{@{}c@{}}4.582\\[-1pt]{\tiny[4.452,4.706]}\end{tabular} & \begin{tabular}[t]{@{}c@{}}0.456\\[-1pt]{\tiny[0.430,0.484]}\end{tabular} & \begin{tabular}[t]{@{}c@{}}0.342\\[-1pt]{\tiny[0.324,0.364]}\end{tabular} & 0.10 \\
\midrule
TIBAV & -- & \begin{tabular}[t]{@{}c@{}}0.615\\[-1pt]{\tiny[0.597,0.634]}\end{tabular} & \begin{tabular}[t]{@{}c@{}}0.878\\[-1pt]{\tiny[0.851,0.904]}\end{tabular} & \begin{tabular}[t]{@{}c@{}}0.782\\[-1pt]{\tiny[0.759,0.805]}\end{tabular} & \begin{tabular}[t]{@{}c@{}}8.13\\[-1pt]{\tiny[8.01,8.24]}\end{tabular} & \begin{tabular}[t]{@{}c@{}}7.60\\[-1pt]{\tiny[7.46,7.74]}\end{tabular} & \begin{tabular}[t]{@{}c@{}}4.906\\[-1pt]{\tiny[4.774,5.033]}\end{tabular} & \begin{tabular}[t]{@{}c@{}}1.062\\[-1pt]{\tiny[1.001,1.126]}\end{tabular} & \begin{tabular}[t]{@{}c@{}}0.512\\[-1pt]{\tiny[0.486,0.540]}\end{tabular} & 0.15 \\
GAE & -- & \begin{tabular}[t]{@{}c@{}}0.533\\[-1pt]{\tiny[0.514,0.551]}\end{tabular} & \begin{tabular}[t]{@{}c@{}}0.743\\[-1pt]{\tiny[0.708,0.779]}\end{tabular} & \begin{tabular}[t]{@{}c@{}}0.649\\[-1pt]{\tiny[0.623,0.677]}\end{tabular} & \begin{tabular}[t]{@{}c@{}}8.13\\[-1pt]{\tiny[8.01,8.25]}\end{tabular} & \begin{tabular}[t]{@{}c@{}}7.62\\[-1pt]{\tiny[7.47,7.76]}\end{tabular} & \begin{tabular}[t]{@{}c@{}}5.098\\[-1pt]{\tiny[4.961,5.231]}\end{tabular} & \begin{tabular}[t]{@{}c@{}}0.525\\[-1pt]{\tiny[0.494,0.558]}\end{tabular} & \begin{tabular}[t]{@{}c@{}}0.368\\[-1pt]{\tiny[0.346,0.391]}\end{tabular} & 0.03 \\
\midrule
KernelSHAP & $N{=}2M{+}2048$, 80 superpixels & \begin{tabular}[t]{@{}c@{}}0.565\\[-1pt]{\tiny[0.548,0.583]}\end{tabular} & \cellbest{\begin{tabular}[t]{@{}c@{}}\textbf{0.961}\\[-1pt]{\tiny[0.945,0.977]}\end{tabular}} & \cellapprox{\begin{tabular}[t]{@{}c@{}}0.823\textsuperscript{$\approx$}\\[-1pt]{\tiny[0.799,0.848]}\end{tabular}} & \begin{tabular}[t]{@{}c@{}}8.29\\[-1pt]{\tiny[8.18,8.40]}\end{tabular} & \begin{tabular}[t]{@{}c@{}}7.50\\[-1pt]{\tiny[7.35,7.65]}\end{tabular} & \begin{tabular}[t]{@{}c@{}}5.140\\[-1pt]{\tiny[5.012,5.261]}\end{tabular} & \begin{tabular}[t]{@{}c@{}}0.897\\[-1pt]{\tiny[0.868,0.928]}\end{tabular} & \begin{tabular}[t]{@{}c@{}}0.650\\[-1pt]{\tiny[0.632,0.667]}\end{tabular} & 6.71 \\
\midrule
RG anchor & $\tau{=}1$ & \begin{tabular}[t]{@{}c@{}}0.509\\[-1pt]{\tiny[0.490,0.528]}\end{tabular} & \begin{tabular}[t]{@{}c@{}}0.724\\[-1pt]{\tiny[0.685,0.759]}\end{tabular} & \begin{tabular}[t]{@{}c@{}}0.699\\[-1pt]{\tiny[0.678,0.720]}\end{tabular} & \begin{tabular}[t]{@{}c@{}}8.19\\[-1pt]{\tiny[8.05,8.31]}\end{tabular} & \begin{tabular}[t]{@{}c@{}}6.97\\[-1pt]{\tiny[6.78,7.15]}\end{tabular} & \begin{tabular}[t]{@{}c@{}}4.854\\[-1pt]{\tiny[4.732,4.972]}\end{tabular} & \cellbest{\begin{tabular}[t]{@{}c@{}}\textbf{0.166}\\[-1pt]{\tiny[0.161,0.172]}\end{tabular}} & \begin{tabular}[t]{@{}c@{}}0.120\\[-1pt]{\tiny[0.117,0.124]}\end{tabular} & 0.10 \\
RG & $\tau{=}0.4,\lambda_{\mathrm{sm}}{=}{-}10,\sigma^2{=}5,\varepsilon{=}0.1$ & \cellbest{\begin{tabular}[t]{@{}c@{}}\textbf{0.711}\\[-1pt]{\tiny[0.692,0.730]}\end{tabular}} & \begin{tabular}[t]{@{}c@{}}0.825\\[-1pt]{\tiny[0.793,0.855]}\end{tabular} & \begin{tabular}[t]{@{}c@{}}0.788\\[-1pt]{\tiny[0.768,0.808]}\end{tabular} & \begin{tabular}[t]{@{}c@{}}8.27\\[-1pt]{\tiny[8.14,8.39]}\end{tabular} & \begin{tabular}[t]{@{}c@{}}7.48\\[-1pt]{\tiny[7.31,7.64]}\end{tabular} & \begin{tabular}[t]{@{}c@{}}5.109\\[-1pt]{\tiny[4.987,5.226]}\end{tabular} & \begin{tabular}[t]{@{}c@{}}0.729\\[-1pt]{\tiny[0.685,0.777]}\end{tabular} & \begin{tabular}[t]{@{}c@{}}0.494\\[-1pt]{\tiny[0.471,0.518]}\end{tabular} & 0.17 \\
RG+Eq anchor & $\tau{=}1,N{=}50$ & \begin{tabular}[t]{@{}c@{}}0.520\\[-1pt]{\tiny[0.501,0.539]}\end{tabular} & \begin{tabular}[t]{@{}c@{}}0.761\\[-1pt]{\tiny[0.726,0.796]}\end{tabular} & \begin{tabular}[t]{@{}c@{}}0.752\\[-1pt]{\tiny[0.729,0.775]}\end{tabular} & \begin{tabular}[t]{@{}c@{}}8.41\\[-1pt]{\tiny[8.29,8.52]}\end{tabular} & \begin{tabular}[t]{@{}c@{}}7.72\\[-1pt]{\tiny[7.56,7.87]}\end{tabular} & \begin{tabular}[t]{@{}c@{}}4.864\\[-1pt]{\tiny[4.742,4.981]}\end{tabular} & \cellapprox{\begin{tabular}[t]{@{}c@{}}0.170\textsuperscript{$\approx$}\\[-1pt]{\tiny[0.165,0.175]}\end{tabular}} & \cellbest{\begin{tabular}[t]{@{}c@{}}\textbf{0.115}\\[-1pt]{\tiny[0.112,0.118]}\end{tabular}} & 4.93 \\
RG+Eq & $\tau{=}0.5,\lambda_{\mathrm{sm}}{=}{-}10,\sigma^2{=}2,N{=}50$ & \begin{tabular}[t]{@{}c@{}}0.698\\[-1pt]{\tiny[0.680,0.717]}\end{tabular} & \begin{tabular}[t]{@{}c@{}}0.855\\[-1pt]{\tiny[0.827,0.883]}\end{tabular} & \cellbest{\begin{tabular}[t]{@{}c@{}}\textbf{0.836}\\[-1pt]{\tiny[0.815,0.857]}\end{tabular}} & \begin{tabular}[t]{@{}c@{}}8.39\\[-1pt]{\tiny[8.28,8.50]}\end{tabular} & \begin{tabular}[t]{@{}c@{}}7.80\\[-1pt]{\tiny[7.65,7.94]}\end{tabular} & \begin{tabular}[t]{@{}c@{}}5.010\\[-1pt]{\tiny[4.889,5.127]}\end{tabular} & \begin{tabular}[t]{@{}c@{}}0.441\\[-1pt]{\tiny[0.421,0.463]}\end{tabular} & \begin{tabular}[t]{@{}c@{}}0.301\\[-1pt]{\tiny[0.290,0.313]}\end{tabular} & 8.66 \\
\bottomrule
\end{tabular}
}
\end{table*}

\paragraph{Setup.}
We evaluate \ac{RG} and \ac{SG} against standard backward baselines (Appendix~\ref{app:baselines})
on VGG-16~\citep{Simonyan2015}, ResNet-50~\citep{He2016resnet}, and
ViT-B/16~\citep{Dosovitskiy2021vit} with torchvision ImageNet-1K
checkpoints~\citep{paszke2019pytorch,Deng2009imagenet,Russakovsky2015ilsvrc}, on a 50-class
ImageNet-S subset~\citep{gao2022luss} at $224{\times}224$ (50-image validation split, 565-image
test split), and on Pascal VOC 2012~\citep{Everingham2010voc}. ImageNet-S provides the pixel masks
required by the localization metrics; ImageNet boxes over-credit mass anywhere inside the box. All
metrics are computed through Quantus~\citep{hedstrom2023quantus} (definitions in
Appendix~\ref{app:metrics}): localization (\ac{AttrLoc}, \ac{PG}, Top-$k$ Intersection (TK);
$\uparrow$), faithfulness (Pixel Flipping area under the curve over the first 5\% and 20\% of flips
(PF$_5$, PF$_{20}$), \ac{Sel}; $\downarrow$), robustness (Maximum Sensitivity (MaxS), Average
Sensitivity (AvgS); $\downarrow$). Baselines run in their reference implementations (lxt, zennit,
captum, pytorch-grad-cam, and the authors' code for \acs{TIBAV} and \acs{GAE}) at the authors'
settings, plus declared grids~\citep{achtibat2024attnlrp, Anders2021zennit, Kokhlikyan2020captum,
Gildenblat2021gradcam, Chefer2021transformer, Chefer2021genatt}. \ac{RG} analytical modes are
reported alongside each configuration: temperature $\tau$, risk aversion $\lambda$ (softmax-shift
variant $\lambda_{\mathrm{sm}}$ on \ac{ViT} attention logits), \ac{LRP} stabilizer $\varepsilon$,
equivariant-Reynolds sample count $N$.

\paragraph{Protocol.}
Methods are evaluated under a two-stage protocol that treats baselines and our games alike. Stage~1
runs a grid search on the validation split (one image per class) for every method with tunable
parameters, including DAVE (Appendix~\ref{app:grids}, Table~\ref{tab:main-search-grid}; \nValGrid{}
configurations in total). Stage~2 picks, for each method, the \emph{loc} configuration with the
best localization rank-sum over $\{\text{AL}, \text{PG}, \text{TK}\}$ and, independently, the
\emph{faith} configuration with the best faithfulness rank-sum over $\{\text{PF}_5, \text{PF}_{20},
    \text{Sel}\}$, and evaluates both once on the test split, with 95\% percentile-bootstrap confidence
intervals~\citep{Efron1993bootstrap} and Holm-corrected~\citep{Holm1979} paired tests against the
best row of each metric (Appendix~\ref{app:stats}); for \ac{RG} / \ac{RG}+Eq we additionally show
the $\alpha\beta$-\ac{LRP} anchor $(\alpha,\beta,\varepsilon,\tau){=}(2,1,0.5,1)$ as a
theorem-driven reference (Theorem~\ref{thm:lrp-recovery}). Localization is a natural target for
modes designed to control routing focus, and object annotations provide external ground truth:
\citet{Kohlbrenner2020} compare \ac{LRP} rule composites by the share of relevance inside
ground-truth bounding boxes, and we use pixel masks. The main table reports the loc row of every
method; since mask agreement is not faithfulness, the faith rows of all methods are reported in
Appendix~\ref{app:results}.

\paragraph{ViT-B/16 (Table~\ref{tab:eval-vit}).}
Under the localization-selected protocol, the \ac{RG} family beats every baseline on attribution
localization, including the prior transformer-specific backward methods
AttnLRP~\citep{achtibat2024attnlrp}, DAVE~\citep{Wrobel2024},
\ac{TIBAV}~\citep{Chefer2021transformer} and \ac{GAE}~\citep{Chefer2021genatt}: the selected mode
reaches AL \vitRGlocAL{} against \vitBestBaseAL{} for the best baseline (\vitBestBaseALname), and
\ac{RG}+Eq \vitRGEqlocAL{}, both significantly ahead of every baseline row after Holm correction.
\ac{RG}+Eq also has the highest TK (\vitRGEqlocTK{}, not significantly above \vitBestBaseTKname{}
at \vitBestBaseTK{}). The pointing game is led by \vitBestBasePGname{} (\vitBestBasePG{}), whose
maps are constant on superpixels: a hit needs only one pixel of the maximal region inside the mask.
Scoring a single pixel drawn from that region instead, KernelSHAP falls below \ac{RG}+Eq, which is
then not significantly below the best baseline (LRP-$\gamma$; Appendix~\ref{app:pg-granularity}).
Faithfulness favors the \ac{LRP} rules: AttnLRP is best on PF$_5$ and PF$_{20}$ and LRP-$\gamma$ on
Sel (PF$_{20}$ \vitBestBasePFtwenty{} versus \vitRGlocPFtwenty{} for \ac{RG}), and the
faithfulness-selected \ac{RG} row narrows but does not close the gap (PF$_{20}$
\vitRGfaithPFtwenty{}). The two anchors are the most robust rows (MaxS \vitRGanchorMaxS{} and
\vitRGEqanchorMaxS{} versus \vitBestBaseMaxS{} for \vitBestBaseMaxSname{}), whereas the
localization-selected mode reaches MaxS \vitRGlocMaxS{}: sharper routing trades robustness and
faithfulness for localization. \ac{ADF} adds modest overhead (\ac{RG}: \vitRGlocTime\,s vs.\
\vitRGanchorTime\,s base), still below sampling baselines like SmoothGrad or IntGrad; Reynolds
averaging is the only expensive variant (\vitRGEqTime\,s), exceeding even the perturbation method
KernelSHAP (\vitSHAPTime\,s).

\paragraph{Higher $\tau$ spreads attribution into context.}
The localization-optimal modes are the sharpest ones: they place mass on the few patches that align
with the segmentation mask. As $\tau$ increases, the same equilibrium spreads routing over more
predecessors; the object remains in focus, but the heatmap also highlights environmental features
the network used. These broader maps can be visually more informative, but on ViT-B/16 attribution
localization worsens monotonically with $\tau$ around the localization-selected setting (AL $0.723$
at $\tau{=}0.3$ down to $0.461$ at $\tau{=}1$), while the faithfulness side is non-monotonic. High
$\tau$ is the clearest example of the gap between metric-optimal and perception-optimal
explanations within the same game.

\paragraph{Cross-dataset transfer (Pascal VOC).}
A natural worry with validation-tuned modes is that the chosen $\tau$, $\lambda_{\mathrm{sm}}$ and
$\sigma^2$ are overfitted to the ImageNet-S 50-class subset. We test transfer on a per-image
disjoint split of Pascal VOC 2012 trainval ($2{,}743$ test (image, foreground-class) samples with
pixel-accurate segmentation masks): \ac{RG} and \ac{RG}+Eq use the \emph{same} mode parameters as
Table~\ref{tab:eval-vit} without any VOC-side re-tuning, while every baseline gets a fresh
per-method validation grid on the new pool. This tests whether the selected modes capture
model-specific rather than dataset-specific structure: the localization lead carries over (AL
\vocRGlocAL{} for \ac{RG} and \vocRGEqlocAL{} for \ac{RG}+Eq versus \vocBestBaseAL{} for
\vocBestBaseALname{}, both significant against every baseline row), while TK and PG are led by
\ac{TIBAV} and the perturbation methods, and faithfulness by the \ac{LRP} rules
(Appendix~\ref{app:voc-cross-data}).

\paragraph{VGG-16 and ResNet-50 (Appendix~\ref{app:results}).}
On the CNNs, where \ac{LRP}-type rules are already strong, the localization-selected \ac{RG} rows
also lead on AL, clearly on VGG-16 (\vggRGlocAL{} versus \vggBestBaseAL{} for \vggBestBaseALname{})
and narrowly but significantly on ResNet-50 (\resRGlocAL{} versus \resBestBaseAL{} for
\resBestBaseALname{}). Pixel flipping favors the \ac{LRP} rules (LRP-$\varepsilon$ on VGG-16,
LRP-$\gamma$ on ResNet-50), while the \ac{CAM} methods and the perturbation methods lead TK and PG.
On VGG-16 the risk-averse \ac{SG} mode is the most robust method (MaxS \vggSGlocMaxS{} versus
\vggBestBaseMaxS{} for \vggBestBaseMaxSname{}).
\section{Discussion \& Conclusion}\label{sec:discussion}

The Stopping and Routing Games unify gradient and $\alpha\beta$-\ac{LRP} attribution as equilibrium
quantities on per-neuron states. Rule variants---temperature $\tau$, risk aversion $\lambda$,
smooth gate $\tilde\Phi$---arise as principled deformations of the game definition around these
exact anchors, and attention modules drop in via an oracle routing policy. This also places
robustness-motivated Softplus replacement~\citep{Dombrowski2020} inside the same picture: entropy
regularization of the activation-state stop/continue choice recovers Softplus-network gradients
(Appendix~\ref{app:softplus-stopping-game}). The games are representation theorems, not mechanistic
inevitabilities: the architecture does not by itself determine one game, policy class, or discount,
and other games may give other insights.

\ch{The modes improve the attribution metrics considered here, but the evidence is specific: exact results cover ReLU networks and attention conditioned on the realized attention pattern, LayerNorm statistics, and GELU gates. Human-mask agreement is only one evaluation axis, and sharper routing can remove network-relevant evidence outside the mask. Within this scope, one adapted ViT-B/16 $\alpha\beta$-\ac{LRP} configuration outperforms all baselines, including prior transformer-specific backward methods, in attribution localization, while the \ac{LRP} rules remain more faithful. Whether this transfers to other modalities, and whether metric gains translate to downstream tasks, remains future work.}

The trajectory distribution $\Pi$ is the object of study; the attribution map is only its pixel
projection. The Hellinger distance $\mathrm{H}(\Pi_A, \Pi_B)$ is computable in one backward pass
\ch{and complements the parameter-insensitivity sanity check~\citep{Adebayo2018sanity,
sixt2020whenexplanationslie} with a calculation-level readout: $\mathrm{H}$ detects randomization
collapse even when projected maps remain similar.}
\ch{Its} empirical utility will require calculation-level benchmarks, but first
results are promising: $\mathrm{H}$ separates intra- from inter-class image pairs for at least one
game on every architecture and increases monotonically with Gaussian input noise
(Appendix~\ref{sec:lmp-experiments}).

\subsection*{AI Use Statement}

In this work, we used generative AI tools for implementing standard baseline methods and evaluation
code, for producing plots and tables of results, for drafting parts of the text, and for editing
(grammar, spelling, word choice, and reformatting).
We have not used generative AI tools to generate research ideas, to formulate the games, or to
produce the theorems and their proofs.
All AI-assisted code was reviewed and tested by the authors, including unit tests that check the
game-theoretic occupation measures against reference gradient and $\alpha\beta$-LRP
implementations; all AI-assisted text was checked by the authors for correctness, and every
reference was verified manually.
We take responsibility for the final content of this work, including text, claims, and artifacts
produced with the aid of generative AI.

\bibliographystyle{plainnat}
\bibliography{references}

\clearpage
\appendix
\thispagestyle{empty}
\onecolumn
\aistatstitle{Playing the Network Backward: A Game-Theoretic\\ Attribution Framework\\
Supplementary Materials}

\section{Notation and Conventions}\label{app:notation}\label{app:notation-identities}

Throughout, $p, q, r \in \{+, -\}$ denote player or turn labels of the two players $P^+$ and $P^-$,
and $p'$ denotes the opponent of $p$ ($+' = -$, $-' = +$); Table~\ref{tab:notation} lists all
symbols, grouped as network, games, Routing Game and attention, modes and smooth activations, and
Hellinger distance; each has a single meaning. We write $\delta_{p,q} = 1$ if $p=q$ and $0$
otherwise, the sign operator $\xi_{p,q} = 2\delta_{p,q} - 1$ ($+1$ if $p=q$, $-1$ if $p \neq q$),
and the XOR operator $p \oplus q = +$ if $p=q$ and $-$ if $p \neq q$; it is commutative and
associative, with $p \oplus p = +$, $p \oplus + = p$, $p \oplus - = p'$, and $p' \oplus q = (p
    \oplus q)'$.
\textbf{Boundary conventions.} We fix $\mathrm{ReLU}'(0) \coloneqq 0$; equivalently, the indicator
$\1[z > 0]$ used throughout returns $0$ at $z = 0$. For a finite tuple $(x_1, \dots, x_m)$,
$\arg\max_k x_k \coloneqq \min\{k : x_k = \max_l x_l\}$, the smallest index attaining the maximum.
We reserve the word \emph{trajectory} for sequences of game states $\zeta = (\zeta_0, \zeta_1,
    \dots)$ and use \emph{path} for sequences of neuron indices.

\begin{table}[H]
    \centering
    \footnotesize
    \caption{Notation. Each symbol has one meaning throughout the paper and appendix, except the local
    attention symbols $Q_t, K_k, e_{tk}, s_{tk}, m, s, c, T_n$ and the token index $t$, defined in
    Appendix~\ref{app:vit-extension}.}\label{tab:notation}
    \setlength{\tabcolsep}{4pt}
    \renewcommand{\arraystretch}{1.12}
    \begin{tabularx}{\textwidth}{@{}l X l X@{}}
        \toprule
        \textbf{Symbol} & \textbf{Meaning} & \textbf{Symbol} & \textbf{Meaning} \\
        \midrule
        $l$, $L$ & layers; input $0$ ($a^{(0)} = x$), output $L$ & $W^{(l)}_{ij}$, $b^{(l)}_j$ & weight
        from $i$ (layer $l{-}1$) to $j$ (layer $l$); bias \\
        $z^{(l)}_j$, $a^{(l)}_j$ & $\sum_i a^{(l-1)}_i W^{(l)}_{ij} + b^{(l)}_j$; $\relu(z^{(l)}_j)$ &
        $f(x) = a^{(L)}_u$ & scalar output (target logit $u$) \\
        $[t]^\pm$, $x_k^\pm$ & $\max\{\pm t, 0\}$, so $t = [t]^+ - [t]^-$; $x_k^\pm = [x_k]^\pm$ &
        $W^{(l,\pm)}$, $b^{(l,\pm)}$ & $[W^{(l)}]^\pm$, $[b^{(l)}]^\pm$ (sign split) \\
        $z^{(l,\pm)}_j$, $a^{(l,\pm)}_{\mathrm{stop},j}$ & sign split of $z$; Stopping Decomposition &
        $\nabla^2 f$ & Hessian \\
        \midrule
        $p, q, r$; $p'$ & players / turns in $\{+,-\}$; opponent & $\kappa \in \{+,-\}$ & stream sign of a
        linear routing state \\
        $s^{(l,\mathrm{act})}_{j,q}$, $s^{(0)}_{k,q}$ & activation state ($q$ in turn); input terminal &
        $s^{(l,\mathrm{lin})}_{j,q,\kappa}$ & \ac{RG} linear routing state \\
        $\dagger$ & cemetery (absorbing) state & $\perp_{\mathrm{lin}}$ & outside action (log-payoff
        $\log\varepsilon$) \\
        $r_p(s,a)$ & payoff of player $p$ & $\pi$, $\pi^\star$ & policy; equilibrium (SPNE) policy \\
        $\tilde V_p(s)$, $\tilde Q_p(s,a)$ & \ac{SG} player-specific value, action value & $V_p(s)$ & game
        value (player $p$ minus $p'$) \\
        $Q_r(s,i)$ & \ac{RG} log-gain of routing to predecessor $i$ & $\Sigma^{(l,\kappa)}_j$ & $\sum_i
            [a^{(l-1)}_i W^{(l)}_{ij}]^\kappa + [b^{(l)}_j]^\kappa$ \\
        $\gamma^{(l)}_j$ & \ac{SG} discount $\sum_i |W^{(l)}_{ij}|$ & $\gamma_{\mathrm{O}}$ & Oracle
        discount ($2$ add./branch, $1$ max-pool) \\
        $\gamma(s)$ & per-state discount: \ac{SG} $\gamma^{(l)}_j$ or $\gamma_{\mathrm{O}}$; \ac{RG} $\in
            \{2\alpha, 2\beta, 1\}$ & $\zeta$ & trajectory $(\zeta_0, \zeta_1, \dots)$ \\
        $\Pi_x$ & equilibrium trajectory law at input $x$ & $d_t(\zeta)$ & cumulative trajectory discount
        \\
        $\Gamma_x(s)$ & discounted occupation measure of state $s$ & $\mathcal{H}$, $\mathcal{H}_\theta$ &
        Shannon entropy; binary form \\
        \midrule
        $\tau$ & \ac{RG} temperature (entropy weight) & $\alpha$, $\beta$ & $\alpha\beta$-\ac{LRP} weights,
        $\alpha - \beta = 1$ \\
        $\varepsilon$ & \ac{LRP} stabilizer & $R^{(l)}_i$ & \ac{LRP} relevance \\
        $A$ & attention matrix (held fixed) & $v$, $W_v$ & Value vectors, projection (sup.\ ``val'') \\
        $n_h$ & number of attention heads & $N$ & equivariant-averaging samples \\
        \midrule
        $\lambda$ & activation-side risk aversion ($\lambda \le 0$) & $\lambda_{\mathrm{sm}}$ &
        attention-logit risk aversion \\
        $\mu$, $\sigma$, $\sigma^2$ & \ac{ADF} mean, std.; input variance & $\lambda_{\mathrm{ent}}$ &
        fan-in entropy-bias mode \\
        $\mathrm{sp}_\theta$, $\mathrm{sp}'_\theta$ & Softplus $\theta\log(1 {+} e^{z/\theta})$; its gate &
        $\Phi$, $\phi$ & normal \ac{CDF}, density \\
        $\theta$; $\tilde\Phi$ & Softplus temperature; soft gate $\Phi(\mu/\sigma+\lambda)$ &
        $\mathcal{R}_\Phi$ & probit regularizer $\phi(\Phi^{-1}(\pi))$ \\
        \midrule
        $\mathrm{H}$, $\mathrm{H}_{\mathrm{surv}}$ & Hellinger distance; survival-conditioned & BC &
        Bhattacharyya coefficient \\
        $T^{(l)}_M$ & backward kernel of model $M$ & $m^{(l)}_M(s)$ & marginal occupation probability \\
        $g^{(l)}(s)$ & overlap mass (backward pass) & $c^{(l)}(s)$ & continuation mass (forward pass) \\
        $\mathrm{H}^2(s)$ & per-terminal-state disagreement & $\psi^{(l)}_M(s)$ & survival probability \\
        \bottomrule
    \end{tabularx}
\end{table}
\newpage
\section{Stopping Game}\label{app:relu-net-game}\label{app:adf-game}

This appendix realizes the activations of the Stopping Decomposition~\eqref{eq:stop-relu} as
player-specific game values of a backward two-player Stopping Game on an extended network graph and
demonstrates that gradients of the original network can be interpreted as integrals over game
trajectories.

\subsection{Stopping Decomposition}\label{app:dc-decomp}\label{app:decompositions}

Throughout we consider a feedforward ReLU network of $L$ layers with input layer $0$, $a^{(0)} =
    x$, pre-activations $z_j^{(l)} = \sum_i a_i^{(l-1)} W_{ij}^{(l)} + b_j^{(l)}$, where $W_{ij}^{(l)}$
connects neuron $i$ in layer $l-1$ to neuron $j$ in layer $l$, activations $a_j^{(l)} =
    \relu(z_j^{(l)})$, and a \emph{single scalar output} $f(x) = a_u^{(L)}$ (for multi-class networks,
fix the target logit $u$).

\paragraph{Canonical sign split.}
We use the element-wise positive/negative decomposition of the original network's weights and
biases,
\begin{align}
    W^{(l,+)} \coloneqq [W^{(l)}]^+, \qquad W^{(l,-)} \coloneqq [W^{(l)}]^-, \qquad b^{(l,+)} \coloneqq
    [b^{(l)}]^+, \qquad b^{(l,-)} \coloneqq [b^{(l)}]^-,
\end{align}
which maps non-negative activation pairs $(a_i^{(l-1,+)}, a_i^{(l-1,-)})$ to the pair of
pre-activations at neuron $j$,
\begin{align}
    z_j^{(l,+)} &\coloneqq b_j^{(l,+)} + \sum_i W_{ij}^{(l,+)} a_i^{(l-1,+)} + \sum_i W_{ij}^{(l,-)}
    a_i^{(l-1,-)}, \label{eq:split-pos}\\
    z_j^{(l,-)} &\coloneqq b_j^{(l,-)} + \sum_i W_{ij}^{(l,+)} a_i^{(l-1,-)} + \sum_i W_{ij}^{(l,-)}
    a_i^{(l-1,+)}. \label{eq:split-neg}
\end{align}
When $a_i^{(l-1,+)} - a_i^{(l-1,-)} = a_i^{(l-1)}$ for all $i$, this yields the identity
$z_j^{(l,+)} - z_j^{(l,-)} = b_j^{(l)} + \sum_i (W_{ij}^{(l,+)} - W_{ij}^{(l,-)}) (a_i^{(l-1,+)} -
    a_i^{(l-1,-)}) = z_j^{(l)}$, so the layer-$l$ pre-activation pair already encodes the original
network's pre-activation through its difference.

The \textbf{Stopping decomposition function} $f_{\mathrm{stop}} \colon \RR_{\ge 0}^2 \to \RR_{\ge
    0}^2$ is
\begin{align}\label{eq:stop-relu}
    f_{\mathrm{stop}}(z^+, z^-) \;\coloneqq\; \bigl(\,\1\!\bigl[z^+ - z^- > 0\bigr]\, z^+,\;\;
    \1\!\bigl[z^+ - z^- > 0\bigr]\, z^-\,\bigr).
\end{align}
Applied at every neuron and propagated by~\eqref{eq:split-pos}--\eqref{eq:split-neg}, it defines
the \textbf{Stopping Decomposition} activations $\bigl(a_{\mathrm{stop},j}^{(l,+)},
    a_{\mathrm{stop},j}^{(l,-)}\bigr) = f_{\mathrm{stop}}\bigl(z_j^{(l,+)}, z_j^{(l,-)}\bigr)$, with
the input-layer initialization
\begin{align}\label{eq:input-split-convention}
    a_{\mathrm{stop},k}^{(0,\pm)} = x_k^{\pm}.
\end{align}

\begin{theorem}[Layerwise Recovery of the Original Activations]\label{thm:decomp-recovery}
For every layer $l$ and neuron $j$,
\begin{align}\label{eq:decomp-recovery-stop-mix}
    a_{\mathrm{stop},j}^{(l,+)} - a_{\mathrm{stop},j}^{(l,-)}
    =
    a_j^{(l)} = \relu\!\bigl(z_j^{(l)}\bigr).
\end{align}
\end{theorem}

\begin{proof}
Induction on $l$. At $l=0$, $x_k^+ - x_k^- = x_k = a_k^{(0)}$ by~\eqref{eq:input-split-convention}.
At layer $l \ge 1$ the identity above gives $z_j^{(l,+)} - z_j^{(l,-)} = z_j^{(l)}$, hence
\(
    a_{\mathrm{stop},j}^{(l,+)} - a_{\mathrm{stop},j}^{(l,-)}
    = \1\{z_j^{(l)} > 0\}\bigl(z_j^{(l,+)} - z_j^{(l,-)}\bigr)
    = \relu(z_j^{(l)})
    = a_j^{(l)}.
\)
\end{proof}

ReLU networks compute continuous piecewise linear functions~\citep{Arora2018understanding}, which
admit a \ac{DC} decomposition~\citep{hiddenMono}; it uses the same sign
split~\eqref{eq:split-pos}--\eqref{eq:split-neg} with the pairwise activation $(\max\{z^+, z^-\},
    z^-)$ in place of $f_{\mathrm{stop}}$ and yields components that are convex in $x$. The Stopping
Decomposition instead gates both streams by the same indicator, so its components are not convex,
but on each linear region they are non-negative affine, hence monotone, in $(x^+, x^-)$.

\paragraph{Parity path decomposition of the split streams.}\label{app:parity-path}
We can decompose the Stopping Decomposition activations into small contributions coming from
individual paths through the network: each non-negative decomposed activation at neuron $j$ is the
sum, over all backward paths from $j$ to the input, of the absolute weight product along the path
times either the positive part $x_{P_0}^+$ or the negative part $x_{P_0}^-$ of the input pixel
where the path terminates, with the choice between $x_{P_0}^+$ and $x_{P_0}^-$ determined by the
parity of the number of negative weights traversed.

\begin{lemma}[Path view on the Stopping Decomposition]\label{lem:parity-path}
Fix a feedforward ReLU network with weights $W^{(l)}$ (biases ignored; see remark at end of proof).
For neuron $j$ in layer $l \ge 1$, call a \emph{backward path} from $j$ to the input layer a list
of neuron indices
\begin{align}
    P = (P_l, P_{l-1}, \dots, P_0), \qquad P_l = j,
\end{align}
with $P_t$ indexing a neuron in layer $t$. The edges of $P$ use the weights $W^{(t)}_{P_{t-1},
    P_t}$ for $t = 1, \dots, l$. Define
\begin{align}
    w(P) \coloneqq \prod_{t=1}^{l} W^{(t)}_{P_{t-1}, P_t},
    \qquad
    |w|(P) \coloneqq \prod_{t=1}^{l} \bigl| W^{(t)}_{P_{t-1}, P_t} \bigr|,
\end{align}
the parity sign
\begin{align}
    \mathrm{sgn}(P) \in \{+,-\},
    \qquad
    \mathrm{sgn}(P) = + \iff \#\bigl\{t : W^{(t)}_{P_{t-1}, P_t} < 0\bigr\} \text{ is even,}
\end{align}
and the \emph{gate factor} (1 iff every gate on $P$ at intermediate layers $1,\dots,l-1$ is open)
\begin{align}
    G_{<l}(P) \coloneqq \prod_{t=1}^{l-1} \1\!\bigl[z_{P_t}^{(t)} > 0\bigr].
\end{align}
Then the decomposed pre-activations~\eqref{eq:split-pos}--\eqref{eq:split-neg} admit the path-sum
representation
\begin{align}
    z_j^{(l,+)} &= \sum_{P:\, P_l = j} G_{<l}(P)\, |w|(P)\, x_{P_0}^{\mathrm{sgn}(P)},
    \label{eq:parity-plus}\\
    z_j^{(l,-)} &= \sum_{P:\, P_l = j} G_{<l}(P)\, |w|(P)\, x_{P_0}^{\mathrm{sgn}(P)'}.
    \label{eq:parity-minus}
\end{align}
\end{lemma}

\begin{proof}
We prove~\eqref{eq:parity-plus}--\eqref{eq:parity-minus} by induction on $l$.

\textbf{Base case ($l=1$):} By the input convention~\eqref{eq:input-split-convention},
$a_{\mathrm{stop},k}^{(0,+)} = x_k^+$ and $a_{\mathrm{stop},k}^{(0,-)} = x_k^-$.
Using~\eqref{eq:split-pos} with biases omitted,
\begin{align}
    z_j^{(1,+)}
    = \sum_k \bigl( W_{kj}^{(1,+)} x_k^+ + W_{kj}^{(1,-)} x_k^- \bigr).
\end{align}
Backward paths with $P_1 = j$ and $P_0 = k$ are in bijection with predecessors $k$; each has
$|w|(P) = |W^{(1)}_{kj}|$, $\mathrm{sgn}(P) = +$ iff $W^{(1)}_{kj} > 0$, and $G_{<1}(P) = 1$ (empty
product). Splitting by the sign of $W^{(1)}_{kj}$:
\begin{itemize}
    \item $W^{(1)}_{kj} > 0$: $W_{kj}^{(1,+)} = |W^{(1)}_{kj}|$ contributes $|W^{(1)}_{kj}| x_k^+$ with
        $\mathrm{sgn}(P) = +$, matching $|w|(P)\, x_{P_0}^{\mathrm{sgn}(P)}$.
    \item $W^{(1)}_{kj} < 0$: $W_{kj}^{(1,-)} = |W^{(1)}_{kj}|$ contributes $|W^{(1)}_{kj}| x_k^-$ with
        $\mathrm{sgn}(P) = -$, again matching.
\end{itemize}
The argument for $z_j^{(1,-)}$ is identical with the parity swapped,
establishing~\eqref{eq:parity-minus}.

\textbf{Inductive step ($l \ge 2$):} The decomposition recursion~\eqref{eq:split-pos} takes
\emph{Stopping Decomposition} activations from layer $l-1$ as inputs,
\begin{align}
    z_j^{(l,+)}
    =
    \sum_i \bigl(
    W_{ij}^{(l,+)}\, a_{\mathrm{stop},i}^{(l-1,+)}
    + W_{ij}^{(l,-)}\, a_{\mathrm{stop},i}^{(l-1,-)}
    \bigr),
\end{align}
and $a_{\mathrm{stop},i}^{(l-1,\pm)} = \1[z_i^{(l-1)} > 0]\, z_i^{(l-1,\pm)}$. Hence, at each
predecessor $i$, the gate indicator $\1[z_i^{(l-1)} > 0]$ multiplies the inductive representation
of $z_i^{(l-1,\pm)}$, upgrading $G_{<l-1}(P')$ to $G_{\le l-1}(P')$ for every sub-path $P'$ ending
at $i$.

Every backward path $P$ with $P_l = j$ and $P_{l-1} = i$ is uniquely obtained by prepending $j$ to
a sub-path $P' = (P'_{l-1}, \dots, P'_0)$ with $P'_{l-1} = i$; the prepended edge uses
$W^{(l)}_{ij}$, contributing $|W^{(l)}_{ij}|$ to $|w|(P)$ and either preserving or flipping parity
depending on $\mathrm{sgn}(W^{(l)}_{ij})$. Crucially, the new gate factor $G_{<l}(P)$ for $P$
equals $G_{\le l-1}(P')$: up to layer $l-1$ the gates are those along $P'$, and there is no gate at
layer $l$ in the $G_{<l}$ product.

\emph{Case $W^{(l)}_{ij} > 0$.} Then $W_{ij}^{(l,+)} = |W^{(l)}_{ij}|$ and $W_{ij}^{(l,-)} = 0$.
Only $W_{ij}^{(l,+)}\, a_{\mathrm{stop},i}^{(l-1,+)}$ contributes, and the new edge preserves
parity, $\mathrm{sgn}(P) = \mathrm{sgn}(P')$. Using the induction hypothesis under the updated gate
factor,
\begin{align}
    W_{ij}^{(l,+)}\, a_{\mathrm{stop},i}^{(l-1,+)}
    &= |W^{(l)}_{ij}| \sum_{P':\, P'_{l-1} = i} G_{\le l-1}(P')\, |w|(P')\, x_{P'_0}^{\mathrm{sgn}(P')}
    \notag\\
    &= \sum_{P:\, P_l = j,\, P_{l-1} = i} G_{<l}(P)\, |w|(P)\, x_{P_0}^{\mathrm{sgn}(P)}.
\end{align}
\emph{Case $W^{(l)}_{ij} < 0$.} Then $W_{ij}^{(l,+)} = 0$ and $W_{ij}^{(l,-)} = |W^{(l)}_{ij}|$.
Only $W_{ij}^{(l,-)}\, a_{\mathrm{stop},i}^{(l-1,-)}$ contributes, and the new edge flips parity,
$\mathrm{sgn}(P) = \mathrm{sgn}(P')'$, so $\mathrm{sgn}(P') = \mathrm{sgn}(P)'$, and the induction
hypothesis for $z_i^{(l-1,-)}$ (with $x_{P'_0}^{\mathrm{sgn}(P')'} = x_{P_0}^{\mathrm{sgn}(P)}$)
gives the same path sum as in the previous case.
\emph{Case $W^{(l)}_{ij} = 0$.} Both terms vanish and no path through $i$ is prepended.

Summing over $i$ yields~\eqref{eq:parity-plus}. The same argument with positive/negative decomposed
activations swapped gives~\eqref{eq:parity-minus}.

\textbf{Signed recovery.} Using $x_k^+ - x_k^- = x_k$ and $\mathrm{sgn}(P) \in \{+,-\}$,
\begin{align}
    z_j^{(l,+)} - z_j^{(l,-)}
    = \sum_P G_{<l}(P)\, |w|(P)\, \bigl(x_{P_0}^{\mathrm{sgn}(P)} - x_{P_0}^{\mathrm{sgn}(P)'}\bigr)
    = \sum_P G_{<l}(P)\, w(P)\, x_{P_0},
\end{align}
which is the standard active-path-sum representation of $z_j^{(l)}$ (paths through an inactive
intermediate neuron contribute zero).

\textbf{Biases.} Adjoining a bias ``stub path'' at each intermediate layer $l'$---starting from
$(l,j)$, terminating at the bias node of neuron $k'$ at layer $l'$, with weight product taken over \ch{the
$l - l'$ edges above and an appropriate gate factor $G_{<l}$ restricted to those edges}---handles the
bias terms by the same parity rule; the argument above goes through verbatim
after absorbing these stub contributions into the path sums.
\end{proof}

\subsection{Game Definition}\label{app:adf-game-def}\label{app:adf-recursion}

\paragraph{Relation to the network games of \citet{Gaubert2025}.}
The Stopping Game adapts the game of Gaubert and Vlassopoulos to the Stopping Decomposition. We
keep their state-space structure---per-neuron $\pm$ states, the bias as the continuation and
network input as final payoff, transitions driven by the positive and negative parts of the
weights, a cemetery state, and a stop-versus-continue action. Two changes are load-bearing for
attribution: (i)~the original formulation tracks only one game value, which equals $\pm a_j^{(l)}$;
we additionally maintain the non-negative payoff per player, and the original game value is
recovered as their difference. (ii)~We replace the signed terminal values $\pm x_k$ by the
non-negative pair $x_k^+, x_k^-$, which keeps every player-specific payoff non-negative and exposes
the parity decomposition of Lemma~\ref{lem:parity-path}.

\begin{figure}[H]
    \centering
    \includegraphics[width=\textwidth]{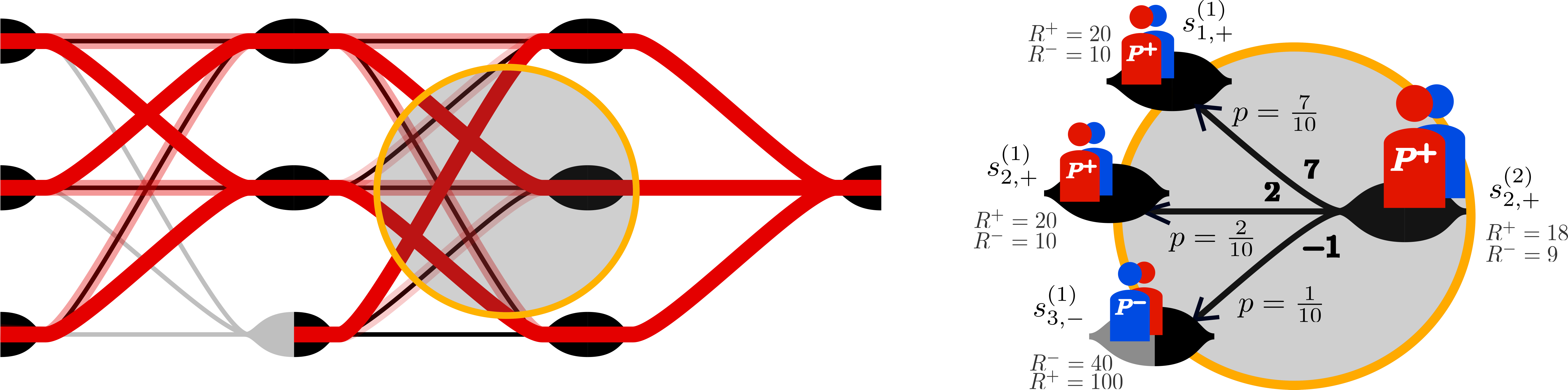}
    \caption{Stopping Game: trajectory distribution and local stopping decisions on a toy subnetwork.
    \emph{Left:} trajectory distribution visualized on the network graph.
    \emph{Right:} local game around $s_{2,+}^{(2)}$. Player $P^+$ decides whether to stop or continue.
    Under continuation, transitions are to $s_{1,+}^{(1)}$ with probability $7/10$, to $s_{2,+}^{(1)}$
    with probability $2/10$ (both turn unchanged), and to $s_{3,-}^{(1)}$ with probability $1/10$ (turn
    switches to $P^-$, reflecting the negative sign of the weight). The labels $R^\pm$ are the players'
    expected payoffs $\tilde Q_\pm(\cdot, \mathrm{cont})$ (at $s_{2,+}^{(2)}$ without the discount
    $\gamma^{(2)}_2 = 10$). At $s_{3,-}^{(1)}$, continuing yields expected payoffs $R^-{=}40$ and
    $R^+{=}100$, giving continuation advantage $R^- - R^+ = -60 < 0$; therefore $P^-$ stops (so
    $V_-(s_{3,-}^{(1)}) = 0$), shown by the grayed-out half-neuron. Correspondingly, in the left panel
    the trajectories terminate at that state (the cemetery state is omitted). Since $P^-$ stops in
    $s_{3,-}^{(1)}$, traversals to this state contribute zero continuation payoff to both players.
    Hence, at $s_{2,+}^{(2, \mathrm{act})}$ the positive player has larger expected payoff under
    continuation than his opponent and chooses to continue.}
    \label{fig:Stopping-Game-illustration}
\end{figure}

\begin{definition}[Stopping Game State Space]\label{def:adf-game-states} The game is played on the following states:
\begin{itemize}
    \item \textbf{Activation States:} $s_{j,q}^{(l, \mathrm{act})}$ where layer $l \in \{1, \dots, L\}$ and
        player $q \in \{+,-\}$ is in turn at neuron $j$.
    \item \textbf{Terminal States:} $s_{k,q}^{(0)}$ where $k$ is an input dimension and $q \in \{+,-\}$ is
        the final turn label.
    \item \textbf{Cemetery State:} $\dagger$, an absorbing state reached upon stopping.
\end{itemize}
We adopt the convention that, when a transition kernel or recursion at layer $l$ references the
predecessor state $s_{i,r}^{(l-1, \mathrm{act})}$, the symbol denotes the activation state at layer
$l-1 \ge 1$ as defined above and, in the boundary case $l = 1$, is interpreted as the terminal
state $s_{i,r}^{(0)}$.
\end{definition}

\begin{definition}[Payoffs, Transitions, and Value Recursion]\label{def:adf-recursion}
Fix an input $x$.

\textbf{Payoffs.} The player-specific \emph{payoff} $r_p$ at activation states is paid out
depending on the action chosen by the active player, and paid out at terminal and cemetery states
without prerequisites:
\begin{itemize}
    \item Cemetery: $r_p(\dagger) \coloneqq 0$.
    \item Active state: $r_p(s_{j,q}^{(l, \mathrm{act})}, \mathrm{stop}) \coloneqq 0$ (stopping pays nothing)
        and $r_p(s_{j,q}^{(l, \mathrm{act})}, \mathrm{cont}) \coloneqq b_j^{(l, p \oplus q)}$ (continuation
        pays a positive bias to the active player $q$---covered by the case $p \oplus q = +$---and the
        absolute value of a negative bias to its opponent $q'$---covered by the case $p \oplus q = -$).
    \item Terminal: $r_p(s_{k,q}^{(0)}) \coloneqq x_k^{p \oplus q}$ (at terminal states the payoff depends on
        state only---the game ends here. A positive input is paid to the active player $q$---covered by the
        case $p \oplus q = +$---whereas the absolute value of a negative input is paid to its opponent
        $q'$---covered by the case $p \oplus q = -$).
\end{itemize}

\textbf{Transition kernel.} Stopping at state $s^{(l, \text{act})}_{j,p}$ deterministically sends
the game to the cemetery; continuation transports backward by one layer with transition
probabilities depending on the fraction of the respective absolute network weight compared to the
sum of absolute incoming weights at neuron $j$ of layer $l$. In case of a positive weight, the
active player $p$ stays in turn, in case of a negative weight, the turn switches to its opponent
$p'$.
\begin{align}\label{eq:adf-kernel-stop}
    \P\bigl(\dagger \,\big|\, s_{j,q}^{(l, \mathrm{act})}, \mathrm{stop}\bigr) \coloneqq 1,
\end{align}
\begin{align}\label{eq:adf-kernel}
    \P\bigl(s_{i,r}^{(l-1, \mathrm{act})} \,\big|\, s_{j,q}^{(l, \mathrm{act})}, \mathrm{cont}\bigr)
    \coloneqq \frac{W_{ij}^{(l, q \oplus r)}}{\gamma_j^{(l)}},
    \qquad
    \gamma_j^{(l)} \coloneqq \sum_i \bigl(W_{ij}^{(l,+)} + W_{ij}^{(l,-)}\bigr) = \sum_i
    \bigl|W_{ij}^{(l)}\bigr|,
\end{align}
where $\gamma_j^{(l)}$ is the local structural discount that restores unit row-sum (if
$\gamma_j^{(l)} = 0$, continuation leads to $\dagger$). For understanding the idea of the
occupation measure, we remark that those $\gamma^{(l)}_j$ are going to weight the trajectories in
the occupation measure.

\textbf{Policy.} A pure, i.e., deterministic policy $\pi = (\pi_+, \pi_-)$ assigns each active
state $s_{j,q}^{(l, \mathrm{act})}$ a binary stop/continue decision by the player in turn,
\begin{align}
    \pi_q\bigl(s_{j,q}^{(l, \mathrm{act})}\bigr) \in \{0, 1\},
    \qquad
    1 = \mathrm{cont},\ 0 = \mathrm{stop}.
\end{align}
\textbf{Player-specific action and game values.} For each policy $\pi$, the player-specific
\textbf{action value} $\tilde{Q}_p^\pi(s, a)$ is the expected discounted payoff for player $p$
starting at $s$ with action $a$ and following $\pi$ thereafter, defined by the standard Bellman
equation
\begin{align}\label{eq:adf-q-bellman}
    \tilde{Q}_p^\pi(s, a) \coloneqq r_p(s, a) + \gamma(s)\, \textstyle\sum_{s'} \P\bigl(s' \,\big|\, s,
    a\bigr)\, \tilde{V}_p^\pi(s'),
\end{align}
where $\tilde{V}_p^\pi$ is the player-specific game value defined below. We use tildes to reserve
$V_p^\pi$ for the advantage further down, which already carries a subtracted $\tilde{V}_{p'}^\pi$.
At active states~\eqref{eq:adf-q-bellman} specializes,
using~\eqref{eq:adf-kernel-stop}--\eqref{eq:adf-kernel}, to
\begin{align}\label{eq:adf-q-stop}
    \tilde{Q}_p^\pi\bigl(s_{j,q}^{(l, \mathrm{act})},\, \mathrm{stop}\bigr) &= 0 + 1 \cdot
    \tilde{V}_p^\pi(\dagger) = 0 \\ \label{eq:adf-q-cont}
    \tilde{Q}_p^\pi\bigl(s_{j,q}^{(l, \mathrm{act})},\, \mathrm{cont}\bigr)
    &= b_j^{(l, p \oplus q)} + \gamma_j^{(l)} \textstyle\sum_{s'} \P\bigl(s' \,\big|\, s_{j,q}^{(l,
    \mathrm{act})}, \mathrm{cont}\bigr)\, \tilde{V}_p^\pi(s') \\
    &= b_j^{(l, p \oplus q)} + \sum_{i,r} W_{ij}^{(l, q \oplus r)}\,
    \tilde{V}_p^\pi\bigl(s_{i,r}^{(l-1, \mathrm{act})}\bigr).
\end{align}
The \textbf{player-specific game value} $\tilde{V}_p^\pi(s)$ is the action value averaged by the
policy. At terminal and cemetery states it equals the local payoff there:
\begin{align}\label{eq:adf-v-terminal}
    \tilde{V}_p^\pi\bigl(s_{k,q}^{(0)}\bigr) \coloneqq x_k^{p \oplus q},
    \qquad
    \tilde{V}_p^\pi(\dagger) \coloneqq 0.
\end{align}
At active states,
\begin{align}\label{eq:adf-v-active}
    \tilde{V}_p^\pi\bigl(s_{j,q}^{(l, \mathrm{act})}\bigr) \coloneqq \pi_q\bigl(s_{j,q}^{(l,
    \mathrm{act})}\bigr)\, \tilde{Q}_p^\pi\bigl(s_{j,q}^{(l, \mathrm{act})},\, \mathrm{cont}\bigr) +
    \bigl(1 - \pi_q\bigl(s_{j,q}^{(l, \mathrm{act})}\bigr)\bigr)\, \tilde{Q}_p^\pi\bigl(s_{j,q}^{(l,
    \mathrm{act})},\, \mathrm{stop}\bigr).
\end{align}
Since $\pi_q$ is binary,~\eqref{eq:adf-v-active} selects either $\tilde{Q}_p^\pi(\cdot,
    \mathrm{cont})$ or $\tilde{Q}_p^\pi(\cdot, \mathrm{stop}) = 0$.

\textbf{Game value.} The quantity the player in turn maximizes at each active state is the
advantage of player $p$ at $s$,
\begin{align}\label{eq:adf-advantage}
    V_p^\pi(s) \coloneqq \tilde{V}_p^\pi(s) - \tilde{V}_{p'}^\pi(s).
\end{align}
\end{definition}

Figure~\ref{fig:Stopping-Game-illustration} shows a worked example of the induced trajectory
distribution and the local stop/continue logic on a small toy network.

\begin{theorem}[Forward Equivalence for Stopping Game]\label{thm:forward-stopping}
The game admits a Subgame Perfect Nash Equilibrium (SPNE) where both players follow the same pure
strategy:
\begin{align} \label{eq:adf-spne-policy}
    \pi_q^{\star}\bigl(s_{j,q}^{(l, \mathrm{act})}\bigr) = \1\bigl[z_j^{(l)} > 0\bigr].
\end{align}
Under this equilibrium, the player-specific game values and the game values satisfy
\begin{align}\label{eq:adf-spne-values}
    \tilde{V}_p^{\pi^\star}\bigl(s_{j,q}^{(l, \mathrm{act})}\bigr) = a_{\mathrm{stop},j}^{(l, p \oplus
    q)},
    \qquad
    V_p^{\pi^\star}\bigl(s_{j,q}^{(l, \mathrm{act})}\bigr) = \xi_{p,q}\, a_j^{(l)}.
\end{align}
\end{theorem}

\begin{remark}[Uniqueness of the SPNE up to tie-breaking] \label{rem:adf-spne-uniqueness}
It is easy to see in the following proof that the SPNE \eqref{eq:adf-spne-policy} is unique up to
tie-breaking at states where continuing yields zero advantage $\tilde
    Q_q^{\pi^\star}(\cdot,\mathrm{cont})-\tilde Q_{q'}^{\pi^\star}(\cdot,\mathrm{cont}) = z_j^{(l)} =
    0$ (the implementation continues at exact ties, which changes nothing unless $z_j^{(l)}=0$
exactly)---representing neurons of preactivation zero. We refer to the policy
\eqref{eq:adf-spne-policy} by \textbf{stop-at-tie-breaks} equilibrium.
\end{remark}

\begin{proof}
We proceed by backward induction on the layer index $l$.

\textbf{Base Case ($l=0$):} By~\eqref{eq:adf-v-terminal}, $\tilde{V}_p^{\pi^\star}(s_{k,q}^{(0)}) =
    x_k^{p \oplus q} = a_{\mathrm{stop},k}^{(0, p \oplus q)}$ (input
convention~\eqref{eq:input-split-convention}). The game value is
\begin{align}
    V_p^{\pi^\star}\bigl(s_{k,q}^{(0)}\bigr)
    = x_k^{p \oplus q} - x_k^{p' \oplus q}
    = \xi_{p,q}\, x_k,
\end{align}
using $x_k^+ - x_k^- = x_k$ and $x_k^- - x_k^+ = -x_k$.

\textbf{Inductive Step:} Assume the subgames starting at states up to layer $l-1$ admit the claimed
SPNE and ~\eqref{eq:adf-spne-values} holds at layer $l-1$. Consider an active state $s_{j,q}^{(l,
    \mathrm{act})}$. The active player $q$ chooses $a \in \{\mathrm{stop}, \mathrm{cont}\}$ to maximize
the advantage $\tilde{Q}_q^{\pi^\star}(s_{j,q}^{(l, \mathrm{act})}, a) -
    \tilde{Q}_{q'}^{\pi^\star}(s_{j,q}^{(l, \mathrm{act})}, a)$. For $a=\mathrm{stop}$ this advantage
is $0$ by~\eqref{eq:adf-q-stop}; for $a=\mathrm{cont}$, using~\eqref{eq:adf-q-cont},
\begin{align}
    \tilde{Q}_q^{\pi^\star}\bigl(s_{j,q}^{(l, \mathrm{act})}, \mathrm{cont}\bigr) -
    \tilde{Q}_{q'}^{\pi^\star}\bigl(s_{j,q}^{(l, \mathrm{act})}, \mathrm{cont}\bigr)
    &= \bigl(b_j^{(l, q \oplus q)} - b_j^{(l, q' \oplus q)}\bigr) + \sum_{i,r} W_{ij}^{(l, q \oplus
    r)}\, V_q^{\pi^\star}\bigl(s_{i,r}^{(l-1, \mathrm{act})}\bigr) \\
    &= \bigl(b_j^{(l,+)} - b_j^{(l,-)}\bigr) + \sum_{i,r} W_{ij}^{(l, q \oplus r)}\,
    V_q^{\pi^\star}\bigl(s_{i,r}^{(l-1, \mathrm{act})}\bigr).
\end{align}
By the inductive hypothesis $V_q^{\pi^\star}(s_{i,r}^{(l-1, \mathrm{act})}) = \xi_{q,r}\,
    a_i^{(l-1)}$, with $\xi_{q,q}=+1$ and $\xi_{q,q'}=-1$,
\begin{align}
    \tilde{Q}_q^{\pi^\star}\bigl(s_{j,q}^{(l, \mathrm{act})}, \mathrm{cont}\bigr) -
    \tilde{Q}_{q'}^{\pi^\star}\bigl(s_{j,q}^{(l, \mathrm{act})}, \mathrm{cont}\bigr)
    &= b_j^{(l)} + \sum_i \bigl( W_{ij}^{(l, q \oplus q)}\, \xi_{q,q} + W_{ij}^{(l, q \oplus q')}\,
    \xi_{q,q'} \bigr) a_i^{(l-1)} \\
    &= b_j^{(l)} + \sum_i W_{ij}^{(l)}\, a_i^{(l-1)} = z_j^{(l)}.
\end{align}
Thus, player $q$ prefers \emph{continue} in case $z_j^{(l)} > 0$ and prefers \emph{stop} in case
$z_j^{(l)} < 0$, and is indifferent at $z_j^{(l)} = 0$, where we tie-break in favor of \emph{stop}.
The rule $\pi_q^\star(s_{j,q}^{(l, \mathrm{act})}) = \1[z_j^{(l)} > 0]$ together with the player
policies of earlier layers thus constitutes an SPNE of the subgame starting at state $s_{j,q}^{(l
    ,\text{act})}$.

To verify~\eqref{eq:adf-spne-values}, substitute $\pi_q^\star$ into~\eqref{eq:adf-v-active} and use
the inductive hypothesis $\tilde{V}_p^{\pi^\star}(s_{i,r}^{(l-1, \mathrm{act})}) =
    a_{\mathrm{stop},i}^{(l-1, p \oplus r)}$ in~\eqref{eq:adf-q-cont}:
\begin{align}
    \tilde{V}_p^{\pi^\star}\bigl(s_{j,q}^{(l, \mathrm{act})}\bigr)
    &= \1\bigl[z_j^{(l)} > 0\bigr]\, \tilde{Q}_p^{\pi^\star}\bigl(s_{j,q}^{(l, \mathrm{act})},
    \mathrm{cont}\bigr) \\
    &= \1\bigl[z_j^{(l)} > 0\bigr] \left[ b_j^{(l, p \oplus q)} + \sum_i \bigl( W_{ij}^{(l,+)}\,
    a_{\mathrm{stop},i}^{(l-1, p \oplus q)} + W_{ij}^{(l,-)}\, a_{\mathrm{stop},i}^{(l-1, p \oplus q')}
    \bigr) \right].
\end{align}
If $p=q$, the bracket is $b_j^{(l,+)} + \sum_i \bigl( W_{ij}^{(l,+)}\,
    a_{\mathrm{stop},i}^{(l-1,+)} + W_{ij}^{(l,-)}\, a_{\mathrm{stop},i}^{(l-1,-)} \bigr) =
    z_j^{(l,+)}$.
If $p=q'$, the bracket is $b_j^{(l,-)} + \sum_i \bigl( W_{ij}^{(l,+)}\,
    a_{\mathrm{stop},i}^{(l-1,-)} + W_{ij}^{(l,-)}\, a_{\mathrm{stop},i}^{(l-1,+)} \bigr) =
    z_j^{(l,-)}$.
In both cases, $\tilde{V}_p^{\pi^\star}(s_{j,q}^{(l, \mathrm{act})}) = \1[z_j^{(l)} > 0]\, z_j^{(l,
    p \oplus q)} = a_{\mathrm{stop},j}^{(l, p \oplus q)}$. Finally,
\begin{align}
    V_p^{\pi^\star}\bigl(s_{j,q}^{(l, \mathrm{act})}\bigr)
    = a_{\mathrm{stop},j}^{(l, p \oplus q)} - a_{\mathrm{stop},j}^{(l, (p \oplus q)')}
    = \xi_{p,q}\, a_j^{(l)}.
    \qed
\end{align}
\end{proof}

\subsection{Skip Connections and Max Pooling}\label{app:skip-connections}

Architectures such as ResNets and CNNs with pooling contain structural operators beyond
single-parent affine maps. To handle these natively in the Stopping Game, we add an
\textbf{Addition Oracle} for skip connections and a deterministic \textbf{Max-Pooling Decision
State} for pooling, mirroring the analogous constructions in the Routing Game
(Appendix~\ref{app:routing-skip}).

\begin{definition}[Addition Oracle]\label{def:addition-subgame}
For a residual node $z = x + y$, the \emph{addition state} $s_{z,q}^{\mathrm{add}}$ has active
player $q \in \{+,-\}$ and operand states $s_{x,q}, s_{y,q}$ at the two summands. The transition
kernel is the uniform Oracle
\begin{align}
    \P\bigl(s_{x,q} \,\big|\, s_{z,q}^{\mathrm{add}}\bigr) \coloneqq \tfrac{1}{2},
    \qquad
    \P\bigl(s_{y,q} \,\big|\, s_{z,q}^{\mathrm{add}}\bigr) \coloneqq \tfrac{1}{2},
\end{align}
with player $q$ retained at the chosen operand. The immediate payoff is
$r_p(s_{z,q}^{\mathrm{add}}) \coloneqq 0$ for both $p \in \{+,-\}$, and the structural discount is
$\gamma_{\mathrm{O}} \coloneqq 2$.
\end{definition}

\begin{proposition}[Forward Equivalence for Addition]\label{prop:addition-equivalence}
The player-specific state value at the addition state equals the sum of the operand state values,
\begin{align}
    \tilde V_p^{\pi}\bigl(s_{z,q}^{\mathrm{add}}\bigr)
    \;=\;
    \tilde V_p^{\pi}(s_{x,q}) + \tilde V_p^{\pi}(s_{y,q}).
\end{align}
Consequently, if the operand states recover the Stopping-Decomposition activations---$\tilde
    V_p^{\pi^\star}(s_{x,q}) = x^{(p\oplus q)}$ and $\tilde V_p^{\pi^\star}(s_{y,q}) = y^{(p\oplus
    q)}$---then $\tilde V_p^{\pi^\star}(s_{z,q}^{\mathrm{add}}) = z^{(p\oplus q)}$, preserving the
forward equivalence of Theorem~\ref{thm:forward-stopping} across the residual sum.
\end{proposition}

\begin{proof}
Apply the Bellman recursion~\eqref{eq:adf-q-bellman}--\eqref{eq:adf-v-active} at
$s_{z,q}^{\mathrm{add}}$ with $r_p = 0$, the uniform Oracle kernel, and $\gamma_{\mathrm{O}} = 2$
from Definition~\ref{def:addition-subgame}:
\begin{align}
    \tilde V_p^{\pi}\bigl(s_{z,q}^{\mathrm{add}}\bigr)
    \;=\; 0 + \gamma_{\mathrm{O}} \cdot \Bigl(\tfrac{1}{2}\,\tilde V_p^{\pi}(s_{x,q}) +
    \tfrac{1}{2}\,\tilde V_p^{\pi}(s_{y,q})\Bigr)
    \;=\; \tilde V_p^{\pi}(s_{x,q}) + \tilde V_p^{\pi}(s_{y,q}),
\end{align}
which is the first identity. Substituting the operand assumption and using $z^{(p\oplus q)} =
    x^{(p\oplus q)} + y^{(p\oplus q)}$ from the additive
convention~\eqref{eq:split-pos}--\eqref{eq:split-neg} of Appendix~\ref{app:dc-decomp} yields the
second.
\end{proof}

\paragraph{On the uniform $1/2$ split.}
The $(\tfrac{1}{2}, \tfrac{1}{2})$ Oracle with $\gamma_{\mathrm{O}} = 2$ is not forced by the
architecture; it is a convention, the simplest choice inside a one-parameter family of
conservation-preserving Oracles $(\rho, 1-\rho)$ with matching discounts $(1/\rho, 1/(1-\rho))$
attached to the branch taken (with the single state discount $\gamma_{\mathrm{O}}$
of~\eqref{eq:adf-q-bellman} and $d_t$, $\rho\gamma_{\mathrm{O}}=(1-\rho)\gamma_{\mathrm{O}}=1$
forces $\rho=\tfrac12$). The uniform split is the maximum-entropy choice---the one that imposes
\emph{no prior} on which summand carries more responsibility. In the Stopping Game,
Proposition~\ref{prop:addition-equivalence} and Theorems~\ref{thm:forward-stopping}
and~\ref{thm:deriv-occ} hold for any such split with branch discounts, since the proofs only use
that routing probability times Oracle discount equals $1$ on each branch; only the trajectory law
$\Pi_x$ (and with it the Hellinger distance of Appendix~\ref{app:path-divergence}) depends on the
split. This invariance does not extend to the Routing Game (Appendix~\ref{app:routing-skip}), where
the split changes the relevance assigned to each summand.

Max pooling is simpler, because the forward operator already selects a single predecessor.

\begin{definition}[Max-Pooling Decision State]\label{def:maxpool-subgame}
For a pooled scalar output $z = \max\{x_1, \dots, x_m\}$, the \emph{pooling state} $s_{z,q}^{\max}$
has active player $q \in \{+, -\}$ and admissible action set $\mathcal{A}(s_{z,q}^{\max}) = \{1,
    \dots, m\}$ indexing the operand states $s_{x_1,q}, \dots, s_{x_m,q}$. Under action $k$, the
transition is deterministic,
\begin{align}
    \P\bigl(s_{x_k,q} \,\big|\, s_{z,q}^{\max}, k\bigr) \coloneqq 1,
\end{align}
with player $q$ retained at the chosen operand. The immediate payoff is $r_p(s_{z,q}^{\max}, k)
    \coloneqq 0$ for both $p \in \{+,-\}$, and the structural discount is $\gamma_{\mathrm{O}}
    \coloneqq 1$.
\end{definition}

\paragraph{Tie-breaking at the pooling state.}
Whenever multiple operands $x_k$ attain the maximum, the active player is indifferent over the
corresponding actions, and any pure best response over the tied indices yields the same forward
value but a different equilibrium policy and hence a different occupation measure on the operand
states. The \emph{stop-at-tie-breaks equilibrium}, fixed by the boundary conventions of
Appendix~\ref{app:notation-identities}, resolves this by the convention $k^\star \coloneqq \min\{k
    : x_k = \max_l x_l\}$, i.e., the active player picks the smallest index among the operands
attaining the maximal payoff. All gradient and occupation-measure statements below are stated for
this equilibrium.

\begin{proposition}[Forward Equivalence for Max Pooling]\label{prop:maxpool-equivalence}
Assume the operand states recover the Stopping-Decomposition activations, $\tilde
    V_p^{\pi^\star}(s_{x_k,q}) = x_k^{(p\oplus q)}$ for all $k$, with non-negative operand activations
$x_k \ge 0$. Then at the SPNE of the subgame rooted at $s_{z,q}^{\max}$,
\begin{align}
    \tilde V_p^{\pi^\star}\bigl(s_{z,q}^{\max}\bigr) \;=\; z^{(p\oplus q)},
\end{align}
so the pooling state inherits the same Stopping-Decomposition identity as its winning operand and
the forward equivalence of Theorem~\ref{thm:forward-stopping} extends across the pooling node.
\end{proposition}

\begin{proof}
The active player $q$ chooses the operand that maximizes its game value $V_q^{\pi^\star} = \tilde
    V_q^{\pi^\star} - \tilde V_{q'}^{\pi^\star}$. By the operand assumption,
$V_q^{\pi^\star}(s_{x_k,q}) = \xi_{q,q}\, x_k = x_k$ for every $k$, so the maximizer is $k^\star
    \in \arg\max_k x_k$ (ties resolved by the smallest index). The Bellman
recursion~\eqref{eq:adf-q-bellman}--\eqref{eq:adf-v-active} with $r_p = 0$, the deterministic
transition, and $\gamma_{\mathrm{O}} = 1$ from Definition~\ref{def:maxpool-subgame} gives
\begin{align}
    \tilde V_p^{\pi^\star}\bigl(s_{z,q}^{\max}\bigr) \;=\; 0 + 1 \cdot \tilde
    V_p^{\pi^\star}(s_{x_{k^\star},q}) \;=\; x_{k^\star}^{(p\oplus q)}.
\end{align}
Since $z = x_{k^\star}$, the right-hand side equals $z^{(p\oplus q)}$.
\end{proof}

\subsection{Discounted Occupation Measures and Gradient Representations}\label{app:occ-deriv}

To connect game trajectories with neural network gradients, we define the \textbf{discounted
occupation measure}. This measure quantifies the expected influence of a neuron on the game's
outcome, weighted by the discount factors accumulated along all possible game trajectories
(including those created by skip connections and max-pooling nodes).

\begin{definition}[Game trajectory and trajectory law]\label{def:trajectory-distribution}
Fix an input $x$ and a root state $s$. A \emph{game trajectory} rooted at $s$ is a finite sequence
of game states
\begin{align}
    \zeta \;=\; (\zeta_0, \zeta_1, \dots, \zeta_T), \qquad \zeta_0 = s,
\end{align}
that follows the admissible transitions of the game (Definition~\ref{def:adf-recursion} together
with the addition/max-pooling Oracle splits of
Definitions~\ref{def:addition-subgame}--\ref{def:maxpool-subgame}) until it terminates either at a
layer-$0$ terminal state $\zeta_T = s_{k,r}^{(0)}$ (reaching an input dimension $k$) or at the
cemetery state $\zeta_T = \dagger$ (reached the first time a player chooses \emph{stop}). The set
of all such trajectories is the trajectory \emph{ground set} $\mathcal{T}_s$. Composing the
transition kernel~\eqref{eq:adf-kernel-stop}--\eqref{eq:adf-kernel}, the Oracle splits, and a
(possibly stochastic) policy $\pi$ at every activation state defines a Markov kernel on game
states, which together with the deterministic start at $s$ induces a probability distribution
$\Pi_x^\pi(\cdot \mid s)$ on $\mathcal{T}_s$. Trajectories that pass through closed gates---neurons
$j$ with $\1[z_j^{(l)} > 0] = 0$---are admissible elements of $\mathcal{T}_s$, but under the
equilibrium policy $\pi^\star$ the player at the activation state stops there, so any trajectory
continuing past a closed gate carries probability zero. We write $\Pi_x(\cdot \mid s) \coloneqq
    \Pi_x^{\pi^\star}(\cdot \mid s)$ for the equilibrium \emph{trajectory law} and $\Pi_x \coloneqq
    \Pi_x(\cdot \mid s_0)$ for the law rooted at the output state $s_0 \coloneqq
    s_{u,+}^{(L,\mathrm{act})}$.
\end{definition}

\begin{definition}[Discounted Occupation Measure]\label{def:occupation}
Let $\Pi_x(\cdot \mid s)$ be the equilibrium trajectory law of
Definition~\ref{def:trajectory-distribution}.
Define the cumulative discount
\begin{align}
    d_t(\zeta) \coloneqq \prod_{m=0}^{t-1} \gamma(\zeta_m), \qquad d_0(\zeta) = 1,
\end{align}
where $\gamma(\zeta_m) = \gamma_{\mathrm{O}} = 2$ for addition Oracle states, $\gamma(\zeta_m) =
    \gamma_{\mathrm{O}} = 1$ for max-pooling states, and $\gamma(\zeta_m) = \gamma_j^{(l)}$ for
continuation at activation states $\zeta_m = s_{j,q}^{(l,\mathrm{act})}$ (per-neuron
weight-magnitude discount). The \emph{discounted occupation measure} of state $v$ for trajectories
rooted at $s$ is
\begin{align}
    \Gamma_x(v \mid s) \coloneq \E_{\zeta \sim \Pi_x(\cdot \mid s)}\left[\sum_t d_t(\zeta)\,\1\{\zeta_t
    = v\}\right],
    \qquad
    \Gamma_x(v) \coloneqq \Gamma_x(v \mid s_0).
\end{align}
\end{definition}

\begin{theorem}[Activation--Occupation Correspondence]\label{thm:deriv-occ}
Fix a neuron state $s = s_{j,q}^{(m, \mathrm{act})}$ and a layer $l \le m$ (for $l = 0$,
$s_{i,\cdot}^{(0,\mathrm{act})}$ denotes the terminal state $s_{i,\cdot}^{(0)}$). Under the
stop-at-tie-breaks equilibrium fixed by the boundary conventions of
Appendix~\ref{app:notation-identities} (uniquely fixing the SPNE at ReLU gates with $z = 0$ and at
max-pooling states with multiple maximizers), the discounted occupation measures exactly represent,
in a multi-trajectory ReLU network (e.g., a ResNet or a CNN with max pooling), the local
Stopping-Decomposition and scalar activation gradients:
\begin{enumerate}
    \item \textbf{Stopping-Decomposition Gradient:}
        \begin{align}
            \frac{\partial a_{\mathrm{stop},j}^{(m, p \oplus q)}}{\partial a_{\mathrm{stop},i}^{(l, k)}}
            \;=\;
            \Gamma_x\!\bigl(s_{i, p \oplus k}^{(l, \mathrm{act})} \bigm| s\bigr).
        \end{align}
    \item \textbf{Network Activation Gradient.} The ordinary-network gradient of $a_j^{(m)}$ with respect to
        the scalar activation $a_i^{(l)}$ is
        \begin{align}
            \frac{\partial a_j^{(m)}}{\partial a_i^{(l)}}
            \;=\;
            \xi_{q, +}\bigl(\Gamma_x(s_{i, +}^{(l, \mathrm{act})} \mid s) - \Gamma_x(s_{i, -}^{(l,
            \mathrm{act})} \mid s)\bigr).
        \end{align}
\end{enumerate}
\end{theorem}

\begin{proof}
We establish Part 1 by backward induction on the layer gap $m - l$ and derive Part 2 from it. Since
$x$ is fixed, we write $\Gamma(v \mid s) \coloneqq \Gamma_x(v \mid s)$.

\textbf{Part 2.} By Theorem~\ref{thm:forward-stopping}, $a_i^{(l)} = a_{\mathrm{stop},i}^{(l,+)} -
    a_{\mathrm{stop},i}^{(l,-)}$ and $a_j^{(m)} = a_{\mathrm{stop},j}^{(m,+)} -
    a_{\mathrm{stop},j}^{(m,-)}$. Differentiating $a_j^{(m)}$ along the $+$ direction (perturb
$a_{\mathrm{stop},i}^{(l,+)}$, keep $a_{\mathrm{stop},i}^{(l,-)}$ fixed) and applying Part 1 at
$k=+$,
\begin{align}
    \frac{\partial a_j^{(m)}}{\partial a_i^{(l)}}
    &= \frac{\partial a_{\mathrm{stop},j}^{(m,+)}}{\partial a_{\mathrm{stop},i}^{(l,+)}}
    - \frac{\partial a_{\mathrm{stop},j}^{(m,-)}}{\partial a_{\mathrm{stop},i}^{(l,+)}}
    = \Gamma\bigl(s_{i, q}^{(l, \mathrm{act})} \bigm| s\bigr) - \Gamma\bigl(s_{i, q'}^{(l,
    \mathrm{act})} \bigm| s\bigr)
    = \xi_{q,+}\bigl(\Gamma(s_{i,+}^{(l, \mathrm{act})} \mid s) - \Gamma(s_{i,-}^{(l, \mathrm{act})}
    \mid s)\bigr),
\end{align}
using $q \oplus + = q$ and $q' \oplus + = q'$ in the second equality.

\textbf{Part 1.} The decomposition recursion of Theorem~\ref{thm:forward-stopping} writes the
upper-layer Stopping-Decomposition activation as
\begin{align}\label{eq:adf-astop-recursion}
    a_{\mathrm{stop},j}^{(m, p \oplus q)}
    = \1\!\bigl[z_j^{(m)} > 0\bigr]\!\left[
    b_j^{(m, p \oplus q)}
    + \sum_h \Bigl( W_{hj}^{(m,+)}\, a_{\mathrm{stop},h}^{(m-1, p \oplus q)} + W_{hj}^{(m,-)}\,
    a_{\mathrm{stop},h}^{(m-1, p \oplus q')} \Bigr)
    \right].
\end{align}

\emph{Base case ($m = l$).} If $s = s_{i,r}^{(l, \mathrm{act})}$ already lies at the target neuron
and decomposed activation, then
\begin{align}
    \frac{\partial a_{\mathrm{stop},i}^{(l, p \oplus r)}}{\partial a_{\mathrm{stop},i}^{(l, k)}}
    = \delta_{p \oplus r,\, k}
    = \1\{r = p \oplus k\}
    = \Gamma\bigl(s_{i, p \oplus k}^{(l, \mathrm{act})} \bigm| s_{i,r}^{(l, \mathrm{act})}\bigr).
\end{align}
\emph{Inductive step (linear routing, $m > l$).} Differentiating~\eqref{eq:adf-astop-recursion}
with respect to $a_{\mathrm{stop},i}^{(l, k)}$ (the bias $b_j^{(m, p \oplus q)}$ does not depend on
lower-layer activations) and substituting the inductive hypothesis at layer $m-1$,
\begin{align}
    \frac{\partial a_{\mathrm{stop},j}^{(m, p \oplus q)}}{\partial a_{\mathrm{stop},i}^{(l, k)}}
    &= \1\!\bigl[z_j^{(m)} > 0\bigr]
    \sum_h \Bigl( W_{hj}^{(m,+)}\, \Gamma\bigl(s_{i, p \oplus k}^{(l, \mathrm{act})} \bigm| s_{h,
    q}^{(m-1, \mathrm{act})}\bigr)
    + W_{hj}^{(m,-)}\, \Gamma\bigl(s_{i, p \oplus k}^{(l, \mathrm{act})} \bigm| s_{h, q'}^{(m-1,
    \mathrm{act})}\bigr) \Bigr).
\end{align}
On the game side, the one-step occupation recursion at $s = s_{j,q}^{(m, \mathrm{act})}$ uses
transition probabilities $W_{hj}^{(m, q \oplus r)} / \gamma_j^{(m)}$ together with the matching
cumulative-discount factor $\gamma_j^{(m)}$ and the equilibrium gate $\1[z_j^{(m)} > 0]$:
\begin{align}
    \Gamma\bigl(s_{i, p \oplus k}^{(l, \mathrm{act})} \bigm| s\bigr)
    = \1\!\bigl[z_j^{(m)} > 0\bigr] \sum_{h, r} W_{hj}^{(m, q \oplus r)}\, \Gamma\bigl(s_{i, p \oplus
    k}^{(l, \mathrm{act})} \bigm| s_{h, r}^{(m-1, \mathrm{act})}\bigr).
\end{align}
The two right-hand sides agree, completing the inductive step at linear routing layers.

\emph{Addition-node step.} For an addition node $z = x + y$, the additivity of the Stopping
Decomposition from Appendix~\ref{app:dc-decomp} gives $a_{\mathrm{stop},z}^{(p \oplus q)} =
    a_{\mathrm{stop},x}^{(p \oplus q)} + a_{\mathrm{stop},y}^{(p \oplus q)}$, so by the inductive
hypothesis,
\begin{align}
    \frac{\partial a_{\mathrm{stop},z}^{(p \oplus q)}}{\partial a_{\mathrm{stop},i}^{(l, k)}}
    = \Gamma\bigl(s_{i, p \oplus k}^{(l, \mathrm{act})} \bigm| s_{x,q}\bigr) + \Gamma\bigl(s_{i, p
    \oplus k}^{(l, \mathrm{act})} \bigm| s_{y,q}\bigr).
\end{align}
The Oracle split of Definition~\ref{def:addition-subgame} sends mass $\tfrac12$ to each branch and
applies the discount $\gamma_{\mathrm{O}} = 2$, so the occupation measure satisfies the matching
identity
\begin{align}
    \Gamma\bigl(s_{i, p \oplus k}^{(l, \mathrm{act})} \bigm| s_{z,q}^{\mathrm{add}}\bigr)
    = 2 \cdot \tfrac{1}{2}\, \Gamma\bigl(s_{i, p \oplus k}^{(l, \mathrm{act})} \bigm| s_{x,q}\bigr)
    + 2 \cdot \tfrac{1}{2}\, \Gamma\bigl(s_{i, p \oplus k}^{(l, \mathrm{act})} \bigm| s_{y,q}\bigr).
\end{align}
\emph{Max-pooling step.} For a max-pooling node $z = \max\{x_1, \dots, x_m\}$ with winner
$k^\star$, Proposition~\ref{prop:maxpool-equivalence} gives $a_{\mathrm{stop},z}^{(p \oplus q)} =
    a_{\mathrm{stop},x_{k^\star}}^{(p \oplus q)}$, and the inductive hypothesis yields
\begin{align}
    \frac{\partial a_{\mathrm{stop},z}^{(p \oplus q)}}{\partial a_{\mathrm{stop},i}^{(l, k)}}
    = \Gamma\bigl(s_{i, p \oplus k}^{(l, \mathrm{act})} \bigm| s_{x_{k^\star},q}\bigr).
\end{align}
The deterministic transition with unit discount in Definition~\ref{def:maxpool-subgame} gives the
matching identity $\Gamma(s_{i, p \oplus k}^{(l, \mathrm{act})} \mid s_{z,q}^{\max}) = \Gamma(s_{i,
    p \oplus k}^{(l, \mathrm{act})} \mid s_{x_{k^\star},q})$. Hence, the correspondence also holds
across skip connections and max-pooling nodes, completing the induction.
\end{proof}
\newpage
\section{Routing Game: Log-Space Formulation}\label{app:routing-formal}

This appendix makes the two-player payoff structure of the Routing Game explicit. The recursion
used in Section~\ref{sec:routing} is the logarithmic representation of an entropy-regularized
two-player game, and admits a Gibbs subgame-perfect equilibrium for every $\tau>0$.

\subsection{Exact Routing Kernel Behind \texorpdfstring{$\alpha\beta$-\ac{LRP}}{alpha-beta LRP}}\label{app:routing-kernel}

The linear routing subgames are entropy-regularized one-step maximization problems whose closed
form is given by the Legendre--Fenchel identity for Shannon entropy $\mathcal{H}(\pi) \coloneqq
    -\sum_i \pi(i)\log\pi(i)$~\citep{BoydVandenberghe2004}: for every $\tau>0$, $n\in\NN$, and payoff
vector $r \in \RR^n$,
\begin{align}\label{eq:gibbs-lse}
    \max_{\pi \in \Delta_n} \Bigl\{ \textstyle\sum_{i} \pi(i)\, r_i + \tau\, \mathcal{H}(\pi) \Bigr\}
    = \tau \log \textstyle\sum_{i} \exp(r_i/\tau),
    \qquad \pi^\star(i) \propto \exp(r_i/\tau).
\end{align}

We write
\begin{align}
    \ell(z) \coloneqq
    \begin{cases}
        \log z, & z > 0, \\
        -\infty, & z \le 0,
    \end{cases}
    \qquad \exp(-\infty) \coloneqq 0,
\end{align}
so that $\exp(\ell(x)) = \text{ReLU}(x)$ for any $x \in \RR$; we use this fact for notational
convenience. Zero or negative routed mass thus receives the stopping value $-\infty$; where a
player compares stopping with continuing (activation states), the comparison is made on the
argument of $\ell$, the signed routed mass, so that negative mass is strictly worse than stopping
(mass $0$). The ReLU-gating itself is fully modeled by the player policies at states modeling the
ReLU gates---just like in the Stopping Game.

\paragraph{State Space and Game Mechanics.}
To correctly model the interaction between player turns and stream signs, we define a structured
state space:
\begin{itemize}
    \item \textbf{Activation States:} $\mathcal{S}_{act}^{(l)} = \{ s_{j,p}^{(l, act)} \}$, where $p \in
            \{+,-\}$ is the player in turn.
    \item \textbf{Linear Routing States:} $\mathcal{S}_{lin}^{(l)} = \{ s_{j,q,\kappa}^{(l, lin)} \}$, where
        $q$ is the player turn and $\kappa \in \{+,-\}$ is the stream sign.
\end{itemize}
At $s_{j,p}^{(l, act)}$, the active player $p$ chooses between \emph{stop} and \emph{continue}.
Under the \emph{continue} action, the environment's transition kernel is uniform over the two sign
branches,
\begin{align}
    \P\bigl(s_{j,p,+}^{(l, lin)} \,\big|\, s_{j,p}^{(l, act)},\, \mathrm{continue}\bigr) =
    \tfrac{1}{2},
    \qquad
    \P\bigl(s_{j,p',-}^{(l, lin)} \,\big|\, s_{j,p}^{(l, act)},\, \mathrm{continue}\bigr) =
    \tfrac{1}{2},
\end{align}
where the first branch keeps player $p$ in control evaluating $\kappa=+$ weights and the second
switches the turn to opponent $p'$ evaluating $\kappa=-$ weights. To compensate for the $1/2$
branching, the environment applies an Oracle structural discount of $\gamma_{\mathrm{O}} = 2$
(distinct from the per-neuron weight-magnitude discount $\gamma_j^{(l)}$ used in the Stopping
Game).

\begin{definition}[Entropy-Regularized Two-Player Routing Game]\label{def:routing-recursion}
Fix an input $x$ and a stabilizer $\varepsilon \ge 0$. At each activation state $s_{j,p}^{(l,
    act)}$, the player in turn chooses \emph{stop} or \emph{continue} with similar mechanics as in the
Stopping Game (Appendix~\ref{app:relu-net-game}) reflecting the open-closed state of a ReLU gate:
stopping sends the game to the cemetery $\dagger$ giving immediate payoff $\ell(0) = -\infty$---and
continuation transitions backward with structural discount $\gamma=2$ and the uniform $1/2$
sign-branch kernel given above. Here happens the critical step: gains of the opponent on the
negative stream state $s_{j,p',-}^{(l, \text{lin})}$ shall be counted as losses of the active
player $p$. At each linear routing state $s_{j,q,\kappa}^{(l, lin)}$, the player in turn instead
chooses a mixed action $\nu \in \Delta(\mathrm{Pred}_j^{(l)} \cup \{\mathrm{bias},
    \perp_{\mathrm{lin}}\})$ over the predecessors of neuron $j$ augmented by a terminal \emph{bias
node} with immediate payoff $\ell([b_j^{(l)}]^\kappa)$ and a fictitious \emph{outside option}
$\perp_{\mathrm{lin}}$ which routes the game's mass to the cemetery with immediate payoff $\ell
    (\varepsilon)$; trajectories entering either are absorbed. Let $\mathcal{H}(\nu) \coloneqq
    -\sum_{a} \nu(a)\log \nu(a)$ be Shannon entropy over this augmented action set. At $\varepsilon =
    0$ the outside option carries $\ell (0) = -\infty$ and is dominated by every positive-mass
predecessor, so the augmented game reduces to the Routing Game over $\mathrm{Pred}_j^{(l)} \cup
    \{\mathrm{bias}\}$.

Throughout, $\mathbf{V}(\cdot) = (V_p(\cdot), V_{p'}(\cdot))^\top$ denotes the game's
backward-induction player-specific log-value function, defined jointly at linear routing states
by~\eqref{eq:routing-linear-value} and at activation states by~\eqref{eq:routing-activation-value}
below. The recursion is well-posed because the layered state graph is a finite DAG: the $l=1$ case
of~\eqref{eq:routing-linear-action} needs no lower-layer $\mathbf{V}$, and every subsequent layer
uses only $\mathbf{V}$ at strictly lower layers.

For a \textbf{linear routing state} at $l \geq 2$ the active player continuing to the $i$th
neuron's activation state receives an immediate payoff of the logarithm of the respective weight
summed with the respective log-value function of the next state, whereas in the first layer the
final payoff is paid out to the active player. The inactive player always obtains an immediate
payoff of $\ell(0) = -\infty$. Let us denote by $Q_r(s, a)$ player $r$'s gain for the game value in
case action $a$ actually happens, whereas $\bar Q_r$ denotes the mixture of the differences between
the payoffs of the opponents including the entropy regularization term given the mixed strategy
chosen by the active player.
\begin{align}\label{eq:routing-linear-action}
    Q_r \bigl( s_{j,q,\kappa}^{(l, lin)}, i \bigr)
    \coloneqq
    \begin{cases}
        \ell\!\bigl(W_{ij}^{(l,\kappa)}\bigr) + V_r\bigl(s_{i,q}^{(l-1, act)}\bigr), & r=q,\ l \ge 2,\\[0.35em]
        \ell\!\bigl([x_iW_{ij}^{(1)}]^\kappa\bigr), & r=q,\ l = 1,\\[0.35em]
        -\infty, & r=q'.
    \end{cases}
\end{align}
The outside option $\perp_{\mathrm{lin}}$ and the bias node have log-scores
$Q_r(s_{j,q,\kappa}^{(l,\mathrm{lin})}, \perp_{\mathrm{lin}}) \coloneqq \log \varepsilon$ and
$Q_r(s_{j,q,\kappa}^{(l,\mathrm{lin})}, \mathrm{bias}) \coloneqq \ell([b_j^{(l)}]^\kappa)$ for $r =
    q$ and $-\infty$ for $r = q'$; they route the $\varepsilon$-fraction and the bias share of the mass
to absorbing terminals instead of to any predecessor.
If player $q$ uses mixed action $\nu$ at this state, the induced \emph{individual routed masses}
are
\begin{align}\label{eq:optimal-linear}
    \bar Q_r \bigl( s_{j,q,\kappa}^{(l, lin)};\nu \bigr)
    &\coloneqq
    \sum_a \nu(a)\, \ell \left( \exp \left( Q_r \bigl( s_{j,q,\kappa}^{(l, lin)}, a \bigr) \right) -
    \exp \left( Q_{r'} \bigl( s_{j,q,\kappa}^{(l, lin)}, a \bigr) \right) \right) \notag\\
    &\quad + \tau \mathcal{H}(\nu),
    \qquad r \in \{q,q'\},
\end{align}
which, as $\exp(-\infty)=0$, evaluates to
\begin{align}
    \bar Q_q \bigl( s_{j,q,\kappa}^{(l, lin)};\nu \bigr)
    &=
    \sum_a \nu(a)\, Q_q \bigl( s_{j,q,\kappa}^{(l, lin)}, a \bigr)
    + \tau \mathcal{H}(\nu) \\
    \bar Q_{q'} \bigl( s_{j,q,\kappa}^{(l, lin)};\nu \bigr)
    &= -\infty.
\end{align}
If $W_{ij}^{(l,\kappa)} = 0$ for every predecessor $i$ and $[b_j^{(l)}]^\kappa = 0$ (i.e.\ the
stream-$\kappa$ part of the fan-in to neuron $j$ is identically zero), then
$Q_q(s_{j,q,\kappa}^{(l,\mathrm{lin})},i) = -\infty$ for all $i$ as well and the
value~\eqref{eq:routing-linear-value} equals $\ell(\varepsilon)$. If $\varepsilon=0$ and every
action scores $-\infty$ (e.g.\ $\Sigma_j^{(l,\kappa)}=0$ at $\tau=1$), every $\nu$ is a best
response; we fix $\nu=\delta_{\perp_{\mathrm{lin}}}$ (absorption), matching the convention
$0/0\coloneqq0$ in~\eqref{eq:lrp-eps-recursion}.
We define the corresponding equilibrium linear-game values by
\begin{align}\label{eq:routing-linear-value}
    V^\star_q \bigl( s_{j,q,\kappa}^{(l, lin)} \bigr)
    &\coloneqq
    \max_{\nu \in \Delta(\mathrm{Pred}_j^{(l)} \cup \{\mathrm{bias}, \perp_{\mathrm{lin}}\})}
    \bar Q_q\bigl(s_{j,q,\kappa}^{(l, lin)};\nu\bigr) \nonumber \\
    &=
    \max_{\nu \in \Delta(\mathrm{Pred}_j^{(l)} \cup \{\mathrm{bias}, \perp_{\mathrm{lin}}\})}
    \left\{
    \sum_{a} \nu(a)\, Q_q \bigl( s_{j,q,\kappa}^{(l, lin)}, a \bigr)
    + \tau \mathcal{H}(\nu)
    \right\} \nonumber \\
    &=
    \tau \cdot \ell \left( \varepsilon^{ \frac{1}{\tau} } +
    \bigl([b_j^{(l)}]^\kappa\bigr)^{\frac{1}{\tau}} + \sum_i \exp\left( Q_q(s_{j,q,\kappa}^{(l, lin)},
    i) \right)^{ \frac{ 1 }{ \tau } } \right), \\
    V^\star_{q'} \bigl( s_{j,q,\kappa}^{(l, lin)} \bigr)
    &\coloneqq -\infty,
\end{align}
where we used~\eqref{eq:gibbs-lse}. At $\tau = 1$ the associated Gibbs policy assigns the outside
option the probability $\varepsilon / (\varepsilon + \sum_{a \neq \perp_{\mathrm{lin}}} \exp
    Q_q(\cdots, a))$ and distributes the remaining mass over the predecessors and the bias node in
proportion to $\exp Q_q$.

For an \textbf{activation state}, continuing exposes both sign branches. Let $s_+ \coloneqq
    s_{j,p,+}^{(l, \text{lin})}$ and $s_- \coloneqq s_{j,p',-}^{(l, \text{lin})}$. If player $p$
chooses \emph{continue} at $s_{j,p}^{(l, act)}$, the game transitions to $s_+$ or to $s_-$ with
equal probability, while \emph{stop} sends the game to the cemetery with game value
$V_r(\dagger)=\ell(0)=-\infty$. Taking the expectation over the uniform sign-branch kernel
($\frac{1}{2}$ each) and respecting the discount $\gamma_O = 2$ interpreted in logarithmic space,
the continuation value becomes
\begin{align}\label{eq:routing-activation-continue}
    Q_r \bigl( s_{j,p}^{(l, act)}, \mathrm{continue} \bigr)
    &\coloneqq
    \ell \left( \sum_{s' \in \{s_+, s_-\}} \P(s' \mid s_{j,p}^{(l, act)}) \cdot \gamma_{\mathrm{O}}
    \cdot \Bigl( \exp V_r(s') - \exp V_{r'}(s') \Bigr) \right) \nonumber \\
    &= \ell \left(
    \xi_{r,p} \Bigl( \exp\left( V_p\left( s_+ \right) \right) - \exp\left( V_{p'}\left( s_- \right)
    \right) \Bigr)
    \right).
\end{align}

Observe that at least one of the two equals $-\infty$.
On the other hand, stopping yields $-\infty$ for both.
The equilibrium activation-game values are
\begin{align}\label{eq:routing-activation-value}
    V_p(s_{j,p}^{(l, act)}) &\coloneqq \max\{-\infty, Q_p(s_{j,p}^{(l, act)}, \text{continue})\},
\end{align}
and, since player $p$ continues only if $Q_p(s_{j,p}^{(l, act)}, \text{continue}) > -\infty$, in
which case $Q_{p'}(s_{j,p}^{(l, act)}, \text{continue}) = -\infty$, we have
\begin{align}
    V_{p'}(s_{j,p}^{(l, act)}) = -\infty.
\end{align}
\end{definition}
Figure~\ref{fig:Routing-Game-illustration} shows a worked example of the local routing subgame.

\begin{figure}[ht]
    \centering
    \makebox[\textwidth][c]{\includegraphics[width=0.6\textwidth]{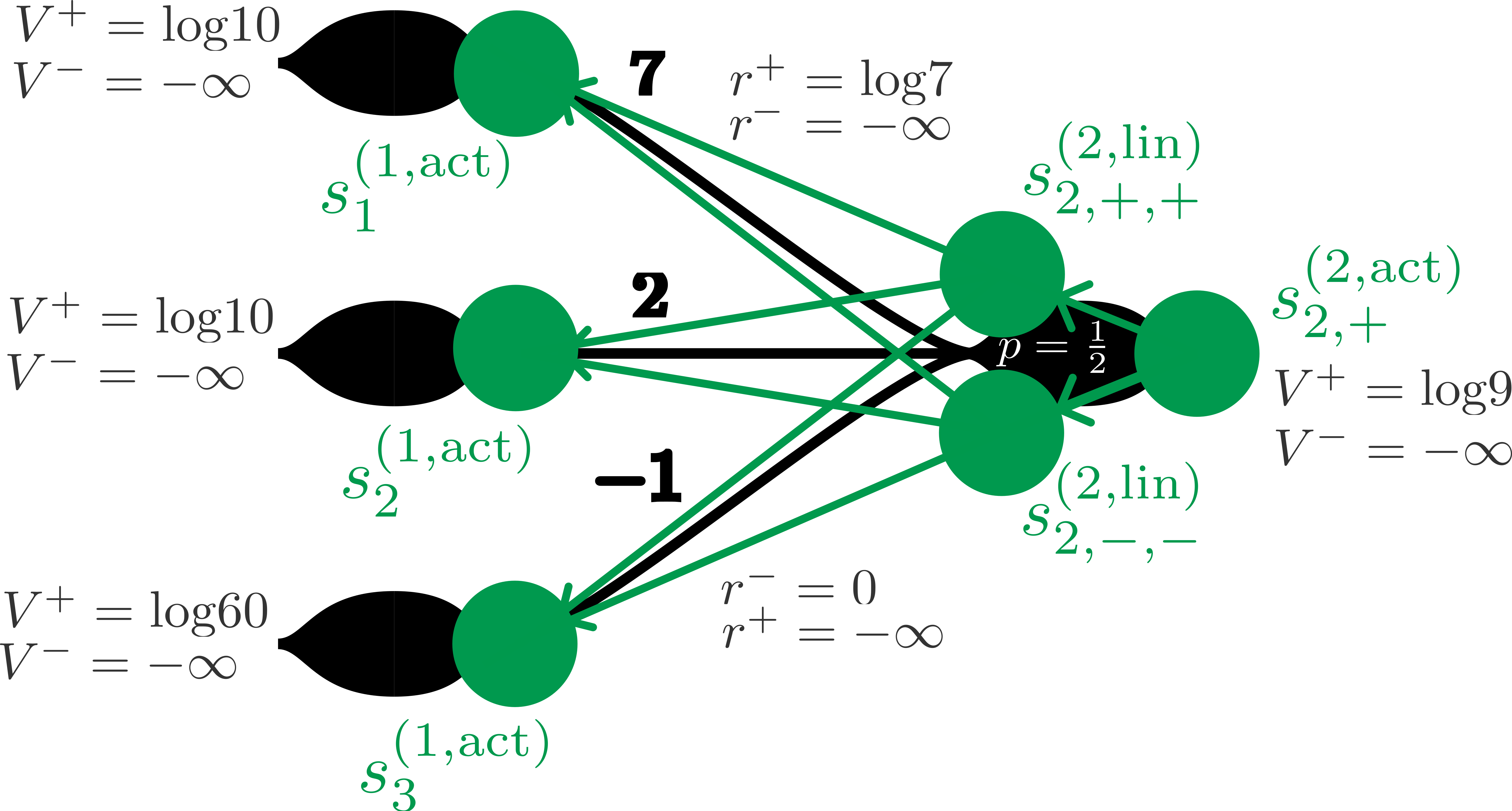}}
    \caption{Routing Game: local routing subgame around $s_{2,+}^{(2,\mathrm{act})}$. This uses the
    same toy subnetwork and numerical values as Figure~\ref{fig:Stopping-Game-illustration}. At this
    activation state, $P^+$ is in turn and continues because the routed mass is positive ($V^+ = \log 9
        > -\infty$ in the depicted example). The game then branches with probability $1/2$ to
    $s_{2,+,+}^{(2,\mathrm{lin})}$ (turn stays with $P^+$) and with probability $1/2$ to
    $s_{2,-,-}^{(2,\mathrm{lin})}$ (turn switches to $P^-$). In each linear state, the active player
    samples a predecessor according to the Gibbs policy from the immediate log-payoff and game value.
    Example: along $s_{2,+,+}^{(2,\mathrm{lin})} \to s_{1,+}^{(1,\mathrm{act})}$, the player on the
    positive stream state gets immediate payoff $\log W^{(2,+)}_{12}=\log 7$ and $P^-$ gets $-\infty$.
    Edges with zero routed mass have payoff $\log W^{(2,-)}_{12}=\log 0 = -\infty$. For readability,
    the $\pm$ activation copies in layer~$1$ are collapsed, and the intermediate payoff $\log 2$ on the
    transition from position~$2$ to $2$ is omitted.}
    \label{fig:Routing-Game-illustration}
\end{figure}

In the following let us denote, with $a^{(0)} = x$,
\begin{align}\label{eq:routing-sigma-def}
    \Sigma_j^{(l,\kappa)}
    \coloneqq \textstyle\sum_i [a_i^{(l-1)}W_{ij}^{(l)}]^\kappa + [b_j^{(l)}]^\kappa,
\end{align}
so that $z_j^{(l)} = \Sigma_j^{(l,+)} - \Sigma_j^{(l,-)}$; for $l \ge 2$, $a^{(l-1)} \ge 0$ gives
$[a_i^{(l-1)}W_{ij}^{(l)}]^\kappa = a_i^{(l-1)} W_{ij}^{(l,\kappa)}$.

\begin{theorem}[SPNE and Forward Equivalence of the Routing Game]\label{thm:routing-forward}
\leavevmode
\begin{description}
    \item[\textnormal{(i) \emph{SPNE existence for every $\tau > 0$ and $\varepsilon \ge 0$.}}]
        For every $\tau > 0$ and every $\varepsilon \ge 0$, the entropy-regularized two-player Routing Game
        of Definition~\ref{def:routing-recursion} (with linear-state entropy weight $\tau\,\mathcal{H}$ and
        outside-option stabilizer $\varepsilon$) admits a subgame-perfect Nash equilibrium
        (see~\eqref{eq:routing-spne-linear}--\eqref{eq:routing-spne-stop} in the proof for the explicit
        form of the equilibrium policy). We write $\cdot^\star$ for equilibrium quantities throughout.

    \item[\textnormal{(ii) \emph{Forward equivalence at the anchor $\tau=1$, $\varepsilon\geq 0$.}}]
        At $\tau=1$ and $\varepsilon \geq 0$, the equilibrium quantities recover the network forward pass:
        for every layer $l$, neuron $j$, player labels $p,q,r$, and stream sign $\kappa\in\{+,-\}$,
        \begin{align}
            V_r\bigl(s_{j,q,\kappa}^{(l, lin)}\bigr) &= \ell\bigl(\delta_{r,q} \, (\Sigma_j^{(l,\kappa)} +
            \varepsilon)\bigr), \label{eq:routing-forward-linear}\\
            Q_r\left( s_{j,p}^{(l, act)}, \text{continue} \right) &= \ell \left( \xi_{r,p} \cdot z^{(l)}_j
            \right), \label{eq:routing-forward-continue} \\
            V_r\bigl(s_{j,p}^{(l, act)}\bigr) &= \ell\bigl( \delta_{r,p}\ \cdot a_j^{(l)} \bigr),
            \label{eq:routing-forward-act}\\
            \pi_p^\star\bigl(s_{j,p}^{(l, act)}\bigr) &= \1[z_j^{(l)} > 0]. \label{eq:routing-forward-stop}
        \end{align}
\end{description}
\end{theorem}

\begin{remark}[Uniqueness]\label{rem:routing-spne-uniqueness}
The linear-state Gibbs best response~\eqref{eq:routing-spne-linear} is the unique maximizer of the
strictly concave entropy-regularized objective whenever some action has a finite score (otherwise
Definition~\ref{def:routing-recursion} fixes $\perp_{\mathrm{lin}}$). At activation states the
active player is indifferent only when the signed continuation mass $\exp V_p(s_+) - \exp
    V_{p'}(s_-)$ is zero ($z_j^{(l)} = 0$ at $\tau = 1$): both stopping and continuing are best
responses, so multiple SPNEs can coexist under different tie-breaking. The theorem specifies the
tie-breaker $\pi_p^\star = \1[Q_p(\cdot,\mathrm{continue}) > -\infty]$, which at $\tau=1$ reduces
to $\1[z_j^{(l)} > 0]$. Just like in the \ac{SG} (Appendix~\ref{app:relu-net-game}), we refer to
this SPNE as the \textbf{stop-at-tie-breaks} equilibrium.
\end{remark}

\begin{proof}
\textbf{SPNE existence for every $\tau > 0$.} The layered state graph is a finite DAG: linear
routing states at layer $l$ transition to activation states at layer $l-1$ (or to terminal input
positions at $l=1$), and activation states at layer $l$ transition to linear routing states at the
same layer $l$ or to the cemetery $\dagger$. We construct the equilibrium by backward induction on
the layer index $l$.

At layer $l=1$ the linear routing subgames have no unresolved game values, so the action values
$Q_q$ in~\eqref{eq:routing-linear-action} are fully specified by the input $x$. At each linear
state the active player solves $\max_{\nu \in \Delta(\mathrm{Pred}_j^{(1)})} \{ \sum_i \nu(i)\,
    Q_q(s_{j,q,\kappa}^{(1, lin)}, i) + \tau\, \mathcal{H}(\nu) \}$, which is strictly concave on the
simplex. Identity~\eqref{eq:gibbs-lse} yields the unique Gibbs best response, and simplifying $\exp
    \circ \ell$ in~\eqref{eq:routing-linear-action} gives the explicit form
\begin{align}\label{eq:routing-spne-linear}
    \pi_q^\star\!\bigl(i \,\big|\, s_{j,q,\kappa}^{(l, lin)}\bigr)
    \propto
    \begin{cases}
        \bigl(W_{ij}^{(l,\kappa)}\bigr)^{1/\tau}\, \exp\!\bigl( V_q(s_{i,q}^{(l-1, act)})/\tau \bigr), & l
        \ge 2, \\[0.25em]
        \bigl([x_i W_{ij}^{(1)}]^\kappa\bigr)^{1/\tau}, & l = 1,
    \end{cases}
\end{align}
with zero mass on routes of zero or wrong sign and closed-form value $V_q(s_{j,q,\kappa}^{(1,lin)})
    = \tau \log \sum_i \exp(Q_q(s_{j,q,\kappa}^{(1,lin)},i)/\tau)$. With these values in hand, the
activation-state subgame at layer $l=1$ is a binary choice between stop (log-payoff $\ell(0) =
    -\infty$) and continue~\eqref{eq:routing-activation-continue}; comparing signed masses and stopping
at ties, the active player's best response is the pure stop/continue strategy
\begin{align}\label{eq:routing-spne-stop}
    \pi_p^\star\bigl(s_{j,p}^{(l, act)}\bigr)
    = \1\!\bigl[ Q_p(s_{j,p}^{(l, act)}, \mathrm{continue}) > -\infty \bigr].
\end{align}

Assume the equilibrium has been constructed at all layers $\le l-1$, yielding definite values
$V_r(s_{i,q}^{(l-1,act)})$. At layer $l$ the linear-state action values $Q_q$
in~\eqref{eq:routing-linear-action} are then specified, and the same strict-concavity argument
gives the Gibbs best response~\eqref{eq:routing-spne-linear} with the analogous log-sum-exp value.
The layer-$l$ activation-state best response is again the binary hard
comparison~\eqref{eq:routing-spne-stop}. Because each subgame's best response is taken to be
optimal given the game values from lower layers, the composed strategy profile is subgame-perfect.
This proves existence of an SPNE for every $\tau > 0$. The argument is unchanged under the
bias-node and outside-option augmentation of Definition~\ref{def:routing-recursion}: the
linear-state action set gains the extra elements $\mathrm{bias}$ and $\perp_{\mathrm{lin}}$ with
constant log-scores $\ell([b_j^{(l)}]^\kappa)$ and $\log \varepsilon$ (finite or $-\infty$), the
strict-concavity / Legendre--Fenchel step goes through verbatim on the augmented simplex, and the
closed-form value picks up the extra terms $([b_j^{(l)}]^\kappa)^{1/\tau}$ and
$\varepsilon^{1/\tau}$ inside the log-sum-exp. Hence, an SPNE exists for every $(\tau, \varepsilon)
    \in \RR_{>0} \times \RR_{\ge 0}$.

\textbf{Forward equivalence at $\tau = 1$.} We now specialize to $\tau = 1$ and show by induction
on $l$ that the equilibrium quantities recover the network forward pass.

\textbf{Base case ($l=1$):} Identical to the inductive step below, with the fused-input action
values $Q_q(s_{j,q,\kappa}^{(1, lin)}, i) = \ell([x_iW_{ij}^{(1)}]^\kappa)$
of~\eqref{eq:routing-linear-action} in place of $\ell(a_i^{(l-1)} W_{ij}^{(l,\kappa)})$.

\textbf{Inductive step ($l \ge 2$):} Assume $V_q(s_{i,q}^{(l-1, act)}) = \ell(a_i^{(l-1)})$ for all
predecessors $i$ and both player labels $q$. Since standard activations satisfy $a_i^{(l-1)} \ge
    0$, the linear routing-state action values become
\begin{align}
    Q_q \bigl( s_{j,q,\kappa}^{(l, lin)}, i \bigr)
    &=
    \ell\!\bigl(W_{ij}^{(l,\kappa)}\bigr) + \ell\!\bigl(a_i^{(l-1)}\bigr)
    =
    \ell\!\bigl(a_i^{(l-1)} W_{ij}^{(l,\kappa)}\bigr).
\end{align}
Applying~\eqref{eq:gibbs-lse} to the linear routing subgame gives the Gibbs best
response~\eqref{eq:routing-spne-linear} and
\begin{align}
    V_q \bigl( s_{j,q,\kappa}^{(l, lin)} \bigr)
    &=
    \ell \left( \sum_a \exp\!\bigl(Q_q(s_{j,q,\kappa}^{(l, lin)}, a)\bigr) \right)
    =
    \ell \left( \sum_i a_i^{(l-1)} W_{ij}^{(l,\kappa)} + [b_j^{(l)}]^\kappa + \varepsilon \right)
    =
    \ell(\Sigma_j^{(l,\kappa)} + \varepsilon),
\end{align}
At the activation state the active player's continuation advantage is again $z_j^{(l)}$
\begin{align}
    Q_p \bigl( s_{j,p}^{(l, act)}, \text{continue} \bigr)
    &=
    \ell \left( \Sigma_j^{(l,+)} - \Sigma_j^{(l,-)} \right)
    =
    \ell(z_j^{(l)}),
\end{align}
whereas for the opponent
\begin{align}
    Q_{p'}\left( s_{j,p}^{(l,\text{act})}, \text{continue} \right)
    = \ell \left( \Sigma_j^{(l,-)} - \Sigma_j^{(l,+)} \right)
    = \ell (-z_j^{(l)}),
\end{align}
again confirming~\eqref{eq:routing-forward-continue}.
Thus, the best response is~\eqref{eq:routing-spne-stop}. Finally,
\begin{align}
    V_r \bigl( s_{j,p}^{(l, act)} \bigr) = \ell\bigl(\delta_{r,p}\, \1[z_j^{(l)} > 0]\, z_j^{(l)}\bigr)
    = \ell\bigl(\delta_{r,p}\, a_j^{(l)}\bigr),
\end{align}
proving~\eqref{eq:routing-forward-act}. Since each subgame has an optimal action chosen by the
player in turn, the resulting strategy profile is subgame-perfect. This completes the induction.
\end{proof}

\subsubsection{Log-Space Addition Subgame, Skip Connections, and Max
Pooling}\label{app:routing-skip}

Residual additions $z = x + y$ and max-pooling nodes $z = \max_k x_k$ enter the \ac{RG} through the
addition state $s^{\mathrm{add}}_{z,p}$ (uniform Oracle to the operand states, player $p$ retained,
$\gamma_{\mathrm{O}} = 2$) and the pooling state $s^{\max}_{z,p}$ (deterministic choice of an
operand, player $p$ retained, $\gamma_{\mathrm{O}} = 1$, ties broken by the smallest index) of
Appendix~\ref{app:skip-connections}, with the values written in log space:
\begin{align}\label{eq:routing-add-max}
    V_r(s^{\mathrm{add}}_{z,p}) = \ell\bigl(\exp V_r(s_{x,p}) + \exp V_r(s_{y,p})\bigr),
    \qquad
    V_p(s^{\max}_{z,p}) = \max_k V_p(s_{x_k,p}).
\end{align}
If the operand states recover their (non-negative) activations, $V_p(s_{x,p}) = \ell(a_x)$ etc.,
then $V_p(s^{\mathrm{add}}_{z,p}) = \ell(a_x + a_y) = \ell(a_z)$ and $V_p(s^{\max}_{z,p}) =
    \ell(a_{x_{k^\star}}) = \ell(a_z)$, so forward equivalence extends across both nodes. Two things
differ from the \ac{SG}. First, $\gamma_{\mathrm{O}} = 2$ enters only the values: the trajectory
discount~\eqref{eq:routing-edge-discount} of the addition transition is $1$, so each operand
receives half of the addition-state occupation (Theorem~\ref{thm:ab-recovery}, item~2). Second, the
relevance split therefore equals the Oracle probabilities: a $(\rho, 1-\rho)$ Oracle with value
discounts $(1/\rho, 1/(1-\rho))$ leaves all forward values unchanged but splits the relevance as
$(\rho, 1-\rho)$, so the uniform split is a convention of the \ac{RG} rule. The implementation
passes the full relevance to both operands (trajectory discount $2$); since every backward
trajectory of a ResNet or \ac{ViT} that reaches the input crosses each of its $B$ addition nodes
exactly once, this rescales the input occupation by the constant $2^B$ and leaves normalized maps
unchanged.

\subsection{Backward Path Integral and \texorpdfstring{$\alpha\beta$-\ac{LRP} Recovery}{alpha-beta LRP Recovery}}\label{app:routing-attribution}

We now show that the projected discounted occupation measure induced by the SPNE on the layered
routing graph realizes the $\alpha\beta$-\ac{LRP}-$\varepsilon$ relevance vector
\ch{of~\eqref{eq:lrp-ab}}. To keep
the game side and the attribution side strictly separate, we use distinct symbols throughout this
subsection:
\begin{itemize}
    \item the \emph{game-side} object is the discounted occupation measure $\Gamma_x(\,\cdot\,)$ on the
        states of the routing graph (Definition~\ref{def:occupation}), now driven by the SPNE trajectory
        law $\Pi_x$ rooted at $s_{u,+}^{(L,\mathrm{act})}$ with unit initial mass:
        $\Gamma_x(s_{u,+}^{(L,\mathrm{act})}) = 1$ and $\Gamma_x(s_{j,p}^{(L,\mathrm{act})}) = 0$ for every
        other output activation state. With non-negative branch discounts
        (cf.~\eqref{eq:routing-edge-discount}) $\Gamma_x$ is itself non-negative; we then form the
        per-neuron difference of the two player copies,
        \begin{align}\label{eq:routing-Gamma-pair}
            \Gamma^{(l)}_j \;\coloneqq\;
            \Gamma_x\!\bigl(s_{j,+}^{(l,\mathrm{act})}\bigr)
            \;-\;
            \Gamma_x\!\bigl(s_{j,-}^{(l,\mathrm{act})}\bigr).
        \end{align}
    \item the \emph{attribution-side} object is the $\alpha\beta$-\ac{LRP}-$\varepsilon$ relevance
        vector $R^{(l)}$ \ch{(the $\alpha\beta$ rule of~\citet{Bach2015} with the stabilization convention
        adopted here, see below)}, a deterministic backward DP on the network's linear layers
        seeded by a chosen output mass $R_u^{(L)}$ and recursing via~\eqref{eq:lrp-eps-recursion} below.
\end{itemize}
The recovery theorem (Theorem~\ref{thm:ab-recovery}) shows that, under the SPNE policy and the
following trajectory discount~\eqref{eq:routing-edge-discount}, the per-neuron difference
$\Gamma^{(l)}_j$ satisfies the same recursion as $R^{(l)}_j$, hence equals $R^{(l)}_j/R^{(L)}_u$.

\paragraph{Trajectory discount.}
We track sign-dependent branch discounts governed by the stream-balance mode $(\alpha, \beta)$
($\alpha-\beta=1$). Let $\zeta = (\zeta_0,\zeta_1,\dots)$ be a backward trajectory on linear
routing, activation, addition or max-pooling states. Define the per-state discount $\gamma$ as a
non-negative function of the entered state,
\begin{align}\label{eq:routing-edge-discount}
    \gamma(s) \coloneqq
    \begin{cases}
        2\alpha, & s = s_{j,q,+}^{(l, lin)} \text{ for some } j,q,l,\\
        2\beta, & s = s_{j,q,-}^{(l, lin)} \text{ for some } j,q,l,\\
        1, & \text{otherwise (activation, addition Oracle, max-pool).}
    \end{cases}
\end{align}
The cumulative discount along the trajectory up to time $t$ is the product
\begin{align}\label{eq:routing-discount}
    d_t(\zeta) \coloneqq \prod_{m=1}^{t} \gamma(\zeta_m),
    \qquad d_0(\zeta) = 1 \text{ (empty product).}
\end{align}
The signed minus of the $\alpha\beta$-\ac{LRP} recursion is produced by the
difference~\eqref{eq:routing-Gamma-pair} between the two player copies.

\begin{remark}[Choice of starting player]\label{rem:starting-player}
Rooting $\Pi_x$ at $s_{u,-}^{(L,\mathrm{act})}$ instead of $s_{u,+}^{(L,\mathrm{act})}$ swaps the
two output-layer occupations and therefore negates $\Gamma^{(l)}_j$ at every layer. We fix the
canonical convention that the game starts at the $+$-state.
\end{remark}

\begin{remark}[Equivalent signed-discount sum formulation]\label{rem:signed-discount-alt}
The same recovery theorem can be phrased with the signed trajectory discount
$\gamma(s_{j,q,+}^{(l,\mathrm{lin})}) = 2\alpha$, $\gamma(s_{j,q,-}^{(l,\mathrm{lin})}) = -2\beta$
and the per-neuron \emph{sum} $\Gamma^{(l)}_j = \Gamma_x(s_{j,+}^{(l,\mathrm{act})}) +
    \Gamma_x(s_{j,-}^{(l,\mathrm{act})})$ in place of~\eqref{eq:routing-Gamma-pair}. We have chosen the
non-negative-discount version because it is easier to comprehend---$\Gamma_x$ remains a genuine
non-negative occupation measure, and the LRP sign is generated by an explicit algebraic difference
rather than absorbed into the trajectory weights.
\end{remark}

\paragraph{Attribution-side: \texorpdfstring{$\alpha\beta$-\ac{LRP}-$\varepsilon$}{alpha-beta-LRP-epsilon} relevance.}
Independently of the game, the $\alpha\beta$-\ac{LRP}-$\varepsilon$ relevance vector
$R^{(l)}$ is defined recursively from $R^{(L)}_j = R_u^{(L)}\,\delta_{j,u}$ at the output by
\begin{align}\label{eq:lrp-eps-recursion}
    R^{(l-1)}_i \;=\; \sum_j \left(
    \alpha \frac{[a_i^{(l-1)} W_{ij}^{(l)}]^+}{\varepsilon + \Sigma_j^{(l,+)}}
    -
    \beta \frac{[a_i^{(l-1)} W_{ij}^{(l)}]^-}{\varepsilon + \Sigma_j^{(l,-)}}
    \right) R^{(l)}_j,
\end{align}
where $\Sigma_j^{(l,\pm)}$ is~\eqref{eq:routing-sigma-def}: the bias enters the denominators as a
terminal bias node, and the relevance routed to it (and to the outside option) is absorbed, as in
the implementation. \ch{Recursion~\eqref{eq:lrp-eps-recursion} is the $\alpha\beta$ rule
of~\citet{Bach2015} with the stabilization convention adopted throughout this paper: the stabilizer
$\varepsilon$ of their \ac{LRP}-$\varepsilon$ rule is added separately to each sign denominator. \citet{Bach2015} present the two rules as alternative
constructions and do not state this combined form; it is the convention of common \ac{LRP}
implementations~\citep{Anders2021zennit}, and Theorem~\ref{thm:ab-recovery} is exact for this form.
At $\varepsilon = 0$ it is the original $\alpha\beta$ rule.}

\begin{theorem}[Recovery of Multi-Path $\alpha\beta$-\ac{LRP}]\label{thm:ab-recovery}
For every $\varepsilon \ge 0$ and every input with $f(x)=a_u^{(L)}>0$ (if $f(x)=0$ the root player
stops and $R^{(l)}=f(x)\Gamma^{(l)}=0$ trivially), under the stop-at-tie-breaks equilibrium
(Remark~\ref{rem:routing-spne-uniqueness}; ties at max-pooling states broken by the smallest index)
at the anchor temperature $\tau = 1$ and stabilizer $\varepsilon$, with the cumulative trajectory
discount~\eqref{eq:routing-discount} and the occupation measure $\Gamma_j^{(l)}$
from~\eqref{eq:routing-Gamma-pair}:
\begin{enumerate}
    \item \textbf{Linear Map.} The per-neuron difference $\Gamma^{(l)}_j$ of~\eqref{eq:routing-Gamma-pair}
        satisfies
        \begin{align}\label{eq:lrp-eps-recovery}
            \Gamma^{(l-1)}_i
            \;=\; \sum_j \left(
            \alpha \frac{[a_i^{(l-1)} W_{ij}^{(l)}]^+}{\varepsilon + \Sigma_j^{(l,+)}}
            -
            \beta \frac{[a_i^{(l-1)} W_{ij}^{(l)}]^-}{\varepsilon + \Sigma_j^{(l,-)}}
            \right) \Gamma^{(l)}_j.
        \end{align}
        Because~\eqref{eq:lrp-eps-recovery} and~\eqref{eq:lrp-eps-recursion} share their computational
        step, $\Gamma^{(l)}_j \cdot R_u^{(L)} = R^{(l)}_j$ at every layer; in particular $R^{(l)}_j =
            f(x)\,\Gamma^{(l)}_j$ for the seed $R^{(L)}_u = f(x)$. The same formula covers the fused input
        layer with $a^{(0)} = x$. At $\varepsilon = 0$ this reduces to the unstabilized
        $\alpha\beta$-\ac{LRP} redistribution (not conservative: bias nodes absorb relevance); at
        $\varepsilon > 0$ \ch{it is the $\alpha\beta$ rule with the stabilization convention adopted here,
        i.e.\ \eqref{eq:lrp-ab} with $\varepsilon$ added to each sign denominator separately.}
    \item \textbf{Residual Addition.} Let $z = x_{\mathrm{op}} + y_{\mathrm{op}}$ be an addition node with
        addition state $s^{\mathrm{add}}_{z,p}$ and operand states $s_{x_{\mathrm{op}},p},
            s_{y_{\mathrm{op}},p}$ as in Appendix~\ref{app:routing-skip}. For every player label $p\in\{+,-\}$,
        with the edge occupation $\Gamma_x(s\to s')\coloneqq\E_{\zeta\sim\Pi_x}\bigl[\sum_t
            d_{t+1}(\zeta)\1\{\zeta_t=s,\zeta_{t+1}=s'\}\bigr]$ (an operand's total occupation adds the edges
        from its other parents),
        \begin{align}
            \Gamma_x\!\bigl(s^{\mathrm{add}}_{z,p}\to s_{x_{\mathrm{op}},p}\bigr) \;=\;
            \tfrac12\,\Gamma_x\!\bigl(s^{\mathrm{add}}_{z,p}\bigr),
            \qquad
            \Gamma_x\!\bigl(s^{\mathrm{add}}_{z,p}\to s_{y_{\mathrm{op}},p}\bigr) \;=\;
            \tfrac12\,\Gamma_x\!\bigl(s^{\mathrm{add}}_{z,p}\bigr),
        \end{align}
        so the operand pair carries the full addition-state mass with no duplication.
    \item \textbf{Max Pooling.} For a pooled output $z = \max\{x_1,\dots,x_m\}$ with winner $k^\star$ and
        pooling state $s_{z,p}^{\max}$ (Appendix~\ref{app:routing-skip}), the deterministic
        value-maximizing transition concentrates all mass on the winner: for every player label $p \in
            \{+,-\}$,
        \begin{align}
            \Gamma_x\!\bigl(s^{\max}_{z,p}\to s_{x_{k^\star},p}\bigr) \;=\;
            \Gamma_x\!\bigl(s_{z,p}^{\max}\bigr),
            \qquad
            \Gamma_x\!\bigl(s^{\max}_{z,p}\to s_{x_{k},p}\bigr) \;=\; 0 \text{ for } k \neq k^\star.
        \end{align}
\end{enumerate}
\end{theorem}

\begin{proof}
We prove $R_u^{(L)} \cdot \Gamma^{(l)}_j = R^{(l)}_j$ at every layer $l$ by backward induction on
$l$.

\emph{Base case ($l = L$):} The trajectory law $\Pi_x$ has unit mass at
$s_{u,+}^{(L,\mathrm{act})}$ by definition, so $\Gamma_x(s_{u,+}^{(L,\mathrm{act})}) = 1$ and
$\Gamma_x(s_{j,p}^{(L,\mathrm{act})}) = 0$ for every $(j,p) \neq (u,+)$.
By~\eqref{eq:routing-Gamma-pair},
\begin{align}
    \Gamma^{(L)}_j \;=\; \Gamma_x\!\bigl(s_{j,+}^{(L,\mathrm{act})}\bigr) -
    \Gamma_x\!\bigl(s_{j,-}^{(L,\mathrm{act})}\bigr) \;=\; \delta_{j,u},
\end{align}
i.e.\ the difference is $1$ at the chosen output neuron $u$ and $0$ elsewhere. Multiplying by the
output mass gives $R_u^{(L)} \cdot \Gamma^{(L)}_j = R_u^{(L)}\,\delta_{j,u} = R^{(L)}_j$, matching
the LRP root specification of~\eqref{eq:lrp-eps-recursion}.

\emph{Inductive step ($l \Rightarrow l-1$):} Assume $R_u^{(L)} \cdot \Gamma^{(l)}_j = R^{(l)}_j$
for every neuron $j$ in layer $l$. From the SPNE construction in Theorem~\ref{thm:routing-forward}
together with the activation-game value reduction $V_p(s_{j,p}^{(l,\mathrm{act})}) =
    \ell(a_j^{(l)})$ (which holds for every $\varepsilon \ge 0$ because the additive
$\varepsilon$-shifts on the two sign branches cancel in the continuation advantage), the
linear-state Gibbs best response on the alive support is
\begin{align}\label{eq:lrp-eps-policy}
    \pi_{j,\kappa}^\star(\perp_{\mathrm{lin}}) = \frac{\varepsilon}{\varepsilon +
    \Sigma_j^{(l,\kappa)}},
    \qquad
    \pi_{j,\kappa}^\star(\mathrm{bias}) = \frac{[b_j^{(l)}]^\kappa}{\varepsilon +
    \Sigma_j^{(l,\kappa)}},
    \qquad
    \pi_{j,\kappa}^\star(i) = \frac{a_i^{(l-1)} W_{ij}^{(l,\kappa)}}{\varepsilon +
    \Sigma_j^{(l,\kappa)}}.
\end{align}
A trajectory reaches the layer-$(l-1)$ activation state $s_{i,r}^{(l-1, \mathrm{act})}$ from a
layer-$l$ activation state through one of two routes. By the activation-state continuation kernel
of Definition~\ref{def:routing-recursion}, departing from $s_{j,p}^{(l, \mathrm{act})}$ (the player
continues: at $l=L$ since $f(x)>0$; at $l<L$, $\Gamma_x(s_{j,\cdot}^{(l,\mathrm{act})})=0$ unless
$a_j^{(l)}>0$ by~\eqref{eq:lrp-eps-policy}) the trajectory takes the $+$ branch to $s_{j,p,+}^{(l,
    \mathrm{lin})}$ with probability $\tfrac12$ (player $p$ retains the turn), and the $-$ branch to
$s_{j,p',-}^{(l, \mathrm{lin})}$ with probability $\tfrac12$ (turn switches to opponent $p'$); from
each linear state the policy~\eqref{eq:lrp-eps-policy} routes to predecessor $i$ with the player
label of that linear state. Combining the two branch discount factors $2\alpha, 2\beta$
from~\eqref{eq:routing-edge-discount} with these $\tfrac12$-probabilities, the trajectory
contributions to $\Gamma_x(s_{i,r}^{(l-1,\mathrm{act})})$ are
\begin{align}
    \Gamma_x\!\bigl(s_{i,r}^{(l-1, \mathrm{act})}\bigr)
    &=
    \sum_j \left(
    \tfrac12 (2\alpha)\, \pi_{j,+}^\star(i)\, \Gamma_x\!\bigl(s_{j,r}^{(l, \mathrm{act})}\bigr)
    \;+\;
    \tfrac12 (2\beta)\, \pi_{j,-}^\star(i)\, \Gamma_x\!\bigl(s_{j,r'}^{(l, \mathrm{act})}\bigr)
    \right) \nonumber \\
    &=
    \sum_j \Bigl(
    \alpha\, \pi_{j,+}^\star(i)\, \Gamma_x\!\bigl(s_{j,r}^{(l, \mathrm{act})}\bigr)
    \;+\;
    \beta\, \pi_{j,-}^\star(i)\, \Gamma_x\!\bigl(s_{j,r'}^{(l, \mathrm{act})}\bigr)
    \Bigr), \label{eq:routing-per-state-recursion}
\end{align}
where $r' \in \{+,-\}$ denotes the opposite player label and the absorbing actions $\mathrm{bias}$
and $\perp_{\mathrm{lin}}$ contribute no mass to any predecessor.

Differencing~\eqref{eq:routing-per-state-recursion} over $r \in \{+,-\}$ (i.e.\ $r{=}+$ minus
$r{=}-$) flips the role of $r$ and $r'$ between the two terms inside the bracket. Using
$\Gamma_x(s_{j,r}^{(l,\mathrm{act})}) - \Gamma_x(s_{j,r'}^{(l,\mathrm{act})}) = +\Gamma^{(l)}_j$
when $r{=}+$, $r'{=}-$, and the same expression with the opposite sign when $r{=}-$, $r'{=}+$, the
per-neuron difference at layer $l-1$ satisfies
\begin{align}
    \Gamma^{(l-1)}_i
    &=
    \Gamma_x\!\bigl(s_{i,+}^{(l-1, \mathrm{act})}\bigr) - \Gamma_x\!\bigl(s_{i,-}^{(l-1,
    \mathrm{act})}\bigr) \nonumber \\
    &=
    \sum_j \Bigl(
    \alpha\, \pi_{j,+}^\star(i)\, \bigl[\Gamma_x(s_{j,+}^{(l,\mathrm{act})}) -
    \Gamma_x(s_{j,-}^{(l,\mathrm{act})})\bigr]
    \;+\;
    \beta\, \pi_{j,-}^\star(i)\, \bigl[\Gamma_x(s_{j,-}^{(l,\mathrm{act})}) -
    \Gamma_x(s_{j,+}^{(l,\mathrm{act})})\bigr]
    \Bigr) \nonumber \\
    &=
    \sum_j \bigl(\alpha\, \pi_{j,+}^\star(i) - \beta\, \pi_{j,-}^\star(i)\bigr)\, \Gamma^{(l)}_j
    \nonumber \\
    &=
    \sum_j \left(
    \alpha \frac{[a_i^{(l-1)} W_{ij}^{(l)}]^+}{\varepsilon + \Sigma_j^{(l,+)}}
    -
    \beta \frac{[a_i^{(l-1)} W_{ij}^{(l)}]^-}{\varepsilon + \Sigma_j^{(l,-)}}
    \right) \Gamma^{(l)}_j,
\end{align}
where the third equality regroups the two square brackets as $+\Gamma^{(l)}_j$ and
$-\Gamma^{(l)}_j$ respectively, and the fourth substitutes~\eqref{eq:lrp-eps-policy} together with
$a_i^{(l-1)} W_{ij}^{(l,\kappa)} = [a_i^{(l-1)} W_{ij}^{(l)}]^\kappa$ for non-negative
$a_i^{(l-1)}$. Multiplying both sides by $R_u^{(L)}$ and using the inductive hypothesis $R_u^{(L)}
    \cdot \Gamma^{(l)}_j = R^{(l)}_j$,
\begin{align}
    R_u^{(L)} \cdot \Gamma^{(l-1)}_i
    \;=\; \sum_j \left(
    \alpha \frac{[a_i^{(l-1)} W_{ij}^{(l)}]^+}{\varepsilon + \Sigma_j^{(l,+)}}
    -
    \beta \frac{[a_i^{(l-1)} W_{ij}^{(l)}]^-}{\varepsilon + \Sigma_j^{(l,-)}}
    \right) R^{(l)}_j
    \;=\; R^{(l-1)}_i.
\end{align}
This closes the induction and proves $R_u^{(L)} \cdot \Gamma^{(l)}_j = R^{(l)}_j$ for every layer
$l$ and neuron $j$.

\emph{Input layer ($l = 1$):} The same calculation applies with the fused first-layer masses $[x_i
    W_{ij}^{(1)}]^\kappa$ in place of $a_i^{(l-1)} W_{ij}^{(l,\kappa)}$, yielding the identical formula
with $a^{(0)} = x$.

\emph{Addition node:} Let $z = x_{\mathrm{op}} + y_{\mathrm{op}}$ be an addition node. By
Appendix~\ref{app:routing-skip}, from $s^{\mathrm{add}}_{z,p}$ the Oracle transitions to
$s_{x_{\mathrm{op}},p}$ and $s_{y_{\mathrm{op}},p}$ with probability $\tfrac12$ each, retaining
player label $p$. Under the trajectory-discount
convention~\eqref{eq:routing-edge-discount}--\eqref{eq:routing-discount}, addition Oracle
transitions carry trajectory-discount factor $1$. Hence
\begin{align}
    \Gamma_x\!\bigl(s^{\mathrm{add}}_{z,p}\to s_{x_{\mathrm{op}},p}\bigr) \;=\;
    \tfrac12\,\Gamma_x\!\bigl(s^{\mathrm{add}}_{z,p}\bigr),
    \qquad
    \Gamma_x\!\bigl(s^{\mathrm{add}}_{z,p}\to s_{y_{\mathrm{op}},p}\bigr) \;=\;
    \tfrac12\,\Gamma_x\!\bigl(s^{\mathrm{add}}_{z,p}\bigr),
\end{align}
proving Item 2 of Theorem~\ref{thm:ab-recovery} for both player labels.

\emph{Max-pooling node:} Let $z = \max\{x_1,\dots,x_m\}$ with winner $k^\star$ and pooling state
$s_{z,p}^{\max}$. By Appendix~\ref{app:routing-skip}, the active player has the unique
value-maximizing action $x_{k^\star}$ at $s_{z,p}^{\max}$, so the transition to $s_{x_{k^\star},p}$
is deterministic, retains player label $p$, and~\eqref{eq:routing-edge-discount} assigns unit
trajectory factor. Therefore $\Gamma_x(s_{x_{k^\star},p}) = \Gamma_x(s_{z,p}^{\max})$ and
$\Gamma_x(s_{x_k,p}) = 0$ for $k \neq k^\star$, proving Item 3.
\end{proof}

\subsection{Temperature \texorpdfstring{$\tau$}{tau} as a Controlled Policy Deformation}\label{app:routing-temp}

Theorem~\ref{thm:routing-forward}(i) establishes an SPNE for every $\tau > 0$ when the entropy
weight $\tau$ is used throughout the backward recursion, but exact $\alpha\beta$-\ac{LRP} recovery
(Theorem~\ref{thm:ab-recovery}) holds only at $\tau=1$. In practice, we deform $\tau$ only
\emph{locally} at each linear subgame's in-round Gibbs readout,
\begin{align}\label{eq:routing-gibbs-general-tau}
    \pi_{j,\kappa}^{(\tau)\star}(i) \propto \exp\!\bigl(Q_q(s_{j,q,\kappa}^{(l, lin)}, i)/\tau\bigr),
\end{align}
while keeping the game values $V(s_{i,q}^{(l-1, act)})$ at their $\tau{=}1$ equilibrium---i.e.\ at
the network activations recovered by Theorem~\ref{thm:routing-forward}(ii). The players are
therefore \emph{unaware} of the $\tau$-deformation (a \emph{myopic} policy): their decisions rest
on the Bellman equations at temperature 1 (the network forward pass), and only the local routing
readout is perturbed. Under non-negative predecessor activations this gives
$\pi_{j,\kappa}^{(\tau)\star}(i) \propto (a_i^{(l-1)} W_{ij}^{(l,\kappa)})^{1/\tau}$, so $\tau<1$
sharpens and $\tau>1$ flattens the route distribution. This controlled deformation is used in the
experiments; it carries no exactness claim at $\tau \neq 1$. The implementation raises activations,
weights and bias-node masses to the power $1/\tau$ but keeps the outside-option weight at
$\varepsilon$ rather than $\varepsilon^{1/\tau}$. On ResNet-50 it passes relevance through the
global average pool uniformly (not in proportion to the activations) and merges batch normalization
into the preceding convolution with an $\varepsilon$-stabilized rule, so the ResNet-50 anchor is
$\alpha\beta$-\ac{LRP} up to these two operations; VGG-16 (whose average pool is the identity at
$224{\times}224$) and ViT-B/16 are not affected.
\newpage
\section{Extension: Vision Transformer Attention}\label{app:vit-extension}\label{app:attention-subgame}

This appendix develops the full state-space formulation for attention modules in the \ac{RG}. The
development proceeds in two stages. \S\ref{app:attention-kl-routing} extends the backward walk of
Theorem~\ref{thm:ab-recovery} through the attention block by treating the softmax row $A_{t,\cdot}$
as a fixed reference policy on the keys, derives the resulting Value-routing equilibrium from the
Legendre--Fenchel transform of the KL divergence, and proves $\alpha\beta$-\ac{LRP} recovery for
the conditioned effective map $X \mapsto O$ in the occupation-measure notation $\Gamma$ of
Theorem~\ref{thm:ab-recovery}. \S\ref{app:attention-oracle} then identifies the fixed reference
$A_{t,\cdot}$ with the policy of an auxiliary entropy-regularized QK oracle player, closing the
loop with the rest of the game-theoretic framework. The risk-averse attention variant is in
\S\ref{app:risk-attn}; forward-pass detachments and gradient-hook conventions used by all ViT
experiments are collected in \S\ref{app:vit-implementation}; the equivariant spatial averaging is
in \S\ref{app:equivariant}.

\subsection{Attention as fixed-reference KL Value routing}\label{app:attention-kl-routing}

\subsubsection{Target rule: \texorpdfstring{$\alpha\beta$-\ac{LRP}}{LRP-alpha-beta} with detached logits}\label{app:attention-target-rule}

We first state the rule that the Routing-Game state-space construction below is designed to
recover. We focus on a single attention head; the multi-head extension carries an additional head
index $h$ and the recovery is per-head. Note that ViT-B/16 implements the Value projection as a
single $D \times D$ linear $v = X W_v$ followed by a reshape into $n_h$ heads (the per-head matrix
$W_v^{(h)} \in \RR^{D \times d_h}$ is the head-$h$ column block of the full $W_v$), so one input
embedding $X_{k,e}$ feeds Value entries of every head and the per-head $R^X$ contributions sum at
$X_{k,e}$ to give the total. The per-head description below applies unchanged. We omit the
$\varepsilon$ stabilizer inside the attention block --- the surrounding ReLU recursion of
\S\ref{app:routing-attribution} carries its own $\varepsilon$ on every linear-routing state and
that mechanism transfers to the $Av{+}W_v$ composite without modification, but tracking it through
the tensor indexing here only obscures the structure.

\paragraph{Forward pass.} The Value projection is applied per token: for each $k \in \{1,\dots,S\}$,
\begin{align}\label{eq:attn-V-forward}
    v_{k,d} \;=\; \sum_e W_{v,e,d}\, X_{k,e} + b_{v,d},
    \qquad d \in \{1,\dots,d_h\},
\end{align}
where the same matrix $W_v \in \RR^{D \times d_h}$ and bias $b_v \in \RR^{d_h}$ are applied
independently to every token. The detached attention matrix $A \in \RR^{S \times S}_{\ge 0}$ (with
$A_{t,k} \ge 0$ and $\sum_k A_{t,k} = 1$ from softmax) then mixes Values across tokens
\emph{independently per feature dimension} $d$: for each output token $t \in \{1,\dots,S\}$ and
feature $d$,
\begin{align}\label{eq:attn-AV-forward}
    O_{t,d} \;=\; \sum_k A_{t,k}\, v_{k,d}
    \;=\; \sum_{k,e} A_{t,k}\, W_{v,e,d}\, X_{k,e} + b_{v,d}.
\end{align}

\paragraph{Composite linear map and $\kappa$-stream split.} Holding $A$, the LayerNorm statistics and the \acs{GELU} gates at their forward values
(\S\ref{app:vit-implementation}) is a modeling assumption about the explanation question---given
the realized attention policy, through which Value routes does relevance flow?---and not a theorem
about relevance through $Q$ and $K$. \ch{Our attention extension is therefore a Value-only attribution
conditioned on the realized attention matrix, resembling the conservative propagation of
\citet{Ali2022} (CP-LRP), which holds the attention matrix constant, and DAVE's conditioned
operators~\citep{Wrobel2024}. It differs from AttnLRP~\citep{achtibat2024attnlrp}, which also
propagates relevance through the attention-generating operations (the softmax and the bilinear QK
product) and so attributes the Query and Key paths.} With $A$ so conditioned, the attention block reduces to the
affine \emph{conditioned effective map} $X \mapsto O$ with composite weights
$W^{\mathrm{eff}}_{(k,e),\,(t,d)} = A_{t,k}\, W_{v,e,d}$ and bias $b_{v,d}$. Since $A_{t,k} \ge 0$,
the sign of the composite weight is the sign of $W_{v,e,d}$:
\begin{align}\label{eq:attn-composite-weight}
    \bigl[\,W^{\mathrm{eff}}_{(k,e),\,(t,d)}\,\bigr]^\kappa
    \;=\; A_{t,k}\, W_{v,e,d}^\kappa,
    \qquad \kappa \in \{+,-\}.
\end{align}
For signed inputs $X$ (LayerNorm outputs), the per-token $\kappa$-stream split of the Value tensor
(as in the fused input layer of the \ac{RG}, with the bias as a bias node) and the
per-output-feature stream sum are
\begin{align}\label{eq:attn-Z-def}
    \tilde{v}_{k,d}^\kappa \;\coloneqq\; \sum_e [X_{k,e} W_{v,e,d}]^\kappa + [b_{v,d}]^\kappa,
    \qquad
    Z_{t,d}^\kappa \;\coloneqq\; \sum_{k} A_{t,k}\, \tilde{v}_{k,d}^\kappa
    \;=\; \sum_{k,e} A_{t,k}\, [X_{k,e} W_{v,e,d}]^\kappa + [b_{v,d}]^\kappa.
\end{align}
Note that $\tilde{v}_{k,d}^\kappa$ is the $\kappa$-stream split inherited from $W_v$ and the sign
of $X$ (not the standard scalar positive part $[v_{k,d}]^+ = \max(0, v_{k,d})$); the two coincide
only when $v_{k,d}$ has a single-signed predecessor decomposition.

\paragraph{Target rule.} Let $R^O_{t,d}$ denote the relevance arriving at the output entry $O_{t,d}$ from the surrounding
LRP recursion. We define the target rule as the $\alpha\beta$-\ac{LRP} \ch{rule~\eqref{eq:lrp-ab}} on the
composite map $X \mapsto O$, written in two stages that carry the $\kappa$-stream split $\tilde
    v_{k,d}^\kappa$ inherited from $W_v$ via~\eqref{eq:attn-composite-weight} through the Value
entries. \ch{It is not sequential (layer-wise) $\alpha\beta$-\ac{LRP} on $X \to v \to O$, which splits
by the sign of the netted Value entry $v_{k,d}$ and applies $(\alpha,\beta)$ at both stages; the two
rules differ whenever a Value entry has mixed-sign predecessors. The sequential variant is the
alternative that this paper does not use: the implementation (Appendix~\ref{app:vit-implementation})
and every ViT result reported here apply the composite rule stated below, and the numerical
identity between the implementation and~\eqref{eq:attn-target-input} is checked by the tests of
Appendix~\ref{app:code-tests}.}

\emph{AV mixing $O \to v$.} Applied to $O_{t,d} = \sum_k A_{t,k}\,v_{k,d}$, the rule gives the
relevance at the Value entry $v_{k,d}$ as the weighted difference of two per-stream contributions:
\begin{align}\label{eq:attn-target-rule}
    R^{\mathrm{val}}_{k,d} \;=\; \alpha\,R^{\mathrm{val},+}_{k,d} \;-\;
    \beta\,R^{\mathrm{val},-}_{k,d},
    \qquad
    R^{\mathrm{val},\kappa}_{k,d} \;\coloneqq\; \sum_t \frac{A_{t,k}\,\tilde
    v_{k,d}^\kappa}{Z_{t,d}^\kappa}\,R^O_{t,d},
    \quad \kappa \in \{+,-\}.
\end{align}

\emph{Value projection $v \to X$.} Applied to $v_{k,d} = \sum_e W_{v,e,d}\,X_{k,e} + b_{v,d}$, the
rule routes each $\kappa$-stream contribution $R^{\mathrm{val},\kappa}_{k,d}$ to the input dim $e$
via the matching $\kappa$-stream ratio $[X_{k,e}W_{v,e,d}]^\kappa/\tilde v_{k,d}^\kappa$ (the bias
share is absorbed):
\begin{align}\label{eq:attn-target-input-via-V}
    R^X_{k,e} \;=\; \sum_d \left(
    \alpha\, \frac{[X_{k,e} W_{v,e,d}]^+}{\tilde v_{k,d}^+}\,R^{\mathrm{val},+}_{k,d}
    \;-\;
    \beta\, \frac{[X_{k,e} W_{v,e,d}]^-}{\tilde v_{k,d}^-}\,R^{\mathrm{val},-}_{k,d}
    \right).
\end{align}

\emph{Aggregate.} Substituting~\eqref{eq:attn-target-rule} into~\eqref{eq:attn-target-input-via-V}
and canceling $\tilde v_{k,d}^\kappa$ in numerator and denominator expresses $R^X_{k,e}$ directly
in terms of $R^O_{t,d}$:
\begin{align}\label{eq:attn-target-input}
    R^X_{k,e}
    \;=\; \sum_{t,d} \left(
    \alpha\, \frac{A_{t,k}\, [X_{k,e} W_{v,e,d}]^+}{Z_{t,d}^+}
    \;-\;
    \beta\, \frac{A_{t,k}\, [X_{k,e} W_{v,e,d}]^-}{Z_{t,d}^-}
    \right) R^O_{t,d},
\end{align}
which coincides with what the standard $\alpha\beta$-\ac{LRP} rule produces directly on the
conditioned effective map $X \mapsto O$ with effective weight $W^{\mathrm{eff}}_{(k,e),(t,d)} =
    A_{t,k}\,W_{v,e,d}$ and bias $b_{v,d}$.

\subsubsection{State space}\label{app:attention-state-space}

We embed the attention block into the routing-graph state space of \S\ref{app:routing-formal} using
the tuple-superscript notation $s_\bullet^{(\mathrm{att},\,\mathrm{Side},\,\mathrm{type})}$
throughout, where $\mathrm{Side} \in \{\mathrm{O}, \mathrm{val}, \mathrm{X}\}$ identifies the
activation slot inside the attention block and $\mathrm{type} \in \{\mathrm{act}, \mathrm{lin}\}$
matches state types of \S\ref{app:routing-formal}. The state indices encode the tensor structure of
the AV mixing and the Value projection, with $t,k \in \{1,\dots,S\}$ ranging over output/key tokens
(sequence axis), $d \in \{1,\dots,d_h\}$ over the within-head feature axis of the focused head
(i.e.\ the head index $h$ is fixed and $d$ runs over that head's $d_h$ feature dimensions), and $e
    \in \{1,\dots,D\}$ over the per-token input embedding dimensions.

\paragraph{States for the composite $Av{+}W_v$ block.} The explanatory walk through the attention block visits three state types beyond the
surrounding ReLU-network's; all are parametrized by the player label in turn.
\begin{itemize}
    \item \textbf{Output activation states} $s_{t,d,p}^{(\mathrm{att},\,\mathrm{O},\,\mathrm{act})}$ at the
        head's output entry $O_{t,d}$, indexed by output token $t \in \{1,\dots,S\}$, within-head feature
        dim $d \in \{1,\dots,d_h\}$, and player $p \in \{+,-\}$. These are the entry points of the
        explanatory walk: the surrounding linear-layer recursion of \S\ref{app:routing-attribution}
        delivers occupation mass to $\{s_{t,d,p}^{(\mathrm{att},\,\mathrm{O},\,\mathrm{act})}\}_{t,d,p}$
        from the layer behind the head.
    \item \textbf{Output sign-branch (linear) states}
        $s_{(t,d),p,\kappa}^{(\mathrm{att},\,\mathrm{O},\,\mathrm{lin})}$, indexed by $(t, d, p, \kappa)$.
        These mirror the linear routing states $s_{j,q,\kappa}^{(l,\mathrm{lin})}$ of
        Definition~\ref{def:routing-recursion} for the AV mixing at output entry $(t, d)$: \emph{this is
        where the QK logits enter the routing} via the reference $\bar\pi_t = A_{t,\cdot}$. The active
        player picks a key token $k$ by a mixed action $\pi_{t,d,\kappa}^\star \in \Delta(\{1,\dots,S\})$
        that combines the AV reference $\bar\pi_t$ with the $\kappa$-stream Value mass
        $\tilde{v}_{k,d}^\kappa$.
    \item \textbf{Value-projection linear states}
        $s_{(k,d),p,\kappa}^{(\mathrm{att},\,\mathrm{val},\,\mathrm{lin})}$, indexed by key token $k$,
        within-head feature dim $d$, player $p$, and sign $\kappa$. These are the linear-routing states of
        the per-token Value projection: the player picks an input dimension $e$ (or the bias node) and the
        game value satisfies
        $V_p\!\bigl(s_{(k,d),p,\kappa}^{(\mathrm{att},\,\mathrm{val},\,\mathrm{lin})}\bigr) =
            \ell(\tilde{v}_{k,d}^\kappa)$ with $\tilde{v}_{k,d}^\kappa$ from~\eqref{eq:attn-Z-def}.
    \item \textbf{Input activation states} $s_{(k,e),p}^{(\mathrm{att},\,\mathrm{X},\,\mathrm{act})}$ at the
        input embedding entry $X_{k,e}$, indexed by token--embedding-dim pair $(k,e)$ and player $p$. $X$
        lives upstream of the Value projection and carries the input-embedding axis $e$; the within-head
        feature axis $d$ enters first at the Value projection's output downstream. They are the exit points
        of the attention block: the surrounding network's recursion resumes from the per-player occupation
        on $\{s_{(k,e),p}^{(\mathrm{att},\,\mathrm{X},\,\mathrm{act})}\}_{(k,e),p}$.
\end{itemize}

\paragraph{Transitions.} The trajectory through the attention block has three steps; all transitions inside the block carry
trajectory discount $1$ except the sign-oracle split, which uses the same $\{2\alpha, 2\beta\}$ as
a ReLU activation state of Definition~\ref{def:routing-recursion}. There is only \emph{one}
sign-oracle split per output entry $(t,d)$, since $A_{t,k} \ge 0$ absorbs the AV mixing into the
$W_v$ sign split.
\begin{enumerate}
    \item \emph{Sign-oracle split (output activation $\to$ output sign-branch, fixed $(t,d)$).} At
        $s_{t,d,p}^{(\mathrm{att},\,\mathrm{O},\,\mathrm{act})}$ an unobserved Oracle transitions uniformly
        to $s_{(t,d),p,+}^{(\mathrm{att},\,\mathrm{O},\,\mathrm{lin})}$ (player $p$ retains the turn,
        trajectory discount $2\alpha$) or to $s_{(t,d),p',-}^{(\mathrm{att},\,\mathrm{O},\,\mathrm{lin})}$
        (turn switches to opponent $p'$, discount $2\beta$), each with probability $\tfrac12$. The feature
        index $d$ is preserved.
    \item \emph{Value-routing policy (output sign-branch $\to$ Value-projection linear).} At
        $s_{(t,d),p,\kappa}^{(\mathrm{att},\,\mathrm{O},\,\mathrm{lin})}$ the active player picks a key
        token $k$ by the mixed action
        \begin{align}\label{eq:attn-V-routing-policy-O}
            \pi_{t,d,\kappa}^\star(k)
            \;=\; \frac{A_{t,k}\, \tilde{v}_{k,d}^\kappa}{Z_{t,d}^\kappa},
        \end{align}
        derived in \S\ref{app:attention-value-routing} as the equilibrium of a KL-regularized log-payoff
        problem against the reference $\bar\pi_t = A_{t,\cdot}$. The trajectory transitions to
        $s_{(k,d),p,\kappa}^{(\mathrm{att},\,\mathrm{val},\,\mathrm{lin})}$ \emph{with the player label,
        the sign $\kappa$, and the within-head feature dim $d$ preserved}; the only thing changing is the
        sequence-axis index from output token $t$ to key token $k$. The trajectory-discount on this
        transition is $1$.
    \item \emph{Value-projection routing (Value-projection linear $\to$ input activation).} At
        $s_{(k,d),p,\kappa}^{(\mathrm{att},\,\mathrm{val},\,\mathrm{lin})}$ the active player picks an
        input dim $e$ by the standard linear-state Gibbs policy of Definition~\ref{def:routing-recursion}
        on the $\kappa$-stream masses $[X_{k,e}W_{v,e,d}]^\kappa$ (the bias node receives
        $[b_{v,d}]^\kappa/\tilde{v}_{k,d}^\kappa$ and is absorbed):
        \begin{align}\label{eq:attn-V-routing-policy}
            \pi_{\mathrm{val},k,d,\kappa}^\star(e)
            \;=\; \frac{[X_{k,e} W_{v,e,d}]^\kappa}{\tilde{v}_{k,d}^\kappa}.
        \end{align}
        The trajectory transitions to $s_{(k,e),p}^{(\mathrm{att},\,\mathrm{X},\,\mathrm{act})}$ with the
        player label preserved and trajectory-discount factor $1$.
\end{enumerate}
The discounted occupation measure $\Gamma_x(\,\cdot\,)$ of \S\ref{app:routing-attribution} is
evaluated at all four state types along the explanatory walk. The recovery theorem of
\S\ref{app:attention-recovery} then compares the per-input-neuron player-difference of $\Gamma_x$
at the input activation states with the target rule~\eqref{eq:attn-target-input}.

\subsubsection{Legendre--Fenchel transform of the KL divergence}\label{app:kl-legendre-fenchel}

The Value-routing equilibrium below is a direct corollary of the convex-conjugate identity for the
KL divergence; we record it as a stand-alone Lemma to make the inheritance explicit.

\begin{lemma}[Legendre--Fenchel transform of KL divergence]\label{lem:kl-legendre-fenchel}
Let $\mathcal{X}$ be a finite set, $\bar\pi \in \Delta(\mathcal{X})$ a reference distribution with
full support on $\mathcal{X}$, and $r : \mathcal{X} \to \RR$ a real-valued payoff. The
KL-regularized linear maximization
\begin{align}\label{eq:kl-lf-objective}
    F(r;\bar\pi) \;\coloneqq\; \max_{\pi \in \Delta(\mathcal{X})}
    \Bigl\{ \E_\pi[r] \;-\; \mathrm{KL}(\pi \,\|\, \bar\pi) \Bigr\}
\end{align}
has the following optimizer and optimum:
\begin{align}\label{eq:kl-lf-optimizer}
    \pi_r^\star(x) \;=\; \frac{\bar\pi(x)\,\exp\!\bigl(r(x)\bigr)}{Z(r;\bar\pi)},
    \qquad
    F(r;\bar\pi) \;=\; \log Z(r;\bar\pi) \;=\; \log\!\sum_{x \in \mathcal{X}}
    \bar\pi(x)\,\exp\!\bigl(r(x)\bigr).
\end{align}
\end{lemma}

Two specializations are central below. When $r(x) = \ell(u(x))$ for some non-negative $u :
    \mathcal{X} \to \RR_{\geq 0}$ (so $\exp(r(x)) = u(x)$ with the convention $\exp(-\infty) \coloneqq
    0$), the value collapses to the $\bar\pi$-weighted linear sum and the optimizer becomes the
$\bar\pi$-tilted normalization of $u$:
\begin{align}\label{eq:kl-lf-linearized}
    F\bigl(\ell(u);\,\bar\pi\bigr) \;=\; \ell\!\Bigl(\sum_{x}\bar\pi(x)\,u(x)\Bigr),
    \qquad
    \pi_{\ell(u)}^\star(x) \;=\; \frac{\bar\pi(x)\,u(x)}{\sum_{x'}\bar\pi(x')\,u(x')}.
\end{align}
When $\bar\pi$ is uniform on $\mathcal{X}$, $\mathrm{KL}(\pi\,\|\,\mathrm{Unif}_\mathcal{X}) =
    \log|\mathcal{X}| - \mathcal{H}(\pi)$, so subtracting the KL penalty in~\eqref{eq:kl-lf-objective}
flips into a $+\mathcal{H}$ entropy bonus:
\begin{align}
    \E_\pi[r] \;-\; \mathrm{KL}\!\bigl(\pi\,\|\,\mathrm{Unif}_\mathcal{X}\bigr)
    \;=\; \E_\pi[r] \;+\; \mathcal{H}(\pi) \;-\; \log|\mathcal{X}|.
\end{align}
Up to the additive constant $-\log|\mathcal{X}|$,~\eqref{eq:kl-lf-objective} therefore reduces to
the entropy-regularized maximization of~\eqref{eq:gibbs-lse}, with value $\log\sum_{x}\exp(r(x))$
and optimizer $\pi^\star(x) \propto \exp(r(x))$.

\begin{proof}
This is the Gibbs variational principle~\citep{BoydVandenberghe2004}: for every $\pi$, $\E_\pi[r] -
    \mathrm{KL}(\pi\,\|\,\bar\pi) = \log Z(r;\bar\pi) - \mathrm{KL}(\pi\,\|\,\pi_r^\star)$, which is
maximal, with value $\log Z(r;\bar\pi)$, exactly at $\pi = \pi_r^\star$.
\end{proof}

\subsubsection{Value-routing equilibrium}\label{app:attention-value-routing}

Fix an output entry $(t, d)$, a sign $\kappa \in \{+,-\}$, and the reference $\bar\pi_t =
    A_{t,\cdot}$ from \S\ref{app:attention-state-space} (the \emph{same} $\bar\pi_t$ for every $d$,
since the AV mixing reuses one row across all feature dims). At the output sign-branch state
$s_{(t,d),p,\kappa}^{(\mathrm{att},\,\mathrm{O},\,\mathrm{lin})}$ the active player $p$ solves the
KL-regularized log-payoff maximization
\begin{align}\label{eq:attn-value-routing}
    V_p\!\bigl(s_{(t,d),p,\kappa}^{(\mathrm{att},\,\mathrm{O},\,\mathrm{lin})}\bigr)
    \;\coloneqq\;
    \max_{\pi \in \Delta(\{1,\dots,S\})}
    \left\{
    \sum_{k=1}^S \pi(k)\, \ell\!\bigl(\tilde{v}_{k,d}^\kappa\bigr)
    \;-\;
    \mathrm{KL}\!\bigl(\pi \,\|\, \bar\pi_t\bigr)
    \right\},
\end{align}
where $\tilde{v}_{k,d}^\kappa \ge 0$ is the $\kappa$-stream split of the Value entry $v_{k,d}$
from~\eqref{eq:attn-Z-def} and $\ell$ is the extended logarithm of \S\ref{app:routing-formal}. The
payoff $\ell(\tilde{v}_{k,d}^\kappa)$ on action $k$ is the active-player value of the predecessor
Value-projection linear state,
$V_p(s_{(k,d),p,\kappa}^{(\mathrm{att},\,\mathrm{val},\,\mathrm{lin})}) =
    \ell(\tilde{v}_{k,d}^\kappa)$, so combining with the AV log-weight $\ell(A_{t,k})$ gives composite
payoff $\ell(A_{t,k}\, \tilde{v}_{k,d}^\kappa)$ at the $(t,d)$-output. Define the per-output-entry
$\bar\pi_t$-weighted Value mass
\begin{align}\label{eq:attn-z-def}
    Z_{t,d}^\kappa \;\coloneqq\; \sum_{m=1}^S \bar\pi_t(m)\, \tilde{v}_{m,d}^\kappa.
\end{align}
The dependence on $d$ enters only through $\tilde{v}_{m,d}^\kappa$; the AV reference $\bar\pi_t$ is
shared across all feature dims.

\begin{proposition}[Value-side optimal policy]\label{prop:attn-value-policy}
If $Z_{t,d}^\kappa > 0$, the unique optimizer of~\eqref{eq:attn-value-routing} is
\begin{align}\label{eq:attn-value-policy}
    \pi_{t,d,\kappa}^\star(k)
    \;=\;
    \frac{\bar\pi_t(k)\, \tilde{v}_{k,d}^\kappa}{Z_{t,d}^\kappa},
\end{align}
with active-player value $V_p(s_{(t,d),p,\kappa}^{(\mathrm{att},\,\mathrm{O},\,\mathrm{lin})}) =
    \ell(Z_{t,d}^\kappa)$ and inactive-player value
$V_{p'}(s_{(t,d),p,\kappa}^{(\mathrm{att},\,\mathrm{O},\,\mathrm{lin})}) = -\infty$ as in
Definition~\ref{def:routing-recursion}. If $Z_{t,d}^\kappa = 0$, no key contributes mass on stream
$\kappa$ and $V_p(s_{(t,d),p,\kappa}^{(\mathrm{att},\,\mathrm{O},\,\mathrm{lin})}) = -\infty$.
\end{proposition}

\begin{proof}
Apply the logit specialization~\eqref{eq:kl-lf-linearized} of Lemma~\ref{lem:kl-legendre-fenchel}
with $\mathcal{X} = \{1,\dots,S\}$, $\bar\pi = \bar\pi_t$, and $u(k) = \tilde{v}_{k,d}^\kappa$, so
that the payoff $r(k) = \ell(\tilde{v}_{k,d}^\kappa)$ has $\exp r(k) = \tilde{v}_{k,d}^\kappa$.
Equation~\eqref{eq:kl-lf-linearized} immediately gives the optimizer $\pi_{t,d,\kappa}^\star$
of~\eqref{eq:attn-value-policy} and value $\ell(Z_{t,d}^\kappa)$ when $Z_{t,d}^\kappa > 0$. The
inactive-player value $-\infty$ is the Routing-Game convention of
Definition~\ref{def:routing-recursion}, inherited from the $r = q'$ branch
of~\eqref{eq:routing-linear-action}. If $Z_{t,d}^\kappa = 0$ then
$\bar\pi_t(k)\,\tilde{v}_{k,d}^\kappa = 0$ for every $k$ and the supremum
of~\eqref{eq:attn-value-routing} over the simplex is $-\infty$, consistent with $\ell(0) =
    -\infty$.
\end{proof}

\paragraph{Specialization to the softmax reference.}
Setting $\bar\pi_t = A_{t,\cdot}$ in~\eqref{eq:attn-value-policy} yields the closed-form
per-$(t,d,\kappa)$ policies used in the \ac{RG}'s backward walk through the attention block,
\begin{align}\label{eq:attn-branch-policy}
    \pi_{t,d,+}^\star(k) \;=\; \frac{A_{t,k}\, \tilde{v}_{k,d}^+}{Z_{t,d}^+},
    \qquad
    \pi_{t,d,-}^\star(k) \;=\; \frac{A_{t,k}\, \tilde{v}_{k,d}^-}{Z_{t,d}^-},
\end{align}
with $Z_{t,d}^\kappa = \sum_m A_{t,m}\, \tilde{v}_{m,d}^\kappa$ from~\eqref{eq:attn-z-def}. The
numerical specification $\bar\pi_t = A_{t,\cdot}$ does not require the oracle interpretation of
\S\ref{app:attention-oracle}; it is enough that the cached softmax row is treated as a fixed input
to~\eqref{eq:attn-value-routing}.

\subsubsection{\texorpdfstring{$\alpha\beta$-\ac{LRP}}{alpha-beta-LRP} recovery for the attention block}\label{app:attention-recovery}

We now show that, under the trajectory through the four state types of
\S\ref{app:attention-state-space} and the trajectory discount of~\eqref{eq:routing-edge-discount},
the player-difference of $\Gamma_x$ at the \emph{input activation states} (the predecessors of the
Value-projection linear states, where the $W_v$ weights actually fire) recovers exactly the target
rule~\eqref{eq:attn-target-input} of \S\ref{app:attention-target-rule}. We keep the game side and
the attribution side strictly separate:
\begin{itemize}
    \item the \emph{game-side} object is the discounted occupation measure $\Gamma_x(\,\cdot\,)$ of
        \S\ref{app:routing-attribution} evaluated at the four attention-block state types of
        \S\ref{app:attention-state-space}. Following the per-neuron difference
        convention~\eqref{eq:routing-Gamma-pair} of the surrounding Routing-Game recursion, we form the
        per-output-entry difference $\Gamma^{O}_{t,d}$ at the entry point (defined
        in~\eqref{eq:attn-Gamma-O} below) and the per-input-neuron difference at the exit
        \begin{align}\label{eq:attn-Gamma-X}
            \Gamma^{X}_{k,e} \;\coloneqq\;
            \Gamma_x\!\bigl(s_{(k,e),+}^{(\mathrm{att},\,\mathrm{X},\,\mathrm{act})}\bigr) -
            \Gamma_x\!\bigl(s_{(k,e),-}^{(\mathrm{att},\,\mathrm{X},\,\mathrm{act})}\bigr).
        \end{align}
    \item the \emph{attribution-side} object is the standard $\alpha\beta$-\ac{LRP} relevance vector
        $R^{X}_{k,e}$ given by the input-level target rule~\eqref{eq:attn-target-input}, seeded by
        $R^{O}_{t,d}$ from the surrounding linear-layer recursion above the head.
\end{itemize}

\begin{theorem}[$\alpha\beta$-\ac{LRP} recovery for attention]\label{thm:attn-ab-recovery}
For signed inputs $X$, under the SPNE of \S\ref{app:attention-state-space} with reference
$\bar\pi_t = A_{t,\cdot}$ and the trajectory discount of~\eqref{eq:routing-edge-discount}
restricted to the (sole) sign-oracle split at $s^{(\mathrm{att},\,\mathrm{O},\,\mathrm{act})}$, the
per-input-neuron player-difference $\Gamma^{X}_{k,e}$ from~\eqref{eq:attn-Gamma-X} satisfies
\begin{align}\label{eq:attn-ab-recovery}
    \Gamma^{X}_{k,e}
    \;=\;
    \sum_{t,d} \left(
    \alpha\, \frac{A_{t,k}\, [X_{k,e} W_{v,e,d}]^+}{Z_{t,d}^+}
    \;-\;
    \beta\, \frac{A_{t,k}\, [X_{k,e} W_{v,e,d}]^-}{Z_{t,d}^-}
    \right) \Gamma^{O}_{t,d},
\end{align}
where the per-output-entry input difference is
\begin{align}\label{eq:attn-Gamma-O}
    \Gamma^{O}_{t,d} \;\coloneqq\;
    \Gamma_x\!\bigl(s_{t,d,+}^{(\mathrm{att},\,\mathrm{O},\,\mathrm{act})}\bigr) \;-\;
    \Gamma_x\!\bigl(s_{t,d,-}^{(\mathrm{att},\,\mathrm{O},\,\mathrm{act})}\bigr),
\end{align}
with $\Gamma^{O}_{t,d} = R^{O}_{t,d}$ by inductive hypothesis from the surrounding recursion, for
the seed $R^{(L)}_u = 1$ (for the seed $f(x)$ every relevance equals $f(x)$ times the corresponding
$\Gamma$, as in Theorem~\ref{thm:ab-recovery}). The right-hand side is exactly the target
rule~\eqref{eq:attn-target-input}, hence $\Gamma^{X}_{k,e} = R^{X}_{k,e}$. Aggregating over $e$ and
the bias node at fixed $(k,d)$ collects the relevance routed through the Value entry $v_{k,d}$:
\begin{align}\label{eq:attn-Gamma-V-corollary}
    \Gamma^{\mathrm{val}}_{k,d}
    \;=\;
    \sum_t \left( \alpha\, \frac{A_{t,k}\, \tilde{v}_{k,d}^+}{Z_{t,d}^+} \;-\; \beta\, \frac{A_{t,k}\,
    \tilde{v}_{k,d}^-}{Z_{t,d}^-} \right) \Gamma^{O}_{t,d}
    \;=\; R^{\mathrm{val}}_{k,d},
\end{align}
recovering the per-Value-entry formula~\eqref{eq:attn-target-rule}.
\end{theorem}

\begin{proof}
The statement is block-local: by hypothesis the surrounding recursion delivers the output-side
per-entry differences $\Gamma^{O}_{t,d} = R^{O}_{t,d}$ at every output position $(t,d)$.

\textbf{Per-state recursion at the input activation states.} Fix an input embedding $(k,e)$ and a
player label $r \in \{+,-\}$. A trajectory reaches
$s_{(k,e),r}^{(\mathrm{att},\,\mathrm{X},\,\mathrm{act})}$ from an output activation
$s_{t,d,p}^{(\mathrm{att},\,\mathrm{O},\,\mathrm{act})}$ in three stages, with all transitions
inside the attention block carrying their discounts and player-label switches as in the kernel of
\S\ref{app:attention-state-space}. The path goes through one specific within-head feature dim $d
    \in \{1,\dots,d_h\}$ (an index along the focused head's $d_h$-wide feature axis): the sign-oracle
is per $(t,d)$, the Value-routing policy $\pi_{t,d,\kappa}^\star$ depends on $d$ via
$\tilde{v}_{k,d}^\kappa$, and the Value-projection routing $\pi_{\mathrm{val},k,d,\kappa}^\star$
also depends on $d$. The $d$ index is dropped at the final transition into $X$, where the input
embedding $X_{k,e}$ has indices $(k,e)$ and the within-head feature axis $d$ enters downstream at
the Value projection's output.
\begin{itemize}
    \item \emph{$\kappa = +$ branch (per $(t,d)$).} Sign-oracle takes $s_{t,d,p}^{(\mathrm{O},\mathrm{act})}
            \to s_{(t,d),p,+}^{(\mathrm{O},\mathrm{lin})}$ with probability $\tfrac12$ and trajectory discount
        $2\alpha$ (player retained). The Value-routing policy $\pi_{t,d,+}^\star(k)$ takes the trajectory
        to $s_{(k,d),p,+}^{(\mathrm{val},\mathrm{lin})}$ (player, sign, $d$ preserved, discount $1$). The
        Value-projection routing $\pi_{\mathrm{val},k,d,+}^\star(e)$ takes the trajectory to
        $s_{(k,e),p}^{(\mathrm{X},\mathrm{act})}$ (player preserved, $d$ dropped, discount $1$). The
        composite per-$(t,d)$ contribution to $\Gamma_x(s_{(k,e),r}^{(\mathrm{X},\mathrm{act})})$ is
        \begin{align}
            \tfrac12 (2\alpha)\, \pi_{t,d,+}^\star(k)\, \pi_{\mathrm{val},k,d,+}^\star(e)\,
            \Gamma_x\!\bigl(s_{t,d,r}^{(\mathrm{O},\mathrm{act})}\bigr)
            &=\; \alpha\, \frac{A_{t,k}\, \tilde{v}_{k,d}^+}{Z_{t,d}^+}\,\frac{[X_{k,e}
            W_{v,e,d}]^+}{\tilde{v}_{k,d}^+}\, \Gamma_x\!\bigl(s_{t,d,r}^{(\mathrm{O},\mathrm{act})}\bigr) \\
            &=\; \alpha\, \frac{A_{t,k}\, [X_{k,e} W_{v,e,d}]^+}{Z_{t,d}^+}\,
            \Gamma_x\!\bigl(s_{t,d,r}^{(\mathrm{O},\mathrm{act})}\bigr),
        \end{align}
        where the $\tilde{v}_{k,d}^+$ in numerator and denominator cancel.
    \item \emph{$\kappa = -$ branch (per $(t,d)$).} Sign-oracle takes $s_{t,d,p}^{(\mathrm{O},\mathrm{act})}
            \to s_{(t,d),p',-}^{(\mathrm{O},\mathrm{lin})}$ with probability $\tfrac12$, discount $2\beta$,
        turn switched to opponent $p'$; the rest of the path preserves $p'$, ending at
        $s_{(k,e),p'}^{(\mathrm{X},\mathrm{act})}$. The contribution to
        $\Gamma_x(s_{(k,e),r}^{(\mathrm{X},\mathrm{act})})$ comes from sources with player $p = r'$ (so the
        switch yields $r$):
        \begin{align}
            \tfrac12 (2\beta)\, \pi_{t,d,-}^\star(k)\, \pi_{\mathrm{val},k,d,-}^\star(e)\,
            \Gamma_x\!\bigl(s_{t,d,r'}^{(\mathrm{O},\mathrm{act})}\bigr)
            \;=\; \beta\, \frac{A_{t,k}\, [X_{k,e} W_{v,e,d}]^-}{Z_{t,d}^-}\,
            \Gamma_x\!\bigl(s_{t,d,r'}^{(\mathrm{O},\mathrm{act})}\bigr).
        \end{align}
\end{itemize}
Summing over $(t,d)$ and over the two branches,
\begin{align}\label{eq:attn-Gamma-X-recursion}
    \Gamma_x\!\bigl(s_{(k,e),r}^{(\mathrm{att},\,\mathrm{X},\,\mathrm{act})}\bigr)
    \;=\; \sum_{t,d} \Bigl[
    \alpha\, \frac{A_{t,k}\, [X_{k,e} W_{v,e,d}]^+}{Z_{t,d}^+}\,
    \Gamma_x\!\bigl(s_{t,d,r}^{(\mathrm{O},\mathrm{act})}\bigr)
    \;+\;
    \beta\, \frac{A_{t,k}\, [X_{k,e} W_{v,e,d}]^-}{Z_{t,d}^-}\,
    \Gamma_x\!\bigl(s_{t,d,r'}^{(\mathrm{O},\mathrm{act})}\bigr)
    \Bigr],
\end{align}
where the $\tfrac12$ from the sign-oracle splits cancels against the $2$ in $2\alpha, 2\beta$.

\textbf{Player difference.} Differencing~\eqref{eq:attn-Gamma-X-recursion} for $r{=}+$ minus
$r{=}-$ uses $\Gamma_x(s_{t,d,+}^{(\mathrm{O},\mathrm{act})}) -
    \Gamma_x(s_{t,d,-}^{(\mathrm{O},\mathrm{act})}) = \Gamma^{O}_{t,d}$ in the $+$ branch (player kept)
and the opposite $-\Gamma^{O}_{t,d}$ in the $-$ branch (player switched):
\begin{align}
    \Gamma^{X}_{k,e}
    \;=\; \alpha\, \sum_{t,d} \frac{A_{t,k}\, [X_{k,e} W_{v,e,d}]^+}{Z_{t,d}^+}\, \Gamma^{O}_{t,d}
    \;-\; \beta\, \sum_{t,d} \frac{A_{t,k}\, [X_{k,e} W_{v,e,d}]^-}{Z_{t,d}^-}\, \Gamma^{O}_{t,d},
\end{align}
which is~\eqref{eq:attn-ab-recovery}.

\textbf{Target-rule identification.} The right-hand side matches~\eqref{eq:attn-target-input}
verbatim. Since $\Gamma^{O}_{t,d} = R^{O}_{t,d}$, $\Gamma^{X}_{k,e} = R^{X}_{k,e}$. Aggregating
over $e$ and the bias node at fixed $(k,d)$ uses $\sum_e [X_{k,e} W_{v,e,d}]^\kappa +
    [b_{v,d}]^\kappa = \tilde{v}_{k,d}^\kappa$, giving~\eqref{eq:attn-Gamma-V-corollary}, which
matches~\eqref{eq:attn-target-rule}.
\end{proof}

\subsection{The reference as the policy of a QK oracle player}\label{app:attention-oracle}

\S\ref{app:attention-kl-routing} treated the softmax row $A_{t,\cdot}$ as a fixed reference policy.
It is itself the entropy-regularized best response of an auxiliary \emph{oracle} player whose
payoff is the detached query--key logit $e_{tk} = Q_t K_k^\top/\sqrt{d_h}$: by~\eqref{eq:gibbs-lse}
at temperature $1$,
\begin{align}\label{eq:attn-qk-oracle}
    A_{t,\cdot} \;=\; \argmax_{\pi \in \Delta(\{1,\dots,S\})} \Bigl\{ \textstyle\sum_{k} \pi(k)\,
    e_{tk} + \mathcal{H}(\pi) \Bigr\},
    \qquad
    \max_{\pi} \Bigl\{ \textstyle\sum_{k} \pi(k)\, e_{tk} + \mathcal{H}(\pi) \Bigr\} = \log
    \textstyle\sum_{k} \exp(e_{tk}).
\end{align}
The oracle acts in the forward pass only and hands $\bar\pi_t = A_{t,\cdot}$ to the Value-routing
player; its states carry no occupation mass, so no relevance flows through $Q$ or $K$
(\S\ref{app:vit-implementation}).

\subsection{Risk-Averse Attention via ADF}\label{app:risk-attn}

When \ac{ADF} variance propagation is active, let $s_{tk}$ denote the uncertainty score attached to
attending from query token $t$ to key token $k$ (in the current implementation, $s_{tk} =
    \hat{\sigma}_k$, the square root of the feature-averaged \ac{ADF} variance of token $k$'s query
projection; it depends only on the key token and is a proxy for, not the standard deviation of,
$e_{tk}$). Risk-averse attention shifts the QK-logits in the direction of reduced uncertainty,
\begin{align}\label{eq:risk-attn-policy}
    \tilde e_{tk} = e_{tk} + \lambda_{\mathrm{sm}}\, s_{tk},
    \qquad \lambda_{\mathrm{sm}} \le 0,
\end{align}
yielding the risk-modified policy
\begin{align}\label{eq:risk-attn-policy-softmax}
    \tilde{\pi}_t(k) = \frac{\exp(\tilde e_{tk})}{\sum_{k'} \exp(\tilde e_{tk'})}.
\end{align}
The sign convention $\lambda_{\mathrm{sm}} \le 0$ is used throughout the paper and in all tables:
more-negative values penalize high-uncertainty keys more strongly, so the shift always redirects
relevance toward tokens whose QK-interactions are stable under input noise. The shift is computed
in the AV-backward hook from cached pre-softmax logits, so the forward pass and model predictions
remain unchanged.

The naming \emph{softmax-shift risk aversion} is a deliberate analog of the activation-side risk
aversion $\lambda$ but acts in a different direction: $\lambda$ is risk-averse on the value side
(it routes through stable activations by lowering mean$-|\lambda|\sigma$), whereas
$\lambda_{\mathrm{sm}}<0$ exponentially tilts the routing policy, $\tilde\pi_t(k) \propto
    A_{t,k}\exp(\lambda_{\mathrm{sm}} s_{tk})$, toward \emph{low}-uncertainty keys: $\sum_k
    \tilde\pi_t(k)\, s_{tk}$ is non-decreasing in $\lambda_{\mathrm{sm}}$, while the entropy of
$\tilde\pi_t$ may rise or fall. We retain the ``risk-averse'' label because what is being avoided
is uncertainty in the routed pathway, but the reader should note that the shift acts on the routing
policy, not on the player's utility curvature.

\subsection{Implementation details}\label{app:vit-implementation}

The forward and backward passes \ch{of the \ac{RG} on ViT-B/16} share four operator-level conventions
that together preserve \ac{LRP} conservation through the non-linearities of the ViT and align the
explanatory walk with the state-space construction of \S\ref{app:attention-state-space}. We refer to this collectively as
the DAVE-style \emph{conditioned forward pass}~\citep{Wrobel2024}, augmented with the QK gradient
detachment that the QK oracle of \S\ref{app:attention-oracle} requires.

\paragraph{LayerNorm detachment.}\label{app:dave-conditioned}
Standard backpropagation through $\text{LN}(x) = w_{\mathrm{LN}} \odot (x - m(x))/s(x) +
    b_{\mathrm{LN}}$ routes gradients through the data-dependent statistics $m(x), s(x)$, which \ch{couple the
feature coordinates of each token (LayerNorm acts per token; it does not mix tokens)} and
violate \ac{LRP} conservation~\citep{Ali2022}. Following the detached-operator
construction of~\citet{Wrobel2024}, we compute $m, s$ from a detached copy $c = x.\text{detach}()$,
so the backward pass sees only the pointwise map $x \mapsto w_{\mathrm{LN}} \odot (x - m(c))/s(c) +
    b_{\mathrm{LN}}$ and conservation is preserved.

\paragraph{GELU gate detachment.}
GELU gates are conditioned analogously: the gate is implemented as $x \cdot \Phi(c)$ with $c =
    x.\text{detach}()$ and $\Phi$ the standard normal CDF, so the backward pass treats the gate as a
pointwise scaling by the (constant) activation $\Phi(c)$.

\paragraph{Softmax detachment.}
Softmax outputs are fully detached. This matches the QK oracle interpretation of
\S\ref{app:attention-oracle}: the cached softmax row $A_{t,\cdot}$ is the equilibrium policy of an
auxiliary entropy-regularized problem, supplied as a fixed input to the Value-routing player; no
trajectory of the explanatory walk enters $Q$ or $K$.

\paragraph{QK gradient detachment.}\label{app:attention-grad-detach}
The softmax and QK-matmul backward hooks return zero gradients, ensuring that the QK pathway is
used \emph{only} for computing the detached kernel $A_{t,\cdot}$ and that no relevance flows
through $Q$ or $K$. The QK oracle therefore supplies the KL reference, and all explanatory mass
through the attention block flows through the Value tensor. The implemented Value rule is the
composite rule of Theorem~\ref{thm:attn-ab-recovery}: each Value entry keeps the stream masses
$\tilde v^\kappa_{k,d}$ of its Value-projection products, the attention step routes stream $\kappa$
in proportion to $A_{t,k}\tilde v^\kappa_{k,d}$, and the Value projection routes stream $\kappa$
only to the contributions $[X_{k,e}W_{v,e,d}]^\kappa$ of the same sign, so the discounts
$(\alpha,\beta)$ act once per block, as in~\eqref{eq:attn-target-input}. \ch{This composite rule is the
only attention--Value rule in the released code (option \texttt{av\_rule='merged'}); the sequential
variant of \S\ref{app:attention-target-rule} was not used for any reported number. The identity
tests of Appendix~\ref{app:code-tests} check the implementation against an independent reference
of~\eqref{eq:attn-target-input}.}

\ch{Together the four conventions realize the detached-operator semantics of~\citet{Wrobel2024} for the
\ac{RG} backward pass, so that its explanatory mass flows only through the linear effective
transformation. They implement the explanation question of \S\ref{app:attention-target-rule}
(Value-only attribution conditioned on the realized attention matrix, as in
CP-LRP~\citep{Ali2022}) and are not the rules of AttnLRP~\citep{achtibat2024attnlrp}, which
propagates relevance through the softmax and the QK product as well.}

\subsection{Equivariant Spatial Averaging}\label{app:equivariant}

The $16{\times}16$ patch tokenization breaks translation equivariance. Following DAVE's
transform-averaging strategy~\citep{Wrobel2024}, \ac{RG}+Eq computes explanations on transformed
inputs, maps each explanation back to the original image coordinates, and averages the aligned maps
(a Monte-Carlo Reynolds average),
\begin{align}\label{eq:equivariant-average}
    \mathcal{M}_N(x) \;=\; \frac{1}{N}\sum_{n=1}^{N} T_n^{-1}\bigl[\mathcal{M}(T_n x)\bigr],
\end{align}
where $\mathcal{M}$ is the \ac{RG} attribution map and the $T_n$ are i.i.d.\ random rotations
(uniform in $\pm 10^\circ$), translations (up to $\pm 5\%$ of the image extent per axis) and
horizontal flips (probability $0.5$), with $N = 50$ and no input noise. The conditioned forward
pass of \S\ref{app:vit-implementation} separates operator variation from the effective
transformation, in the same spirit as DAVE's conditioned-operator argument~\citep{Wrobel2024}. The
transforms do not form a group, so the average is an optional smoothing mode without an exactness
claim.
\newpage
\section{Analytical Modes}\label{app:modes}

This appendix collects the formal derivations for the analytical modes that the game framework
exposes on top of the $\alpha\beta$-\ac{LRP} anchor: the local temperature deformation $\tau$ on
the linear-state Gibbs readout (Appendix~\ref{app:routing-temp}), the activation-side risk-aversion
$\lambda$ via \ac{ADF} variance propagation, the noisy-observation gate $\tilde\Phi$ on the
player's stop/continue comparison, and the fan-in entropy bias $\lambda_{\mathrm{ent}}$. Each mode
maps a concept in game space (focus, risk aversion) to an explicit modification of the backward
propagation rule; none carries an exactness claim, and each recovers the anchor at its default
($\tau = 1$, $\lambda = 0$, $\tilde\Phi$ off, $\lambda_{\mathrm{ent}} = 0$). The mode-by-mode
empirical sweeps live in Appendix~\ref{app:dual-sweeps}.

\subsection{Defaults, costs and order of use}\label{app:mode-guide}

Table~\ref{tab:mode-guide} lists every mode with its default, its effect and its measured cost
(median seconds per image, one A100 GPU, batch size 1; Appendix~\ref{app:compute-resources}). At
the defaults, \ac{SG} is the gradient and \ac{RG} is $\alpha\beta$-\ac{LRP}. A practical order of
use: (i) start at the anchor; (ii) vary $\tau$ first, since it is free and moves routing between
focused and diffuse maps (Table~\ref{tab:ablations-vit}); (iii) add risk aversion ($\lambda$ on
\acp{CNN}, $\lambda_{\mathrm{sm}}$ on \acp{ViT}, with the input variance $\sigma^2$) when the maps
should discount routes through high-variance units, at the price of one \ac{ADF} pass; (iv) use
equivariant averaging on \acp{ViT} only when token-grid artifacts matter, since it multiplies the
cost by $N$. The validation selections of Appendix~\ref{app:val-selection} show which combination
the protocol picks per model and criterion.

\begin{table}[H]
    \centering
    \caption{Modes: default (anchor) value, effect and measured cost per image.}
    \label{tab:mode-guide}
    \small
    \setlength{\tabcolsep}{3pt}
    \begin{tabularx}{\linewidth}{@{}l l >{\raggedright\arraybackslash}X
        >{\raggedright\arraybackslash}p{0.33\linewidth}@{}}
        \toprule
        Mode & Default & Effect & Cost (s/image) \\
        \midrule
        $(\alpha,\beta)$, $\varepsilon$ & $(2,1)$, $0.5$ & \ac{LRP} stream weights and stabilizer (exact) &
        none \\
        $\tau$ & $1$ & sharper ($\tau<1$) or more diffuse ($\tau>1$) routing & none (VGG-16
        $0.035\to0.037$, ResNet-50 $0.089\to0.094$) \\
        $\lambda$, $\sigma^2$ & $\lambda{=}0$ & attenuates routes through high-variance units & one
        \ac{ADF} pass (ResNet-50 $0.089\to0.140$) \\
        $\lambda_{\mathrm{sm}}$, $\sigma^2$ & $0$ & downweights high-variance keys in attention routing &
        one \ac{ADF} pass (ViT-B/16 $0.097\to0.172$) \\
        $\tilde\Phi$ & off & smooth stop/continue gate & needs the \ac{ADF} $\sigma$ (one \ac{ADF} pass) \\
        $N$ & off & averages over input transformations (token-grid artifacts) & $\times N$ (ViT-B/16,
        $N{=}50$: $4.93$; with $\lambda_{\mathrm{sm}}$: $8.66$) \\
        \bottomrule
    \end{tabularx}
\end{table}

\subsection{Activation-Side Risk Aversion \texorpdfstring{$\lambda$}{lambda} via ADF Variance Propagation}\label{app:adf-variance}

The activation-side risk-aversion mode replaces the deterministic activation $a_i^{(l)} =
    \relu(z_i^{(l)})$ that the player observes by a mean--variance confidence bound $V_\lambda = \mu +
    \lambda\sigma$ ($\lambda \le 0$) computed from an \ac{ADF} variance propagation
pass~\citep{gast2018lightweight}. We propagate the exact mean $\E[\relu(X)]$ and the variance
approximation $\sigma^2\Phi(\mu/\sigma)$ of $\mathrm{Var}[\relu(X)]$ for $X \sim \mathcal{N}(\mu,
    \sigma^2)$~\citep{gast2018lightweight}.

The propagated variance is initialized at the input by the scale $\sigma^2$ at every input neuron.
Negative $\lambda$ shifts the player's observed game value at stopping layers (not linear layers!)
by a lower-confidence amount, attenuating routes through high-variance neurons; this is the
activation-side analog of confidence-bound exploration in bandit / RL
settings~\citep{auer2002finite, zhang2021risk}. The softmax-shift variant $\lambda_{\mathrm{sm}}$
on the attention-logit oracle (\acs{ViT}) is treated separately in Appendix~\ref{app:risk-attn}.

\subsection{Noisy-Observation Gate \texorpdfstring{$\tilde\Phi$}{Phi}}\label{app:noisy-gate-mode}

The noisy-observation mode replaces the player's pure stop/continue policy at every activation
state by a \emph{mixed} policy. At $s_{i,q}^{(l, \text{act})}$, the player continues with
probability
\begin{align}\label{eq:noisy-gate-defn}
    \tilde\Phi(z_i^{(l)}) \;\coloneqq\; \Phi\bigl(\mu_i^{(l)}/\sigma_i^{(l)} + \lambda\bigr)
\end{align}
and stops with the complementary probability $1 - \tilde\Phi(z_i^{(l)})$, where $\Phi$ is the
standard normal \ac{CDF} and $\mu_i^{(l)}$, $\sigma_i^{(l)}$ are the \ac{ADF} mean and standard
deviation of $z_i^{(l)}$ (Appendix~\ref{app:adf-variance}). The analysis below takes
$\mu_i^{(l)}=z_i^{(l)}$, $\sigma_i^{(l)}=1$ and $\lambda=0$; the general case follows with $z
    \mapsto \mu+\lambda\sigma$ and $\mathcal{R}_\Phi \mapsto \sigma\mathcal{R}_\Phi$. The mixed
policy~\eqref{eq:noisy-gate-defn} arises as the unique solution of a regularized maximization that
the player can solve directly. Define the probit regularizer
\begin{align}\label{eq:noisy-gate-regulariser}
    \mathcal{R}_\Phi(\pi) \;\coloneqq\; \phi\!\bigl(\Phi^{-1}(\pi)\bigr),
    \qquad \pi \in [0,1],
\end{align}
where $\phi(y) = (2\pi)^{-1/2}\,e^{-y^2/2}$ is the standard normal probability density and
$\Phi^{-1}$ is the quantile function (inverse of the standard normal \ac{CDF}).

\begin{lemma}[Probit dual of the soft gate]\label{lem:noisy-gate-dual}
For every $z\in\RR$, the regularized maximization
\begin{align}\label{eq:noisy-gate-objective}
    \max_{\pi \in [0,1]}\; \bigl\{\,\pi z \;+\; \mathcal{R}_\Phi(\pi)\,\bigr\}
\end{align}
admits the unique maximizer $\pi^\star = \Phi(z)$ and optimum value
\begin{align}\label{eq:noisy-gate-optimum}
    z\,\Phi(z) + \phi(z) \;=\; \E_{\eta\sim\mathcal{N}(0,1)}\!\bigl[\max(z+\eta,\,0)\bigr].
\end{align}
\end{lemma}

\begin{proof}
Using $\phi'(y) = -y\,\phi(y)$ and the inverse-function derivative $(\Phi^{-1})'(\pi) =
    1/\phi(\Phi^{-1}(\pi))$,
\begin{align}\label{eq:noisy-gate-deriv}
    \frac{\mathrm{d}\mathcal{R}_\Phi}{\mathrm{d}\pi}
    \;=\; \phi'\!\bigl(\Phi^{-1}(\pi)\bigr) \cdot \frac{1}{\phi(\Phi^{-1}(\pi))}
    \;=\; -\Phi^{-1}(\pi).
\end{align}
The first-order condition $z - \Phi^{-1}(\pi) = 0$ yields $\pi^\star = \Phi(z)$; substituting back
and using $\phi(\Phi^{-1}(\Phi(z))) = \phi(z)$ gives the left-hand side
of~\eqref{eq:noisy-gate-optimum}. The right-hand identity is the standard Gaussian smoothing
$\E[\max(z+\eta,0)] = z\Phi(z) + \phi(z)$.
\end{proof}

\begin{remark}[Optimal policy under perceived payoff noise]\label{rem:noisy-gate-perception}
The optimum $\pi^\star = \Phi(z_i^{(l)})$ admits an equivalent interpretation as the marginal of an
optimal pure rule: the player perceives the continuation payoff at $s_{i,q}^{(l)}$ as the noisy
quantity $z_i^{(l)} + \eta$ with $\eta \sim \mathcal{N}(0,1)$ independent across states, and
applies the noiseless-optimal pure rule \emph{``continue iff the perceived payoff is strictly
positive''} on this perception. Marginalizing the rule over $\eta$ at fixed $z_i^{(l)}$ recovers
the continue probability,
\begin{align}\label{eq:noisy-gate-marginal}
    \P_{\eta\sim\mathcal{N}(0,1)}\!\bigl(z_i^{(l)} + \eta > 0\bigr) \;=\; \Phi(z_i^{(l)}).
\end{align}
The realized continuation payoff is the deterministic $z_i^{(l)}$ from the actual network forward
pass: the noise $\eta$ lives in the player's perception of the payoff while the environment pays
out $z_i^{(l)}$ on continuation, so the expected backward Bellman comparison at $s_{i,q}^{(l)}$ is
the soft mass $\Phi(z_i^{(l)}) \cdot z_i^{(l)}$ in place of the deterministic $\1[z_i^{(l)} >
    0]\cdot z_i^{(l)}$.
\end{remark}

The regularizer $\mathcal{R}_\Phi$ is symmetric around $\pi = \tfrac12$ and peaks there at $\phi(0)
    = 1/\sqrt{2\pi}$ (where $\Phi^{-1}(\tfrac12) = 0$), tapering to $0$ at the boundary
$\pi\in\{0,1\}$. It rewards \emph{indecisiveness} --- exactly the role Shannon entropy
$\mathcal{H}(\pi)$ plays in the Softplus variant of \S\ref{app:softplus-stopping-game}, where the
entropy bonus is the active player's surplus from being allowed to mix. The optimum $\pi^\star =
    \Phi(z) = \E_{\eta\sim\mathcal{N}(0,1)}[\,\1(z+\eta > 0)\,]$ is the hard ReLU gate averaged over a
Gaussian shift of its threshold, aligning with the original stochastic-gate motivation for
\acs{GELU}~\citep{Hendrycks2016gelu}. The optimal value~\eqref{eq:noisy-gate-optimum} is the
Gaussian-smoothed \acs{ReLU} $z\Phi(z) + \phi(z)$, whose derivative is $\Phi$; it is not
$\mathrm{GELU}(z) = z\Phi(z)$, which is not convex and hence not the conjugate of any regularizer.

On \ac{ReLU} networks the forward pass still computes $\relu(z_i^{(l)})$ exactly, so the downstream
game values inherited from Theorem~\ref{thm:forward-stopping} are unchanged: the player evaluates
the soft gate $\tilde\Phi$ on the \emph{current} layer and reuses the original network activations
for everything below. There, $\tilde\Phi$ is an optional smoothing mode (used in the VGG-16 \ac{SG}
row) without an exactness claim.

\paragraph{\acs{GELU} networks.}\label{app:gelu-noisy-gate}
On a \acs{GELU} network, $\tilde\Phi$ multiplies the occupation-measure recursion by
$\Phi(z_i^{(l)})$ at every layer, which is the chain rule of the gate-conditioned map $a_i^{(l)} =
    \Phi(\bar z_i^{(l)})\, z_i^{(l)}$ with the gate argument $\bar z_i^{(l)}$ held at its forward
value, the conditioning used for \acs{ViT}s (Appendix~\ref{app:vit-implementation}). It is not the
\acs{GELU} gradient, which needs $\mathrm{GELU}'(z) = \Phi(z) + z\phi(z)$; the substitution
$\1[z>0] \to \Phi(z)$ is exact only for the Gaussian-smoothed-\acs{ReLU} network, as the
entropy-regularized gate is for Softplus (Appendix~\ref{app:softplus-stopping-game}).

\subsection{\texorpdfstring{Fan-in entropy bias $\lambda_{\mathrm{ent}}$}{Fan-in entropy bias lambda\_ent}}\label{app:entropy-mode}

The fan-in entropy mode lowers the bias, i.e.\ the continuation payoff (\ac{SG}) or bias-node
payoff (\ac{RG}), of neurons with diffuse fan-in: $b_j^{(l)} \mapsto b_j^{(l)} -
    \lambda_{\mathrm{ent}}\tilde{\mathcal{H}}_j^{(l)}$ with $\tilde{\mathcal{H}}_j^{(l)} = -\sum_i
    p_{ij}\log p_{ij}/\log n_j$, $p_{ij} = |W_{ij}^{(l)}|/\gamma_j^{(l)}$ and $n_j$ the fan-in of
neuron $j$. It depends on the weights only (negligible cost) and changes the forward pass of the
explained network.
\newpage
\section{Softplus Variant of the Stopping Game}\label{app:softplus-stopping-game}

We introduce the entropy regularization that \citet{Gaubert2025} apply at activation states to
model networks with Softplus instead of ReLU activations and, similarly to the ReLU case, adapt it
to the Stopping Decomposition of Appendix~\ref{app:dc-decomp}: the player-specific game values
agree with the activations of a Softplus Decomposition, and the gradient of the Softplus network is
recovered by a linear combination of occupation measures.

\paragraph{Modified rule: entropy-regularized activation-state policy.}\label{app:softplus-rules}
The technical tool throughout this section, and re-used by the Routing Game in
Appendix~\ref{app:routing-formal} and the attention block of Appendix~\ref{app:vit-extension}, is
the standard Legendre--Fenchel identity for Shannon entropy: for every $\theta>0$, $n\in\NN$, and
payoff vector $r \in \RR^n$,
\begin{align}\label{eq:legendre-fenchel-app}
    \max_{\pi \in \Delta_n} \left\{ \sum_{i=1}^n \pi(i)\, r_i + \theta\, \mathcal{H}(\pi) \right\}
    = \theta\, \log \sum_{i=1}^n \exp\!\left(\frac{r_i}{\theta}\right),
    \qquad \pi^\star(i) \propto \exp\!\left(\frac{r_i}{\theta}\right),
\end{align}
with $\mathcal{H}(\pi) \coloneqq -\sum_i \pi(i)\log \pi(i)$; the unique maximizer is the Gibbs
measure at inverse temperature $1/\theta$. Identity~\eqref{eq:legendre-fenchel-app} is the Fenchel
conjugate of the (scaled) negative entropy~\citep{BoydVandenberghe2004}.

Only the active player's stop/continue choice changes. At every activation state $s = s_{j,q}^{(l,
    \mathrm{act})}$, replace the binary maximization~\eqref{eq:adf-v-active} by an entropy-regularized
maximization over mixed policies $\pi_q(s) \in \Delta(\{\mathrm{stop},\mathrm{cont}\})$ at
temperature $\theta$. With $\tilde Q_q^\pi(s, \mathrm{stop}) = \tilde Q_{q'}^\pi(s, \mathrm{stop})
    = 0$ (cf.~\eqref{eq:adf-q-stop}) and $\tilde Q_p^\pi(s, \mathrm{cont})$ as
in~\eqref{eq:adf-q-cont}, the soft Bellman recursion for the active-player advantage is
\begin{align}\label{eq:softplus-soft-bellman}
    V_q^\pi(s)
    \coloneqq
    \max_{\nu \in \Delta(\{\mathrm{stop},\mathrm{cont}\})} \Bigl\{
    \textstyle\sum_{a} \nu(a)\bigl(\tilde Q_q^\pi(s, a) - \tilde Q_{q'}^\pi(s, a)\bigr)
    + \theta\, \mathcal{H}(\nu)
    \Bigr\}.
\end{align}
Applying~\eqref{eq:legendre-fenchel-app} to the binary action set $\{\mathrm{stop},\mathrm{cont}\}$
with payoffs $0$ and $\tilde Q_q^\pi(s, \mathrm{cont}) - \tilde Q_{q'}^\pi(s, \mathrm{cont})$
immediately yields the unique optimizer as the Gibbs policy at inverse temperature $1/\theta$ and
the closed-form soft advantage,
\begin{align}\label{eq:softplus-policy}
    \pi_q^\star\!\bigl(\mathrm{cont} \,\big|\, s\bigr)
    = \mathrm{sp}'_\theta\bigl(\tilde Q_q^\pi(s, \mathrm{cont}) - \tilde Q_{q'}^\pi(s,
    \mathrm{cont})\bigr),
    \qquad
    V_q^\pi(s)
    = \mathrm{sp}_\theta\bigl(\tilde Q_q^\pi(s, \mathrm{cont}) - \tilde Q_{q'}^\pi(s,
    \mathrm{cont})\bigr),
\end{align}
with $\mathrm{sp}_\theta(z) = \theta\log(1+e^{z/\theta})$ and $\mathrm{sp}'_\theta(z) =
    1/(1+e^{-z/\theta})$. The transition kernel~\eqref{eq:adf-kernel-stop}--\eqref{eq:adf-kernel}, the
payoffs $r_p$, the Bellman equations~\eqref{eq:adf-q-bellman}--\eqref{eq:adf-q-cont} for $\tilde
    Q_p^\pi$ and the value recursion~\eqref{eq:adf-v-active} for $\tilde V_p^\pi$ are unchanged, except
that $\pi$ is no longer pure and the entropy bonus $\theta\mathcal{H}(\pi_q(s))$ is credited to the
active player $q$.

\begin{theorem}[Softplus forward equivalence]\label{thm:softplus-forward}
The entropy-regularized Stopping Game of~\eqref{eq:softplus-soft-bellman} admits an SPNE whose
activation-state policy is the Gibbs policy~\eqref{eq:softplus-policy}; it is unique on networks
without max-pooling nodes, since~\eqref{eq:softplus-soft-bellman} is strictly concave (ties at
max-pooling states are resolved as in Appendix~\ref{app:notation-identities}). Under this
equilibrium, at every activation state $s_{j,q}^{(l, \mathrm{act})}$,
\begin{align}\label{eq:softplus-spne-tildes}
    \tilde V_p^{\pi^\star}\!\bigl(s_{j,q}^{(l, \mathrm{act})}\bigr)
    \;=\;
    a_{\mathrm{soft},j}^{(l, p \oplus q,\,\theta)},
    \qquad
    V_q^{\pi^\star}\!\bigl(s_{j,q}^{(l, \mathrm{act})}\bigr)
    \;=\;
    \mathrm{sp}_\theta\!\bigl(z_j^{(l)}\bigr),
\end{align}
where $a_{\mathrm{soft},j}^{(l, \pm,\,\theta)}$ are the Softplus-Decomposition activations of the
Softplus network with activation $\mathrm{sp}_\theta$, propagated
by~\eqref{eq:split-pos}--\eqref{eq:split-neg} from $a_{\mathrm{soft},k}^{(0,\pm,\theta)} = x_k^\pm$
with the decomposition function
\begin{align}\label{eq:softplus-decomp}
    f_{\mathrm{soft}}^{\,\theta}(z^+, z^-) \;\coloneqq\; \bigl(\,\mathrm{sp}'_\theta(z^+ - z^-)\, z^+ +
    \theta\, \mathcal{H}_\theta(z^+ - z^-),\;\; \mathrm{sp}'_\theta(z^+ - z^-)\, z^-\,\bigr),
\end{align}
and $\mathcal{H}_\theta(z) \coloneqq \mathcal{H}(\mathrm{sp}'_\theta(z))$ the binary entropy of the
Gibbs weight.
\end{theorem}

\begin{proof}
By backward induction in $l$, mirroring the proof of Theorem~\ref{thm:forward-stopping}, with the
same base case $\tilde V_p^{\pi^\star}(s_{k,q}^{(0)}) = x_k^{p\oplus q} = a_{\mathrm{soft},k}^{(0,
    p\oplus q, \theta)}$. For $l \ge 1$, assume~\eqref{eq:softplus-spne-tildes} holds at layer $l-1$,
so the decomposed activations obey $a_{\mathrm{soft},i}^{(l-1, +,\theta)} -
    a_{\mathrm{soft},i}^{(l-1, -,\theta)} = a_i^{(l-1)}$. Substituting into~\eqref{eq:adf-q-cont}, the
Q-difference of~\eqref{eq:softplus-policy} computes to
$b_j^{(l)} + \sum_{i,r} W_{ij}^{(l, q\oplus r)}\, \xi_{q,r}\, a_i^{(l-1)} = z_j^{(l)}$, using
$\sum_r W_{ij}^{(l, q\oplus r)}\xi_{q,r} = W_{ij}^{(l)}$ exactly as in the proof of
Theorem~\ref{thm:forward-stopping}. Hence, the Gibbs policy is $\pi_q^\star(\mathrm{cont}\mid
    \cdot) = \mathrm{sp}'_\theta(z_j^{(l)})$, and~\eqref{eq:softplus-policy} gives
$V_q^{\pi^\star}(s_{j,q}^{(l, \mathrm{act})}) = \mathrm{sp}_\theta(z_j^{(l)})$, the second identity
of~\eqref{eq:softplus-spne-tildes}.

For the first identity, for the active player ($p = q$, so $p \oplus q = +$), the
entropy-regularized optimization~\eqref{eq:softplus-soft-bellman} attributes both the routed
expectation and the entropy bonus to player $q$:
\(
    \tilde V_q^{\pi^\star}(s_{j,q}^{(l, \mathrm{act})})
    = \mathrm{sp}'_\theta(z_j^{(l)})\, z_j^{(l, +)} + \theta\, \mathcal{H}_\theta(z_j^{(l)})
    = a_{\mathrm{soft},j}^{(l, +, \theta)},
\)
where $\tilde Q_q^{\pi^\star}(\cdot, \mathrm{cont}) = z_j^{(l, +)}$ follows from the inductive
hypothesis via~\eqref{eq:split-pos}. For the opponent ($p = q'$, so $p \oplus q = -$), the entropy
bonus is absent (it is the active player's surplus from being allowed to mix) and only the routed
expectation contributes:
\(
    \tilde V_{q'}^{\pi^\star}(s_{j,q}^{(l, \mathrm{act})})
    = \mathrm{sp}'_\theta(z_j^{(l)})\, z_j^{(l, -)}
    = a_{\mathrm{soft},j}^{(l, -, \theta)}
\)
by~\eqref{eq:split-neg}. The two identities together close the induction.
\end{proof}

\paragraph{Gradient recovery on Softplus networks.}\label{app:softplus-gradient}
The discounted occupation measure of Definition~\ref{def:occupation} is unchanged in form, but the
trajectory law $\Pi_x^\theta$ is now generated by the Gibbs policies~\eqref{eq:softplus-policy} in
place of the deterministic stop/continue rule.

\begin{theorem}[Softplus gradient recovery]\label{thm:softplus-gradient}
Under the SPNE of Theorem~\ref{thm:softplus-forward} on the Softplus network with activation
$\mathrm{sp}_\theta$, the discounted occupation difference recovers the Softplus-network gradient
exactly as in Theorem~\ref{thm:deriv-occ}:
\begin{align}
    \frac{\partial a_j^{(m)}}{\partial a_i^{(l)}}
    \;=\;
    \xi_{q,+}\Bigl(\Gamma_x\bigl(s_{i,+}^{(l, \mathrm{act})} \bigm| s\bigr) -
    \Gamma_x\bigl(s_{i,-}^{(l, \mathrm{act})} \bigm| s\bigr)\Bigr),
    \qquad s = s_{j,q}^{(m, \mathrm{act})},
\end{align}
with $\Gamma_x$ now driven by $\Pi_x^\theta$.
\end{theorem}

\begin{proof}
Write $D(s_{h,r}^{(m',\mathrm{act})}) \coloneqq \Gamma_x(s_{i,r}^{(l,\mathrm{act})} \mid
    s_{h,r}^{(m',\mathrm{act})}) - \Gamma_x(s_{i,r'}^{(l,\mathrm{act})} \mid
    s_{h,r}^{(m',\mathrm{act})})$, the right-hand side above. The one-step occupation recursion in the
proof of Theorem~\ref{thm:deriv-occ} holds with the gate $\1[z_j^{(m)} > 0]$ replaced by the Gibbs
continuation probability $\mathrm{sp}'_\theta(z_j^{(m)})$; taking its difference over the two
targets gives
\(
    D(s_{j,q}^{(m,\mathrm{act})}) = \mathrm{sp}'_\theta(z_j^{(m)}) \sum_h \bigl( W_{hj}^{(m,+)}
    D(s_{h,q}^{(m-1,\mathrm{act})}) - W_{hj}^{(m,-)} D(s_{h,q'}^{(m-1,\mathrm{act})}) \bigr).
\)
By induction on $m-l$, $D(s_{h,r}^{(m-1,\mathrm{act})}) = \partial a_h^{(m-1)}/\partial a_i^{(l)}$
for both $r$, so the right-hand side is $\mathrm{sp}'_\theta(z_j^{(m)}) \sum_h W_{hj}^{(m)}\,
    \partial a_h^{(m-1)}/\partial a_i^{(l)}$, the chain rule at a Softplus gate; the base case
$D(s_{i,r}^{(l,\mathrm{act})}) = 1$ and the addition and max-pooling steps are unchanged. Unlike
Part~1 of Theorem~\ref{thm:deriv-occ}, the individual occupation measures are not derivatives of
the decomposed activations~\eqref{eq:softplus-decomp}, because the soft gate depends on them; only
their difference is.
\end{proof}

\paragraph{Probit analog and GELU.}
Replacing $\theta\mathcal{H}$ in~\eqref{eq:softplus-soft-bellman} by the probit regularizer
$\mathcal{R}_\Phi(\pi) = \phi(\Phi^{-1}(\pi))$ gives, by Lemma~\ref{lem:noisy-gate-dual}, the
continuation probability $\Phi(z)$ and the value $z\Phi(z) + \phi(z) = \E_{\eta \sim
    \mathcal{N}(0,1)}[\relu(z + \eta)]$, so the proofs above carry over verbatim with
$\mathrm{sp}_\theta$ replaced by this Gaussian-smoothed ReLU, $\mathrm{sp}'_\theta$ by its
derivative $\Phi$ and $\theta\mathcal{H}_\theta(z)$ by $\phi(z)$: the probit-regularized game
recovers Gaussian-smoothed-ReLU networks and their exact gradient. \acs{GELU}, $z\Phi(z)$, is
\emph{not} covered: it is not convex, so it is not the value $\max_\pi\{\pi z + \mathcal{R}(\pi)\}$
(a Fenchel conjugate) of any regularized stop/continue choice. Substituting the gate $\Phi(z)$ into
Theorem~\ref{thm:deriv-occ} on a \acs{GELU} network yields the gradient of the gate-conditioned map
$a = \Phi(z^{\mathrm{fwd}})\, z$, not the \acs{GELU} gradient $\Phi(z) + z\phi(z)$
(Appendix~\ref{app:noisy-gate-mode}).

\paragraph{Relation to the curvature-based defense of Dombrowski et
al.}\label{app:softplus-dombrowski}
\citet{Dombrowski2020} train Softplus networks from scratch as a defense against
gradient-explanation manipulation. Their argument bounds the Frobenius norm of the network Hessian:
replacing $\mathrm{ReLU}$, whose distributional second derivative $\delta(z)$ leaves $\Vert
    \nabla^2 f \Vert_F$ unbounded, by $\mathrm{sp}_\theta$ with $|\mathrm{sp}_\theta''(z)| \le
    1/(4\theta)$ yields a uniform bound on $\Vert \nabla^2 f \Vert_F$, and the gradient theorem then
bounds the maximal change of the gradient explanation under input perturbations.
Theorems~\ref{thm:softplus-forward}--\ref{thm:softplus-gradient} place this Softplus replacement
inside our framework as a single change in the game: entropy-regularizing the activation-state
stop/continue policy at temperature $\theta$. Whether the Hellinger response of the resulting
trajectory law (Appendix~\ref{app:path-divergence}) yields bounds sharper than, or differently
targeted from, the Lipschitz bounds of \citet{Dombrowski2020} is left for future work.
\newpage
\section{Hellinger Distance and Disagreement Map}\label{app:path-divergence}

We compare two layered Markov processes on a common state graph and seek both a scalar divergence
of the trajectory distributions and an exact per-pixel decomposition of this divergence at the
input layer.  We use the squared Hellinger distance because the geometric mean $\sqrt{PQ}$
preserves the multiplicative structure of Markov trajectory distributions and yields an exact
linear-time dynamic program.  Unlike Jeffreys divergence, it remains finite under support mismatch;
unlike Jensen--Shannon, it admits an exact recursion on the layered graph.  Later, we decompose the
squared Hellinger distance as a sum of non-negative per-pixel contributions over all input-layer
states, giving a disagreement map at the input layer (Theorem~\ref{thm:per-pixel}).

\subsection{Layered Markov Process Foundations}\label{sec:lmp-foundations}

We collect here the layered-Markov-process abstraction that underlies every Hellinger statement in
this appendix: a finite layered graph augmented with a cemetery $\dagger$, and the act-state
\acp{MP} that the two games of the main paper instantiate on it. The natural \ac{RG} carries an
extra linear-routing state per neuron between consecutive activation states; we marginalize it out
and check in Remark~\ref{rem:rg-natural-equivalence} that this preserves trajectory Bhattacharyya
and Hellinger exactly.

\subsubsection{Layered Markov Process with Cemetery}\label{sec:lmp}

\paragraph{State graph.}
Fix an integer $L \geq 1$.  For each $l \in \{0, 1, \ldots, L\}$ let $\mathcal{S}^{(l)}$ be a
finite set of \emph{ordinary states} at layer~$l$.  Adjoin a single absorbing \emph{cemetery state}
$\dagger$ shared across all layers and write $\mathcal{S}^{(l)}_+ \coloneqq \mathcal{S}^{(l)} \cup
    \{\dagger\}$.
The full state graph is the disjoint union $\bigsqcup_l \mathcal{S}^{(l)}_+$, and transitions are
only permitted from layer~$l$ to layer~$l{-}1$ (or into the cemetery, which then absorbs the
trajectory forever).  The states at layer~$0$, the input layer, are referred to as \emph{pixels}
below; the per-pixel decomposition will sum exactly over $\mathcal{S}^{(0)}_+$.

\paragraph{Transition kernel.}
For each layer $l \geq 1$, a \emph{backward transition kernel} $T^{(l)} : \mathcal{S}^{(l)}_+ \to
    \Delta\bigl(\mathcal{S}^{(l-1)}_+\bigr)$ assigns to each state $s \in \mathcal{S}^{(l)}_+$ a
probability distribution $T^{(l)}(s \to \cdot)$ over $\mathcal{S}^{(l-1)}_+$, with the constraint
that the cemetery is absorbing:
\begin{align}\label{eq:cemetery-absorbing}
    T^{(l)}(\dagger \to \dagger) = 1, \qquad T^{(l)}(\dagger \to s') = 0 \;\;\text{for}\;\; s' \in
    \mathcal{S}^{(l-1)}.
\end{align}
Ordinary states may transition either to the next layer's ordinary states or into the cemetery
(early termination).  We write $T^{(l)}(s \to \dagger)$ for the probability that an ordinary state
at layer~$l$ jumps to the cemetery rather than to a layer-$(l{-}1)$ ordinary state.

\paragraph{Initial distribution and trajectories.}
Fix a deterministic starting state $s_0 \in \mathcal{S}^{(L)}$.  A \emph{trajectory} (indexed by
layer, $\zeta_l \in \mathcal{S}^{(l)}_+$) is a sequence
\begin{align}
    \zeta = (\zeta_L, \zeta_{L-1}, \ldots, \zeta_0) \;\in\; \mathcal{S}^{(L)}_+ \times
    \mathcal{S}^{(L-1)}_+ \times \cdots \times \mathcal{S}^{(0)}_+,
\end{align}
and the kernel $T^{(\cdot)}$ together with $\delta_{s_0}$ at layer~$L$ induces a probability
distribution on trajectories via
\begin{align}\label{eq:path-measure}
    \Pi(\zeta_L, \zeta_{L-1}, \ldots, \zeta_0) \;=\; \mathbf{1}\{\zeta_L = s_0\} \prod_{k=1}^{L}
    T^{(k)}(\zeta_k \to \zeta_{k-1}).
\end{align}
For any $l \in \{0,\ldots,L\}$ the prefix distribution $\Pi^{(l)}$ over $(\zeta_L, \ldots,
    \zeta_l)$ is the marginal of~$\Pi$ over the trailing coordinates.

\paragraph{Two parallel processes.}
We compare two such processes that share the state graph and the start state~$s_0$, but use
different kernels: $M = A$ with $T^{(l)}_A, \Pi_A^{(l)}$, and $M = B$ with $T^{(l)}_B,
    \Pi_B^{(l)}$.
The kernels may differ either because the underlying network weights differ (e.g.\ randomization),
the activations differ (e.g.\ different inputs), or both.

\subsubsection{How the Two Games Fit This Abstraction}\label{sec:games-fit}

We instantiate the layered MP of Section~\ref{sec:lmp} for both games in this subsection. State
labels $s^{(l,\mathrm{type})}_{i,p,\kappa}$ follow Appendices~\ref{app:relu-net-game}
and~\ref{app:routing-formal}. We evaluate the SPNE policies of Theorems~\ref{thm:forward-stopping}
(\ac{SG}) and~\ref{thm:routing-forward} (\ac{RG}) at $\tau = 1$ (so activation-state decisions are
deterministic, continue iff $z_j^{(l)} > 0$); for the \ac{RG} we further fix $\varepsilon = 0$.

\paragraph{\ac{SG} MP.}
Ordinary states at layers $l \in \{1,\ldots,L\}$ are activation states
$s^{(l,\mathrm{act})}_{i,p}$, and layer-$0$ ordinary states are pixel-player terminal pairs
$s^{(0)}_{j,q}$ (each pixel splits into two terminals, mirroring the player tag the trajectory
carries to it). The deterministic start state is $s_0 = s^{(L,\mathrm{act})}_{u,+}$. The SPNE
backward kernel keeps the player on positive weights and flips it on negative weights:
\begin{align}
    T^{(l)}_M\bigl(s^{(l,\mathrm{act})}_{j,p} \to s^{(l-1,\mathrm{act})}_{i,p}\bigr)
    &= \frac{W^{(l,+)}_M[i,j]}{\gamma^{(l)}_{M,j}}\,\1\{z_j^{(l)} > 0\},
    \nonumber \\
    T^{(l)}_M\bigl(s^{(l,\mathrm{act})}_{j,p} \to s^{(l-1,\mathrm{act})}_{i,p'}\bigr)
    &= \frac{W^{(l,-)}_M[i,j]}{\gamma^{(l)}_{M,j}}\,\1\{z_j^{(l)} > 0\},
    \qquad 2 \le l \le L,
    \label{eq:adf-T}
\end{align}
with $\gamma^{(l)}_{M,j} = \sum_i |W^{(l)}_M[i,j]|$, $p' = -p$, cemetery row
$T^{(l)}_M(s^{(l,\mathrm{act})}_{j,p} \to \dagger) = \1\{z_j^{(l)} \le 0\}$, and the analogous form
at $l = 1$ landing on the pixel-player terminals $s^{(0)}_{i,p}, s^{(0)}_{i,p'}$. This is the
natural \ac{SG} MP of Appendix~\ref{app:relu-net-game} verbatim.

\paragraph{\ac{RG} MP.}
Ordinary states at layers $l \in \{1,\ldots,L\}$ are activation states
$s^{(l,\mathrm{act})}_{i,p}$, and layer-$0$ ordinary states are pixel-stream terminal pairs
$s^{(0)}_{j,\kappa}$ (each pixel splits into two terminals tagged by the layer-$1$ routing-edge
stream). The deterministic start state is $s_0 = s^{(L,\mathrm{act})}_{u,+}$. The kernel composes
the activation-state stop/continue decision and the subsequent $\kappa$-branched Gibbs routing into
one act-to-act transition:
\begin{align}
    T^{(l)}_M\bigl(s^{(l,\mathrm{act})}_{j,p} \to s^{(l-1,\mathrm{act})}_{i,p}\bigr)
    &= \tfrac12\,\pi^{(l)}_{M,+}(i \mid j)\,\1\{z_j^{(l)} > 0\},
    \nonumber \\
    T^{(l)}_M\bigl(s^{(l,\mathrm{act})}_{j,p} \to s^{(l-1,\mathrm{act})}_{i,p'}\bigr)
    &= \tfrac12\,\pi^{(l)}_{M,-}(i \mid j)\,\1\{z_j^{(l)} > 0\},
    \qquad 2 \le l \le L,
    \label{eq:rg-T-act}
\end{align}
and at $l = 1$ the destinations are pixel-stream terminals carrying the routing-edge $\kappa$,
\begin{align}\label{eq:rg-T-pixel}
    T^{(1)}_M\bigl(s^{(1,\mathrm{act})}_{j,p} \to s^{(0)}_{i,\kappa}\bigr)
    \;=\; \tfrac12\,\pi^{(1)}_{M,\kappa}(i \mid j)\,\1\{z_j^{(1)} > 0\},
    \qquad \kappa \in \{+,-\},
\end{align}
with the same cemetery row $\1\{z_j^{(l)} \le 0\}$. The Gibbs policy is the temperature-$\tau$
distribution
\begin{align}\label{eq:rg-T}
    \pi^{(l)}_{M,\kappa}(i \mid j)
    \;\propto\;
    \bigl[\,a^{(l-1)}_{M,i}\;W^{(l,\kappa)}_{M,ij}\,\bigr]^{1/\tau},
    \qquad i \in \mathrm{Pred}_j^{(l)},
\end{align}
(at $l = 1$, $a_{M,i}^{(0)}\,W^{(1,\kappa)}_{M,ij}$ is replaced by $[x_i\,W^{(1)}_{M,ij}]^\kappa$;
with $\varepsilon > 0$, mass $\varepsilon^{1/\tau}/Z^{(l)}_{j,\kappa}$ is added to the cemetery
row; as in the implementation, the bias node of Appendix~\ref{app:routing-formal} is omitted).
Equations~\eqref{eq:rg-T-act}--\eqref{eq:rg-T-pixel} are the act-to-act marginal of the natural
\ac{RG} MP of Appendix~\ref{app:routing-formal}, which carries an extra linear-routing state per
neuron between consecutive act states; Remark~\ref{rem:rg-natural-equivalence} gives a
one-paragraph derivation and shows that this marginalization preserves Bhattacharyya and Hellinger
distances on trajectories exactly.

\subsubsection{Marginal Occupation Measures}\label{sec:lmp-marginals}

For each model $M \in \{A,B\}$, the marginal probability that the chain visits state $s$ at
layer~$l$ is
\begin{align}\label{eq:lmp-alpha}
    m^{(l)}_M(s) \;=\; \sum_{\zeta: \zeta_l = s} \Pi^{(l)}_M(\zeta),
    \qquad s \in \mathcal{S}^{(l)}_+.
\end{align}
This satisfies the deterministic boundary condition
\begin{align}\label{eq:lmp-alpha-boundary}
    m^{(L)}_M(s_0) = 1, \qquad m^{(L)}_M(s) = 0 \;\;\text{for}\;\; s \neq s_0,
\end{align}
and the one-step recursion
\begin{align}\label{eq:lmp-alpha-recursion}
    m^{(l-1)}_M(s') \;=\; \sum_{s \in \mathcal{S}^{(l)}_+} m^{(l)}_M(s)\; T^{(l)}_M(s \to s'),
    \qquad s' \in \mathcal{S}^{(l-1)}_+.
\end{align}
We refer to the single backward propagation~\eqref{eq:lmp-alpha-recursion} as the \emph{Marginal
Pass}.  At the input layer, $m^{(0)}_M(s)$ is the probability that a backward trajectory from~$s_0$
terminates at pixel~$s$ under model~$M$, and $m^{(0)}_M(\dagger)$ is the total mass absorbed by the
cemetery along the way.

\subsection{Hellinger Distance from Bhattacharyya Coefficients}\label{sec:lmp-hellinger-distance}

\paragraph{Properties.}
$\mathrm{BC}^{(l)}$ and $\mathrm{H}^{(l)}$ lie in $[0,1]$ for every $l$, even under support
mismatch.  If a state is alive in one model and dead in the other, its geometric-mean row vanishes
and its mass contributes directly to the Hellinger increment.  Since $\mathrm{H}$ is a metric and
satisfies $\mathrm{H}^2 \le d_{\mathrm{TV}} \le \sqrt{2}\,\mathrm{H}$ and $\mathrm{H}^2 \le
    \tfrac{1}{2}\mathrm{KL}$, small Hellinger distance controls TV and small KL controls H as well.
\emph{Normalization conventions.} Here $d_{\mathrm{TV}}(\Pi_A, \Pi_B) = \tfrac{1}{2}\sum_{\zeta}
    |\Pi_A(\zeta)-\Pi_B(\zeta)| \in [0,1]$ is the probability-metric (half-$L^1$) total-variation
distance, and $\mathrm{KL}$ is measured in nats; the displayed inequalities follow this convention.
Readers using the unnormalized $L^1$ total variation $\|\Pi_A-\Pi_B\|_1 = 2 d_{\mathrm{TV}}$ should
absorb the factor of~$2$ accordingly.

\paragraph{Alternative divergences.}\phantomsection\label{app:why-hellinger}
KL and Jeffreys divergence have an exact Markov-chain chain rule but are unbounded and become
infinite under support mismatch, e.g.\ when randomization opens or closes routes.  Jensen--Shannon
is bounded and finite, but $\tfrac{1}{2}(\Pi_A + \Pi_B)$ is not Markov on the layered graph, so no
exact linear-time dynamic program is available.  Wasserstein distances need a ground cost between
neuron-labeled trajectories (across layers, signs and permuted neurons), for which there is no
canonical choice, and the resulting transport problem has no exact local recursion.  Hellinger is
bounded, finite, and exactly tractable: the overlap factorizes along the layered law
(Lemma~\ref{lem:bc-chain-rule}), which gives a dynamic program linear in the number of edges
(Algorithm~\ref{alg:hellinger}) and the non-negative per-state decomposition of
Theorem~\ref{thm:per-pixel}.  Note that $\mathrm{H}$ compares the \emph{undiscounted} law $\Pi$:
the discount $d_t$ does not enter the kernels, so $\mathrm{H}$ is blind to the path weights
$(\alpha,\beta)$ and to the \ac{SG} weight magnitudes $\gamma^{(l)}_j$, which enter only through
$d_t$; it compares \emph{where} trajectories go, not how much weight they carry.

\paragraph{Computation.}
Compute the marginal occupations $m_M^{(l)}$, run the Hellinger pass
(Algorithm~\ref{alg:hellinger}), and read off $\mathrm{H}^2(s)$ from~\eqref{eq:per-pixel-h2}; the
total cost is linear in the number of edges.

\subsubsection{Bhattacharyya Coefficient and the Geometric-Mean Kernel}\label{sec:lmp-bc}

\paragraph{Definitions.}
For two probability distributions $P, Q$ on a common finite set, the \emph{Bhattacharyya
coefficient} is the geometric overlap
\begin{align}\label{eq:bc-def}
    \mathrm{BC}(P, Q) \;\coloneqq\; \sum_{y} \sqrt{P(y)\, Q(y)} \;\in\; [0, 1],
\end{align}
and the \emph{squared Hellinger distance} is
\begin{align}\label{eq:hellinger-def}
    \mathrm{H}^2(P, Q) \;\coloneqq\; \tfrac{1}{2} \sum_y \bigl(\sqrt{P(y)} - \sqrt{Q(y)}\bigr)^2 \;=\;
    1 - \mathrm{BC}(P, Q) \;\in\; [0, 1].
\end{align}
The Hellinger distance $\mathrm{H}(P,Q) = \sqrt{1 - \mathrm{BC}(P,Q)}$ is a true metric: $0 \leq
    \mathrm{H} \leq 1$, with $\mathrm{H} = 0$ iff $P = Q$ and $\mathrm{H} = 1$ iff $P$ and $Q$ have
disjoint supports.  Crucially, $\sqrt{0 \cdot Q(y)} = 0$, so mismatched supports cause no
infinity---only a smaller overlap.

\paragraph{Geometric-mean kernel.}
For each layer $l \geq 1$ and pair of states $(s, s') \in \mathcal{S}^{(l)}_+ \times
    \mathcal{S}^{(l-1)}_+$, define the \emph{geometric-mean kernel}
\begin{align}\label{eq:geom-kernel}
    \widetilde{T}^{(l)}(s, s') \;\coloneqq\; \sqrt{T^{(l)}_A(s \to s')\; T^{(l)}_B(s \to s')}.
\end{align}
This kernel is non-negative but in general sub-stochastic.  Its row sum
\begin{align}\label{eq:geom-rowsum}
    \widetilde{R}^{(l)}(s) \;\coloneqq\; \sum_{s' \in \mathcal{S}^{(l-1)}_+} \widetilde{T}^{(l)}(s, s')
\end{align}
satisfies $0 \leq \widetilde{R}^{(l)}(s) \leq 1$ by Cauchy--Schwarz, with equality iff $T^{(l)}_A(s
    \to \cdot) = T^{(l)}_B(s \to \cdot)$.  Cemetery handling is automatic: $T^{(l)}_A(\dagger \to
    \dagger) = T^{(l)}_B(\dagger \to \dagger) = 1$, so $\widetilde{T}^{(l)}(\dagger, \dagger) = 1$ and
$\widetilde{R}^{(l)}(\dagger) = 1$.

\subsubsection{Multiplicative Chain Rule}\label{sec:lmp-chain-rule}

\begin{lemma}[Chain rule]\label{lem:bc-chain-rule}
For every $l \in \{0, 1, \ldots, L\}$, the Bhattacharyya coefficient of the prefix distributions
factorizes along trajectories:
\begin{align}\label{eq:bc-chain-rule}
    \mathrm{BC}\!\bigl(\Pi^{(l)}_A,\, \Pi^{(l)}_B\bigr) \;=\; \sum_{(\zeta_L, \ldots, \zeta_l)}
    \mathbf{1}\{\zeta_L = s_0\} \prod_{k=l+1}^{L} \widetilde{T}^{(k)}(\zeta_k, \zeta_{k-1}).
\end{align}
\end{lemma}

\begin{proof}
Substitute the Markov factorization $\Pi^{(l)}_M(\zeta) = \mathbf{1}\{\zeta_L =
    s_0\}\prod_{k=l+1}^L T^{(k)}_M(\zeta_k, \zeta_{k-1})$ for $M \in \{A,B\}$ into the definition of
$\mathrm{BC}$:
\begin{align}
    \mathrm{BC}\!\bigl(\Pi^{(l)}_A,\, \Pi^{(l)}_B\bigr)
    &\;=\; \sum_{(\zeta_L, \ldots, \zeta_l)} \mathbf{1}\{\zeta_L = s_0\}\,
    \sqrt{\prod_{k=l+1}^{L} T^{(k)}_A(\zeta_k, \zeta_{k-1})\, T^{(k)}_B(\zeta_k, \zeta_{k-1})} \\
    &\;=\; \sum_{(\zeta_L, \ldots, \zeta_l)} \mathbf{1}\{\zeta_L = s_0\}\,
    \prod_{k=l+1}^{L} \widetilde{T}^{(k)}(\zeta_k, \zeta_{k-1}),
\end{align}
using $\sqrt{\prod_k a_k b_k} = \prod_k \sqrt{a_k b_k}$ for non-negative factors.
\end{proof}

\begin{lemma}[Monotonicity]\label{lem:bc-monotone}
Let $\mathrm{BC}^{(l)} \coloneqq \mathrm{BC}(\Pi^{(l)}_A, \Pi^{(l)}_B)$.  Then $\mathrm{BC}^{(l)}$
is monotonically non-increasing as $l$ decreases from $L$ to $0$:
\begin{align}
    1 \;=\; \mathrm{BC}^{(L)} \;\geq\; \mathrm{BC}^{(L-1)} \;\geq\; \cdots \;\geq\; \mathrm{BC}^{(0)}
    \;\geq\; 0.
\end{align}
Equivalently, the squared Hellinger distance $\mathrm{H}^{2}_{(l)} = 1 - \mathrm{BC}^{(l)}$ is
monotonically non-decreasing.
\end{lemma}

\begin{proof}
Define
\begin{align}
    g^{(l)}(s) \;\coloneqq\; \sum_{(\zeta_L,\ldots,\zeta_l)} \mathbf{1}\{\zeta_L = s_0\}\,
    \prod_{k=l+1}^L \widetilde{T}^{(k)}(\zeta_k, \zeta_{k-1})\, \mathbf{1}\{\zeta_l = s\}.
\end{align}
From Lemma~\ref{lem:bc-chain-rule}, $\mathrm{BC}^{(l)} = \sum_{s \in \mathcal{S}^{(l)}_+}
    g^{(l)}(s)$.  Grouping the layer-$(l-1)$ sum by $\zeta_l = s$,
\begin{align}
    \mathrm{BC}^{(l-1)} \;=\; \sum_{s \in \mathcal{S}^{(l)}_+} g^{(l)}(s)\; \widetilde{R}^{(l)}(s).
\end{align}
Since $\widetilde{R}^{(l)}(s) \in [0, 1]$ and $g^{(l)}(s) \geq 0$, we get $\mathrm{BC}^{(l-1)} \leq
    \mathrm{BC}^{(l)}$.  The base case $\mathrm{BC}^{(L)} = 1$ holds because $\Pi^{(L)}_A = \Pi^{(L)}_B
    = \delta_{s_0}$ and $\sqrt{1 \cdot 1} = 1$.
\end{proof}

\subsubsection{Hellinger backward pass}\label{sec:lmp-hellinger-backward}

The chain rule (Lemma~\ref{lem:bc-chain-rule}) yields a backward DP that propagates the per-state
cumulative $\widetilde{T}$-mass $g^{(l)}(s)$ from the start state $s_0$ down through the layers.
The prefix Bhattacharyya coefficient at each layer is the row sum $\mathrm{BC}^{(l)} = \sum_{s}
    g^{(l)}(s)$, the prefix Hellinger distance is $\mathrm{H}^{(l)} = \sqrt{1 - \mathrm{BC}^{(l)}}$,
and the per-terminal-state disagreement map is read off at $l = 0$ from~\eqref{eq:per-pixel-h2}.
The map is defined on the full terminal set $\mathcal{S}^{(0)}_+ = \mathcal{S}^{(0)} \cup
    \{\dagger\}$; the visualized pixel attribution map is its restriction to ordinary pixels
$\mathcal{S}^{(0)}$, and its sum is therefore $\mathrm{H}^2_{\mathrm{ord}} = \mathrm{H}^2 -
    \mathrm{H}^2(\dagger)$, not the full $\mathrm{H}^2$ (Theorem~\ref{thm:per-pixel}(iii)).

\begin{algorithm}[H]
    \caption{Hellinger backward pass (with per-terminal-state disagreement map)}
    \label{alg:hellinger}\label{alg:hellinger-backward}
    \begin{algorithmic}[1]
\Require Backward kernels $T^{(l)}_A, T^{(l)}_B$ for all $l$; start state $s_0$; marginal occupations $m^{(0)}_A, m^{(0)}_B$ from the Marginal Pass.
\Ensure Bhattacharyya coefficients $\mathrm{BC}^{(l)}$, Hellinger distances $\mathrm{H}^{(l)}$, and per-terminal-state disagreement map $\mathrm{H}^2(\cdot)$ on $\mathcal{S}^{(0)}_+$.
\State $g^{(L)}(s_0) \gets 1$; \quad $g^{(L)}(s) \gets 0$ for every $s \neq s_0$
\State $\mathrm{BC}^{(L)} \gets 1$
\For{$l = L, L{-}1, \ldots, 1$}
    \For{each state $s' \in \mathcal{S}^{(l-1)}_+$}
        \State $g^{(l-1)}(s') \;\gets\; \displaystyle\sum_{s \in \mathcal{S}^{(l)}_+}\; g^{(l)}(s)\; \widetilde{T}^{(l)}(s, s')$
            \Comment{$\widetilde{T}^{(l)}(s, s') = \sqrt{T^{(l)}_A(s \to s')\, T^{(l)}_B(s \to s')}$}
    \EndFor
    \State $\mathrm{BC}^{(l-1)} \;\gets\; \displaystyle\sum_{s' \in \mathcal{S}^{(l-1)}_+} g^{(l-1)}(s')$
\EndFor
\For{$l = L, L{-}1, \ldots, 0$}
    \State $\mathrm{H}^{(l)} \gets \sqrt{1 - \mathrm{BC}^{(l)}}$
\EndFor
\State \textbf{// Per-terminal-state disagreement map (Theorem~\ref{thm:per-pixel}).}
\State \textbf{// The visualized pixel attribution map is the restriction to $\mathcal{S}^{(0)} \subset \mathcal{S}^{(0)}_+$.}
\For{each terminal state $s \in \mathcal{S}^{(0)}_+$}
    \State $\mathrm{H}^2(s) \;\gets\; \tfrac{1}{2}\bigl(m^{(0)}_A(s) + m^{(0)}_B(s)\bigr) - g^{(0)}(s)$
\EndFor
\end{algorithmic}
\end{algorithm}

\begin{lemma}[Backward pass correctness]\label{lem:hellinger-backward-correctness}
Algorithm~\ref{alg:hellinger} computes $g^{(l)}$, hence $\mathrm{BC}^{(l)}$, $\mathrm{H}^{(l)}$,
and $\mathrm{H}^2(s)$ correctly for every $l$ and $s$.  In particular,
\begin{align}
    \sum_{s \in \mathcal{S}^{(0)}_+} \mathrm{H}^2(s) \;=\; \mathrm{H}^{2}_{(0)}.
\end{align}
\end{lemma}

\begin{proof}
\emph{Correctness of $g^{(l)}$ by downward induction on $l$.}  At $l = L$, the initialization
$g^{(L)}(s_0) = 1$ and $g^{(L)}(s) = 0$ for $s \neq s_0$ matches the definition $g^{(L)}(s) =
    \mathbf{1}\{s = s_0\}$ in Lemma~\ref{lem:bc-monotone} (the empty product equals~1).

Assume $g^{(l)}$ stored by the algorithm equals the $g^{(l)}$ defined in
Lemma~\ref{lem:bc-monotone}.  In iteration $l$, the algorithm computes
\begin{align}
    g^{(l-1)}(s') \;=\; \sum_{s \in \mathcal{S}^{(l)}_+} g^{(l)}(s)\, \widetilde{T}^{(l)}(s, s'),
\end{align}
which by direct expansion equals
\begin{align}
    \begin{aligned}
        &\sum_s \Bigl(\sum_{(\zeta_L,\ldots,\zeta_l)} \mathbf{1}\{\zeta_L =
        s_0\}\!\!\prod_{k=l+1}^L\!\!\widetilde{T}^{(k)}(\zeta_k, \zeta_{k-1})\, \mathbf{1}\{\zeta_l =
        s\}\Bigr) \widetilde{T}^{(l)}(s, s') \\
        &\qquad=\; \sum_{(\zeta_L, \ldots, \zeta_{l-1})} \mathbf{1}\{\zeta_L =
        s_0\}\!\!\prod_{k=l}^L\!\!\widetilde{T}^{(k)}(\zeta_k, \zeta_{k-1})\, \mathbf{1}\{\zeta_{l-1} =
        s'\}.
    \end{aligned}
\end{align}
This equals $g^{(l-1)}(s')$ from Lemma~\ref{lem:bc-monotone}, closing the induction.

\emph{Correctness of $\mathrm{BC}^{(l)}$.}  After the loop, the algorithm sets $\mathrm{BC}^{(l)} =
    \sum_{s'} g^{(l)}(s')$, which by the proof of Lemma~\ref{lem:bc-monotone} equals
$\mathrm{BC}\!\bigl(\Pi^{(l)}_A, \Pi^{(l)}_B\bigr)$.  Then $\mathrm{H}^{(l)} = \sqrt{1 -
    \mathrm{BC}^{(l)}}$ by definition~\eqref{eq:hellinger-def}.

\emph{Correctness of $\mathrm{H}^2(\cdot)$.} At $l = 0$, the stored value satisfies
\begin{align}
    g^{(0)}(s)
    =
    \sum_{\zeta: \zeta_0 = s}
    \sqrt{\Pi^{(0)}_A(\zeta) \Pi^{(0)}_B(\zeta)},
\end{align}
which is exactly the quantity used in the proof of Theorem~\ref{thm:per-pixel}. Substituting
into~\eqref{eq:per-pixel-h2} gives the per-pixel disagreement values from
Theorem~\ref{thm:per-pixel}, and~\eqref{eq:per-pixel-sum} guarantees they sum to
$\mathrm{H}^{2}_{(0)}$.
\end{proof}

\begin{remark}[\ac{RG} act-only \ac{MP} matches the natural game]\label{rem:rg-natural-equivalence}
The natural \ac{RG} \ac{MP} of Appendix~\ref{app:routing-formal} alternates activation and
linear-routing states: each $s^{(l,\mathrm{act})}_{j,p}$ either stops to $\dagger$ (if $z_j^{(l)}
    \le 0$) or branches uniformly to $s^{(l,\mathrm{lin})}_{j,p,+}$ or $s^{(l,\mathrm{lin})}_{j,p',-}$
(intra-layer act-to-lin), and each $s^{(l,\mathrm{lin})}_{j,q,\kappa}$ routes to a layer-$(l-1)$
activation (or, at $l=1$, to the layer-$0$ pixel-stream terminal $s^{(0)}_{i,\kappa}$) via the
Gibbs policy $\pi^{(l)}_{M,\kappa}(\,\cdot\mid j)$ retaining the player (inter-layer lin-to-act).
Multiplying these two kernels along the unique intermediate lin state yields the act-to-act
rows~\eqref{eq:rg-T-act}--\eqref{eq:rg-T-pixel} term by term. Deleting every lin state from a
natural trajectory is a bijection between natural and act-only trajectories starting at $s_0 =
    s^{(L,\mathrm{act})}_{u,+}$ (each lin coordinate is recovered: $j,q$ from the act states it sits
between, and $\kappa$ from the player update---kept $\Rightarrow +$, switched $\Rightarrow -$---or,
at $l=1$, directly from the pixel-stream terminal). A bijection on trajectory supports preserves
Bhattacharyya and Hellinger by the trivial reindexing $\sum_{\zeta} \sqrt{\Pi_A(\zeta)\Pi_B(\zeta)}
    = \sum_{\widehat{\zeta}} \sqrt{\widehat\Pi_A(\widehat{\zeta})\widehat\Pi_B(\widehat{\zeta})}$, so
all $\mathrm{BC}^{(l)}$ and $\mathrm{H}^{(l)}$ computed from the act-only kernel coincide with the
natural-game ones.
\end{remark}

\subsection{Per-Pixel Disagreement Map}\label{sec:lmp-disagreement}

The squared Hellinger distance $\mathrm{H}^2$ admits an exact non-negative decomposition over
input-layer states: each input pixel carries a calibrated share of the total trajectory
disagreement (Theorem~\ref{thm:per-pixel}), giving a disagreement \emph{map} that is itself an
attribution of the trajectory mismatch back to the input. The conditioned-on-survival variant
removes the cemetery contribution that otherwise bounds plain $\mathrm{H}$ from below on randomized
models.

\subsubsection{Per-Pixel Decomposition: The Disagreement Map}\label{sec:lmp-per-pixel}

The squared Hellinger distance at layer~$0$ decomposes exactly over terminal pixels.

\begin{theorem}[Per-pixel decomposition]\label{thm:per-pixel}
Let $\mathrm{H}^2 \coloneqq \mathrm{H}^2\!\bigl(\Pi^{(0)}_A,\, \Pi^{(0)}_B\bigr)$.  For $s \in
    \mathcal{S}^{(0)}_+$, define
\begin{align}\label{eq:per-pixel-h2}
    \mathrm{H}^2(s) \;\coloneqq\; \tfrac{1}{2}\bigl(m^{(0)}_A(s) + m^{(0)}_B(s)\bigr) \;-\; g^{(0)}(s),
\end{align}
where $g^{(0)}(s)$ is the cumulative $\widetilde{T}$-mass at $s$.  Then:
\begin{enumerate}
    \item[(i)] $\mathrm{H}^2(s) \geq 0$ for every $s \in \mathcal{S}^{(0)}_+$.
    \item[(ii)]
        \begin{align}\label{eq:per-pixel-sum}
            \mathrm{H}^2 \;=\; \sum_{s \in \mathcal{S}^{(0)}_+} \mathrm{H}^2(s).
        \end{align}
    \item[(iii)]
        \begin{align}\label{eq:ordinary-h2}
            \mathrm{H}^2_{\mathrm{ord}} \;\coloneqq\; \sum_{s \in \mathcal{S}^{(0)}} \mathrm{H}^2(s) \;=\;
            \mathrm{H}^2 - \mathrm{H}^2(\dagger)
        \end{align}
\end{enumerate}
\end{theorem}

\begin{proof}
Apply the algebraic identity $(\sqrt{a} - \sqrt{b})^2 = a + b - 2\sqrt{ab}$ to each trajectory:
\begin{align}
    \mathrm{H}^2 \;=\; \tfrac{1}{2}\sum_{\zeta}\bigl(\sqrt{\Pi^{(0)}_A(\zeta)} -
    \sqrt{\Pi^{(0)}_B(\zeta)}\bigr)^2
    \;=\; \tfrac{1}{2}\sum_{\zeta}\bigl[\Pi^{(0)}_A(\zeta) + \Pi^{(0)}_B(\zeta) -
    2\sqrt{\Pi^{(0)}_A(\zeta)\Pi^{(0)}_B(\zeta)}\bigr].
\end{align}
Group the trajectory sum by the terminal pixel $\zeta_0 = s$:
\begin{align}
    \mathrm{H}^2 \;=\; \tfrac{1}{2}\sum_{s \in \mathcal{S}^{(0)}_+}\Bigl[\sum_{\zeta: \zeta_0 =
    s}\Pi^{(0)}_A(\zeta)
    + \sum_{\zeta: \zeta_0 = s}\Pi^{(0)}_B(\zeta)
    - 2\sum_{\zeta: \zeta_0 = s}\sqrt{\Pi^{(0)}_A(\zeta)\Pi^{(0)}_B(\zeta)}\Bigr].
\end{align}
By definition, $\sum_{\zeta: \zeta_0 = s}\Pi^{(0)}_M(\zeta) = m^{(0)}_M(s)$ for $M \in \{A,B\}$.
By Lemma~\ref{lem:bc-chain-rule} restricted to trajectories terminating at~$s$ (i.e.\ tracing
through Lemma~\ref{lem:bc-monotone}'s definition of $g^{(l)}$), $\sum_{\zeta: \zeta_0 =
    s}\sqrt{\Pi^{(0)}_A(\zeta)\Pi^{(0)}_B(\zeta)} = g^{(0)}(s)$.  Substituting,
\begin{align}
    \mathrm{H}^2 \;=\; \tfrac{1}{2}\sum_s\bigl[m^{(0)}_A(s) + m^{(0)}_B(s) - 2g^{(0)}(s)\bigr]
    \;=\; \sum_{s \in \mathcal{S}^{(0)}_+} \mathrm{H}^2(s),
\end{align}
proving (ii).

For (i), apply Cauchy--Schwarz to the sum defining $g^{(0)}(s)$:
\begin{align}
    g^{(0)}(s)
    = \sum_{\zeta: \zeta_0 = s} \sqrt{\Pi^{(0)}_A(\zeta) \Pi^{(0)}_B(\zeta)}
    \leq \sqrt{\Bigl(\sum_{\zeta: \zeta_0 = s}\Pi^{(0)}_A(\zeta)\Bigr)\Bigl(\sum_{\zeta: \zeta_0 =
    s}\Pi^{(0)}_B(\zeta)\Bigr)}
    = \sqrt{m^{(0)}_A(s)\,m^{(0)}_B(s)}.
\end{align}
Combined with the AM--GM inequality $\sqrt{m^{(0)}_A(s)m^{(0)}_B(s)} \leq \tfrac{1}{2}(m^{(0)}_A(s)
    + m^{(0)}_B(s))$, this gives $g^{(0)}(s) \leq \tfrac{1}{2}(m^{(0)}_A(s) + m^{(0)}_B(s))$, so
$\mathrm{H}^2(s) \geq 0$.  Equality holds iff $\Pi^{(0)}_A(\cdot \mid \zeta_0 = s) =
    \Pi^{(0)}_B(\cdot \mid \zeta_0 = s)$ \emph{and} $m^{(0)}_A(s) = m^{(0)}_B(s)$.

Item (iii) is immediate from (ii) by separating the cemetery from the ordinary pixels.
\end{proof}

\paragraph{The disagreement map.}
Theorem~\ref{thm:per-pixel} defines $\mathrm{H}^2(\cdot)$ on the full terminal set
$\mathcal{S}^{(0)}_+ = \mathcal{S}^{(0)} \cup \{\dagger\}$: this is the \emph{full terminal-state
decomposition}, with $\sum_{s \in \mathcal{S}^{(0)}_+} \mathrm{H}^2(s) = \mathrm{H}^2_{(0)}$, and
the cemetery contribution $\mathrm{H}^2(\dagger)$ is part of it. The \emph{Hellinger disagreement
map} is the ordinary-pixel restriction $s \mapsto \mathrm{H}^2(s)$ on $\mathcal{S}^{(0)}$, which is
the quantity we visualize. It is non-negative, sums to $\mathrm{H}^2_{\mathrm{ord}}$ over ordinary
pixels, and is computed from $g^{(0)}, m_A^{(0)}, m_B^{(0)}$.  The quantity
$\mathrm{H}^2_{\mathrm{ord}}$ is an unnormalized ordinary-state subsum; the conditioned-on-survival
Hellinger distance that normalizes it is introduced below. Recall from Section~\ref{sec:games-fit}
that layer-$0$ ordinary states of the two games are \emph{tagged} pixel pairs $s^{(0)}_{j,\pm}$,
tagged by the player $q$ for the \ac{SG} and by the stream sign $\kappa$ for the \ac{RG}; the
per-pixel attribution displayed in figures aggregates the two tags,
\begin{align}
    \mathrm{H}^2_{\mathrm{pixel}}(j) \;=\; \mathrm{H}^2\bigl(s^{(0)}_{j,+}\bigr) +
    \mathrm{H}^2\bigl(s^{(0)}_{j,-}\bigr),
\end{align}
which remains non-negative and sums to $\mathrm{H}^2_{\mathrm{ord}}$ over pixels.  By the
Cauchy--Schwarz step in the proof of Theorem~\ref{thm:per-pixel}(i), $g^{(0)}(s) \leq
    \sqrt{m^{(0)}_A(s)\,m^{(0)}_B(s)}$, so $\mathrm{H}^2(s)$ is bounded below by the marginal surrogate
$\tfrac{1}{2}\bigl(m^{(0)}_A(s) + m^{(0)}_B(s)\bigr) - \sqrt{m^{(0)}_A(s)\,m^{(0)}_B(s)}$, the
per-pixel squared Hellinger distance between the layer-0 marginal occupations; the trajectory map
is strictly larger whenever the conditional distributions of trajectories reaching pixel~$s$ differ
between models.

\subsubsection{Conditioning on Survival to the Ordinary Input
Layer}\label{sec:lmp-conditioned-survival}

When the cemetery absorbs a large fraction of trajectory mass, the total Hellinger distance mixes
two effects: disagreement about \emph{whether} a trajectory survives and disagreement about
\emph{where surviving trajectories go}.  To isolate the second effect, we condition each model's
trajectory law on the event that the trajectory reaches an ordinary input-layer state.  The
resulting scalar compares the two \emph{surviving} trajectory distributions directly and removes
the common cemetery floor.  Because survival to the input layer factorizes along the layered graph,
this conditioning remains Markov and is implemented by a per-state Doob $h$-transform (with $h =
    \psi_M$).

\begin{theorem}[Conditioned-survival trajectory law and conditioned Hellinger distance]\label{thm:conditioned-survival}
Let $\mathsf{Surv} \coloneqq \{(\zeta_L,\ldots,\zeta_0) : \zeta_0 \in \mathcal{S}^{(0)}\}$ be the
event, in the trajectory space $\prod_{l=0}^{L} \mathcal{S}^{(l)}_+$, that the backward trajectory
ends in an ordinary input-layer state — equivalently, the event that the trajectory never enters
the cemetery~$\dagger$.  For each model $M \in \{A,B\}$ define the survival function
\begin{align}\label{eq:survival-function}
    \psi_M^{(l)}(s)
    \;\coloneqq\;
    \mathbb{P}_M(\mathsf{Surv} \mid \zeta_l = s),
    \qquad s \in \mathcal{S}^{(l)}_+, \;\; l \in \{0,\ldots,L\}.
\end{align}
Then:
\begin{enumerate}
    \item[(i)] The boundary and recursion are
        \begin{align}\label{eq:survival-boundary}
            \psi_M^{(0)}(s) = 1 \;\; \text{for } s \in \mathcal{S}^{(0)},
            \qquad
            \psi_M^{(0)}(\dagger) = 0,
        \end{align}
        and for every $l \ge 1$,
        \begin{align}\label{eq:survival-recursion}
            \psi_M^{(l)}(s)
            \;=\;
            \sum_{s' \in \mathcal{S}^{(l-1)}_+} T_M^{(l)}(s \to s')\, \psi_M^{(l-1)}(s')
            \;=\;
            \sum_{s' \in \mathcal{S}^{(l-1)}} T_M^{(l)}(s \to s')\, \psi_M^{(l-1)}(s').
        \end{align}
        In particular,
        \begin{align}\label{eq:survival-mass}
            Z_M \;\coloneqq\; \mathbb{P}_M(\mathsf{Surv}) \;=\; \psi_M^{(L)}(s_0) \;=\; \sum_{s \in
            \mathcal{S}^{(0)}} m_M^{(0)}(s) \;=\; 1 - m_M^{(0)}(\dagger).
        \end{align}
    \item[(ii)] Assume $Z_M > 0$.  For every ordinary state $s \in \mathcal{S}^{(l)}$ with $\psi_M^{(l)}(s) >
            0$, define the conditioned kernel on ordinary states by
        \begin{align}\label{eq:conditioned-kernel}
            T_{M,\mathrm{surv}}^{(l)}(s \to s')
            \;\coloneqq\;
            \frac{T_M^{(l)}(s \to s')\, \psi_M^{(l-1)}(s')}{\psi_M^{(l)}(s)},
            \qquad s' \in \mathcal{S}^{(l-1)}.
        \end{align}
        Then $T_{M,\mathrm{surv}}^{(l)}(s \to \cdot)$ is row-stochastic on $\mathcal{S}^{(l-1)}$, and the
        conditional trajectory distribution $\Pi_{M,\mathrm{surv}} \coloneqq \Pi_M(\,\cdot \mid
            \mathsf{Surv}\,)$ factorizes as the layered Markov law
        \begin{align}\label{eq:conditioned-path-measure}
            \Pi_{M,\mathrm{surv}}(\zeta_L,\ldots,\zeta_0)
            = \mathbf{1}\{\zeta_L = s_0\} \prod_{k=1}^{L} T_{M,\mathrm{surv}}^{(k)}(\zeta_k \to \zeta_{k-1}),
            \qquad (\zeta_L,\ldots,\zeta_0) \in \mathsf{Surv}.
        \end{align}
        Its layer-$l$ marginals satisfy
        \begin{align}\label{eq:conditioned-alpha}
            m_{M,\mathrm{surv}}^{(l)}(s)
            \;=\;
            \frac{m_M^{(l)}(s)\, \psi_M^{(l)}(s)}{Z_M},
            \qquad s \in \mathcal{S}^{(l)}.
        \end{align}
    \item[(iii)] Define the conditioned geometric-mean kernel
        \begin{align}\label{eq:conditioned-geom-kernel}
            \widetilde{T}_{\mathrm{surv}}^{(l)}(s,s')
            \coloneqq \sqrt{T_{A,\mathrm{surv}}^{(l)}(s \to s')\, T_{B,\mathrm{surv}}^{(l)}(s \to s')}
            = \frac{\widetilde{T}^{(l)}(s,s')\, \sqrt{\psi_A^{(l-1)}(s')\,
            \psi_B^{(l-1)}(s')}}{\sqrt{\psi_A^{(l)}(s)\, \psi_B^{(l)}(s)}}.
        \end{align}
        The Bhattacharyya coefficient of the conditioned trajectory laws is
        \begin{align}\label{eq:conditioned-bc-kernel}
            \mathrm{BC}\!\bigl(\Pi_{A,\mathrm{surv}}, \Pi_{B,\mathrm{surv}}\bigr)
            \;=\;
            \sum_{(\zeta_L,\ldots,\zeta_0) \in \mathsf{Surv}}
            \mathbf{1}\{\zeta_L = s_0\}
            \prod_{k=1}^{L} \widetilde{T}_{\mathrm{surv}}^{(k)}(\zeta_k, \zeta_{k-1}),
        \end{align}
        so it is computed by the same Hellinger pass (Algorithm~\ref{alg:hellinger}), but on the
        ordinary-state graph with the conditioned kernel.
    \item[(iv)] Let
        \begin{align}\label{eq:ordinary-bc}
            \mathrm{BC}_{\mathrm{ord}}^{(0)}
            \;\coloneqq\;
            \sum_{s \in \mathcal{S}^{(0)}} g^{(0)}(s)
            \;=\;
            \sum_{(\zeta_L,\ldots,\zeta_0) \in \mathsf{Surv}} \sqrt{\Pi_A(\zeta_L,\ldots,\zeta_0)\,
            \Pi_B(\zeta_L,\ldots,\zeta_0)}.
        \end{align}
        Then
        \begin{align}\label{eq:conditioned-bc-posthoc}
            \mathrm{BC}\!\bigl(\Pi_{A,\mathrm{surv}}, \Pi_{B,\mathrm{surv}}\bigr)
            \;=\;
            \frac{\mathrm{BC}_{\mathrm{ord}}^{(0)}}{\sqrt{Z_A Z_B}},
        \end{align}
        and the conditioned-survival Hellinger distance is
        \begin{align}\label{eq:conditioned-h2}
            \mathrm{H}_{\mathrm{surv}}^2
            \;\coloneqq\;
            \mathrm{H}^2\!\bigl(\Pi_{A,\mathrm{surv}}, \Pi_{B,\mathrm{surv}}\bigr)
            \;=\;
            1 - \frac{\mathrm{BC}_{\mathrm{ord}}^{(0)}}{\sqrt{Z_A Z_B}}.
        \end{align}
\end{enumerate}
\end{theorem}

\begin{proof}
For (i),~\eqref{eq:survival-boundary} is immediate from the definition of $\mathsf{Surv}$: at
layer~$0$ every ordinary state survives and the cemetery does not.  For $l \ge 1$, conditioning on
the next step gives
\begin{align}
    \psi_M^{(l)}(s)
    =
    \sum_{s' \in \mathcal{S}^{(l-1)}_+}
    T_M^{(l)}(s \to s')\, \mathbb{P}_M(\mathsf{Surv} \mid \zeta_{l-1} = s')
    =
    \sum_{s' \in \mathcal{S}^{(l-1)}_+}
    T_M^{(l)}(s \to s')\, \psi_M^{(l-1)}(s').
\end{align}
Since $\psi_M^{(l-1)}(\dagger)=0$, the cemetery term drops out and~\eqref{eq:survival-recursion}
follows.  Evaluating at $s_0$ gives $Z_M = \psi_M^{(L)}(s_0)$.  The identities with $m_M^{(0)}$ are
just the layer-$0$ marginal of the event $\mathsf{Surv}$.

For (ii), row-stochasticity follows from~\eqref{eq:survival-recursion}:
\begin{align}
    \sum_{s' \in \mathcal{S}^{(l-1)}} T_{M,\mathrm{surv}}^{(l)}(s \to s')
    =
    \frac{1}{\psi_M^{(l)}(s)}
    \sum_{s' \in \mathcal{S}^{(l-1)}} T_M^{(l)}(s \to s')\, \psi_M^{(l-1)}(s')
    = 1.
\end{align}
Now fix $(\zeta_L,\ldots,\zeta_0) \in \mathsf{Surv}$.  By definition of conditional probability,
\begin{align}
    \Pi_{M,\mathrm{surv}}(\zeta_L,\ldots,\zeta_0)
    =
    \frac{\Pi_M(\zeta_L,\ldots,\zeta_0)}{Z_M}
    =
    \frac{\mathbf{1}\{\zeta_L=s_0\}}{Z_M}
    \prod_{k=1}^{L} T_M^{(k)}(\zeta_k \to \zeta_{k-1}).
\end{align}
Multiplying the conditioned kernels and telescoping the survival factors gives
\begin{align}
    \prod_{k=1}^{L} T_{M,\mathrm{surv}}^{(k)}(\zeta_k \to \zeta_{k-1})
    &=
    \prod_{k=1}^{L}
    T_M^{(k)}(\zeta_k \to \zeta_{k-1})
    \frac{\psi_M^{(k-1)}(\zeta_{k-1})}{\psi_M^{(k)}(\zeta_k)} \\
    &=
    \frac{\psi_M^{(0)}(\zeta_0)}{\psi_M^{(L)}(s_0)}
    \prod_{k=1}^{L} T_M^{(k)}(\zeta_k \to \zeta_{k-1}) \\
    &=
    \frac{1}{Z_M}
    \prod_{k=1}^{L} T_M^{(k)}(\zeta_k \to \zeta_{k-1}),
\end{align}
because $\zeta_0 \in \mathcal{S}^{(0)}$ implies $\psi_M^{(0)}(\zeta_0)=1$.  This
is~\eqref{eq:conditioned-path-measure}.  For the marginals,
\begin{align}
    m_{M,\mathrm{surv}}^{(l)}(s)
    =
    \mathbb{P}_M(\zeta_l = s \mid \mathsf{Surv})
    =
    \frac{\mathbb{P}_M(\zeta_l = s, \mathsf{Surv})}{Z_M}
    =
    \frac{m_M^{(l)}(s)\, \psi_M^{(l)}(s)}{Z_M},
\end{align}
which is~\eqref{eq:conditioned-alpha}.

For (iii), substitute~\eqref{eq:conditioned-kernel} into the square root to
obtain~\eqref{eq:conditioned-geom-kernel}.  Then apply Lemma~\ref{lem:bc-chain-rule} to the
conditioned trajectory laws~\eqref{eq:conditioned-path-measure} on the ordinary-state graph.  This
gives~\eqref{eq:conditioned-bc-kernel} directly.

For (iv), expand the Bhattacharyya coefficient of the conditioned laws:
\begin{align}
    \mathrm{BC}\!\bigl(\Pi_{A,\mathrm{surv}}, \Pi_{B,\mathrm{surv}}\bigr)
    &=
    \sum_{(\zeta_L,\ldots,\zeta_0)\in\mathsf{Surv}}
    \sqrt{\frac{\Pi_A(\zeta_L,\ldots,\zeta_0)}{Z_A}\,
    \frac{\Pi_B(\zeta_L,\ldots,\zeta_0)}{Z_B}} \\
    &=
    \frac{1}{\sqrt{Z_A Z_B}}
    \sum_{(\zeta_L,\ldots,\zeta_0)\in\mathsf{Surv}}
    \sqrt{\Pi_A(\zeta_L,\ldots,\zeta_0)\, \Pi_B(\zeta_L,\ldots,\zeta_0)}.
\end{align}
Grouping the last sum by the ordinary terminal state $\zeta_0=s$ and using the definition of
$g^{(0)}(s)$ exactly as in the proof of Theorem~\ref{thm:per-pixel} yields~\eqref{eq:ordinary-bc},
hence~\eqref{eq:conditioned-bc-posthoc}.  Equation~\eqref{eq:conditioned-h2} is then the identity
$\mathrm{H}^2 = 1 - \mathrm{BC}$ applied to the conditioned trajectory laws.
\end{proof}

\paragraph{Implementation consequence.}
Theorem~\ref{thm:conditioned-survival} gives two exactly equivalent computations of the conditioned
scalar:
\begin{enumerate}
    \item[(a)] \emph{Kernel-side conditioning:} replace each row by the reweighted
        kernel~\eqref{eq:conditioned-kernel}, form the geometric mean~\eqref{eq:conditioned-geom-kernel},
        and run the ordinary Hellinger backward pass on $\mathcal{S}^{(L)} \times \cdots \times
            \mathcal{S}^{(0)}$.
    \item[(b)] \emph{Post-hoc normalization:} run the original pass once, extract the ordinary overlap
        $\mathrm{BC}_{\mathrm{ord}}^{(0)} = \sum_{s \in \mathcal{S}^{(0)}} g^{(0)}(s)$ and the survival
        masses $Z_A, Z_B$, and apply~\eqref{eq:conditioned-h2}.
\end{enumerate}
The first route makes the conditioned trajectory law explicit; the second is cheaper when only the
scalar is needed, but for deep networks it becomes numerically infeasible when the survival
probability $Z_M$ of either model $M \in \{A,B\}$ gets too small.

\paragraph{Conditioned disagreement map.}
Applying Theorem~\ref{thm:per-pixel} to the conditioned kernels $T_{M,\mathrm{surv}}^{(l)}$ yields
the conditioned disagreement map
\begin{align}\label{eq:conditioned-per-pixel-h2}
    \mathrm{H}_{\mathrm{surv}}^2(s)
    \;\coloneqq\;
    \tfrac{1}{2}\bigl(m_{A,\mathrm{surv}}^{(0)}(s) + m_{B,\mathrm{surv}}^{(0)}(s)\bigr)
    -
    g_{\mathrm{surv}}^{(0)}(s),
    \qquad s \in \mathcal{S}^{(0)},
\end{align}
which is non-negative and satisfies $\sum_{s \in \mathcal{S}^{(0)}} \mathrm{H}_{\mathrm{surv}}^2(s)
    = \mathrm{H}_{\mathrm{surv}}^2$.

\subsection{Experiments}\label{sec:lmp-experiments}

This subsection puts the trajectory-space machinery to two empirical tests, holding the model
fixed: how the Hellinger trajectory distance reacts to input perturbations, and whether the
resulting distance separates inputs from different ImageNet-S classes from one another and from
random noise. Here \ac{SG} is the Stopping-Game kernel~\eqref{eq:adf-T} without risk aversion
($\lambda = 0$; named ``ADF'' in the code) and \ac{RG} is the Routing-Game
kernel~\eqref{eq:rg-T-act} at $\tau = 1$.

\subsubsection{Input-Noise Sensitivity}\label{sec:input-noise-hellinger}

For each architecture and each noise standard deviation $\sigma \in \{0.005, 0.01, 0.02, 0.03,
    0.05, 0.1, 0.2, 0.3, 0.5\}$ we draw one standard-Gaussian tensor $\eta$ per image and compute
$\mathrm{H}(\Pi_A(x), \Pi_A(x + \sigma\eta))$ under both game kernels with the argmax target of the
clean input. Figure~\ref{fig:hellinger-input-noise} reports the mean $\mathrm{H}$ and
$\mathrm{H}_{\mathrm{surv}}$ curves over the swept $\sigma$; both increase with $\sigma$ for both
games on all three architectures. On VGG-16 the \ac{RG} $\mathrm{H}$ and
$\mathrm{H}_{\mathrm{surv}}$ are equal, as the kernel routes no trajectories to states belonging to
dead neurons. On ResNet-50 and ViT-B/16, skip connections route trajectories to dead-neuron states,
and on ViT-B/16 the GELU gate also sends some trajectories to the cemetery at every state.

\begin{figure}[H]
    \centering
    \begin{tikzpicture}
\begin{groupplot}[
    group style={group size=3 by 1, horizontal sep=0.90cm},
    scale only axis,
    height=3.25cm,
    ymin=0,
    enlarge x limits=false,
    enlarge y limits=false,
    axis x line*=bottom,
    axis y line*=left,
    axis line style={draw=black, line width=0.8pt},
    tick style={black, line width=0.7pt},
    xlabel={Input-noise $\sigma$},
    ylabel={Hellinger distance $\mathrm{H}$},
    xlabel style={font=\footnotesize},
    ylabel style={font=\footnotesize},
    title style={font=\footnotesize},
    x tick label style={font=\scriptsize},
    y tick label style={font=\scriptsize},
    tick align=outside,
    scaled x ticks=false,
    scaled y ticks=false,
    grid=both,
    major grid style={draw=gray!25},
    minor grid style={draw=gray!15},
    minor y tick num=1,
]
\nextgroupplot[title={VGG-16}, width=3.60cm, xmode=log, log basis x=10, xmax=0.500, xmin=0.0045, log ticks with fixed point, xtick={0.005, 0.01, 0.02, 0.03, 0.05, 0.1, 0.2, 0.3, 0.5}, x tick label style={rotate=90, anchor=east, /pgf/number format/.cd, fixed, precision=3}, ymax=1.0, ytick={0,0.2,0.4,0.6,0.8,1.0}, yticklabel style={font=\scriptsize, /pgf/number format/.cd, fixed, precision=1, zerofill}]
\addplot+[color=blue!70!black, mark=none, line width=1.5pt] coordinates {(0.005000, 0.0553140999) (0.010000, 0.0856729081) (0.020000, 0.1308442531) (0.030000, 0.1647212310) (0.050000, 0.2165156573) (0.100000, 0.2942934630) (0.200000, 0.3756936446) (0.300000, 0.4278060399) (0.500000, 0.5033394820)};
\addplot+[color=blue!70!black, mark=none, line width=1.5pt, densely dotted] coordinates {(0.005000, 0.4797141252) (0.010000, 0.6291917364) (0.020000, 0.7834284200) (0.030000, 0.8614684689) (0.050000, 0.9347354570) (0.100000, 0.9844335426) (0.200000, 0.9978674384) (0.300000, 0.9995125510) (0.500000, 0.9999488152)};
\addplot+[color=red!80!black,  mark=none, line width=1.5pt] coordinates {(0.005000, 0.2776190381) (0.010000, 0.3922999387) (0.020000, 0.5496846290) (0.030000, 0.6581073065) (0.050000, 0.7958674986) (0.100000, 0.9337734384) (0.200000, 0.9888058215) (0.300000, 0.9973915839) (0.500000, 0.9997982972)};
\addplot+[color=red!80!black, mark=none, line width=1.5pt, densely dotted] coordinates {(0.005000, 0.2775959566) (0.010000, 0.3922755395) (0.020000, 0.5496613847) (0.030000, 0.6580863276) (0.050000, 0.7958493750) (0.100000, 0.9337572963) (0.200000, 0.9887889032) (0.300000, 0.9973742499) (0.500000, 0.9997809263)};
\nextgroupplot[title={ResNet-50}, width=3.60cm, xmode=log, log basis x=10, xmax=0.500, xmin=0.0045, log ticks with fixed point, xtick={0.005, 0.01, 0.02, 0.03, 0.05, 0.1, 0.2, 0.3, 0.5}, x tick label style={rotate=90, anchor=east, /pgf/number format/.cd, fixed, precision=3}, ymax=1.0, ytick={0,0.2,0.4,0.6,0.8,1.0}, ylabel={}, yticklabel style={font=\scriptsize, /pgf/number format/.cd, fixed, precision=1, zerofill}]
\addplot+[color=blue!70!black, mark=none, line width=1.5pt] coordinates {(0.005000, 0.0452152578) (0.010000, 0.0642118959) (0.020000, 0.0907002623) (0.030000, 0.1095971441) (0.050000, 0.1325101551) (0.100000, 0.1642478283) (0.200000, 0.1974226976) (0.300000, 0.2196046893) (0.500000, 0.2537678290)};
\addplot+[color=blue!70!black, mark=none, line width=1.5pt, densely dotted] coordinates {(0.005000, 0.3833659495) (0.010000, 0.5009903215) (0.020000, 0.6361947477) (0.030000, 0.7155634523) (0.050000, 0.8049244294) (0.100000, 0.8970243870) (0.200000, 0.9539478054) (0.300000, 0.9741317490) (0.500000, 0.9894004723)};
\addplot+[color=red!80!black,  mark=none, line width=1.5pt] coordinates {(0.005000, 0.1178935577) (0.010000, 0.1738609446) (0.020000, 0.2586063418) (0.030000, 0.3201846700) (0.050000, 0.3993576777) (0.100000, 0.4971650731) (0.200000, 0.5790294239) (0.300000, 0.6243918148) (0.500000, 0.6840088309)};
\addplot+[color=red!80!black, mark=none, line width=1.5pt, densely dotted] coordinates {(0.005000, 0.2716503111) (0.010000, 0.3841185327) (0.020000, 0.5427390559) (0.030000, 0.6517284764) (0.050000, 0.7838204504) (0.100000, 0.9166907404) (0.200000, 0.9796536898) (0.300000, 0.9933897417) (0.500000, 0.9989774765)};
\nextgroupplot[title={ViT-B/16}, width=3.60cm, xmode=log, log basis x=10, xmax=0.500, xmin=0.0045, log ticks with fixed point, xtick={0.005, 0.01, 0.02, 0.03, 0.05, 0.1, 0.2, 0.3, 0.5}, x tick label style={rotate=90, anchor=east, /pgf/number format/.cd, fixed, precision=3}, ymax=1.0, ytick={0,0.2,0.4,0.6,0.8,1.0}, ylabel={}, yticklabel style={font=\scriptsize, /pgf/number format/.cd, fixed, precision=1, zerofill}, legend style={font=\scriptsize, draw=black, fill=white, fill opacity=0.92, text opacity=1, cells={anchor=west}, row sep=1pt, inner xsep=4pt, inner ysep=3pt, at={(0.02,0.98)}, anchor=north west}, legend columns=1]
\addplot+[color=blue!70!black, mark=none, line width=1.5pt] coordinates {(0.005000, 0.0133444159) (0.010000, 0.0265228240) (0.020000, 0.0543220247) (0.030000, 0.0813419521) (0.050000, 0.1233459404) (0.100000, 0.1851983074) (0.200000, 0.2442760116) (0.300000, 0.2769362447) (0.500000, 0.3133687389)};
\addlegendentry{SG $\mathrm{H}$}
\addplot+[color=blue!70!black, mark=none, line width=1.5pt, densely dotted] coordinates {(0.005000, 0.0329339274) (0.010000, 0.0657790490) (0.020000, 0.1341236114) (0.030000, 0.2007487602) (0.050000, 0.3038745886) (0.100000, 0.4582102226) (0.200000, 0.6060417373) (0.300000, 0.6873592496) (0.500000, 0.7752400235)};
\addlegendentry{SG $\mathrm{H}_{\mathrm{surv}}$}
\addplot+[color=red!80!black,  mark=none, line width=1.5pt] coordinates {(0.005000, 0.0274186654) (0.010000, 0.0522607684) (0.020000, 0.1018637235) (0.030000, 0.1512062222) (0.050000, 0.2300948891) (0.100000, 0.3514355790) (0.200000, 0.4755279295) (0.300000, 0.5472851915) (0.500000, 0.6320482010)};
\addlegendentry{RG $\mathrm{H}$}
\addplot+[color=red!80!black, mark=none, line width=1.5pt, densely dotted] coordinates {(0.005000, 0.0637267384) (0.010000, 0.1212389912) (0.020000, 0.2323997733) (0.030000, 0.3368058769) (0.050000, 0.4977321302) (0.100000, 0.7227154583) (0.200000, 0.8925281441) (0.300000, 0.9512227121) (0.500000, 0.9858833916)};
\addlegendentry{RG $\mathrm{H}_{\mathrm{surv}}$}
\end{groupplot}
\end{tikzpicture}
    \caption{Input-noise Hellinger trajectory distance. Solid: plain $\mathrm{H}$; dotted:
    $\mathrm{H}_{\mathrm{surv}}$, conditioned on survival to an ordinary input-layer state. $x$-axis is
    log-scale input-noise $\sigma$; same $(A,A)$ model on both sides, only the input differs.}
    \label{fig:hellinger-input-noise}
\end{figure}

\subsubsection{Class Separability and Reference Baselines}\label{sec:input-noise-baselines}

\textbf{Can the Hellinger distance differentiate between inputs of different classes?} Partly; the
reference baselines in Table~\ref{tab:hellinger-input-random} make the answer architecture- and
kernel-specific:
\begin{itemize}
    \item On \textbf{VGG-16}, the \ac{SG} $\mathrm{H}$ is the kernel that does the work: clean-intraclass
        $\mathrm{H}=0.513$ sits below clean-interclass $\mathrm{H}=0.607$. The \ac{RG} kernel saturates at
        $\mathrm{H}=1$ on every row on VGG-16, so it cannot resolve class structure on this architecture.
    \item On \textbf{ResNet-50}, the \ac{RG} kernel is the discriminating one: clean-intraclass
        $\mathrm{H}=0.786$ vs.\ clean-interclass $\mathrm{H}=0.902$. The \ac{SG} kernel does not separate
        classes on ResNet-50 (intraclass $0.272$ vs.\ interclass $0.268$).
    \item On \textbf{ViT-B/16}, the \ac{RG} kernel gives intraclass $0.789$ vs.\ interclass $0.882$; the
        \ac{SG} gap ($0.422 \pm 0.023$ vs.\ $0.435 \pm 0.017$) is smaller than the spread over pairs, so
        the \ac{SG} does not separate classes there.
\end{itemize}
Random-vs-random is the lowest row in every column whose rows differ, and clean-vs-random is the
highest row of every plain-$\mathrm{H}$ column except the \ac{SG} on ResNet-50 ($0.269$ vs.\
intraclass $0.272$) and the saturated \ac{RG} on VGG-16, consistent with the model routing Gaussian
inputs differently from in-distribution images. The conditioned variant
$\mathrm{H}_{\mathrm{surv}}$ is at least $0.99$ on every row of VGG-16 and ResNet-50 and for the
\ac{RG} on ViT-B/16, so it carries no class-separability signal there. For the \ac{SG} on ViT-B/16
it separates random-vs-random ($0.852$) from the clean rows ($0.957$--$0.968$), but neither
intraclass from interclass pairs ($0.963$ vs.\ $0.968$) nor clean-vs-random ($0.957$) from the
clean pairs.

\begin{table}[H]
    \centering\small
    \setlength{\tabcolsep}{4pt}
    \begin{tabular}{l l cc cc}
  \toprule
  & & \multicolumn{2}{c}{Stopping Game} & \multicolumn{2}{c}{Routing Game} \\
  \cmidrule(lr){3-4} \cmidrule(lr){5-6}
  Architecture & Baseline & $\mathrm{H}$ & $\mathrm{H}_{\mathrm{surv}}$ & $\mathrm{H}$ & $\mathrm{H}_{\mathrm{surv}}$ \\
  \midrule
    \multirow{4}{*}{VGG-16} & clean-intraclass   & $0.513 \pm 0.071$ & $1.000 \pm 0.000$ & $1.000 \pm 0.000$ & $1.000 \pm 0.000$ \\
                               & clean-interclass   & $0.607 \pm 0.032$ & $1.000 \pm 0.000$ & $1.000 \pm 0.000$ & $1.000 \pm 0.000$ \\
                               & clean-vs-random    & $0.663 \pm 0.020$ & $1.000 \pm 0.000$ & $1.000 \pm 0.000$ & $1.000 \pm 0.000$ \\
                               & random-vs-random   & $0.320 \pm 0.019$ & $1.000 \pm 0.000$ & $1.000 \pm 0.000$ & $1.000 \pm 0.000$ \\
    \midrule
    \multirow{4}{*}{ResNet-50} & clean-intraclass   & $0.272 \pm 0.056$ & $0.999 \pm 0.000$ & $0.786 \pm 0.044$ & $1.000 \pm 0.000$ \\
                               & clean-interclass   & $0.268 \pm 0.045$ & $1.000 \pm 0.000$ & $0.902 \pm 0.033$ & $1.000 \pm 0.000$ \\
                               & clean-vs-random    & $0.269 \pm 0.035$ & $1.000 \pm 0.000$ & $0.943 \pm 0.010$ & $1.000 \pm 0.000$ \\
                               & random-vs-random   & $0.152 \pm 0.008$ & $0.992 \pm 0.000$ & $0.739 \pm 0.011$ & $0.999 \pm 0.000$ \\
    \midrule
    \multirow{4}{*}{ViT-B/16} & clean-intraclass   & $0.422 \pm 0.023$ & $0.969 \pm 0.008$ & $0.789 \pm 0.067$ & $1.000 \pm 0.000$ \\
                               & clean-interclass   & $0.435 \pm 0.017$ & $0.974 \pm 0.005$ & $0.882 \pm 0.010$ & $1.000 \pm 0.000$ \\
                               & clean-vs-random    & $0.439 \pm 0.014$ & $0.963 \pm 0.006$ & $0.891 \pm 0.007$ & $1.000 \pm 0.000$ \\
                               & random-vs-random   & $0.377 \pm 0.013$ & $0.867 \pm 0.013$ & $0.476 \pm 0.018$ & $0.997 \pm 0.000$ \\
  \bottomrule
\end{tabular}

    \caption{Reference baselines for the input-noise sweep, averaged over $R$ pairs drawn from the
    $565$-image ImageNet-S test split (clean rows: $R=200$ sampled pairs; clean-vs-random: $R=565$
    images paired with a fresh Gaussian each; random-vs-random: $R=200$ Gaussian pairs); entries are
    mean $\pm$ standard deviation over pairs. \textbf{clean-intraclass}: $\mathrm{H}(\Pi_A(x_1),
        \Pi_A(x_2))$ with $x_1, x_2$ from the same ground-truth class; \textbf{clean-interclass}: same,
    with distinct classes; \textbf{clean-vs-random}: one side clean, other Gaussian;
    \textbf{random-vs-random}: both sides Gaussian.}
    \label{tab:hellinger-input-random}
\end{table}
\newpage
\section{Randomization Sanity Checks}\label{app:sanity-check}

This appendix collects the theoretical and empirical material behind the cascading-randomization
experiments of Section~\ref{sec:hellinger-sanity}.  \S\ref{sec:lmp-randomization} carries the
protocol of~\citet{Adebayo2018sanity} over to the trajectory-space Hellinger diagnostic of
Appendix~\ref{app:path-divergence}; \S\ref{app:sanity-heatmap} is the pixel-level
heatmap-similarity sanity check on the main-paper attribution methods;
\S\ref{sec:lmp-rand-double-random} gives the fully random reference;
\S\ref{sec:per-image-orthogonality} compares trajectory and map diagnostics per image.

\subsection{Cascading Randomization in Trajectory Space}\label{sec:lmp-randomization}

We apply the Hellinger framework of Appendix~\ref{app:path-divergence} to the parameter
randomization test of~\citet{Adebayo2018sanity}.  Replacing one weight matrix yields two kernels on
the same layered graph, and the Hellinger pass quantifies the resulting trajectory-space change.
Notation follows Appendix~\ref{app:path-divergence}: $\widetilde{T}^{(l)}$ is the geometric-mean
kernel and $g^{(l)}(s)$ the overlap mass reaching state $s$, so $\mathrm{BC}^{(l)} = \sum_{s \in
    \mathcal{S}^{(l)}_+} g^{(l)}(s)$; $\mathrm{BC}^{(l)}_{\mathrm{ord}}$ denotes the same sum over
ordinary states $s \in \mathcal{S}^{(l)}$ only.

\begin{lemma}[Alive-set mismatch]\label{lem:alive-mismatch}
For backward kernels $T^{(l)}_A, T^{(l)}_B$ governed by alive sets via $T^{(l)}_M(s_i \to \dagger)
    = \1\{i \notin \mathrm{alive}^{(l)}_M\}$, every state in the alive symmetric difference satisfies
$\sum_{s'}\widetilde{T}^{(l)}(s_i, s') = 0$ and contributes nothing to $g^{(l-1)}$.
\end{lemma}

\begin{proof}
WLOG $i \in \mathrm{alive}^{(l)}_A \setminus \mathrm{alive}^{(l)}_B$, so $T^{(l)}_A(s_i, \cdot)$ is
supported on ordinary states and $T^{(l)}_B(s_i, \cdot) = \delta_\dagger$.  Then
$\widetilde{T}^{(l)}(s_i, s') = \sqrt{T^{(l)}_A(s_i, s') \cdot 0} = 0$ for ordinary $s'$ and
$\widetilde{T}^{(l)}(s_i, \dagger) = \sqrt{0 \cdot 1} = 0$.
\end{proof}

Fix a ReLU MLP with hidden widths $n_1,\ldots,n_L$, an input $x$, and start state $s_0$ at the
target output.  Let $B$ replace $W^{(l^*)}$ by an i.i.d.\ matrix from a continuous distribution
symmetric about zero, and assume that for each $l \ge l^*$ and every neuron $i$,
$\mathbb{P}_B\bigl(i \in \mathrm{alive}^{(l)}_B \mid \text{kernels above layer } l\bigr) \le
    1-\delta$ for some $\delta \in (0,\tfrac12]$.  Then $T_A^{(l)} = T_B^{(l)}$ for $l < l^*$; at $l =
    l^*$ the row weights differ and the alive sets may differ; for $l > l^*$ the perturbation
propagates through the activations.  The same reduction applies to convolutional and transformer
blocks interpreted as single \ac{LMP} layers.  Write $\eta \coloneqq 1 - \delta \in
    \bigl[\tfrac{1}{2}, 1\bigr)$ and $N \coloneqq L - l^* + 1$.

\begin{theorem}[Randomization in the Stopping Game]\label{thm:rand-adf}
$\mathbb{E}_B\bigl[\mathrm{BC}^{(0)}_{\mathrm{ord}}(\Pi_A, \Pi_B)\bigr] \leq \eta^{N}$.  Hence, for
every $z > 0$, the conditioned-survival distance of Theorem~\ref{thm:conditioned-survival}
satisfies $\mathbb{E}_B\bigl[\mathrm{H}^2_{\mathrm{surv}}\,\1\{Z_A Z_B \geq z\}\bigr] \geq
    \mathbb{P}_B(Z_A Z_B \geq z) - \eta^{N}/\sqrt{z}$, while the plain distance is $\mathrm{H}^2 = 1 -
    \mathrm{BC}^{(0)}_{\mathrm{ord}} - g^{(0)}(\dagger)$.
\end{theorem}

\begin{proof}
At a layer $l \geq l^*$, ordinary overlap mass at a state dead in both models moves into the common
cemetery, mass on the alive symmetric difference vanishes (Lemma~\ref{lem:alive-mismatch}), and on
states alive in both models the row mass is at most $1$ by Cauchy--Schwarz.  Hence
$\mathrm{BC}^{(l-1)}_{\mathrm{ord}} \leq \sum_{s \in \mathcal{S}^{(l)}} g^{(l)}(s)\,\1\{s \in
    \mathrm{alive}^{(l)}_B\}$, and, since $g^{(l)}$ is fixed by the kernels above layer~$l$, the
conditional bound gives $\mathbb{E}_B[\mathrm{BC}^{(l-1)}_{\mathrm{ord}}] \leq
    \eta\,\mathbb{E}_B[\mathrm{BC}^{(l)}_{\mathrm{ord}}]$.  Iterating from
$\mathrm{BC}^{(L)}_{\mathrm{ord}} = 1$ over $l = L, \ldots, l^*$ and using that
$\mathrm{BC}^{(l)}_{\mathrm{ord}}$ is non-increasing below $l^*$ yields the bound.  The
$\mathrm{H}_{\mathrm{surv}}$ statement follows from $\mathrm{H}^2_{\mathrm{surv}} = 1 -
    \mathrm{BC}^{(0)}_{\mathrm{ord}}/\sqrt{Z_A Z_B} \geq 1 - \mathrm{BC}^{(0)}_{\mathrm{ord}}/\sqrt{z}$
on $\{Z_A Z_B \geq z\}$ (Theorem~\ref{thm:conditioned-survival}), the plain one from
$\mathrm{BC}^{(0)} = \mathrm{BC}^{(0)}_{\mathrm{ord}} + g^{(0)}(\dagger)$.
\end{proof}

\begin{remark}[Cemetery floor]\label{rem:adf-cemetery-floor}
Theorem~\ref{thm:rand-adf} bounds only the ordinary-state overlap.  Overlap mass that reaches a
neuron dead in both models enters the cemetery in both, $g^{(l-1)}(\dagger) = g^{(l)}(\dagger) +
    \sum_{s \in \mathrm{dead}^{(l)}_A \cap \mathrm{dead}^{(l)}_B} g^{(l)}(s)$, and is never thinned
again, so plain $\mathrm{H}$ saturates below one.  In a mean-field approximation with uniform dead
fraction $q$ under $A$ and alive probability $p$ under $B$, the ordinary overlap shrinks by
$(1-q)p$ per layer while a fraction $q(1-p)$ is absorbed, so $g^{(0)}(\dagger) \to
    q(1-p)/(1-(1-q)p)$ as $N \to \infty$; for illustrative values $q = 0.2$, $p = \tfrac{1}{2}$ this
gives $\mathrm{H} \approx 0.91$, the same order as the observed $\drFloorMin$--$\drFloorMax$
(Table~\ref{tab:double-random-hellinger}).  The conditioned-survival distance removes the cemetery
floor exactly, and $\mathbb{E}_B[\mathrm{H}^2_{\mathrm{surv}} \mid Z_A Z_B \geq z] \to 1$ for every
fixed $z > 0$ with $\mathbb{P}_B(Z_A Z_B \geq z)$ bounded away from zero.  For the Routing Game the
activation-weighted kernel concentrates $g^{(l)}$ on neurons with largest $|a W|$, which rarely
coincide with $\mathrm{dead}_A \cap \mathrm{dead}_B$, so $\mathrm{H}$ is nearly identical to
$\mathrm{H}_{\mathrm{surv}}$ in practice.
\end{remark}

\begin{theorem}[Randomization in the Routing Game]\label{thm:rand-routing}
Under the same setup without the alive-set assumption, take the \ac{RG} kernel at temperature $\tau
    > 0$ and assume that for every $l \geq l^*$ each state alive in both models has row mass
$\sum_{s'}\widetilde{T}^{(l)}(s, s') \leq \bar{b} < 1$.  Then
$\mathrm{BC}^{(0)}_{\mathrm{ord}}(\Pi_A, \Pi_B) \leq \bar{b}^{N}$.
\end{theorem}

\begin{proof}
Since $a^{(l-1)}_{M,i} = 0$ for $i \notin \mathrm{alive}^{(l-1)}_M$, the Gibbs
policy~\eqref{eq:rg-T} sends no mass to dead neurons, so for $1 \le l < L$ the overlap $g^{(l)}$
vanishes outside $\mathrm{alive}^{(l)}_A \cap \mathrm{alive}^{(l)}_B$, and the start state is alive
in both, in the symmetric difference (Lemma~\ref{lem:alive-mismatch}) or dead in both.  Hence
$\mathrm{BC}^{(l-1)}_{\mathrm{ord}} \leq \bar{b}\,\mathrm{BC}^{(l)}_{\mathrm{ord}}$ for every $l
    \geq l^*$, with no alive-set factor (the alive mismatch at layer $l-1$ is already part of the row
masses at layer $l$); iterate as in Theorem~\ref{thm:rand-adf}.  The assumption reflects the
structure of the \ac{RG}: each act-to-act row factors as the source-deadness indicator
$\1\{z_j^{(l)} > 0\}$, the constant $\tfrac12$ uniform $\kappa$-branch, and the Gibbs policy
$\pi_{M,\kappa}^{(l)}(\,\cdot\mid j) \propto [a^{(l-1)}_M[\,\cdot\,]\,W^{(l,\kappa)}_M[\,\cdot,
    j]]^{1/\tau}$.  At $l = l^*$ the two policies share $a^{(l^*-1)}$, but $W^{(l^*)}_B$ drawn from a
continuous distribution makes the normalized Gibbs rows differ a.s.; for $l > l^*$ the activation
factors differ as well.  By the equality case of Cauchy--Schwarz every such row mass is then $< 1$
a.s., whereas \ac{SG} rows on states alive in both models coincide for $l > l^*$.
\end{proof}

The input-layer marginal $m^{(0)}_M$ recovers gradient and $\alpha\beta$-\ac{LRP} maps
(Theorems~\ref{thm:deriv-occ} and~\ref{thm:ab-recovery}).
Theorems~\ref{thm:rand-adf}--\ref{thm:rand-routing} show that $\Pi_M$ changes drastically under
randomization, while Table~\ref{tab:sanity-check} shows that $m^{(0)}_M$ can remain similar. \ch{The two
diagnostics therefore detect different aspects of parameter sensitivity: the map-level test
of~\citet{Adebayo2018sanity} asks whether the explanation delivered to the user depends on the
learned parameters, the trajectory-level test whether the calculation behind it does. A map can fail
the first while its calculation passes the second; the trajectory-level result locates where the
sensitivity is lost (between the calculation and its pixel projection), but it does not make such a
map adequate for model debugging~\citep{Adebayo2018sanity, sixt2020whenexplanationslie}.}

\subsection{Heatmap similarity at full randomization}\label{app:sanity-heatmap}

We apply the cascading parameter randomization test of \citet{Adebayo2018sanity} to the attribution
methods of Table~\ref{tab:sanity-check}. Starting from a pretrained network, we cumulatively
re-initialize the convolutional and linear layers from the output toward the input (weights Kaiming
normal~\citep{He2015delving}, biases $\mathcal{N}(0, 0.01^2)$); BatchNorm (ResNet-50) and
LayerNorm, the packed query/key/value projection of the attention, the positional embedding and the
class token (ViT-B/16) keep their pretrained values, here and in the Hellinger cascade of
Section~\ref{sec:hellinger-cascading}. At each randomization depth we compute the attribution for
every test image under both the original and the (partially) randomized model, then measure the
similarity between the two attribution maps. A method that genuinely reflects the learned
parameters should produce increasingly dissimilar maps as more layers are randomized; a method that
merely responds to low-level input structure (edges, textures) will show high similarity even at
full randomization.

\paragraph{Metrics.}
We report two complementary similarity measures between attribution maps:
\begin{itemize}
    \item \textbf{Spearman $\rho(|a|)$}: The Spearman rank correlation of the absolute attribution values,
        measuring whether the two maps rank pixels in the same order of importance. We use absolute values
        because sign conventions may differ between original and randomized attributions; the correlation
        itself is signed.
    \item \textbf{HOG $r$}: The Pearson correlation between Histograms of Oriented Gradients (HOG)
        descriptors~\citep{Dalal2005hog} extracted from the attribution maps, capturing structural
        similarity at the level of local gradient orientations.
\end{itemize}
We intentionally do \emph{not} report SSIM here: \citet{Binder2022shortcomings} show that
randomization-based similarity scores can remain deceptively high because input-dependent structure
and activation scales survive top-down randomization.
All values are averaged over the 565-image test split of ImageNet-S (50 classes); a map that is
constant over the image (LIME and KernelSHAP on the fully randomized ResNet-50, whose output no
longer depends on the superpixels) has similarity $0$ by convention. Table~\ref{tab:sanity-check}
reports the similarity at \emph{full randomization} (all convolutional and linear layers
re-initialized).

\paragraph{Discussion.}
Several patterns emerge. First, Gradient and Grad-CAM stay closest to zero on the \acp{CNN} (and
LayerCAM on VGG-16), and Gradient, SmoothGrad and \ac{GAE} on ViT-B/16: they are sensitive to the
parameters, though this does not imply that they are faithful. Second, LRP-$\varepsilon$ keeps the
most rank structure on VGG-16 ($\sanityVggLRPepsRho$); on ResNet-50 the \ac{LRP} rules
(LRP-$\gamma$ $\sanityResLRPgammaRho$, LRP-$\varepsilon$ $\sanityResLRPepsRho$) are close to
Integrated Gradients and LayerCAM, which keep the most there. Third, LIME and KernelSHAP keep a HOG
correlation of about $0.86$ on VGG-16 and ViT-B/16 at low Spearman correlation: their superpixels
are computed from the input image, so the maps inherit its edge structure; on ResNet-50 the fully
randomized model's output no longer depends on the superpixels and their maps are constant
($^{\ddagger}$). Fourth, the \ac{RG} rows lie between Gradient and LRP-$\varepsilon$ on VGG-16,
slightly below the \ac{LRP} rules on ResNet-50, and keep the highest Spearman correlation of all
methods on ViT-B/16 ($\sanityVitRGanchorRho$ for the anchor, $\sanityVitRGlocRho$ for the selected
mode). Two maps become anti-correlated rather than unrelated (Grad-CAM++ on ResNet-50, \ac{TIBAV}
on ViT-B/16). The complementary trajectory-space analysis of \S\ref{sec:lmp-randomization} measures
divergence of the underlying trajectory distributions rather than pixel-level map similarity.

\begin{table}[H]
\centering
\small
\caption{Parameter randomization sanity check~\citep{Adebayo2018sanity}. Each cell reports the average similarity between the attribution of the original pretrained model and the attribution of a fully randomized copy (every convolutional and linear layer re-initialized; normalization layers, the attention's query/key/value projection and the ViT embeddings stay pretrained). Values near zero indicate that the method detects the parameter change; the Spearman correlation is signed, so negative values occur. Bold marks the value closest to zero per metric within each architecture, among maps that are not constant; $^{\ddagger}$: the map is constant at full randomization (similarity $0$ by convention). Rows: the reported configuration of every method (Section~\ref{sec:experiments}).}
\label{tab:sanity-check}
\begin{tabular}{l c c}
\toprule
Method & Spear.\ $\rho(|a|)$ & HOG $r$ \\
\midrule
\multicolumn{3}{l}{\textbf{VGG-16}} \\
\midrule
Gradient & 0.013 & \textbf{0.009} \\
SmGrad & 0.084 & 0.101 \\
IntGrad & 0.207 & 0.348 \\
DeepLift & 0.216 & 0.340 \\
LRP-$\varepsilon$ & 0.341 & 0.457 \\
LRP-$\gamma$ & 0.092 & 0.370 \\
GCAM & \textbf{-0.007} & 0.069 \\
GCAM++ & 0.180 & 0.129 \\
LCAM & -0.011 & 0.052 \\
LIME & 0.257 & 0.864 \\
KernelSHAP & 0.075 & 0.868 \\
SG ($\lambda{=}{-}1,\sigma^2{=}5$, gate on) & 0.289 & 0.456 \\
RG anchor & 0.299 & 0.535 \\
RG ($\tau{=}0.3,\varepsilon{=}0.1$) & 0.284 & 0.287 \\
\midrule
\multicolumn{3}{l}{\textbf{ResNet-50}} \\
\midrule
Gradient & 0.148 & 0.134 \\
SmGrad & 0.233 & 0.214 \\
IntGrad & 0.277 & 0.424 \\
DeepLift & 0.224 & 0.351 \\
LRP-$\varepsilon$ & 0.250 & 0.393 \\
LRP-$\gamma$ & 0.254 & 0.130 \\
GCAM & \textbf{0.020} & \textbf{0.072} \\
GCAM++ & -0.502 & 0.151 \\
LCAM & 0.266 & 0.145 \\
LIME$^{\ddagger}$ & 0.000 & 0.000 \\
KernelSHAP$^{\ddagger}$ & 0.000 & 0.000 \\
RG anchor & 0.217 & 0.344 \\
RG ($\tau{=}0.9,\varepsilon{=}1$) & 0.180 & 0.316 \\
\midrule
\multicolumn{3}{l}{\textbf{ViT-B/16}} \\
\midrule
Gradient & \textbf{-0.003} & 0.041 \\
SmGrad & 0.010 & \textbf{0.038} \\
IntGrad & 0.272 & 0.362 \\
LRP-$\varepsilon$ & 0.165 & 0.300 \\
LRP-$\gamma$ & 0.027 & 0.234 \\
AttnLRP & 0.124 & 0.176 \\
DAVE & 0.178 & 0.323 \\
TIBAV & -0.175 & 0.150 \\
GAE & -0.009 & 0.130 \\
AR & 0.120 & 0.153 \\
LIME & 0.248 & 0.861 \\
KernelSHAP & 0.131 & 0.867 \\
RG anchor & 0.454 & 0.469 \\
RG ($\tau{=}0.4,\lambda_{\mathrm{sm}}{=}{-}10,\sigma^2{=}5,\varepsilon{=}0.1$) & 0.450 & 0.310 \\
\bottomrule
\end{tabular}

\par\smallskip
{\footnotesize Spear.\ $\rho(|a|)$ = Spearman rank correlation of absolute attribution values; HOG $r$ = Pearson correlation of HOG descriptors. RG anchor = the $\alpha\beta$-\ac{LRP} representation ($\tau{=}1$, $\varepsilon{=}0.5$); the second \ac{RG} row and the \ac{SG} row are the localization-selected configurations and list the modes that differ ($\tau$ = temperature; $\lambda$, $\lambda_{\mathrm{sm}}$ = risk aversion; $\sigma^2$ = \ac{ADF} input variance; $\varepsilon$ = \ac{LRP} stabilizer; $\tilde\Phi$ = noisy-observation gate).}
\end{table}

\subsection{Double-Random Empirical Baseline}\label{sec:lmp-rand-double-random}

To calibrate the asymptotic value that Theorems~\ref{thm:rand-adf}--\ref{thm:rand-routing} approach
under \emph{full} randomization, we compare two \emph{independently random-initialized} copies of
each architecture.  This endpoint matches the deepest step of the cascading test
of~\citet{Adebayo2018sanity} while eliminating any residual correlation with the pretrained
weights.

\paragraph{Protocol.}
For each architecture we run $R = 500$ independent rounds.  Each round: \emph{(i)} instantiate two
copies $A_r, B_r$ of the pretrained checkpoint and re-initialize every
\texttt{Conv2d}/\texttt{Linear} layer of each copy independently as in the cascade, matching its
endpoint; the remaining parameters (normalization layers, the attention's packed query/key/value
projection, positional embedding and class token) keep their pretrained values in both copies
(starting from the checkpoint rather than from \texttt{weights=None}, whose ViT classification head
is identically zero and would make the two trajectory laws trivially indistinguishable; the same
construction is used for all three architectures).  \emph{(ii)} Draw one image index with
replacement from the $565$-image ImageNet-S test split (the $500$ draws hit $\drDistinctImages$
distinct images under the fixed seed we use).  \emph{(iii)} Compute the plain Hellinger
$\mathrm{H}(\Pi_{A_r}, \Pi_{B_r})$ and conditioned-survival $\mathrm{H}_{\mathrm{surv}}$ for both
the Stopping Game and Routing Game kernels using the ground-truth label as target class.
Table~\ref{tab:double-random-hellinger} reports the round-to-round mean and standard deviation; the
round-to-round variance mixes two sources (image choice and the random-pair realization).

\begin{table}[H]
    \centering
    \small
    \begin{tabular}{lcccc}
  \toprule
  & \multicolumn{2}{c}{Stopping Game} & \multicolumn{2}{c}{Routing Game} \\
  \cmidrule(lr){2-3} \cmidrule(lr){4-5}
  Architecture & $\mathrm{H}$ & $\mathrm{H}_{\mathrm{surv}}$ & $\mathrm{H}$ & $\mathrm{H}_{\mathrm{surv}}$ \\
  \midrule
    VGG-16 & $0.939 \pm 0.003$ & $1.000 \pm 0.000$ & $1.000 \pm 0.000$ & $1.000 \pm 0.000$ \\
    ResNet-50 & $0.938 \pm 0.003$ & $1.000 \pm 0.000$ & $0.973 \pm 0.002$ & $1.000 \pm 0.000$ \\
    ViT-B/16 & $0.958 \pm 0.002$ & $1.000 \pm 0.000$ & $0.956 \pm 0.001$ & $1.000 \pm 0.000$ \\
  \bottomrule
\end{tabular}

    \caption{Hellinger trajectory distance between two independently random-initialized copies of the
    same architecture, averaged over $R = 500$ rounds (one fresh ImageNet-S test image per round).
    Plain $\mathrm{H}$ concentrates at an architecture-dependent cemetery floor of
    $\drFloorMin$--$\drFloorMax$; conditioned-survival $\mathrm{H}_{\mathrm{surv}}$ saturates at unity.
    The VGG-16 Routing Game cell is the one case in which the plain distance also saturates at
    $1.000$.}
    \label{tab:double-random-hellinger}
\end{table}

\paragraph{Interpretation.}
Plain $\mathrm{H}$ concentrates at a model-dependent floor of $\drFloorMin$--$\drFloorMax$,
matching the cemetery-floor prediction of Remark~\ref{rem:adf-cemetery-floor}: once both random
kernels produce uncorrelated dead masks, the shared dead-in-both coincidence rate bounds
$\mathrm{BC}^{(0)}$ below by a strictly positive constant and plain $\mathrm{H}$ cannot reach
unity.  The conditioned-survival distance removes the cemetery floor by renormalizing on survival
and correctly reports the two independent random trajectory laws as essentially support-disjoint on
ordinary input states ($\mathrm{H}_{\mathrm{surv}} \approx 1$ throughout).  The
pretrained-vs-randomized Hellinger distance that the cascading sanity check reports should
therefore be read against this architecture-specific asymptote rather than against unity: values
close to the double-random baseline already indicate that the game-theoretic trajectory law has
become fully decoupled from the pretrained one.

\subsection{Per-Image Trajectory-vs-Map Independence}\label{sec:per-image-orthogonality}

Theorems~\ref{thm:rand-adf}--\ref{thm:rand-routing} and the double-random baseline of
Table~\ref{tab:double-random-hellinger} bound the \emph{aggregate} behavior of the Hellinger
trajectory distance and the pixel-level attribution-map similarity.  A sharper version of the
argument of \S\ref{sec:lmp-randomization} is the claim that the two diagnostics disagree \emph{per
image}: at fixed randomization depth, whether a given image's trajectory distribution is strongly
displaced by the randomization need not predict whether its attribution map becomes more
dissimilar.  Table~\ref{tab:per-image-corr} tests this directly.

\paragraph{Protocol.}
For each architecture we pick a cascading step at which the Hellinger run and the
attribution-sanity run have randomized the same part of the network, so per-image values are
directly comparable: the layers down to \texttt{features.26} on VGG-16 and down to
\texttt{layer4.1.conv3} on ResNet-50 (identical sets of re-initialized layers), and the last
encoder block on ViT-B/16, where the Hellinger cascade re-initializes the MLP and the output
projection of the block and the attribution cascade additionally its $Q$, $K$ and $V$ projections.
The mean $\mathrm{H}$ / $\mathrm{H}_{\mathrm{surv}}$ at these steps can be read off
Figure~\ref{fig:hellinger-sanity}.  For each of the $\perImageN$ ImageNet-S test images we then
record four Hellinger scalars ($\mathrm{H}$ plain/surv for the Stopping and Routing Games) and four
attribution-map scalars (Spearman $\rho(|a|)$ and HOG, each computed between the pretrained and
partially-randomized model, for Gradient and for the $\alpha\beta$-\ac{RG} anchor at $\tau{=}1$;
these are the attribution methods whose marginal laws line up with the Hellinger game kernels).
Table~\ref{tab:per-image-corr} reports Spearman rank correlations across the $\perImageN$ images.

\begin{table}[H]
    \centering
    \scriptsize
    \setlength{\tabcolsep}{3.5pt}
    \begin{tabular}{ll cccc cccc}
  \toprule
  & & \multicolumn{4}{c}{Within Hellinger} & \multicolumn{2}{c}{Gradient} & \multicolumn{2}{c}{$\mathrm{RG}$ anchor} \\
  \cmidrule(lr){3-6} \cmidrule(lr){7-8} \cmidrule(lr){9-10}
  Architecture & Hellinger variant
    & $\mathrm{H}_{\mathrm{RG,pl}}$ & $\mathrm{H}_{\mathrm{RG,surv}}$
    & $\mathrm{H}_{\mathrm{SG,pl}}$ & $\mathrm{H}_{\mathrm{SG,surv}}$
    & Spear.\,$\rho(|a|)$ & HOG\,$r$ & Spear.\,$\rho(|a|)$ & HOG\,$r$ \\
  \midrule
    \multirow{4}{*}{VGG-16} & $\mathrm{H}_{\mathrm{RG,plain}}$ & $-$     & $+1.00$ & $-0.33$ & $+0.85$ & $+0.07$ & $+0.10$ & $-0.01$ & $+0.04$ \\
                             & $\mathrm{H}_{\mathrm{RG,surv}}$ & $+1.00$ & $-$     & $-0.33$ & $+0.85$ & $+0.07$ & $+0.10$ & $-0.01$ & $+0.04$ \\
                             & $\mathrm{H}_{\mathrm{SG,plain}}$ & $-0.33$ & $-0.33$ & $-$     & $-0.30$ & $+0.12$ & $+0.11$ & $+0.00$ & $-0.02$ \\
                             & $\mathrm{H}_{\mathrm{SG,surv}}$ & $+0.85$ & $+0.85$ & $-0.30$ & $-$     & $+0.01$ & $+0.05$ & $-0.10$ & $-0.10$ \\
    \midrule
    \multirow{4}{*}{ResNet-50} & $\mathrm{H}_{\mathrm{RG,plain}}$ & $-$     & $+0.58$ & $-0.08$ & $+0.70$ & $+0.05$ & $-0.01$ & $-0.08$ & $-0.19$ \\
                             & $\mathrm{H}_{\mathrm{RG,surv}}$ & $+0.58$ & $-$     & $+0.16$ & $+0.89$ & $+0.15$ & $+0.08$ & $+0.05$ & $-0.05$ \\
                             & $\mathrm{H}_{\mathrm{SG,plain}}$ & $-0.08$ & $+0.16$ & $-$     & $+0.16$ & $-0.04$ & $+0.07$ & $-0.01$ & $+0.02$ \\
                             & $\mathrm{H}_{\mathrm{SG,surv}}$ & $+0.70$ & $+0.89$ & $+0.16$ & $-$     & $+0.16$ & $+0.07$ & $-0.01$ & $-0.10$ \\
    \midrule
    \multirow{4}{*}{ViT-B/16} & $\mathrm{H}_{\mathrm{RG,plain}}$ & $-$     & $+0.96$ & $+0.29$ & $+0.30$ & $+0.02$ & $+0.00$ & $-0.04$ & $-0.08$ \\
                             & $\mathrm{H}_{\mathrm{RG,surv}}$ & $+0.96$ & $-$     & $+0.25$ & $+0.28$ & $+0.02$ & $+0.02$ & $-0.07$ & $-0.07$ \\
                             & $\mathrm{H}_{\mathrm{SG,plain}}$ & $+0.29$ & $+0.25$ & $-$     & $+0.99$ & $-0.12$ & $-0.06$ & $-0.03$ & $-0.05$ \\
                             & $\mathrm{H}_{\mathrm{SG,surv}}$ & $+0.30$ & $+0.28$ & $+0.99$ & $-$     & $-0.12$ & $-0.06$ & $-0.04$ & $-0.05$ \\
  \bottomrule
\end{tabular}

    \caption{Per-image Spearman rank correlations across the $\perImageN$-image ImageNet-S test split
    at the aligned cascading steps of the Hellinger and attribution-sanity pipelines.  Left block:
    within-Hellinger cross-correlations of the plain ($\mathrm{H}_{\cdot,\mathrm{pl}}$) and
    survival-conditioned ($\mathrm{H}_{\cdot,\mathrm{surv}}$) distances under the \ac{RG} and \ac{SG}
    kernels (self-entries shown as $-$; trivially $+1$).  Right block: Hellinger-vs-map
    cross-correlations.  The Hellinger-vs-map block is uniformly $|\rho| \leq \perImageCrossMax$:
    per-image Hellinger and per-image attribution-map similarity rank images essentially
    independently.}
    \label{tab:per-image-corr}
\end{table}

\paragraph{Readings.}
The Hellinger-vs-map cross-block in Table~\ref{tab:per-image-corr} is uniformly near zero ($|\rho|
    \leq \perImageCrossMax$ on all three architectures): per-image trajectory displacement and
per-image attribution-map drift rank images essentially independently.

The within-Hellinger block shows that per-image $\mathrm{H}$ carries a stable ordering of the
images. Where little trajectory mass reaches the cemetery, the plain and conditioned distances of
the same game rank images almost identically: $\rho = +0.96$ (\ac{RG}) and $+0.99$ (\ac{SG}) on
ViT-B/16, and identical values for the \ac{RG} on VGG-16, whose kernel routes no trajectories to
states of dead neurons. Under the \ac{SG} on the \ac{ReLU} backbones the two normalizations
decouple ($\rho = -0.30$ on VGG-16, $+0.16$ on ResNet-50): the cemetery share of plain $\mathrm{H}$
varies with the image's dead masks, and the renormalization by $1/\sqrt{Z_A Z_B}$ removes exactly
this share. The conditioned distances of the two games agree strongly on the \ac{ReLU} backbones
($\rho = +0.85$ on VGG-16, $+0.89$ on ResNet-50) and only weakly on ViT-B/16 ($+0.28$), where the
two kernels read partly different per-image structure.

A secondary observation from the within-method pattern: Spearman $\rho(|a|)$ and HOG within a
single attribution method agree well ($\rho = 0.65$--$0.82$, computed from the same per-image
joins), but Gradient and $\alpha\beta$-\ac{RG} similarities barely agree across images ($\rho$
between $-0.10$ and $+0.30$ depending on architecture and similarity metric).  Together with the
near-zero Hellinger-vs-map block, this tightens the conclusion of \S\ref{sec:lmp-randomization}:
trajectory-space and pixel-level diagnostics look at orthogonal per-image structure, and within
pixel-level diagnostics the choice of attribution method changes \emph{which} images a
randomization sanity check flags.
\newpage
\providecommand{\nValGrid}{\textcolor{red}{??}}%

\section{Evaluation Protocol}\label{app:protocol}

\subsection{Data and splits}

We evaluate on a \emph{custom} 50-class subset of ImageNet-S~\citep{gao2022luss}, sampled once at
the start of the project and held fixed across all experiments. The subset has two partitions: a
\textbf{custom validation split of 50 images} (one image per class, seed $42$) used only for
configuration selection, and the \textbf{test split} of the remaining images, used for every
reported test number. Images whose object mask lies entirely outside the $224{\times}224$ center
crop are excluded after the split, because the mask-based metrics are undefined for them; this
removes one test image and leaves \textbf{565 test images}. The Pascal VOC 2012 transfer split is
described in Appendix~\ref{app:voc-cross-data}.

\paragraph{Backbone and dataset provenance.}
All main-paper experiments use the pretrained torchvision checkpoints shipped with the library
release used for evaluation. Specifically, VGG-16 uses \texttt{VGG16\_Weights.IMAGENET1K\_V1}
(\url{https://download.pytorch.org/models/vgg16-397923af.pth}; recipe
\url{https://github.com/pytorch/vision/tree/main/references/classification#alexnet-and-vgg}),
ResNet-50 uses \texttt{ResNet50\_Weights.IMAGENET1K\_V2}
(\url{https://download.pytorch.org/models/resnet50-11ad3fa6.pth}; recipe
\url{https://github.com/pytorch/vision/issues/3995#issuecomment-1013906621}), and ViT-B/16 uses
\texttt{ViT\_B\_16\_Weights.IMAGENET1K\_V1}
(\url{https://download.pytorch.org/models/vit_b_16-c867db91.pth}; recipe
\url{https://github.com/pytorch/vision/tree/main/references/classification#vit_b_16}). These are
ImageNet-1K classifiers trained on the ImageNet / ILSVRC
benchmark~\citep{Deng2009imagenet,Russakovsky2015ilsvrc}. Quantitative scoring is performed with
Quantus~\citep{hedstrom2023quantus}.

\paragraph{Licenses.}\label{app:licences}
ImageNet images are used under the ImageNet terms of access (non-commercial research and
educational use)~\citep{Deng2009imagenet}. The ImageNet-S masks~\citep{gao2022luss} are released by
their authors without an explicit license; their images are ImageNet images and fall under the same
terms. The Pascal VOC 2012 images~\citep{Everingham2010voc} are subject to the Flickr terms of use,
as stated by the dataset providers. The checkpoints are distributed with torchvision (BSD-3-Clause)
and were trained on ImageNet-1K. Libraries: PyTorch and torchvision (BSD-3-Clause), Quantus
(GPL-3.0), zennit (GPL-3.0/LGPL-3.0), lxt, captum and scikit-image (BSD-3-Clause), pytorch-grad-cam
(MIT); the reference code of \acs{TIBAV} and \acs{GAE} (MIT) is included with its license notice.

\subsection{Methods}\label{app:baselines}

\paragraph{Gradient-based block.}
\textbf{Gradient} is the raw input-gradient saliency map~\citep{simonyan2014deep}, the maximum over
color channels of the absolute gradient. \textbf{SmGrad} is SmoothGrad, i.e.\ an average over noisy
gradient maps~\citep{Smilkov2017}. \textbf{IntGrad} is Integrated
Gradients~\citep{sundararajan2017axiomatic} with a black-image baseline. \textbf{DeepLift}
propagates activation differences relative to a reference input with the Rescale
rule~\citep{Shrikumar2017} (CNNs only). \textbf{DAVE} is the distribution-aware ViT gradient
decomposition of~\citet{Wrobel2024}. \textbf{SG} is our \ac{SG}, included in this block because it
recovers the gradient as a Stopping-Game occupation measure (Theorem~\ref{thm:deriv-occ}).

\paragraph{LRP / relevance-propagation block.}
\textbf{LRP-$\varepsilon$} denotes the stabilized Layer-wise Relevance Propagation
rule~\citep{Bach2015}. \textbf{LRP-$\gamma$} is the positive-weight-amplified LRP variant used in
best-practice comparisons~\citep{Kohlbrenner2020}. \textbf{AttnLRP} is transformer LRP \ch{that
propagates relevance through the softmax and the bilinear attention product as well as the Value
path}~\citep{achtibat2024attnlrp}. \textbf{RG} is our \ac{RG}.
\textbf{RG+Eq} is the same method with equivariant Reynolds averaging, used mainly on ViT to
suppress tokenization-grid artifacts. \textbf{TIBAV}~\citep{Chefer2021transformer},
\textbf{GAE}~\citep{Chefer2021genatt} and Attention Rollout (\textbf{AR}) are the
transformer-specific attention-map baselines.

\paragraph{CAM block.}
\textbf{GCAM} is Grad-CAM~\citep{Selvaraju2017gradcam}, \textbf{GCAM++} is
Grad-CAM++~\citep{Chattopadhyay2018gradcampp}, and \textbf{LCAM} is LayerCAM, which weights
class-activation evidence per location~\citep{Jiang2021layercam}. All three are computed at the
last convolutional block of the CNNs and upsampled bilinearly.

\paragraph{Perturbation block.}
\textbf{LIME} is Local Interpretable Model-agnostic Explanations~\citep{Ribeiro2016lime}.
\textbf{KShap} denotes KernelSHAP, the model-agnostic weighted-regression estimator of SHAP
values~\citep{Lundberg2017}. Both perturb SLIC superpixels~\citep{Achanta2012slic}. These methods
explain predictions by perturbing the input and fitting a local surrogate, so they are separated
from the backward-propagation methods above.

\subsection{Metric definitions}\label{app:metrics}

Let $x \in \RR^{C \times H \times W}$ be an image, let $y$ be its target class, and let $a \in
    \RR^{H \times W}$ be the scalar attribution map produced by an explanation method after channel
reduction. Write $d \coloneqq HW$ and let $\text{vec}(a) \in \RR^d$ denote the flattened map. For
any binary segmentation mask $m \in \{0,1\}^{H \times W}$, write $\text{vec}(m) \in \{0,1\}^d$.
When a metric is configured with \texttt{abs=true} and \texttt{normalise=true}, the evaluator
applies the max-abs normalization (not min--max)
\begin{align}
    \psi(a) = \frac{|\text{vec}(a)|}{\max_i |\text{vec}(a)_i|},
\end{align}
with the convention that the all-zero map remains zero. We denote the preprocessed attribution
vector by $\bar{a}$. All metrics below use $\bar{a}$; faithfulness metrics use raw logits (no
softmax) to avoid saturation effects, and the perturbation baseline is the channel-wise mean.

\paragraph{Attribution localization (AL).} The inside-mass ratio of~\citet{Kohlbrenner2020}, who compute it with ground-truth bounding boxes;
we use the ImageNet-S pixel masks:
\begin{align}
    \text{AL}(a, m) = \frac{\sum_{i=1}^d \bar{a}_i m_i}{\sum_{i=1}^d \bar{a}_i + 10^{-9}}.
\end{align}

\paragraph{Pointing game (PG).} Following the Excitation Backprop protocol~\citep{Zhang2018excitation}, the pointing-game config
uses \texttt{abs=true} and no normalization. Let $M(a) \coloneqq \{i \in \{1, \dots, d\} :
    \bar{a}_i = \max_j \bar{a}_j\}$ be the entire set of maximizers. Then
\begin{align}
    \text{PG}(a, m) = \1\bigl\{M(a) \cap \{i : m_i = 1\} \neq \varnothing\bigr\}.
\end{align}
Ties are scored as a hit whenever at least one maximizer lies inside the object mask. PG is binary
per image, so its test mean is a multiple of $1/565$.

\paragraph{Top-$k$ intersection (TK).} Following~\citet{Theiner2022geolocation}, with $k = 1000$ and the permutation $\pi$ that sorts
$\bar{a}$ in descending order:
\begin{align}
    \text{TK}(a, m) = \frac{1}{1000} \sum_{r=1}^{1000} m_{\pi_r}.
\end{align}

\paragraph{Pixel-flipping AUC (PF$_5$, PF$_{20}$).} Pixel flipping~\citep{Bach2015, Samek2017evaluating} removes the most relevant features first
and tracks the class score. The evaluator uses \texttt{pixel\_flipping\_abs} with
\texttt{features\_in\_step = 1568} and \texttt{perturb\_baseline = mean}. For $T \coloneqq \lceil d
    / 1568 \rceil$, define perturbed inputs $x^{(t)}$ by replacing features $\pi_{(t-1) \cdot 1568 +
    1}, \dots, \pi_{\min\{t \cdot 1568, d\}}$ (and all earlier ones) with the per-sample mean baseline.
The evaluator computes
\begin{align}
    f_t \coloneqq F(x^{(t)})_y, \qquad t = 1, \dots, T,
\end{align}
and PF$_p$ is the trapezoidal area over the first $p\%$ of this flipping curve, i.e.\ over its
first $K_p$ points with the abscissa rescaled to $[0,1]$: $\text{PF}_p =
    \text{trapz}\bigl(\{(t-1)/(K_p-1), f_t\}_{t=1}^{K_p}\bigr)$, $K_p = \max(2, \lceil pT/100 \rceil)$;
with $T = 96$ steps of 1568 features (the attribution is broadcast over the three color channels),
$K_5 = 5$ and $K_{20} = 20$. Lower is better.

\paragraph{Selectivity (Sel).} Following~\citet{Montavon2018methods} and its Quantus implementation~\citep{hedstrom2023quantus},
$\bar{a}$ is summed over square patches of $39{\times}39$ pixels (side
$\lfloor\sqrt{1568}\rfloor$); non-overlapping patches are removed (replaced by the per-sample mean)
in decreasing order of their summed attribution, and the target logit is recorded after each
removal. Sel is the trapezoidal area under this curve, with the fraction of removed patches on the
abscissa. Lower is better.

\paragraph{Max and average sensitivity (MaxS, AvgS).} Following \citet{yeh2019fidelity}, the evaluator regenerates attributions on the same $N_s = 50$
inputs for both metrics, $x + \xi_k$ with uniform noise $\xi_k \sim \mathcal{U}[-r, r]^{C \times H
    \times W}$, $r = 0.2$, and reports the maximum (MaxS) or mean (AvgS) over $k$ of
$\|\psi(a(x+\xi_k)) - \psi(a(x))\|_F / \|\psi(a(x))\|_F$. Lower is better, as it indicates
stability of the explanation to input noise.
The noise is seeded per image, so every row sees the same perturbed inputs. MaxS and AvgS are
computed on a fixed random subset of 200 test images (seed 0) that is identical for all rows and
models; all other metrics use the full test split.

\subsection{Declared validation grids}\label{app:grids}

Table~\ref{tab:main-search-grid} lists the grid of every method. Functionally duplicate rows (e.g.\
different $\sigma^2$ when $\lambda = \lambda_{\mathrm{sm}} = 0$) are evaluated once; in total
\nValGrid{} validation configurations are evaluated. On ViT-B/16 the activation-side $\lambda$ is
fixed to $0$, which keeps GELU gates conditioned (Section~\ref{sec:attention}). At attention--Value
products, \ac{RG} applies $\alpha\beta$-\ac{LRP}-$\varepsilon$ to the composite map of the
attention product and the Value projection (Theorem~\ref{thm:attn-ab-recovery};
Appendix~\ref{app:vit-implementation}). Baselines without a declared grid are run at the settings
of their reference implementations.

\begin{table}[H]
    \centering
    \caption{Per-method validation grids (all architectures unless stated). Symbols:
    Appendix~\ref{app:config-symbols}.}
    \label{tab:main-search-grid}
    \small
    \renewcommand{\arraystretch}{1.15}
    \begin{tabular}{@{}l p{0.73\textwidth}@{}}
        \toprule
        Method & Grid \\
        \midrule
        Gradient & single config \\
        SmGrad & $N{=}50$, noise level $\sigma/(x_{\max}{-}x_{\min}) \in \{0.1, 0.15, 0.2\}$, mean
        aggregation \quad (3 configs) \\
        IntGrad & black-image baseline, $N_{\text{steps}} \in \{20, 50, 100, 300\}$ \quad (4 configs) \\
        DeepLift & Rescale rule, single config \quad (CNNs only) \\
        DAVE & $N \in \{25, 50\} \times \theta_{\max} \in \{0, 10, 20\} \times s_{\max} \in \{0, 0.05,
            0.1\} \times t_{\max} \in \{0, 0.25, 0.5\}$, horizontal flips with probability $0.5$, averaged
        effective map multiplied by the input \quad (54 configs; ViT-B/16 only) \\
        SG & $\lambda \in \{-5, -2, -1, -0.3, 0\} \times \sigma^2 \in \{1, 2, 5, 10\} \times$ noisy gate
        $\tilde\Phi \in \{\text{off}, \text{on}\}$ \quad (40 configs; VGG-16 only) \\
        LRP-$\varepsilon$ & $\varepsilon \in \{0.01, 0.05, 0.1, 0.25, 0.5, 1, 2\}$ \quad (7 configs) \\
        LRP-$\gamma$ & $(\varepsilon, \gamma) \in \{0.1, 0.25, 0.5\}\times\{0.1, 0.25, 0.5, 1, 2\}$;
        ViT-B/16: $\gamma$ on all Linear layers (15 configs); CNNs: both zennit's EpsilonGammaBox composite
        ($z^{\mathcal{B}}$ input layer, $\gamma$ on convolutions, $\varepsilon$ on dense layers) and
        $\gamma$ on all layers (30 configs) \\
        AttnLRP & $\gamma$-composite of~\citet{achtibat2024attnlrp}: reference setting
        $(\gamma_{\text{conv}}, \gamma_{\text{lin}}, \varepsilon_{\text{attn}}){=}(0.25, 0.05, 10^{-6})$ +
        univariate sweeps $\gamma_{\text{lin}}\!\in\!\{0.01, 0.025, 0.1, 0.25\}$,
        $\gamma_{\text{conv}}\!\in\!\{0.1, 0.5, 1\}$, $\varepsilon_{\text{attn}}\!\in\!\{0.01, 0.05,
            0.25\}$ \quad (11 configs; ViT-B/16 only) \\
        TIBAV & single config \quad (ViT-B/16 only) \\
        GAE & single config \quad (ViT-B/16 only) \\
        AR & residual weight $w_r \in \{0, 0.25, 0.5, 0.75\}$ \quad (4 configs; ViT-B/16 only) \\
        RG & $\tau \in \{0.3, 0.4, \ldots, 1.0\}$, $(\alpha,\beta){=}(2,1)$, $\sigma^2 \in \{1, 2, 5,
            10\}$, $\varepsilon \in \{0.1, 0.25, 0.5, 1\}$; ViT-B/16: $\lambda{=}0$, $\lambda_{\mathrm{sm}} \in
            \{-50, -10, -5, -2, -1, 0\}$, composite Value rule of Theorem~\ref{thm:attn-ab-recovery} (672
        configs); CNNs: $\lambda \in \{-5, -2, -1, -0.3, 0\}$ (544 configs) \\
        RG+Eq & $\tau \in \{0.3, 0.4, \ldots, 1.0\} \times \lambda_{\mathrm{sm}} \in \{0, -10\}$,
        $\sigma^2{=}2$, $\varepsilon{=}0.5$, $N{=}50$, composite Value rule \quad (16 configs; ViT-B/16
        only) \\
        GCAM, GCAM++, LCAM & last convolutional block (VGG-16: \texttt{features.29}, ResNet-50:
        \texttt{layer4}), single config \quad (CNNs only) \\
        LIME & 80 SLIC superpixels, $N{=}1000$, ridge surrogate, exponential kernel of width $0.25$ on the
        cosine distance of the binary superpixel vectors, removed superpixels filled with their mean color;
        fixed reference row \\
        KShap & 80 SLIC superpixels, $N{=}2M{+}2048$ for $M$ superpixels, removed superpixels filled with
        the image mean; fixed reference row \\
        Anchor & $\alpha\beta$-\ac{LRP} anchor $(\alpha,\beta,\varepsilon,\tau){=}(2,1,0.5,1)$ of \ac{RG}
        and \ac{RG}+Eq, reported as the theorem-driven reference row (not selected) \\
        \bottomrule
    \end{tabular}
\end{table}

\subsection{Selection}\label{app:selection}

For every tunable method, baselines and ours alike, two configurations are selected on the
validation split: \texttt{loc} minimizes the rank-sum over $\{\text{AL}, \text{PG}, \text{TK}\}$
and \texttt{faith} the rank-sum over $\{\text{PF}_5, \text{PF}_{20}, \text{Sel}\}$. Ranks are
computed within each method's grid, respecting metric directions; ties are broken in favor of the
default (for \ac{RG}, the anchor), then of fewer non-default parameters, then of grid order. A
configuration is admissible only if its map is non-zero and every metric is defined on every
validation image; inadmissible configurations are excluded before ranking (5 of the 40 \ac{SG}
configurations on VGG-16, all with strong risk aversion and the deterministic gate, whose maps
vanish on some images; no other configuration). MaxS and AvgS are not used for selection because
each evaluation re-explains every image $2N_s$ times. The \texttt{loc} row is the one reported in
the main tables, and the \texttt{faith} row of every method is reported alongside
(Appendix~\ref{app:results}): localization is a natural target because the analytical modes are
designed to control routing focus, because mask localization has an external ground truth whereas
perturbation metrics evaluate off-distribution inputs, and because comparing \ac{LRP} rule choices
by object localization has precedent, there with bounding boxes~\citep{Kohlbrenner2020};
localization is not faithfulness, so both rows are reported and the trade-off stays visible.

\subsection{Statistics}\label{app:stats}

Every test metric is reported with a 95\% percentile bootstrap confidence interval over test
images~\citep{Efron1993bootstrap} ($10{,}000$ resamples of the images; the same resampled image
indices are used for all rows, so comparisons are paired). MaxS and AvgS are computed on the fixed
200-image subset of Appendix~\ref{app:metrics}, and their intervals resample the images of that
subset. For each metric, every row is tested against the best row for that metric: continuous
metrics with a two-sided paired bootstrap test of the mean per-image difference (the bootstrap
distribution of the mean difference is centered at zero, and $p$ is the fraction of centered
resampled differences at least as large in absolute value as the observed one), and the binary PG
with McNemar's test~\citep{McNemar1947} ($\chi^2$ statistic with continuity correction). The
$p$-values are Holm-corrected~\citep{Holm1979} within each metric: over the rows shown in the main
table, over all rows of a column in the appendix tables, and, for the statement that a row is
significantly ahead of every baseline row, over its comparisons with the baseline rows. We call a
difference significant at level $0.05$ after correction and describe all other differences as
comparable.

\subsection{Compute resources}\label{app:compute-resources}

The attribution evaluations are inference-only. The main cost is the repeated backward attribution
pass over the validation and test images across the method grids, including the
equivariant-Reynolds and noise-averaged configurations. Runtime measurements were taken on a single
NVIDIA A100-SXM4-40GB with batch size 1; evaluation sweeps were distributed over several GPUs of
the same type, with up to two evaluation processes per GPU. The full pipeline (validation grids,
test evaluation, robustness, Pascal VOC, ablations and randomization diagnostics) required about
450 GPU-hours.

\subsection{\texorpdfstring{\ch{Code and correctness tests}}{Code and correctness tests}}\label{app:code-tests}

\ch{The code (\coderepo) contains unit tests that pin the implemented backward rules to the statements
of the appendix. Two of them concern the attention block.
\texttt{tests/test\_attention\_ab\_recovery.py} runs the \ac{RG} hooks at $\tau = 1$ on a small
attention block with signed inputs, a Value bias and $\varepsilon > 0$, in float64, and compares the
result with an independent implementation of $\alpha\beta$-\ac{LRP}-$\varepsilon$ on the composite
map $X \mapsto A(XW_v^\top + b_v)$, i.e.\ the target rule~\eqref{eq:attn-target-input} of
Theorem~\ref{thm:attn-ab-recovery}, for a single head and for a multi-head block in the layout of
\texttt{nn.MultiheadAttention}; it also asserts that the result differs from the sequential variant
of \S\ref{app:attention-target-rule}, which the released code does not offer (\texttt{av\_rule}
accepts only \texttt{'merged'}).
\texttt{tests/test\_hellinger\_merged\_attention.py} enumerates every backward trajectory of a small
attention block under two models, at $\tau \in \{0.5, 1, 2\}$, and checks the per-input masses, the
Bhattacharyya mass and the cemetery masses of the Hellinger pass (Appendix~\ref{app:path-divergence})
against this brute-force reference; a pre-norm residual block built from \texttt{nn.MultiheadAttention}
must give the same distance as the same block with its attention matrix supplied as a constant.}
\newpage
\makeatletter
\providecommand{\captionof}[1]{\def\@captype{#1}\caption}
\providecommand{\appinputgen}[1]{\begingroup\def\^^M{\\}\input{#1}\endgroup}
\providecommand{\appinputgenlt}[1]{\begingroup\def\^^M{\\}\input{#1}\endgroup\addtocounter{table}{-1}}
\providecommand{\appinputgenci}[1]{\begingroup\fontsize{8.5}{10}\selectfont\def\tiny{\fontsize{5}{5.5}\selectfont}\appinputgenlt{#1}\endgroup}
\providecommand{\appneedspace}[1]{\par\penalty-100\begingroup\dimen0=\pagegoal\advance\dimen0-\pagetotal
\ifdim\dimen0<#1\relax\newpage\fi\endgroup}
\newsavebox{\appgenbox}
\providecommand{\appfitgen}[1]{\global\setbox\appgenbox=\hbox{\def\^^M{\\}\input{#1}}\ifdim\wd\appgenbox>\linewidth\resizebox{\linewidth}{!}{\usebox{\appgenbox}}\else\usebox{\appgenbox}\fi}
\makeatother%

\section{Additional Results}\label{app:results}

The validation, test and VOC tables of this appendix are generated by one pipeline from the
declared grids and the selection rules of Appendix~\ref{app:protocol}.

\paragraph{Configuration symbols.}\label{app:config-symbols}
The configuration columns of all tables use the following parameters. \ac{RG}, \ac{RG}+Eq and
\ac{SG} (Table~\ref{tab:modes}): temperature $\tau$, activation-side risk aversion $\lambda$,
attention-logit risk aversion $\lambda_{\mathrm{sm}}$, \ac{ADF} input variance $\sigma^2$, \ac{LRP}
stabilizer $\varepsilon$, stream weights $(\alpha,\beta)$ (default $(2,1)$), noisy-observation gate
$\tilde\Phi$ (``gate on/off''), and the number $N$ of transformed inputs averaged by \ac{RG}+Eq.
SmoothGrad: $N$ noisy samples with noise standard deviation $\sigma$ relative to the input range
($\sigma/\mathrm{range}$). Integrated Gradients: $N_{\mathrm{steps}}$ integration steps from a
black-image baseline. DAVE: $N$ transformed inputs with rotations up to $\theta_{\max}$ degrees,
shifts up to $s_{\max}$ of the image extent, and variance-preserving Gaussian noise of
interpolation weight up to $t_{\max}$. LRP-$\varepsilon$ and LRP-$\gamma$: stabilizer $\varepsilon$
and positive-weight boost $\gamma$; on the CNNs, LRP-$\gamma$ is zennit's composite with the
$z^{\mathcal{B}}$ rule at the input layer. AttnLRP: $\gamma$-rule parameters
$\gamma_{\mathrm{conv}}$ and $\gamma_{\mathrm{lin}}$ of the convolutional and linear layers and the
stabilizer $\varepsilon_{\mathrm{attn}}$ of the attention layers. AR (attention rollout): residual
weight $w_r$. LIME and KernelSHAP: $N$ perturbation samples ($2M{+}2048$ for $M$ superpixels) on 80
SLIC superpixels; LIME fits a ridge surrogate. DeepLift: Rescale rule. Grad-CAM, Grad-CAM++ and
LayerCAM: the target layer (torchvision module name). Row labels in parentheses give the selection
criterion (\texttt{loc}, \texttt{faith}); single-configuration methods have no criterion.

\subsection{Validation selections}\label{app:val-selection}

Tables~\ref{tab:val-selection-vit}--\ref{tab:val-selection-resnet} list, for every method, the
configurations selected on the 50-image validation split under the \texttt{loc} and \texttt{faith}
criteria of Appendix~\ref{app:selection}, together with both validation rank-sums.

\begin{table}[H]
    \centering
    \caption{ViT-B/16: validation-selected configurations per method and criterion (\texttt{loc}:
    minimal rank-sum over AL, PG, TK; \texttt{faith}: over PF$_5$, PF$_{20}$, Sel; rank-sums over the
    configurations of one method, lower is better; the number of configurations is given after the
    method name). Symbols: Appendix~\ref{app:config-symbols}.}
    \label{tab:val-selection-vit}
    \small
    \begin{tabular}{@{}l l l c c@{}}
\toprule
Method & Criterion & Selected configuration & Loc.\ rank-sum & Faith.\ rank-sum \\
\midrule
Gradient (1) & fixed & -- & -- & -- \\
\addlinespace
SmGrad (3) & loc & $N{=}50,\sigma/\mathrm{range}{=}0.2$ & 3.5 & 3 \\
 & faith & $N{=}50,\sigma/\mathrm{range}{=}0.2$ & 3.5 & 3 \\
\addlinespace
IntGrad (4) & loc & $N_{\mathrm{steps}}{=}100$, black baseline & 3.5 & 4 \\
 & faith & $N_{\mathrm{steps}}{=}100$, black baseline & 3.5 & 4 \\
\addlinespace
LRP-$\varepsilon$ (7) & loc & $\varepsilon{=}0.25$ & 5.5 & 7 \\
 & faith & $\varepsilon{=}0.1$ & 6.5 & 3 \\
\addlinespace
LRP-$\gamma$ (15) & loc & $\varepsilon{=}0.1,\gamma{=}2$ & 6 & 25 \\
 & faith & $\varepsilon{=}0.1,\gamma{=}0.1$ & 32 & 12 \\
\addlinespace
AttnLRP (11) & loc & $\gamma_{\mathrm{conv}}{=}0.25,\gamma_{\mathrm{lin}}{=}0.05,\varepsilon_{\mathrm{attn}}{=}0.25$ & 3 & 29 \\
 & faith & $\gamma_{\mathrm{conv}}{=}0.25,\gamma_{\mathrm{lin}}{=}0.05,\varepsilon_{\mathrm{attn}}{=}0.01$ & 14 & 7 \\
\addlinespace
DAVE (54) & loc & $N{=}50,\theta_{\max}{=}0,s_{\max}{=}0,t_{\max}{=}0.5$ & 14 & 6 \\
 & faith & $N{=}50,\theta_{\max}{=}0,s_{\max}{=}0,t_{\max}{=}0.5$ & 14 & 6 \\
\addlinespace
TIBAV (1) & fixed & -- & -- & -- \\
\addlinespace
GAE (1) & fixed & -- & -- & -- \\
\addlinespace
AR (4) & loc & $w_r{=}0.75$ & 3.5 & 3 \\
 & faith & $w_r{=}0.75$ & 3.5 & 3 \\
\addlinespace
LIME (1) & fixed & $N{=}1000$, 80 superpixels, ridge & -- & -- \\
\addlinespace
KernelSHAP (1) & fixed & $N{=}2M{+}2048$, 80 superpixels & -- & -- \\
\addlinespace
RG (672) & loc & $\tau{=}0.4,\lambda_{\mathrm{sm}}{=}{-}10,\sigma^2{=}5,\varepsilon{=}0.1$ & 47 & 788 \\
 & faith & $\tau{=}0.7,\lambda_{\mathrm{sm}}{=}{-}5,\sigma^2{=}5,\varepsilon{=}0.1$ & 1156 & 22 \\
 & anchor & $\tau{=}1$ & 1420 & 344 \\
\addlinespace
RG+Eq (16) & loc & $\tau{=}0.5,\lambda_{\mathrm{sm}}{=}{-}10,\sigma^2{=}2,N{=}50$ & 6.5 & 29 \\
 & faith & $\tau{=}0.7,\lambda_{\mathrm{sm}}{=}{-}10,\sigma^2{=}2,N{=}50$ & 14 & 11 \\
 & anchor & $\tau{=}1,N{=}50$ & 47 & 31 \\
\bottomrule
\end{tabular}

\end{table}

\begin{table}[H]
    \centering
    \caption{VGG-16: validation-selected configurations per method and criterion, as in
    Table~\ref{tab:val-selection-vit}.}
    \label{tab:val-selection-vgg}
    \small
    \begin{tabular}{@{}l l l c c@{}}
\toprule
Method & Criterion & Selected configuration & Loc.\ rank-sum & Faith.\ rank-sum \\
\midrule
Gradient (1) & fixed & -- & -- & -- \\
\addlinespace
SmGrad (3) & loc & $N{=}50,\sigma/\mathrm{range}{=}0.1$ & 3 & 3 \\
 & faith & $N{=}50,\sigma/\mathrm{range}{=}0.1$ & 3 & 3 \\
\addlinespace
IntGrad (4) & loc & $N_{\mathrm{steps}}{=}20$, black baseline & 6 & 3 \\
 & faith & $N_{\mathrm{steps}}{=}20$, black baseline & 6 & 3 \\
\addlinespace
DeepLift (1) & fixed & Rescale & -- & -- \\
\addlinespace
LRP-$\varepsilon$ (7) & loc & $\varepsilon{=}0.5$ & 6 & 9 \\
 & faith & $\varepsilon{=}1$ & 6 & 4 \\
\addlinespace
LRP-$\gamma$ (30) & loc & $\varepsilon{=}0.5,\gamma{=}0.25$, $z^{\mathcal{B}}$/$\gamma$/$\varepsilon$ composite & 13 & 43 \\
 & faith & $\varepsilon{=}0.25,\gamma{=}0.1$ & 36 & 5 \\
\addlinespace
GCAM (1) & fixed & \texttt{features.29} & -- & -- \\
\addlinespace
GCAM++ (1) & fixed & \texttt{features.29} & -- & -- \\
\addlinespace
LCAM (1) & fixed & \texttt{features.29} & -- & -- \\
\addlinespace
LIME (1) & fixed & $N{=}1000$, 80 superpixels, ridge & -- & -- \\
\addlinespace
KernelSHAP (1) & fixed & $N{=}2M{+}2048$, 80 superpixels & -- & -- \\
\addlinespace
SG (40) & loc & $\lambda{=}{-}1,\sigma^2{=}5$, gate on & 37 & 10 \\
 & faith & $\lambda{=}{-}0.3,\sigma^2{=}1$, gate off & 49.5 & 6 \\
\addlinespace
RG (544) & loc & $\tau{=}0.3,\varepsilon{=}0.1$ & 78 & 194 \\
 & faith & $\tau{=}0.8,\varepsilon{=}0.1$ & 173.5 & 5 \\
 & anchor & $\tau{=}1$ & 339 & 60 \\
\bottomrule
\end{tabular}

\end{table}

\begin{table}[H]
    \centering
    \caption{ResNet-50: validation-selected configurations per method and criterion, as in
    Table~\ref{tab:val-selection-vit}.}
    \label{tab:val-selection-resnet}
    \small
    \begin{tabular}{@{}l l l c c@{}}
\toprule
Method & Criterion & Selected configuration & Loc.\ rank-sum & Faith.\ rank-sum \\
\midrule
Gradient (1) & fixed & -- & -- & -- \\
\addlinespace
SmGrad (3) & loc & $N{=}50,\sigma/\mathrm{range}{=}0.15$ & 4 & 5 \\
 & faith & $N{=}50,\sigma/\mathrm{range}{=}0.1$ & 8 & 4 \\
\addlinespace
IntGrad (4) & loc & $N_{\mathrm{steps}}{=}100$, black baseline & 3 & 6 \\
 & faith & $N_{\mathrm{steps}}{=}50$, black baseline & 7.5 & 6 \\
\addlinespace
DeepLift (1) & fixed & Rescale & -- & -- \\
\addlinespace
LRP-$\varepsilon$ (7) & loc & $\varepsilon{=}0.5$ & 5.5 & 9 \\
 & faith & $\varepsilon{=}2$ & 8 & 4 \\
\addlinespace
LRP-$\gamma$ (30) & loc & $\varepsilon{=}0.5,\gamma{=}2$, $z^{\mathcal{B}}$/$\gamma$/$\varepsilon$ composite & 12.5 & 33 \\
 & faith & $\varepsilon{=}0.5,\gamma{=}0.5$, $z^{\mathcal{B}}$/$\gamma$/$\varepsilon$ composite & 36 & 22.5 \\
\addlinespace
GCAM (1) & fixed & \texttt{layer4} & -- & -- \\
\addlinespace
GCAM++ (1) & fixed & \texttt{layer4} & -- & -- \\
\addlinespace
LCAM (1) & fixed & \texttt{layer4} & -- & -- \\
\addlinespace
LIME (1) & fixed & $N{=}1000$, 80 superpixels, ridge & -- & -- \\
\addlinespace
KernelSHAP (1) & fixed & $N{=}2M{+}2048$, 80 superpixels & -- & -- \\
\addlinespace
RG (544) & loc & $\tau{=}0.9,\varepsilon{=}1$ & 19 & 224 \\
 & faith & $\tau{=}1,\lambda{=}{-}1,\sigma^2{=}1,\varepsilon{=}0.25$ & 466 & 38 \\
 & anchor & $\tau{=}1$ & 212.5 & 140 \\
\bottomrule
\end{tabular}

\end{table}

\subsection{Test results with confidence intervals}\label{app:test-ci}

Tables~\ref{tab:test-ci-vit}--\ref{tab:test-ci-resnet} list every row evaluated on the 565-image
test split (the \texttt{loc}- and \texttt{faith}-selected configuration of every method, the
anchors and the single-configuration methods) with 95\% bootstrap confidence intervals and the
tests of Appendix~\ref{app:stats} (bold and dark shading: best row of a column; $^{\approx}$ and
light shading: not significantly different from it). Per-image runtimes come from the single-GPU
benchmark of Appendix~\ref{app:compute-resources}.

\appneedspace{12\baselineskip}
\begin{center}
    \captionof{table}{ViT-B/16, ImageNet-S test split: all rows with 95\% bootstrap confidence
    intervals. Metrics as in Table~\ref{tab:eval-vit}; configuration symbols and row labels:
    Appendix~\ref{app:config-symbols}.}
    \label{tab:test-ci-vit}
\end{center}
\appinputgenci{tables/generated/test_vit_b16_ci_longtable}

\appneedspace{12\baselineskip}
\begin{center}
    \captionof{table}{VGG-16, ImageNet-S test split: all rows with 95\% bootstrap confidence intervals,
    as in Table~\ref{tab:test-ci-vit}.}
    \label{tab:test-ci-vgg}
\end{center}
\appinputgenci{tables/generated/test_vgg16_ci_longtable}

\appneedspace{12\baselineskip}
\begin{center}
    \captionof{table}{ResNet-50, ImageNet-S test split: all rows with 95\% bootstrap confidence
    intervals, as in Table~\ref{tab:test-ci-vit}.}
    \label{tab:test-ci-resnet}
\end{center}
\appinputgenci{tables/generated/test_resnet50_ci_longtable}

\subsection{Trade-off sweeps}\label{app:tradeoff-sweeps}\label{app:dual-sweeps}

Table~\ref{tab:ablations-vit} includes one-at-a-time sweeps of $\tau$ and $\lambda_{\mathrm{sm}}$
around the \texttt{loc}-selected ViT-B/16 \ac{RG} configuration. These sweep rows are evaluated on
the test split only after selection and never feed back into it.

\subsection{Cross-Dataset Evaluation: Pascal VOC 2012}\label{app:voc-cross-data}

This part details the protocol behind the ViT-B/16 cross-dataset row of the summary table in
Section~\ref{sec:experiments} (all rows in Table~\ref{tab:eval-vit-voc}). The goal is twofold: (i)
test whether the \ac{RG} / \ac{RG}+Eq configurations selected on ImageNet-S transfer to a different
dataset \emph{without} re-tuning, and (ii) put baselines on equal footing by re-running the same
per-method validation selection on the new pool.

\paragraph{Dataset.}
We use Pascal VOC 2012~\citep{Everingham2010voc} segmentation annotations: the trainval pool of
$2{,}913$ images with pixel-accurate \texttt{Seg\-mentationClass} masks across $20$ foreground
classes plus background. Each image is expanded to one (image, foreground-class) sample per
distinct VOC class label present in its mask, yielding $3{,}422$ total samples after dropping the
\texttt{person} class (no faithful ImageNet-1K target exists for it) and discarding sliver masks
whose pixel area on the \emph{original} (pre-resize) annotation is below $64$ pixels. Each (image,
class) sample carries the same input image and an explanation target equal to the most semantically
central ImageNet-1K index for that VOC class (e.g.\ \texttt{horse}$\,\to\,$\textsc{n02389026
sorrel}; \texttt{tvmonitor}$\,\to\,$\textsc{n03782006 monitor}); the curated mapping is part of the
released code.

\paragraph{Per-image disjoint val/test split.} The trainval pool is partitioned randomly with seed $42$ at the \emph{image} level (not the (image,
class) level) into a $1{:}5$ val/test split: the trainval stems are shuffled deterministically and
the first $1/6$ is assigned to val. All (image, class) samples for a given image stay in the same
partition, so leakage between val tuning and test evaluation is structurally impossible. As on
ImageNet-S, samples whose object mask lies entirely outside the $224{\times}224$ center crop are
excluded after the split (18 val and 105 test samples). The resulting partitions:
\begin{center}
    \begin{tabular}{lcc}
        \toprule
        & Images & (image, class) samples \\
        \midrule
        val (tuning)  & $460$  & $556$ \\
        test (eval)   & $2{,}293$ & $2{,}743$ \\
        \bottomrule
    \end{tabular}
\end{center}

\paragraph{Pixel masks for localization.} \ac{AttrLoc}, \ac{PG}, and Top-$k$ Intersection are computed against the per-(image, class)
\emph{binary VOC2012 segmentation pixel mask} for that class, not bounding boxes. The mask is
resized via NEAREST interpolation and center-cropped to $224{\times}224$ to match the input
pipeline (\texttt{Resize(256)$\,\to\,$CenterCrop(224)}), then passed verbatim into
Quantus~\citep{hedstrom2023quantus} alongside the attribution map. This avoids the well-known bbox
over-credit issue~\citep{Selvaraju2017gradcam} where a method that puts mass anywhere inside the
box wins the metric; pixel masks reward only mass on the actual foreground object.

\paragraph{Tuning protocol.}
The protocol mirrors the ImageNet-S protocol of Appendix~\ref{app:protocol}, with one deliberate
asymmetry: \ac{RG} and \ac{RG}+Eq do \emph{not} get a fresh grid search on VOC.
\begin{itemize}
    \item \textbf{Baselines.} Every row of the per-method ViT-B/16 grid of Table~\ref{tab:main-search-grid}
        is evaluated on the $556$-sample val pool, and the \texttt{loc} and \texttt{faith} configurations
        are selected per method with the criteria of Appendix~\ref{app:selection}. The selected rows are
        re-evaluated on the $2{,}743$-sample test pool.
    \item \textbf{\ac{RG} / \ac{RG}+Eq (transfer test).} The anchor and the \texttt{loc}- and
        \texttt{faith}-selected ImageNet-S configurations are evaluated on the VOC test pool unchanged,
        without any VOC-side re-tuning.
\end{itemize}
To save compute, VOC is evaluated on localization and faithfulness only (no MaxS/AvgS and no
runtime column). Any tuning advantage flows to the baselines, not to us. Wins, ties and losses are
counted per metric with the tests of Appendix~\ref{app:stats}.

\appneedspace{8\baselineskip}
\begin{center}
    \captionof{table}{ViT-B/16 on the Pascal VOC 2012 test pool ($2{,}743$ samples): all rows with 95\%
    bootstrap confidence intervals. \ac{RG} / \ac{RG}+Eq use the ImageNet-S-selected configurations
    without re-tuning; baselines are selected on the VOC val pool. Metrics as in
    Table~\ref{tab:eval-vit}; configuration symbols and row labels: Appendix~\ref{app:config-symbols}.}
    \label{tab:eval-vit-voc}
\end{center}
\IfFileExists{tables/generated/test_voc_vit_b16_ci_longtable.tex}{\appinputgenci{tables/generated/test_voc_vit_b16_ci_longtable}}{%
\begin{center}\textcolor{red}{To be generated by the cluster run.}\end{center}}

\providecommand{\absval}[2]{
\begin{tabular}[c]{@{}c@{}}#1\\\textcolor{gray!70}{\scriptsize #2}\end{tabular}}

\subsection{One-factor ablations}\label{app:abl-one-factor}

Tables~\ref{tab:ablations-vit}--\ref{tab:ablations-resnet} vary one parameter at a time around each
validation-selected configuration of \ac{RG}, \ac{RG}+Eq and \ac{SG} (bold: the selected value),
over the values of the declared grid; all other parameters stay at their selected values. The rows
are evaluated on the 565-image test split with the localization and faithfulness metrics and never
feed back into the selection. On ViT-B/16, ``+Eq'' evaluates \ac{RG}+Eq at the parameters of the
selected \ac{RG} row. Configurations that are inadmissible on the validation split
(Appendix~\ref{app:selection}) are not included.

\appneedspace{8\baselineskip}
\begin{center}
    \captionof{table}{ViT-B/16, ImageNet-S test split: one-factor ablations around the selected
    configurations. Center: the selected configuration (criterion in parentheses) around which the
    parameter in the Factor column is varied over the Value column, all other parameters fixed; bold:
    the center's own value. Symbols: Appendix~\ref{app:config-symbols}.}
    \label{tab:ablations-vit}
\end{center}
{\small\appinputgenlt{tables/generated/ablations_vit_b16_longtable}}

\appneedspace{8\baselineskip}
\begin{center}
    \captionof{table}{VGG-16, ImageNet-S test split: one-factor ablations around the selected
    configurations, as in Table~\ref{tab:ablations-vit}.}
    \label{tab:ablations-vgg}
\end{center}
{\small\appinputgenlt{tables/generated/ablations_vgg16_longtable}}

\appneedspace{8\baselineskip}
\begin{center}
    \captionof{table}{ResNet-50, ImageNet-S test split: one-factor ablations around the selected
    configurations, as in Table~\ref{tab:ablations-vit}.}
    \label{tab:ablations-resnet}
\end{center}
{\small\appinputgenlt{tables/generated/ablations_resnet50_longtable}}

\subsection{Pointing game at matched granularity}\label{app:pg-granularity}

The pointing game of Quantus counts a hit whenever any pixel attaining the map maximum lies in the
mask. For piecewise-constant maps (superpixel methods; attention maps on the $14{\times}14$ patch
grid) the maximum is a region of hundreds to thousands of pixels, for pixel-level maps usually a
single pixel. Table~\ref{tab:pg-granularity} reports, for the reported rows, the median size of the
maximum set, the expected hit rate of a single maximum pixel drawn uniformly from that set
(PG$_{1\mathrm{px}}$), and the pointing game after average-pooling every map to the same
$16{\times}16$-pixel grid (PG$_{16\mathrm{px}}$). Drawing a single maximum pixel removes most of
the advantage of the superpixel methods LIME and KernelSHAP. At the common $16{\times}16$ grid the
leading methods move closer together: on ViT-B/16 LRP-$\gamma$ and AttnLRP lead, on VGG-16 the
localization-selected \ac{SG} row and the faithfulness-selected \ac{RG} row are not significantly
below the best baseline (Grad-CAM++), and on ResNet-50 Grad-CAM, LayerCAM and LRP-$\gamma$ lead
(McNemar tests against the best baseline row).

\begin{table}[H]
    \centering
    \caption{Pointing game under the Quantus tie rule (PG, as in the main tables) and at matched
    granularity, ImageNet-S test split. max.\ set: median number of pixels attaining the map maximum;
    PG$_{1\mathrm{px}}$: expected hit rate of one pixel drawn uniformly from that set;
    PG$_{16\mathrm{px}}$: PG after average-pooling the map to $16{\times}16$-pixel cells. Rows: the
    \texttt{loc}-selected or fixed configuration of every method, plus the \texttt{faith} rows of our
    games. \textbf{Bold}: best per model and column.}
    \label{tab:pg-granularity}
    \small
    \begin{tabular}{@{}l l c c c c@{}}
\toprule
Model & Method & PG & max.\ set (px) & PG$_{1\mathrm{px}}$ & PG$_{16\mathrm{px}}$ \\
\midrule
ViT-B/16 & Gradient & 0.361 & 1 & 0.361 & 0.382 \\
 & SmGrad & 0.706 & 1 & 0.706 & 0.835 \\
 & IntGrad & 0.513 & 1 & 0.513 & 0.543 \\
 & LRP-$\varepsilon$ & 0.834 & 1 & 0.834 & 0.904 \\
 & LRP-$\gamma$ & 0.873 & 1 & \textbf{0.873} & \textbf{0.968} \\
 & AttnLRP & 0.855 & 1 & 0.855 & 0.926 \\
 & DAVE & 0.717 & 1 & 0.717 & 0.781 \\
 & TIBAV & 0.878 & 256 & 0.784 & 0.878 \\
 & GAE & 0.743 & 256 & 0.670 & 0.743 \\
 & AR & 0.565 & 256 & 0.516 & 0.565 \\
 & LIME & 0.897 & 1104 & 0.693 & 0.786 \\
 & KernelSHAP & \textbf{0.961} & 1031 & 0.833 & 0.903 \\
 & RG (faith) & 0.770 & 1 & 0.770 & 0.899 \\
 & RG & 0.825 & 1 & 0.825 & 0.888 \\
 & RG+Eq (faith) & 0.812 & 1 & 0.812 & 0.919 \\
 & RG+Eq & 0.855 & 1 & 0.855 & 0.919 \\
\midrule
VGG-16 & Gradient & 0.738 & 1 & 0.738 & 0.812 \\
 & SmGrad & 0.782 & 1 & 0.782 & 0.915 \\
 & IntGrad & 0.770 & 1 & 0.770 & 0.860 \\
 & DeepLift & 0.701 & 1 & 0.701 & 0.784 \\
 & LRP-$\varepsilon$ & 0.832 & 1 & \textbf{0.832} & 0.906 \\
 & LRP-$\gamma$ & 0.692 & 1 & 0.692 & 0.908 \\
 & GCAM & 0.830 & 1 & 0.823 & 0.920 \\
 & GCAM++ & 0.823 & 1 & 0.815 & \textbf{0.924} \\
 & LCAM & 0.720 & 1 & 0.715 & 0.834 \\
 & LIME & \textbf{0.972} & 1103 & 0.826 & 0.899 \\
 & KernelSHAP & 0.954 & 1011 & 0.826 & 0.894 \\
 & SG (faith) & 0.821 & 1 & 0.821 & 0.883 \\
 & SG & 0.770 & 1 & 0.770 & 0.908 \\
 & RG (faith) & 0.712 & 1 & 0.712 & 0.915 \\
 & RG & 0.736 & 1 & 0.736 & 0.873 \\
\midrule
ResNet-50 & Gradient & 0.556 & 1 & 0.556 & 0.630 \\
 & SmGrad & 0.789 & 1 & 0.789 & 0.919 \\
 & IntGrad & 0.635 & 1 & 0.635 & 0.743 \\
 & DeepLift & 0.727 & 1 & 0.727 & 0.793 \\
 & LRP-$\varepsilon$ & 0.775 & 1 & 0.775 & 0.880 \\
 & LRP-$\gamma$ & 0.842 & 1 & 0.842 & 0.958 \\
 & GCAM & 0.873 & 1 & \textbf{0.862} & \textbf{0.959} \\
 & GCAM++ & 0.834 & 1 & 0.827 & 0.945 \\
 & LCAM & 0.869 & 1 & 0.860 & 0.958 \\
 & LIME & 0.926 & 1058 & 0.729 & 0.844 \\
 & KernelSHAP & \textbf{0.965} & 1014 & 0.825 & 0.912 \\
 & RG (faith) & 0.743 & 1 & 0.743 & 0.903 \\
 & RG & 0.763 & 1 & 0.763 & 0.912 \\
\bottomrule
\end{tabular}

\end{table}

\subsection{Gradient-fallback attention}\label{app:abl-gradient-fallback}

The \ac{RG} backward operates by default through a routing rule that holds the attention policy
fixed and routes relevance only through the Value path, with zero gradient on $Q$, $K$ (Value-only
routing \ch{conditioned on the realized attention matrix, as in CP-LRP~\citep{Ali2022}}; see
Appendix~\ref{app:vit-implementation}).  The natural alternative is a gradient
fallback: pure autograd through the attention block, recovering the zennit-style
$\varepsilon$-\ac{LRP} attention treatment.  Table~\ref{tab:abl-gradient-fallback} swaps the
attention rule at the $\alpha\beta$-\ac{LRP} anchor and at the two validation-selected \ac{RG}
configurations and evaluates both rules on the 565-image test split.  The softmax risk
$\lambda_\text{sm}$ acts on the Value-path routing and is therefore inactive under the gradient
fallback.  The gradient fallback lowers all three localization metrics at every configuration, AL
and TK significantly throughout, and it does not improve faithfulness at any configuration: PF$_5$,
PF$_{20}$ and Sel are worse in every cell, significantly so in four of the nine.

\begin{table}[H]
    \centering
    \caption{\ac{ViT}-B/16, \ac{RG}: attention rule ablation on the ImageNet-S test split (565 images),
    at the anchor and at the loc- and faith-selected configurations.  Each cell shows the gradient
    fallback on top and the composite Value rule of Theorem~\ref{thm:attn-ab-recovery} (gray) below;
    configuration symbols: Appendix~\ref{app:config-symbols}; $^{\dagger}$: paired difference
    significant (paired bootstrap / McNemar for PG, Holm-corrected over the table, $p<0.05$).}
    \label{tab:abl-gradient-fallback}
    \setlength{\tabcolsep}{2pt}
    \renewcommand{\arraystretch}{0.9}
    \small
    \begin{tabular}{@{}l c c c c c c@{}}
\toprule
Center & AL$\uparrow$ & PG$\uparrow$ & TK$\uparrow$ & PF$_5$$\downarrow$ & PF$_{20}$$\downarrow$ & Sel$\downarrow$ \\
\midrule
anchor ($\tau{=}1$) & \absval{0.504$^{\dagger}$}{0.509} & \absval{0.722}{0.724} & \absval{0.694$^{\dagger}$}{0.699} & \absval{8.19}{8.19} & \absval{6.98$^{\dagger}$}{6.97} & \absval{4.877$^{\dagger}$}{4.854} \\
loc ($\tau{=}0.4,\lambda_{\mathrm{sm}}{=}{-}10,\sigma^2{=}5,\varepsilon{=}0.1$) & \absval{0.624$^{\dagger}$}{0.711} & \absval{0.743$^{\dagger}$}{0.825} & \absval{0.672$^{\dagger}$}{0.788} & \absval{8.31}{8.27} & \absval{7.50}{7.48} & \absval{5.364$^{\dagger}$}{5.109} \\
faith ($\tau{=}0.7,\lambda_{\mathrm{sm}}{=}{-}5,\sigma^2{=}5,\varepsilon{=}0.1$) & \absval{0.543$^{\dagger}$}{0.549} & \absval{0.729$^{\dagger}$}{0.770} & \absval{0.706$^{\dagger}$}{0.741} & \absval{8.13}{8.11} & \absval{6.87$^{\dagger}$}{6.78} & \absval{4.816}{4.808} \\
\bottomrule
\end{tabular}

\end{table}

\subsection{Qualitative comparison}\label{app:qualitative}

Figure~\ref{fig:qualitative} compares attribution maps for six ImageNet-S examples on ViT-B/16.
Each column shows a method at its \texttt{loc}-selected configuration
(Appendix~\ref{app:selection}), or at its fixed configuration if the method has no tunable
parameters (Gradient, TIBAV): \ac{RG} with
$\tau{=}0.4,\lambda_{\mathrm{sm}}{=}{-}10,\sigma^2{=}5,\varepsilon{=}0.1$ and \ac{RG}+Eq with
$\tau{=}0.5,\lambda_{\mathrm{sm}}{=}{-}10,\sigma^2{=}2,N{=}50$, both with the merged
attention--Value rule, i.e.\ the composite rule of Theorem~\ref{thm:attn-ab-recovery}, which routes
through the conditioned attention block as one map $X \mapsto O$; SmoothGrad with
$N{=}50,\sigma/\mathrm{range}{=}0.2$; Integrated Gradients with $N_{\mathrm{steps}}{=}100$ and a
black baseline; DAVE with $N{=}50,\theta_{\max}{=}0,s_{\max}{=}0,t_{\max}{=}0.5$; LRP-$\gamma$ with
$\varepsilon{=}0.1,\gamma{=}2$; and AttnLRP with
$\gamma_{\mathrm{conv}}{=}0.25,\gamma_{\mathrm{lin}}{=}0.05,\varepsilon_{\mathrm{attn}}{=}0.25$. As
in the evaluation, every map explains the ImageNet-S label, which is the top-1 prediction for all
images except the kuvasz (predicted: English setter). The first image is from the test split, the
other five are from the 50-image validation split used for the selection. Heatmaps show the raw
evaluator-side maps, before the absolute value and max-abs normalization of
Appendix~\ref{app:metrics} (the $14{\times}14$ TIBAV map is upsampled by nearest neighbor, as for
the metrics): red marks positive and blue negative attribution; each sign is divided by its 99th
percentile over the image and clipped, so the top 1\% of pixels saturate and concentrated maps stay
legible in print; the same display transform is used for every method.

\begin{figure}[H]
    \centering
    \setlength{\tabcolsep}{0.6pt}
    \renewcommand{\arraystretch}{0.3}
    \resizebox{\textwidth}{!}{
    \begin{tabular}{@{}r *{11}{c}@{}}
        & {\scriptsize Orig.}
        & {\scriptsize Mask}
        & {\scriptsize RG (loc)}
        & {\scriptsize RG+Eq (loc)}
        & {\scriptsize Grad}
        & {\scriptsize SmGrad}
        & {\scriptsize IntGrad}
        & {\scriptsize DAVE}
        & {\scriptsize LRP-$\gamma$}
        & {\scriptsize AttnLRP}
        & {\scriptsize TIBAV} \\[2pt]
        \rotatebox{90}{\scriptsize\hspace{2pt}Red fox (snow)}
        & \includegraphics[width=0.085\textwidth]{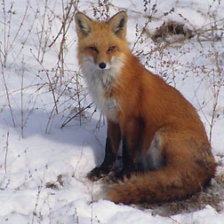}
        & \includegraphics[width=0.085\textwidth]{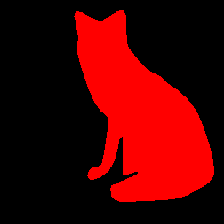}
        & \includegraphics[width=0.085\textwidth]{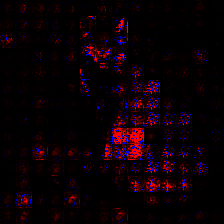}
        & \includegraphics[width=0.085\textwidth]{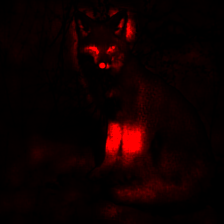}
        & \includegraphics[width=0.085\textwidth]{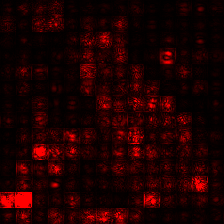}
        & \includegraphics[width=0.085\textwidth]{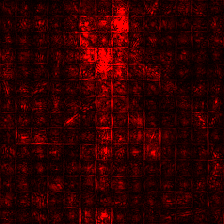}
        & \includegraphics[width=0.085\textwidth]{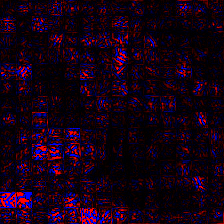}
        & \includegraphics[width=0.085\textwidth]{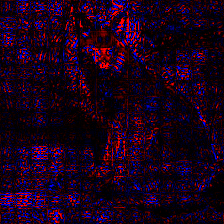}
        & \includegraphics[width=0.085\textwidth]{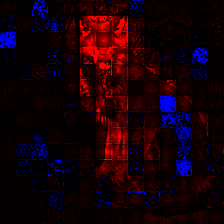}
        & \includegraphics[width=0.085\textwidth]{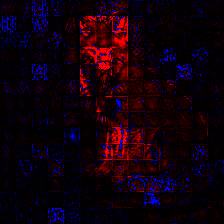}
        & \includegraphics[width=0.085\textwidth]{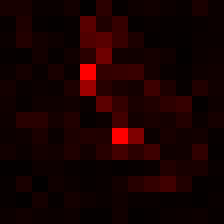} \\[1pt]
        \rotatebox{90}{\scriptsize\hspace{6pt}Red fox}
        & \includegraphics[width=0.085\textwidth]{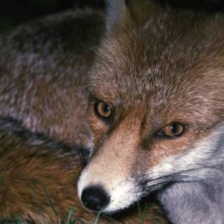}
        & \includegraphics[width=0.085\textwidth]{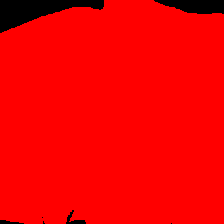}
        & \includegraphics[width=0.085\textwidth]{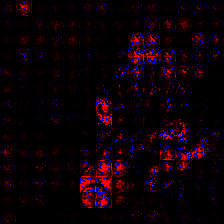}
        & \includegraphics[width=0.085\textwidth]{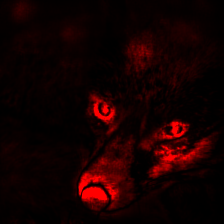}
        & \includegraphics[width=0.085\textwidth]{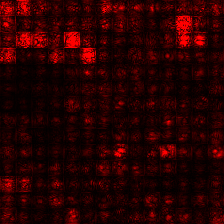}
        & \includegraphics[width=0.085\textwidth]{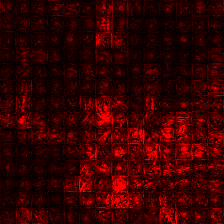}
        & \includegraphics[width=0.085\textwidth]{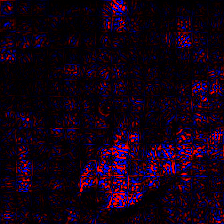}
        & \includegraphics[width=0.085\textwidth]{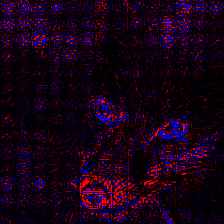}
        & \includegraphics[width=0.085\textwidth]{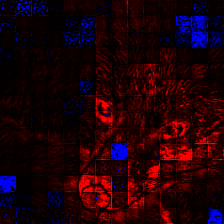}
        & \includegraphics[width=0.085\textwidth]{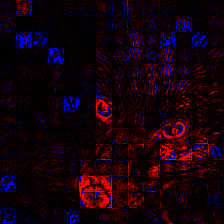}
        & \includegraphics[width=0.085\textwidth]{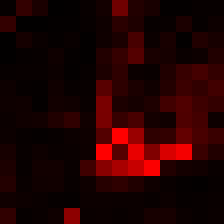} \\[1pt]
        \rotatebox{90}{\scriptsize\hspace{9pt}Kuvasz}
        & \includegraphics[width=0.085\textwidth]{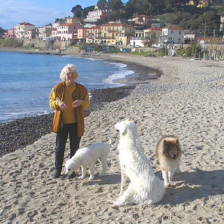}
        & \includegraphics[width=0.085\textwidth]{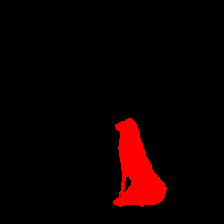}
        & \includegraphics[width=0.085\textwidth]{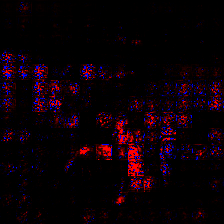}
        & \includegraphics[width=0.085\textwidth]{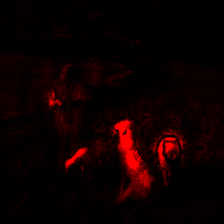}
        & \includegraphics[width=0.085\textwidth]{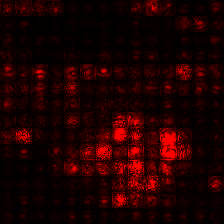}
        & \includegraphics[width=0.085\textwidth]{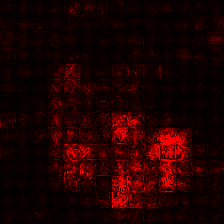}
        & \includegraphics[width=0.085\textwidth]{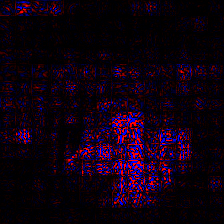}
        & \includegraphics[width=0.085\textwidth]{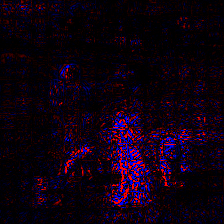}
        & \includegraphics[width=0.085\textwidth]{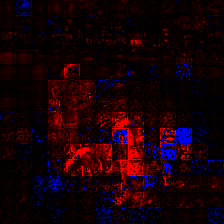}
        & \includegraphics[width=0.085\textwidth]{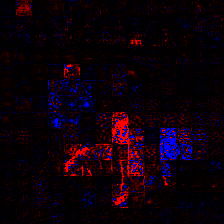}
        & \includegraphics[width=0.085\textwidth]{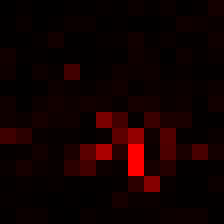} \\[1pt]
        \rotatebox{90}{\scriptsize\hspace{7pt}Wild boar}
        & \includegraphics[width=0.085\textwidth]{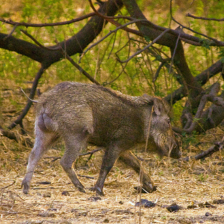}
        & \includegraphics[width=0.085\textwidth]{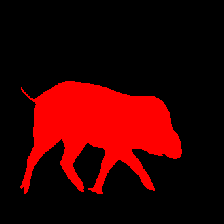}
        & \includegraphics[width=0.085\textwidth]{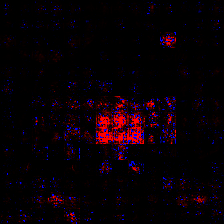}
        & \includegraphics[width=0.085\textwidth]{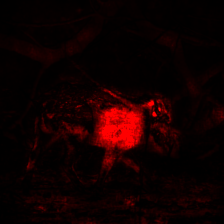}
        & \includegraphics[width=0.085\textwidth]{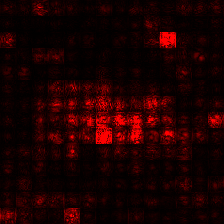}
        & \includegraphics[width=0.085\textwidth]{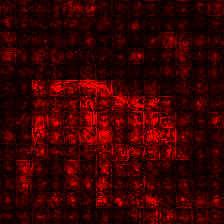}
        & \includegraphics[width=0.085\textwidth]{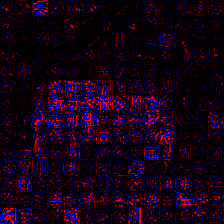}
        & \includegraphics[width=0.085\textwidth]{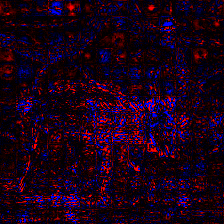}
        & \includegraphics[width=0.085\textwidth]{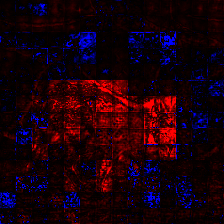}
        & \includegraphics[width=0.085\textwidth]{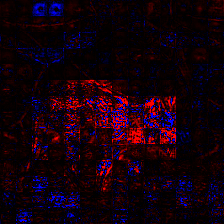}
        & \includegraphics[width=0.085\textwidth]{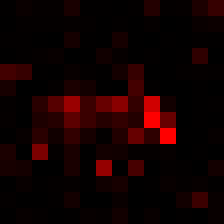} \\[1pt]
        \rotatebox{90}{\scriptsize\hspace{5pt}Black bear}
        & \includegraphics[width=0.085\textwidth]{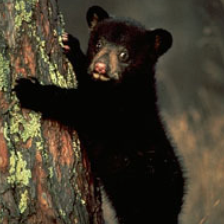}
        & \includegraphics[width=0.085\textwidth]{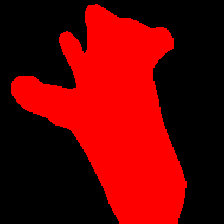}
        & \includegraphics[width=0.085\textwidth]{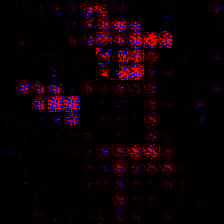}
        & \includegraphics[width=0.085\textwidth]{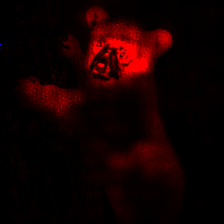}
        & \includegraphics[width=0.085\textwidth]{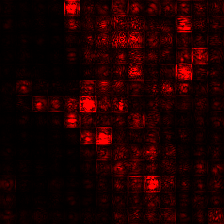}
        & \includegraphics[width=0.085\textwidth]{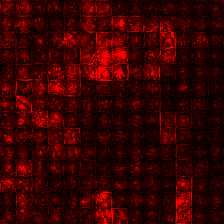}
        & \includegraphics[width=0.085\textwidth]{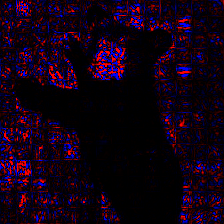}
        & \includegraphics[width=0.085\textwidth]{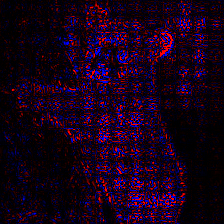}
        & \includegraphics[width=0.085\textwidth]{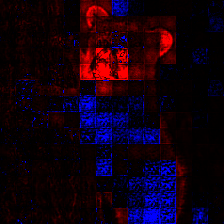}
        & \includegraphics[width=0.085\textwidth]{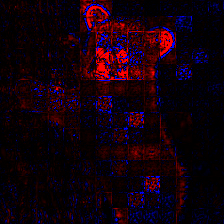}
        & \includegraphics[width=0.085\textwidth]{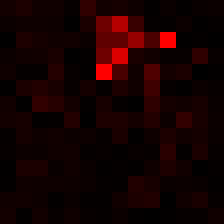} \\[1pt]
        \rotatebox{90}{\scriptsize\hspace{9pt}Airliner}
        & \includegraphics[width=0.085\textwidth]{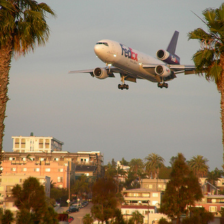}
        & \includegraphics[width=0.085\textwidth]{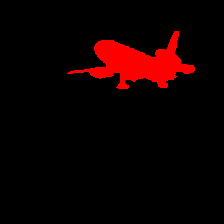}
        & \includegraphics[width=0.085\textwidth]{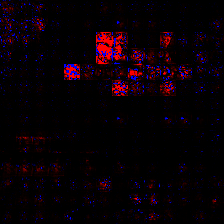}
        & \includegraphics[width=0.085\textwidth]{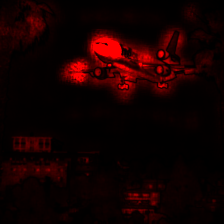}
        & \includegraphics[width=0.085\textwidth]{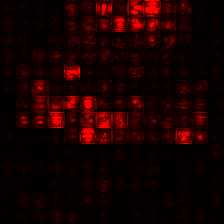}
        & \includegraphics[width=0.085\textwidth]{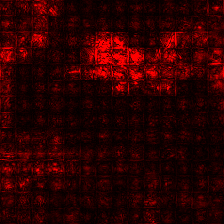}
        & \includegraphics[width=0.085\textwidth]{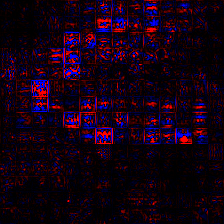}
        & \includegraphics[width=0.085\textwidth]{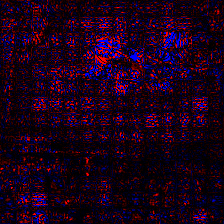}
        & \includegraphics[width=0.085\textwidth]{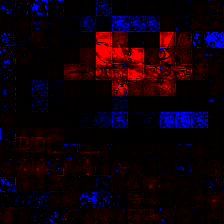}
        & \includegraphics[width=0.085\textwidth]{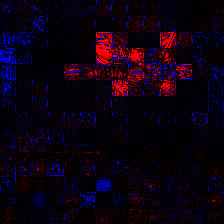}
        & \includegraphics[width=0.085\textwidth]{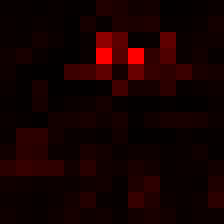} \\
    \end{tabular}
    }
    \caption{Dense qualitative comparison on six ImageNet-S examples (ViT-B/16). Columns: Original;
    ImageNet-S pixel Mask; \ac{RG} and \ac{RG}+Eq (merged attention--Value rule); Gradient, SmoothGrad,
    Integrated Gradients, DAVE, LRP-$\gamma$, AttnLRP and TIBAV. Every method is shown at its
    \texttt{loc}-selected configuration, Gradient and TIBAV at their fixed configurations (values in
    the text). Red: positive, blue: negative attribution, each sign divided by its 99th percentile over
    the image and clipped.}
    \label{fig:qualitative}
\end{figure}
\newpage
\section{Broader Impact}\label{app:broader-impact}

The work is methodological: it studies attribution methods for existing image classifiers and does
not introduce a deployed decision system, a new dataset of people, or a high-risk generative model.
Its intended positive impact is to make explanation methods more diagnostically accountable. In
particular, trajectory-level comparisons can reveal cases where visually similar attribution maps
arise from different backward calculations, and can therefore support more careful auditing of
explanation pipelines used in scientific or engineering workflows.

The same capability can also have negative or misleading uses. Sharper or more stable attribution
maps may create unwarranted confidence in a model if users treat explanations as causal guarantees
rather than as post-hoc diagnostics. In safety-critical or fairness-sensitive settings,
improved-looking explanations could be used to justify a classifier whose predictions remain
brittle, biased, or dataset-dependent. The Hellinger trajectory distance also compares internal
calculation paths rather than social validity; small trajectory distance does not imply that the
underlying model is fair, robust, or appropriate for deployment.

We therefore view the framework as an analysis tool rather than a deployment certificate. The
experiments use public benchmark classifiers and ImageNet-derived evaluation data, and the results
should be interpreted as evidence about attribution behavior under these benchmarks, not as
evidence that the underlying models are suitable for consequential use.

\end{document}